%% file: main.tex
\documentclass[11pt]{article} 

\pdfoutput=1

\usepackage{etoolbox}
\newtoggle{arxiv}
\toggletrue{arxiv}
\newtoggle{arxiv_upload}
\togglefalse{arxiv_upload}
\newtoggle{icml}
\togglefalse{icml}

\newcommand{\icml}[1]{\iftoggle{icml}{#1}{}}
\newcommand{\arxiv}[1]{\iftoggle{icml}{}{#1}}

\input{packages} 
\input{macros} 
\input{macros2}

\icml{
\renewcommand{\cite}[1]{\citet{#1}}
}

\addauthor{df}{red}
\addauthor{ms}{blue}

\newlength\tindent
\def\finalversion{1}

\ifx\finalversion\undefined
\else
\fi

\title{\huge Naive Exploration is Optimal for Online LQR}

\author{Max Simchowitz\\UC Berkeley\\\texttt{msimchow@berkeley.edu}\and Dylan J. Foster\\MIT\\\texttt{dylanf@mit.edu}}
\date{}

\begin{document}
\maketitle
\begin{abstract}
\iftoggle{arxiv_upload}{\input{abstract}}{\input{body/abstract}}
\end{abstract}

\iftoggle{arxiv_upload}{\input{introduction}}{\input{body/introduction}}

\iftoggle{arxiv_upload}{\input{problem_setting}}{\input{body/problem_setting}}

\section{Main Results}
\label{sec:main_results}
\iftoggle{arxiv_upload}{\input{intuition}}{\input{body/intuition}}
\section{Perturbation Bounds via the Self-Bounding ODE Method}
\label{sec:self_bounding}
\iftoggle{arxiv_upload}{\input{self_bounding}}{\input{body/self_bounding}}

\section{Proof of Lower Bound (\Cref{thm:main_lb})}
\label{sec:formal_lb}
\iftoggle{arxiv_upload}{\input{lower_bound}}{\input{body/lower_bound}}
\section{Algorithm and Proof of Upper Bound (\Cref{thm:main_ub})}
\label{sec:upper_bound_main}
\iftoggle{arxiv_upload}{\input{upper_bound}}{\input{body/upper_bound}}

\section{Conclusion}
\label{sec:conclusion}
\iftoggle{arxiv_upload}{\input{conclusion}}{\input{body/conclusion}}

\iftoggle{arxiv_upload}{\input{acks}}{\input{body/acks}}


\clearpage
\bibliographystyle{plainnat}
\bibliography{main}
\clearpage

\clearpage
\appendix


\ifx\finalversion\undefined
\tableofcontents
\else
\renewcommand{\contentsname}{Contents of Appendix}
\tableofcontents
\addtocontents{toc}{\protect\setcounter{tocdepth}{2}}
\fi

\section{Organization and Notation}
\label{sec:org_notation}
\iftoggle{arxiv_upload}{\input{app_org_notation}}{\input{appendix/app_org_notation}}

\newpage
\part{Technical Tools}
\label{part:technical_tools}
\section{Main Perturbation Bounds}
\label{app:perturbation}
\iftoggle{arxiv_upload}{\input{smoothness}}{\input{appendix/smoothness}}

\iftoggle{arxiv_upload}{\input{derivative_comp}}{\input{appendix/derivative_comp}}

\section{Self-Bounding ODE Method}
\label{app:self_bounding}
\iftoggle{arxiv_upload}{\input{app_self_bounding}}{\input{appendix/app_self_bounding}}

\section{Concentration and Estimation Bounds}
\label{app:ols}
\iftoggle{arxiv_upload}{\input{ols_appendix}}{\input{appendix/ols_appendix}}

\part{Proof of Upper and Lower Bounds}
\label{part:ub_lb_details}

\section{Proofs for Lower Bound (\Cref{sec:formal_lb})}
\label{app:lb_proofs}
\iftoggle{arxiv_upload}{\input{lb_appendix}}{\input{appendix/lb_appendix}}

\section{Proofs for Upper Bound (\Cref{sec:upper_bound_main})}
\label{app:ub_proofs}
\iftoggle{arxiv_upload}{\input{appendix_upper_help}}{\input{appendix/appendix_upper_help}}



\end{document}

%% file: packages.tex
\usepackage[utf8]{inputenc}
\usepackage{mathrsfs}
\usepackage{microtype}
\usepackage{subfigure}
\usepackage{booktabs} 
\usepackage{longtable,makecell}
\usepackage{amsmath, amsfonts, amsthm, amssymb, graphicx}
\usepackage{mathtools,verbatim}
\usepackage{enumitem}
\usepackage{bm}
\usepackage{thm-restate}
\usepackage{natbib}
\usepackage{xspace}
\usepackage[dvipsnames]{xcolor}

\iftoggle{arxiv}
{
	\usepackage{color-edits}
}
{
	\iftoggle{icml}
		{
			\usepackage[suppress]{color-edits}
		}
		{
			\usepackage{color-edits}
			\usepackage{showlabels}
		}

}
\addauthor{df}{ForestGreen}
\addauthor{ms}{orange}

\usepackage[noend]{algpseudocode}
	\usepackage[linesnumbered,ruled,noend]{algorithm2e}

\iftoggle{arxiv}
{
	\usepackage[scaled=.90]{helvet}
	
	\usepackage[vmargin=1.00in,hmargin=1.00in,centering,letterpaper]{geometry}
	\usepackage{fancyhdr, lastpage}

	\setlength{\headsep}{.10in}
	\setlength{\headheight}{15pt}
	\cfoot{}
	\lfoot{}
	\rfoot{\sc Page\ \thepage\ of\ \protect\pageref{LastPage}}

}

\usepackage{hyperref}
\usepackage{cleveref}
\hypersetup{colorlinks,citecolor=blue}
\newtoggle{upperbound}
\togglefalse{upperbound}



%% file: macros.tex
\arxiv{
\newcommand{\algcomment}[1]{\textcolor{blue!70!black}{\footnotesize{\texttt{\textbf{//
          #1}}}}}
\newcommand{\algcommentbig}[1]{\textcolor{blue!70!black}{\footnotesize{\texttt{\textbf{/*
          #1~*/}}}}}
}
\icml{
\newcommand{\algcomment}[1]{\textcolor{blue!70!black}{\scriptsize{\texttt{\textbf{//
          #1}}}}}
\newcommand{\algcommentbig}[1]{\textcolor{blue!70!black}{\scriptsize{\texttt{\textbf{/*
          #1~*/}}}}}
}

\newcommand{\icmlminpt}{\iftoggle{icml}{\vspace{-.1in}}{}}
\newcommand{\epspack}{\epsilon_{\mathrm{pack}}}
\newcommand{\epserr}{\gamma_{\mathrm{err}}}

\newcommand{\ckerr}{c_{\mathrm{err}}}

\newcommand{\Differential}{\mathsf{D}}
\newcommand{\prm}{^{\prime}}
\newcommand{\dprm}{^{\prime\prime}}
\newcommand{\prmtop}{^{\prime\top}}

\newcommand{\epsfro}{\epsilon_{\fro}}

\newcommand{\Bsafe}{B_{\mathrm{safe}}}

\newcommand{\Asafe}{A_{\mathrm{safe}}}

\newcommand{\delK}{\Delta_K}

\newcommand{\norms}{\crl{\op,F}}

\newcommand{\epsls}{\gamma_{\mathrm{ls}}}
\newcommand{\Ke}{K_e}
\newcommand{\Acl}{A_{\mathrm{cl}}}
\newcommand{\delA}{\Delta_{A}}
\newcommand{\delB}{\Delta_{B}}
\newcommand{\epscirc}{\epsilon_{\circ}}
\newcommand{\circnorm}[1]{\|#1\|_{\circ}}

\newcommand{\epsop}{\epsilon_{\op}}
\newcommand{\delAcl}{\Delta_{\Acl}}
\newcommand{\sigmain}{\sigma_{\mathrm{in}}}
\newcommand{\calP}{\mathcal{P}}
\newcommand{\QftT}{\Qf_{t;T}}
\newcommand{\VftT}{\Vf_{t;T}}
\newcommand{\pist}{\pi_{\star}}
\newcommand{\Jabinf}[1]{\Jfunc_{A,B,\infty}[#1]}
\newcommand{\cls}{c_{\,\mathrm{LS}}}
\newcommand{\DARE}{\mathsf{DARE}}

\iftoggle{arxiv}
{
	
}
{
	
}

\newcommand{\Delt}{\eta_{t}}
\newcommand{\DelTmt}{\eta_{T-t}}

\newcommand{\SimpleRegret}{\Exp\mathrm{Regret}}

\newcommand{\Kls}{\widehat{K}_{\mathrm{LS}}}

\newcommand{\cost}{c}

\newcommand{\Qf}{\mathbf{Q}}
\newcommand{\Vf}{\mathbf{V}}

\newcommand{\Popt}{P_{\infty}}

\newcommand{\Jfunc}{\mathcal{J}}
\newcommand{\Jfuncopt}{\Jfunc^{\star}}

\newcommand{\Rx}{R_{\matx}}
\newcommand{\Ru}{R_{\matu}}
\newcommand{\Jst}{\Jfunc_{\star}}

\newcommand{\Jinf}{\Jfunc_{\infty}}

\newcommand{\Regret}{\mathrm{Regret}}

\newcommand{\Fdare}{\mathcal{F}_{\,\DARE}}
\newcommand{\Torus}{\mathbb{T}}
\newcommand{\matzbar}{\overline{\matz}}

\newcommand{\Thetahat}{\widehat{\Theta}}
\newcommand{\Thetast}{\Theta_\star}
\newcommand{\matLambar}{\overline{\matLam}}

\newcommand{\calB}{\mathcal{B}}

\newcommand{\calI}{\mathcal{I}}
\newcommand{\Aclst}{A_{\mathrm{cl},\star}}
\newcommand{\Aclhat}{A_{\mathrm{cl},\widehat{\star}}}
\newcommand{\Aclnot}{A_{\mathrm{cl},0}}

\newcommand{\Hinfty}{\mathcal{H}_{\infty}}

\newcommand{\Symd}{\mathbb{S}^{d}}

\newcommand{\Sympldx}{\mathbb{S}_{++}^{\dimx}}

\newcommand{\Symdx}{\mathbb{S}^{\dimx}}

\newcommand{\Hinf}[1]{\|#1\|_{\mathcal{H}_{\infty}}}

\newcommand{\matE}{\boldsymbol{E}}
\newcommand{\matkappa}{\boldsymbol{\kappa}}

\newcommand{\poly}{\mathrm{poly}}
\newcommand{\calX}{\mathcal{X}}
\newcommand{\calS}{\mathcal{S}}
\newcommand{\calT}{\mathcal{T}}
\newcommand{\matLam}{\boldsymbol{\Lambda}}
\newcommand{\dlyap}{\mathsf{dlyap}}
\newcommand{\matxbar}{\overline{\matx}}

\newcommand{\vecq}{q}

\newcommand{\Kinfmate}{K_{\mate,\infty}}

\newcommand{\opnorm}[1]{\| #1\|_{\op}}
\newcommand{\fronorm}[1]{\| #1\|_{\fro}}

\newcommand{\uvec}{w}
\newcommand{\vvec}{v}
\newcommand{\espace}{\{\minone,1\}^{[n]\times[m]}}

\newcommand{\matx}{\mathbf{x}}
\newcommand{\dimx}{d_{\matx}}
\newcommand{\dimu}{d_{\matu}}
\newcommand{\matu}{\mathbf{u}}
\newcommand{\matdel}{\boldsymbol{\delta}}

\numberwithin{equation}{section}

\newcommand{\calU}{\mathcal{U}}

\newcommand{\mate}{\boldsymbol{e}}
\newcommand{\matehat}{\widehat{\mate}}
\newcommand{\ehat}{\widehat{e}}

\newcommand{\fro}{\mathrm{F}}

\newcommand{\Kstinf}{K_{\star}}

\newcommand{\calV}{\mathcal{V}}

\newcommand{\Ast}{A_{\star}}
\newcommand{\Bst}{B_{\star}}
\newcommand{\Kst}{K_{\star}}
\newcommand{\Pst}{P_{\star}}
\newcommand{\Kinf}{K_{\infty}}
\newcommand{\Pinf}{P_{\infty}}
\newcommand{\clb}{C_{\mathrm{lb}}}

\newcommand{\KTmt}{K_{T-t;T}}

\newcommand{\Pt}{P_{t;T}}
\newcommand{\Ptpl}{P_{t+1;T}}
\newcommand{\Sigmat}{\Sigma_{t;T}}

\newcommand{\textmin}{\text{-}}
\newcommand{\minone}{\text{-}1}

\newcommand{\Ksterr}{\mathrm{K}_{\star}\text{-}{\mathrm{Err}}}

\newcommand{\Kerr}{\mathrm{K}\text{-}{\mathrm{Err}}}

\newcommand{\calN}{\mathcal{N}}
\newcommand{\Khat}{\widehat{K}}

\newcommand{\matDel}{\boldsymbol{\Delta}}

\newcommand{\grad}{\nabla}

\newcommand{\calE}{\mathcal{E}}
\newcommand{\boldknot}{\boldsymbol{k}_0}

\DeclarePairedDelimiter\ceil{\lceil}{\rceil}

{
 \theoremstyle{plain}
      \newtheorem{asm}{Assumption}
}
\creflabelformat{asm}{#2#1#3}

\theoremstyle{plain}
\newtheorem{thm}{Theorem}
\newtheorem{lem}{Lemma}[section]
\newtheorem{cor}{Corollary}
\newtheorem{claim}[lem]{Claim}

\newtheorem{prop}[thm]{Proposition}
\theoremstyle{definition}
\newtheorem{defn}{Definition}[section]

\newtheorem{obs}{Observation}[section]

\renewcommand{\Pr}{\mathbb{P}}
\newcommand{\Exp}{\mathbb{E}}

\newcommand{\unifsim}{\overset{\mathrm{unif}}{\sim}}
\newcommand{\iidsim}{\overset{\mathrm{i.i.d.}}{\sim}}

\newcommand{\KL}{\mathrm{KL}}

\newcommand{\TV}{\mathrm{TV}}
\newcommand{\rmd}{\mathrm{d}}

\newcommand{\range}{\mathrm{range}}

\newcommand{\tr}{\mathrm{tr}}
\newcommand{\rank}{\mathrm{rank}}

\newcommand{\diag}{\mathrm{diag}}
\newcommand{\Diag}{\mathrm{Diag}}
\newcommand{\sign}{\mathrm{sign\ }}
\newcommand{\trn}{\top}

\newcommand{\op}{\mathrm{op}}

\renewcommand{\implies}{\text{ implies }}

\newcommand{\matX}{\mathbf{X}}

\newcommand{\matw}{\mathbf{w}}
\newcommand{\maty}{\mathbf{y}}

\newcommand{\N}{\mathbb{N}}

\newcommand{\C}{\mathbb{C}}
\newcommand{\R}{\mathbb{R}}
\newcommand{\I}{\mathbb{I}}

\DeclareMathOperator{\BigOm}{\mathcal{O}}

\DeclareMathOperator{\BigOmtil}{\widetilde{\mathcal{O}}}

\newcommand{\BigOh}[1]{\BigOm\left({#1}\right)}

\newcommand{\BigOhTil}[1]{\BigOmtil\left({#1}\right)}

\DeclareMathOperator*{\argmin}{arg\,min}

\newcommand{\eigmin}{\lambda_{\mathrm{min}}}

\newcommand{\calF}{\mathcal{F}}

\newcommand{\Alg}{\mathsf{Alg}}

\newcommand{\psdgeq}{\succeq}

\newcommand{\approxgeq}{\gtrsim}

\newcommand{\frakp}{\mathfrak{p}}

\newcommand{\calR}{\mathcal{R}}
\newcommand{\dham}{d_{\mathrm{ham}}}

\newcommand{\OLS}{\mathsf{OLS}}
\newcommand{\ksafe}{k_{\mathrm{safe}}}
\newcommand{\Ballsafe}{\calB_{\mathrm{safe}}}

\newcommand{\SafeRoundInit}{\mathsf{SafeRoundInit}}
\newcommand{\Psibst}{\Psi_{\Bst}}
\newcommand{\matg}{\mathbf{g}}
\newcommand{\Ahat}{\widehat{A}}
\newcommand{\Bhat}{\widehat{B}}
\newcommand{\Conf}{\mathsf{Conf}}
\newcommand{\Atil}{\widetilde{A}}
\newcommand{\Btil}{\widetilde{B}}
\newcommand{\basin}{\mathsf{safe}}

\newcommand{\matz}{\mathbf{z}}
\newcommand{\sftrue}{\mathsf{True}}
\newcommand{\sffalse}{\mathsf{False}}
\newcommand{\Csafe}{C_{\mathrm{safe}}}
\newcommand{\Cest}{C_{\mathrm{est}}}

\newcommand{\Mbar}{\Psi}
\newcommand{\Mbarst}{\Mbar_{\star}}
\newcommand{\gamstab}{\gamma_{\mathrm{sta}}}

      \newcommand{\Cost}{\mathsf{Cost}}
\newcommand{\Ebound}{\mathcal{E}_{\mathrm{bound}}}

\newcommand{\tauls}{\tau_{\mathrm{ls}}}
\newcommand{\Els}{\mathcal{E}_{\mathrm{ls}}}
\newcommand{\deff}{d_{\mathrm{eff}}}
\newcommand{\epsafe}{\epsilon_{\mathrm{safe}}}

\newcommand{\Jbar}{\Jfunc_{\max}}
\newcommand{\Esafe}{\mathcal{E}_{\mathrm{safe}}}
\newcommand{\matgbar}{\overline{\matg}}

\newcommand{\kmax}{k_{\max}}

\newcommand{\ProjMat}{\mathsf{P}}
\newcommand{\gtrsimst}{\gtrsim_{\star}}

\newcommand{\Lamcross}{\Lambda_{\mathrm{cross}}}

\newcommand{\matwbar}{\overline{\matw}}
\newcommand{\Aclk}{A_{\mathrm{cl},k}}

  \newcommand{\Rdiag}{\left[\begin{array}{c} \diag(R_1) \\ \hline \diag(R_2) \end{array}\right]_{\diag}}
  \newcommand{\Lamgbar}{\Lambda_{\matgbar}}
  \newcommand{\Lamgdiag}{\Lambda_{\matgbar,\diag}}

  \newcommand{\Lamxone}{\Lambda_{\matx_1}}

\newcommand{\kfin}{k_{\mathrm{fin}}}

 \newcommand{\Toep}{\mathsf{Toep}}
  \newcommand{\ToepCol}{\mathsf{ToepCol}}
  \newcommand{\calD}{\mathcal{D}}

  \newcommand{\Ereg}{\calE_{\mathrm{reg}}}



%% file: macros2.tex
\DeclarePairedDelimiter{\abs}{\lvert}{\rvert} %
\DeclarePairedDelimiter{\brk}{[}{]}
\DeclarePairedDelimiter{\crl}{\{}{\}}
\DeclarePairedDelimiter{\prn}{(}{)}
\DeclarePairedDelimiter{\nrm}{\|}{\|}

\DeclareMathOperator{\En}{\mathbb{E}}

\newcommand{\wt}[1]{\widetilde{#1}}
\newcommand{\wh}[1]{\widehat{#1}}

\def\ddefloop#1{\ifx\ddefloop#1\else\ddef{#1}\expandafter\ddefloop\fi}
\def\ddef#1{\expandafter\def\csname bb#1\endcsname{\ensuremath{\mathbb{#1}}}}
\ddefloop ABCDEFGHIJKLMNOPQRSTUVWXYZ\ddefloop
\def\ddefloop#1{\ifx\ddefloop#1\else\ddef{#1}\expandafter\ddefloop\fi}
\def\ddef#1{\expandafter\def\csname b#1\endcsname{\ensuremath{\mathbf{#1}}}}
\ddefloop ABCDEFGHIJKLMNOPQRSTUVWXYZ\ddefloop
\def\ddef#1{\expandafter\def\csname c#1\endcsname{\ensuremath{\mathcal{#1}}}}
\ddefloop ABCDEFGHIJKLMNOPQRSTUVWXYZ\ddefloop
\def\ddef#1{\expandafter\def\csname h#1\endcsname{\ensuremath{\widehat{#1}}}}
\ddefloop ABCDEFGHIJKLMNOPQRSTUVWXYZ\ddefloop
\def\ddef#1{\expandafter\def\csname hc#1\endcsname{\ensuremath{\widehat{\mathcal{#1}}}}}
\ddefloop ABCDEFGHIJKLMNOPQRSTUVWXYZ\ddefloop
\def\ddef#1{\expandafter\def\csname t#1\endcsname{\ensuremath{\widetilde{#1}}}}
\ddefloop ABCDEFGHIJKLMNOPQRSTUVWXYZ\ddefloop
\def\ddef#1{\expandafter\def\csname tc#1\endcsname{\ensuremath{\widetilde{\mathcal{#1}}}}}
\ddefloop ABCDEFGHIJKLMNOPQRSTUVWXYZ\ddefloop

\newcommand{\ls}{\ell}

\newcommand{\eps}{\epsilon}
\newcommand{\veps}{\varepsilon}

\newcommand{\ldef}{\vcentcolon=}


%% file: body/abstract.tex

We consider the problem of online adaptive control of the linear quadratic regulator, where the true system parameters are unknown. We prove new upper and lower bounds demonstrating that the optimal regret scales as $\widetilde{\Theta}({\sqrt{d_{\mathbf{u}}^2 d_{\mathbf{x}} T}})$, where $T$ is the number of time steps, $d_{\mathbf{u}}$ is the dimension of the input space, and $d_{\mathbf{x}}$ is the dimension of the system state.  Notably, our lower bounds rule out the possibility of a $\mathrm{poly}(\log{}T)$-regret algorithm, which had been conjectured due to the apparent strong convexity of the problem. Our upper bound is attained by a simple variant of \emph{certainty equivalent control}, where the learner selects control inputs according to  the optimal controller for their estimate of the system while injecting exploratory random noise
\iftoggle{icml}{ \cite{mania2019certainty}.}{. While this approach was shown to achieve $\sqrt{T}$-regret by \cite{mania2019certainty}, we show that if the learner continually refines their estimates of the system matrices, the method attains optimal dimension dependence as well.}

Central to our upper and lower bounds is a new approach for controlling perturbations of Riccati equations called the \emph{self-bounding ODE method}, which we use to derive suboptimality bounds for the certainty equivalent controller synthesized from estimated system dynamics. 
This in turn enables regret upper bounds which hold for \emph{any stabilizable instance} and scale with natural control-theoretic quantities.


%% file: body/introduction.tex
\section{Introduction}


Reinforcement learning has recently achieved great success in application domains including Atari \citep{mnih2015human}, Go
\citep{silver2016mastering}, and robotics \citep{lillicrap2015continuous}. All of these breakthroughs leverage data-driven methods for continuous control in large state spaces. Their success, along with challenges in deploying RL in the real world, has led to renewed interest on developing continuous control algorithms with improved reliability and sample efficiency.  In particular, on the theoretical side, there has been a push to develop a non-asymptotic theory of data-driven continuous control, with an emphasis on understanding key algorithmic principles and fundamental limits.

In the non-asymptotic theory of reinforcement learning, much attention has been focused on the so-called ``tabular'' setting where states and actions are discrete, and the optimal rates for this setting are by now relatively well-understood \citep{jaksch2010near,dann2015sample,azar2017minimax}. Theoretical results for continuous control setting have been more elusive, with progress spread across various models \citep{kakade2003exploration,munos2008finite,jiang2017contextual,jin2019provably}, but the linear-quadratic regulator (LQR) problem has recently emerged as a candidate for a standard benchmark for continuous control and RL. For tabular reinforcement learning problems, it is widely understood that careful exploration is essential for sample efficiency. Recently, however, it was shown that for the online variant of the LQR problem, relatively simple exploration strategies suffice to obtain the best-known performance guarantees \citep{mania2019certainty}. In this paper, we address a curious question raised by these results: Is sophisticated exploration helpful for LQR, or is linear control in fact substantially easier than the general reinforcement learning setting? More broadly, we aim to shed light on the question:
\iftoggle{icml}{\vspace{-.2in}}{}
\begin{quote}
\begin{center}
\textit{To what extent to do sophisticated exploration strategies improve learning in online linear-quadratic control?}
\end{center}
\end{quote}
\iftoggle{icml}{\vspace{-.1in}}{}
\paragraph{Is $\veps$-Greedy Optimal for Online LQR?} In the LQR problem, the system state $\matx_t$ evolves according \iftoggle{icml}{to}{to the dynamics}
\begin{align}
\matx_{t+1} = A \matx_t + B \matu_t + \matw_t, \quad\text{where}\quad \matx_1 = 0 \label{eq:dynamics},
\end{align}
and where $\matu_t\in\bbR^{\dimu}$ is the learner's control input, $\matw_t \in \R^{\dimx}$ is a noise process drawn as $\matw_t \iidsim \calN(0, I)$, and $A\in \R^{\dimx\times\dimx}$, $B\in\R^{\dimx\times{}\dimu}$ are unknown system matrices.

Initially the learner has no knowledge of the system dynamics, and their goal is to repeatedly select control inputs and observing states over $T$ rounds so as to minimize their total cost  $\sum_{t=1}^{T}c(\matx_t,\matu_t)$, where $c(x,u)=x^{\trn}\Rx{}x+u^{\trn}\Ru{}u$ is a known quadratic function. In the online variant of the LQR problem, we measure performance via \emph{regret} to the optimal linear controller:
\begin{align}
\Regret_{A,B,T}[\pi] = \left[\sum_{t=1}^T\cost(\matx_t,\matu_t)\right] - T\, \min_{K}\Jfunc_{A,B}[K],\label{eq:reg_def}
\end{align}
where $K$ is a linear state feedback policy \iftoggle{icml}{}{of the form $\matu_t =  K\matx_t$ }and---letting $\Exp_{A,B,K}[\cdot]$ denote expectation under this policy---where
\begin{align*}
\Jfunc_{A,B}[K] := \lim_{T \to \infty}\frac{1}{T}\Exp_{A,B,K}\left[ \sum_{t=1}^T \cost(\matx_t,\matu_t)  \right],
\end{align*}
is the average infinite-horizon cost of $K$, which is finite as long as $K$ is \emph{stabilizing} in the sense that $\rho(A + B{}K) < 1$, where $\rho(\cdot)$ denotes the spectral radius.\footnote{For potentially asymmetric matrix $A \in \R^{d \times d}$, $\rho(A)\ldef\max\crl*{\abs*{\lambda}\mid\text{$\lambda$ is an eigenvalue for $A$}}$.} We further define $\Jfuncopt_{A,B} := \min_K \Jfunc_{A,B}[K] $.


This setting has enjoyed substantial development beginning with the work of \cite{abbasi2011regret}, and following a line of successive improvements \citep{dean2018regret,faradonbeh2018input,cohen2019learning,mania2019certainty}, the best known algorithms for online LQR have regret scaling as $\sqrt{T}$.

We investigate a question that has emerged from this research: The role of exploration in linear control. The first approach in this line of work, \cite{abbasi2011regret}, proposed a sophisticated though computationally inefficient strategy based on \emph{optimism in the face of uncertainty}, upon which \cite{cohen2019learning} improved to ensure optimal $\sqrt{T}$-regret and polynomial runtime. Another approach which enjoys $\sqrt{T}$-regret, due to \cite{mania2019certainty}, employs a variant of the classical $\veps$-greedy exploration strategy \citep{sutton2018reinforcement} known in control literature as \emph{certainty equivalence}: At each timestep, the learner computes the greedy policy for the current estimate of the system dynamics, then follows this policy, adding exploration noise proportional to $\veps$. While appealing in its simplicity, $\veps$-greedy has severe drawbacks for general reinforcement learning problems: For tabular RL, it leads to exponential blowup in the time horizon \citep{kearns2000approximate}, and for multi-armed bandits\iftoggle{icml}{, bandit linear optimization, and contextual bandits}{ and variants such as bandit linear optimization and contextual bandits}, it leads to suboptimal dependence on the time horizon $T$ \citep{langford2007epoch}. 

This begs the question: Can we improve beyond $\sqrt{T}$ regret for online LQR using more sophisticated exploration strategies? Or is exploration in LQR simply much easier than in general reinforcement learning settings? One natural hope would be to achieve logarithmic (i.e. $\poly(\log{}T)$) regret. After all, online LQR has strongly convex loss functions, and this is a sufficient condition for logarithmic regret in many simpler online learning and optimization problems  \citep{vovk2001competitive,hazan2007logarithmic,rakhlin2014online}, as well as LQR with known dynamics but potentially changing costs \citep{agarwal2019logarithmic}. More subtly, the $\sqrt{T}$ online LQR regret bound of \cite{mania2019certainty} requires that the pair $(\Ast,\Bst)$ be \emph{controllable};\footnote{$(\Ast,\Bst)$ are said to be controllable if and only the \emph{controllability Gramian} $\mathcal{C}_n\mathcal{C}_n^\top :=  \sum_{i=0}^n \Ast^i \Bst\Bst^\top (\Ast^i)^\top$ is strictly positive definite for some $n \ge 0$.  For any $n$ for which $\mathcal{C}_n \succ 0$, the upper bounds of \cite{mania2019certainty}  scale polynomially in  $n, 1/\lambda_{\min}(\mathcal{C}_n\mathcal{C}_n^\top)$. Controllability implies stabilizability, but the converse is not true.} it was not known if naive exploration attains this rate for arbitrary \emph{stabilizable} problem instances, or if it necessarily leverages controllability to ensure its efficiency.

\subsection{Contributions}
We prove new upper and lower bounds which characterize the minimax optimal regret for online LQR as $\wt{\Theta}(\sqrt{\dimu^2 \dimx T})$. Beyond dependence on the horizon $T$, dimensions $\dimx,\dimu$, and logarithmic factors, our bounds depend only on \emph{operator} norms of transparent, control theoretic quantities,  which do not hide additional dimension dependence. Our main lower bound is \Cref{thm:main_lb}, which implies that no algorithm can improve upon $\sqrt{T}$ regret for online LQR, and so simple $\veps$-greedy exploration is indeed \emph{rate-optimal}.
\newtheorem*{thm:lb_informal}{Theorem \ref*{thm:main_lb} (informal)}
\begin{thm:lb_informal}
  For every sufficiently non-degenerate problem instance and every (potentially randomized) algorithm, there exists a nearby problem instance on which the algorithm must suffer regret at least $\wt{\Omega}(\sqrt{\dimu^2\dimx{}T})$.
\end{thm:lb_informal}
Perhaps more surprisingly, our main upper bound shows that a simple variant of certainty equivalence is also \emph{dimension-optimal}, in that it asymptotically matches the $\sqrt{\dimu^2 \dimx T}$ lower bound of \Cref{thm:main_lb}.
\newtheorem*{thm:ub_informal}{Theorem \ref*{thm:main_ub} (informal)}
\begin{thm:ub_informal}
  Certainty equivalent control with continual $\veps$-greedy exploration (\Cref{alg:ce}) has regret at most $\wt{O}\prn*{
    \sqrt{\dimu^{2}\dimx{}T} + \dimx^{2}
    }$ for every stabilizable online LQR instance.
\end{thm:ub_informal}
Our upper bound \emph{does not} require controllability, and is the first bound for \emph{any} algorithm to attain the optimal dimension dependence. In comparison, result of \cite{mania2019certainty} guarantees $\sqrt{(\dimx+\dimu)^3 T}$ regret \emph{and} imposes strong additional assumptions. In the many control settings where $\dimu \ll \dimx$, our bound constitutes a significant improvement. Other approaches \emph{not} based on certainty equivalence suffer considerably larger dimension dependence \citep{cohen2019learning}. Together, \Cref{thm:main_lb} and \Cref{thm:main_ub} characterize the asymptotic minimax regret for online LQR, showing that there is little room for improvement over naive exploration.

Our results leverage a new perturbation bound for controllers synthesized via certainty equivalence. Unlike prior bounds due to \cite{mania2019certainty}, our guarantee depends only on natural control-theoretic quantities, and crucially does not require controllability of the system.
\newtheorem*{thm:ce_informal}{Theorem \ref*{thm:main_perturb_simple} (informal)}
\begin{thm:ce_informal}
  Fix an instance $(A,B)$.  Let $(\Ahat,\Bhat)$, and let $\Khat$ denote the optimal infinite horizon  controller from instance $(\Ahat,\Bhat)$. Then if $(\Ahat,\Bhat)$ are sufficiently close to $(A,B)$, we have
 	\begin{align*}
 	\Jfunc_{A,B}[\Khat] - \Jfuncopt_{A,B} \le 142\|P\|_{\op}^{8}\cdot(\|\Ahat - A\|_{\fro}^2 + \|\Bhat - B\|_{\fro}^2),
 	\end{align*}
 	where $P$ is the solution to the $\DARE$ for the system $(A,B)$.
\end{thm:ce_informal}
For simplicity, the bound above assumes the various normalization conditions on the noise and cost matrices, described in \Cref{sec:problem_setting}. With these conditions, our perturbation bound only requires that the operator norm distance between $(\Ahat,\Bhat)$ and $(A,B)$ be at most $1/\mathrm{poly}(\|P\|_{\op})$. Hence, we establish perturbation bounds for which both the scaling of the deviation and the region in which the bound applies can be quantified in terms of a single quantity: the norm of $\DARE$ solution $P$. We prove this bound through a new technique we term the \emph{Self-Bounding ODE method}, described below. Beyond removing the requirement of controllability, we believe this method is simpler and more transparent than past approaches. 

\subsection{Our Approach}
Both our lower and upper bounds are facilitated by the \emph{self-bounding ODE method}, a new technique for establishing perturbation bounds for the Riccati equations that characterize the optimal value function and controller for LQR. The method sharpens existening perturbation bounds, weakens controllability and stability assumptions required by previous work \citep{dean2018regret,faradonbeh2018input,cohen2019learning,mania2019certainty}, and yields an upper bound whose leading terms depend only on the horizon $T$, dimension parameters $\dimx,\dimu$, and the control-theoretic parameters sketched in the prequel.  

In more detail, if $(A,B)$ is stabilizable and $\Rx,\Ru \succ 0$, there exists a unique \arxiv{positive semidefinite}\icml{PSD} solution $\Pinf(A,B)$ for the \emph{discrete algebraic Riccati equation} (\arxiv{or, }$\DARE$),
\iftoggle{icml}
	{\begin{multline}
	\label{eq:dare}
	P  = A^\top P A + \Rx \\
	- A^\top P B (\Ru + B^\top P B)^{-1}B^\top P A 
	\end{multline}
	}
	{
	\begin{align}
	\label{eq:dare}
	P  = A^\top P A - A^\top P B (\Ru + B^\top P B)^{-1}B^\top P A + \Rx.
	\end{align}
	}
	The unique optimal infinite-horizon controller \arxiv{$\Kinf(A,B):=  $ $\argmin_{K} \Jabinf{K}$ }is given by \[\Kinf(A,B) = -(\Ru + B^\top \Popt(A,B) B)^{-1}B^\top \Popt(A,B) A,\] and the matrix $\Pinf(A,B)$ induces a positive definite quadratic form which can be interpreted as a value function  for the LQR problem.

Both our upper and lower bounds make use of novel perturbation bounds to control the change in $\Pinf$ and $\Kinf$ when we move from a nominal instance $(A,B)$ to a nearby instance $(\Ahat,\Bhat)$. For our upper bound, these are used to show that a good estimator for the nominal instance leads to a good controller, while for our lower bounds, they show that the converse is true.  The self-bounding ODE method allows us to prove perturbation guarantees that depend only on the norm of the value function $\nrm*{\Popt(A,B)}_{\op}$ for the nominal instance, which is a weaker assumption that subsumes previous conditions. The key observation underpinning the method is that the norm of the directional derivative of $\frac{d}{dt}\Popt(A(t),B(t))\big{|}_{t=u}$ at a point $t = u$ along a line $(A(t),B(t))$ 
is bounded in terms of the magnitude of $\|\Popt(A(u),B(u))\|$; we call this the \emph{self-bounding} property. From this relation, we show that bounding the norm of the derivatives reduces to solving a scalar ordinary differential equation, whose derivative saturates the scalar analogue of this self-bounding property. Notably, this technique does not require that the system be controllable, and in particular does not yield guarantees which depend on the smallest singular value of the controllability matrix as in \cite{mania2019certainty}. Moreover, given estimates  $(\Ahat,\Bhat)$ and an upper-bound on their deviation from the true system $(\Ast,\Bst)$, our bound allows the learner to check whether the certainty-equivalent controller synthesized from $\Ahat,\Bhat$ stabilizes the true system and satisfies the preconditions for our perturbation bounds.  


{On the lower bound side, we begin with a nominal instance $(A_0,B_0)$ and consider a packing of alternative instances within a small neighborhood. Specifically, if $K_0$ is the optimal controller for $(A_0,B_0)$, we consider perturbations of the form $(A_{\Delta},B_{\Delta}) = (A_0 - \Delta K_0, B_0 + \Delta)$ for $\Delta \in \R^{\dimu\dimx}$.  The self-bounding ODE method facilitates a perturbation analysis which implies that the optimal controller $K_{\Delta}$ on each alternative $(A_{\Delta},B_{\Delta})$ deviates from $K_0$ by $\|K_0 - K_{\Delta}\|_{\fro} \ge \Omega(\|\Delta\|_{\fro})$ for non-degenerate instances. Using this reasoning, we show that any low-regret algorithm can approximately recover the perturbation $\Delta$.}

On the other hand, if the learner selects inputs   $\matu_t = K_0 \matx_t$ according to the optimal control policy for the nominal instance, all alternatives are \emph{indistinguishable} from the nominal instance. Indeed, the structure of our perturbations ensures that $A_\Delta + B_\Delta K_0 = A_0 + B_0 K_0$ for all choices of $\Delta$. Thus, since low regret implies identification of the perturbation, any low regret learner must substantially deviate  from the nominal controller $K_0$. Equivalently, this can be understood as a consequence of the fact that playing $\matu_t = K_0 \matx_t$ yields a degenerate covariance matrix for the random variable $(\matx_t,\matu_t)$, and thus some deviation from $K_0$ is required to ensure this covariance is full rank. The regret scales proportionally to the deviation from $K_0$, which scales proportionally to the minimum eigenvalues of the aforementioned covariance matrix, but the estimation error rate scales as $1/T$ (the typical ``fast rate'') times the \emph{inverse} of these eigenvalues. 
Balancing the tradeoffs leads to the ``slow'' $\sqrt{T}$ lower bound. {Crucially, our argument exploits a fundamental tension between control and indentification in linear systems, first described by \citet{polderman1986necessity}, and summarized in \citet{polderman1989adaptive}.}




Our upper bound refines the certainty equivalent control strategy proposed in \cite{mania2019certainty} by re-estimating the system parameters on a doubling epoch schedule to advantage of the endogenous excitation supplied by the $\matw_t$-sequence. A careful analysis of the least squares estimator shows that the error in a $\dimx \dimu$-dimensional subspace decays as $\BigOh{1/\sqrt{t}}$, and in the remaining $\dimx^2$ dimensions decays at a \emph{fast rate} of $\BigOh{1/t}$. \iftoggle{icml}{}{The former rate yields the desired regret bound, and the latter contributes at most logarithmically. The novel perturbation analysis described above obviates the need for additional assumptions or prior knowledge about the system.\footnote{Without the doubling epoch schedule, we can obtain a near optimal rate of $\sqrt{(\dimx+\dimu)\dimx\dimu T}$, slightly improving upon \cite{mania2019certainty} in dimension dependence, and also without the need for additional assumptions}
}

\paragraph{Related Work}
Non-asymptotic guarantees for learning linear dynamical
systems have been the subject of intense recent interest \citep{dean2017sample,hazan2017learning,tu2017least,hazan2018spectral,simchowitz2018learning,sarkar2018fast,simchowitz2019learning,mania2019certainty,sarkar2019fast}. The online LQR setting we study was introduced by \cite{abbasi2011regret}, which considers the problem of controlling an unknown linear system under stationary stochastic noise.\footnote{A more recent line of work studies a more general \emph{non-stochastic} noise regime (see \cite{agarwal2019online} et seq.), which we do not consider in this work.}  They showed that an algorithm based on the optimism in the face of uncertainty (OFU) principle enjoys $\sqrt{T}$, but their algorithm is computationally inefficient and their regret bound depends exponentially on dimension. The problem was revisited by \cite{dean2018regret}, who showed that an explicit explore-exploit scheme based on $\veps$-greedy exploration and certainty equivalence achieves $T^{2/3}$ regret efficiently, and left the question of obtaining $\sqrt{T}$ regret efficiently as an open problem. This issue was subsequently addressed by \cite{faradonbeh2018input} and \cite{mania2019certainty}, who showed that certainty equivalence obtains $\sqrt{T}$ regret, and \cite{cohen2019learning}, who achieve $\sqrt{T}$ regret using a semidefinite programming relaxation for the OFU scheme. The regret bounds in \cite{faradonbeh2018input} do not specify dimension dependence, and (for $\dimx \ge \dimu$), the dimension scaling of \cite{cohen2019learning} can be as large as $\sqrt{\dimx^{16} T}$;\footnote{The regret bound of \cite{cohen2019learning} scales as $\dimx^3\sqrt{T}\cdot (\Jfuncopt_{\Ast,\Bst})^5$; typically, $\Jfuncopt_{\Ast,\Bst}$ scales linearly in $\dimx$} \cite{mania2019certainty} incurs an almost-optimal dimension dependence of $\sqrt{\dimx^3 T}$ (suboptimal when $\dimu \ll \dimx$), but at the expense of imposing a strong controllability assumption. 

The question of whether regret for online LQR could be improved further (for example, to $\log{}T$) remained open, and was left as a conjecture by \cite{faradonbeh2018optimality}. Our lower bounds resolve this conjecture by showing that $\sqrt{T}$-regret is optimal. Moreover, by refining the upper bounds of \cite{mania2019certainty}, our results show that the asymptotically optimal regret is $\wt{\Theta}(\sqrt{\dimu^2\dimx{}T})$, and that this achieved by certainty equivalence. Beyond attaining the optimal dimension dependence, our upper bounds also enjoy refined dependence on problem parameters, and do not require a-priori knowledge of these parameters.

Logarithmic regret bounds are ubiquitous in online learning and optimization problems with strongly convex loss functions \citep{vovk2001competitive,hazan2007logarithmic,rakhlin2014online}.  \cite{agarwal2019logarithmic} demonstrate that for the problem of controlling an \emph{known} linear dynamic system  with adversarially chosen, strongly convex costs, logarithmic regret is also attainable. Our $\sqrt{T}$ lower bound shows that the situation for the online LQR with an \emph{unknown} system parallels that of bandit convex optimization, where \cite{shamir2013complexity} showed that $\sqrt{T}$ is optimal even for strongly convex quadratics. That is, in spite of strong convexity of the losses, issues of partial observability prevent fast rates in both settings.

Our lower bound carefully exploits the online LQR problem structure to show that $\sqrt{T}$ is optimal. To obtain optimal dimension dependence for the lower bound, we build on well-known lower bound technique for adaptive sensing based on Assouad's lemma \citep{arias2012fundamental} (see also \cite{assouad1983deux,yu1997assouad}).

Finally, a parallel line of research provides Bayesian and frequentist regret bounds for online LQR based on Thompson sampling \citep{ouyang2017control,abeille2017thompson}, with \cite{abeille2018improved} demonstrating $\sqrt{T}$-regret for the scalar setting. Unfortunately, Thompson sampling is not computationally efficient for the LQR.

\subsection{Organization}
\Cref{sec:problem_setting} introduces basic notation and definitions. \Cref{sec:main_results} introduces our main results: In \Cref{ssec:main_lb} and \Cref{ssec:main_ub} we state our main lower and upper bounds respectively and give an overview of the proof techniques, and in \Cref{ssec:main_results_strong} we instantiate and compare these bounds for the simple special case of strongly stable systems. 
\iftoggle{icml}
{
	In \Cref{sec:self_bounding} we introduce the self-bounding ODE method and show how it is used to prove key perturbation bounds used in our main results. All additional proofs and proof details are given in the appendix, whose organization is described at length in \Cref{sec:org_notation}. Future directions and open problems are discussed in \Cref{sec:conclusion}.
}
{
	The remainder of the paper is devoted to proving these results. In \Cref{sec:self_bounding} we introduce the self-bounding ODE method and show how it is used to prove key perturbation bounds used in our main results; additional details are given in \Cref{part:technical_tools} of the appendix. Detailed proofs for the lower and upper bound are given in \Cref{sec:formal_lb} and \Cref{sec:upper_bound_main}, with additional proofs deferred to \Cref{app:lb_proofs} and \Cref{app:ub_proofs}. Finally, future directions and open problems are discussed in \Cref{sec:conclusion}.

}


%% file: body/problem_setting.tex

\subsection{Preliminaries}
\label{sec:problem_setting}

\paragraph{Assumptions}
We restrict our attention \emph{stabilizable} systems $(A,B)$ for which there exists a stabilizing controller $K$ such that $\rho(A+BK)<1$. Note that this does not require that the system be controllable. We further assume that $\Ru =  I$ and  $\Rx \succeq I$. The first can be enforced by a change of basis in input space, and the second can be enforced by rescaling the state space, increasing the regret by at most a multiplicative factor of  $\min\left\{1,1/\sigma_{\min}(\Rx)\right\}$. We also assume that the process noise $\matw_t$ has identity covariance. We note that non-identity noise can be adressed via a change of variables, and  in \Cref{app:general_noise} we sketch extensions of our results to (a) independent, sub-Gaussian noise with bounded below covariance, and (b) more general martingale noise, where we remark on how to achieve optimal rates in the regime $\dimx \lesssim \dimu^2$.

\paragraph{Algorithm Protocol and Regret} Formally, the learner's (potentially randomized) decision policy is modeled as a sequence of mappings $\pi = (\pi_t)_{t=1}^{T}$, where each function $\pi_t$ maps the history $(\matx_1,\dots,\matx_t,\matu_1,\dots,\matu_{t-1})$ and an internal random seed $\xi$ to an output control signal $\matu_t$. For a linear system evolving according to Eq.~\eqref{eq:dynamics} and policy $\pi$, we let $\Pr_{A,B,\pi}$ and $\Exp_{A,B,\pi}\left[ \cdot \right]$ denote the probability and expectation with respect to the dynamics~\eqref{eq:dynamics} and randomization of $\pi$. For such a policy, we use the notation $\Regret_{A,B,T}[\pi]$ as in Eq.~\eqref{eq:reg_def} for regret, which is a random variable with law $\Pr_{A,B,\pi}[\cdot]$. We prove high-probability upper bounds on $\Regret_{A,B,T}[\pi]$, and prove lower bounds on the expected regret $\SimpleRegret_{A,B,T}[\pi] := \Exp_{A,B,\pi}[\Regret_{A,B,T}[\pi]]$.\footnote{One might consider as a stronger benchmark described the expected loss of the optimal policy for \emph{fixed horizon} $T$. A fortiori, our lower bounds apply for this benchmark as well: In view of the proof of Lemma~\ref{lem:Kerr_lem} in \Cref{ssec:lem:Kerr_lem}, this benchmark differs from  $T \, \Jfuncopt_{A,B}$ by a constant factor which depends on $(A,B)$ but \emph{does not grow with } $T$.}

\paragraph{Additional Notation}
For vectors $x\in \R^d$, $\nrm*{x}$ denotes the $\ls_2$ norm. For matrices $X\in\R^{d_1\times{}d_2}$, $\nrm*{X}_{\op}$ denotes
the spectral norm, and $\nrm*{X}_{\fro}$ the Frobenius norm.  When $d_1\le d_2$,
$\sigma_1(X),\ldots,\sigma_{d_1}(X)$ denote the singular values of $X$, arranged in decreasing order. We say $f \lesssim g$ to denote that $f(x) \le C g(x) $ for a universal constant $C$, and $f \lessapprox g$ to denote informal inequality. We write $f \eqsim g$ if $g \lesssim f \lesssim g$.

For ``starred'' systems $(\Ast,\Bst)$, we adopt the shorthand $\Pst := \Pinf(\Ast,\Bst)$, $\Kst := \Kinf(\Ast,\Bst)$ for the optimal controller, $\Jst := \Jfuncopt_{\Ast,\Bst} := \Jfunc_{\Ast,\Bst}[\Kst]$ for optimal cost, and $\Aclst := \Ast + \Bst \Kst$ for the optimal closed loop system. We define $\Mbarst := \max\{1,\opnorm{\Ast},\opnorm{\Bst}\}$ and $\Psibst := \max\{1,\opnorm{\Bst}\}$. For systems $(A_0, B_0)$, we let  $\calB_{\op}(\epsilon;A_0,B_0)=\crl{(A,B)\mid{}\nrm*{A-A_0}_{\op}\vee\nrm*{B-B_0}_{\op}\leq{}\eps}$ denote the set of nearby systems in operator norm.


%% file: body/intuition.tex

We now state our main upper and lower bounds for online LQR and give a high-level overview of the proof techniques behind both results. At the end of the section, we instantiate and compare the two bounds for the simple special case of strongly stable systems.

Both our upper and lower bounds \iftoggle{icml}{start with}{are motivated by} the following question: Suppose that the learner is selecting near optimal control inputs $\matu_t \approx \Kst \matx_t$, where $\Kst = \Kinf(\Ast,\Bst)$ is the optimal controller for the system $(\Ast,\Bst)$. What  information can she glean about the system?

\subsection{Lower Bound}
\label{ssec:main_lb}
We provide a \emph{local
  minimax} lower bound, which captures the difficulty of ensuring low
regret on both a \emph{nominal instance} $(\Ast,\Bst)$ and on the hardest
nearby alternative. For a distance parameter $\eps>0$, we define the local
minimax complexity at scale $\eps$ as
\iftoggle{icml}
{
  \begin{multline*}
    \calR_{\Ast,\Bst,T}(\eps) :=  \min_{\pi}\max_{A,B}\Big\{\SimpleRegret_{A,B,T}[\pi]:\\
\fronorm{A - \Ast}^2 \vee \fronorm{B - \Bst}^2  \le \epsilon\Big\}.
\end{multline*}
}
{
  \begin{align*}
\calR_{\Ast,\Bst,T}(\eps) :=  \min_{\pi}\max_{A,B}\left\{\SimpleRegret_{A,B,T}[\pi]:\fronorm{A - \Ast}^2 \vee \fronorm{B - \Bst}^2  \le \epsilon\right\}.
\end{align*}
}
Local minimax complexity captures the idea certain instances $(\Ast,\Bst)$ are more difficult than others, and allows us to provide lower bounds that
scale only with control-theoretic parameters of the nominal instance. Of
course, the local minimax lower bound immediately implies a lower
bound on the global minimax complexity as well.\footnote{Some
  care must be taken in defining the global complexity, or it may well
  be infinite. One sufficient definition, which captures prior work,
  is to consider minimax regret over all instances subject to a global
  bound on $\nrm*{\Pst}$, $\nrm*{\Bst}$, and so on.}

\paragraph{Intuition Behind the Lower Bound.}
We show that if the learner plays near-optimally on every instance in the neighborhood of $(\Ast,\Bst)$, then there is a $\dimx \dimu$-dimensional subspace of system parameters that the learner must explore by deviating from $\Kst$ when the underlying instance is $(\Ast,\Bst)$. Even though the system parameters can be estimated at a fast rate, such deviations preclude logarithmic regret.


In more detail, if the learner plays near-optimally, she is not be able to distinguish between whether the instance she is interacting with is $(\Ast,\Bst)$, or another system of the form
  \begin{align}
  (A,B) = (\Ast - \Kst \Delta, \Bst + \Delta),\label{eq:B_perturbation}
  \end{align}
  for some perturbation $\Delta\in\bbR^{\dimx\times\dimu}$. This is because all the obsevations $(\matx_t,\matu_t)$ generated by the optimal controller lie in the subspace $\crl*{(x,u) : u - \Kst x = 0}$, and likewise all observations generated by any near-optimal controller approximately lie in this subspace. Since the learner cannot distinguish between $(\Ast,\Bst)$ and $(A,B)$, she will also play $\matu_t \approx \Kst \matx_t$ on $(A,B)$. This leads to poor regret when the  instance is $(A,B)$, since the optimal controller in this case has $\matu_t =\Kinf(A,B)\matx_t$. This is made concrete by the next lemma, which shows to a first-order approximation that if $\Delta$ is large, the distance between $\Kst$ and $\Kinf(A,B)$ must also be large.
  \begin{restatable}[Derivative Computation (\citet{abeille2018improved}, Proposition 2)] {lem}{lowerbounddercomp}\label{lem:lower_bound_dercomp}Let $(\Ast,\Bst)$ be stabilizable\iftoggle{icml}
  {, and recall $\Aclst := \Ast + \Bst\Kst$. Then,
    \begin{multline*}
    \frac{d}{dt}\Kinf(\Ast - t\Delta \Kst, \Bst + t\Delta)\big{|}_{t=0} \\
    = -(\Ru + \Bst^\top \Pst \Bst)^{-1} \cdot \Delta^\top \Pst \Aclst.
    \end{multline*}
  }
  {. Then
    \begin{align*}
    \frac{d}{dt}\Kinf(\Ast - t\Delta \Kst, \Bst + t\Delta)\big{|}_{t=0} = -(\Ru + \Bst^\top \Pst \Bst)^{-1} \cdot \Delta^\top \Pst \Aclst, 
    \end{align*}
     where we recall $\Aclst := \Ast + \Bst\Kst$.
  }

  \end{restatable} 

  In particular, when the closed loop system $\Aclst$ is (approximately) well-conditioned,
the optimal controllers for $(\Ast,\Bst)$ and for $(A,B)$ are $\Omega(\fronorm{\Delta})$-apart, and so the learner cannot satisfy both $\matu_t \approx \Kst\matx_t$ and $\matu_t \approx \Kinf(A,B)\matx_t$ simultaneously. More precisely, for the learner to ensure $\sum_{t}\|\matx_t - \Kinf(A,B)\matu_t\|_{\fro}^2 \lessapprox \dimx \dimu \epsilon^2$ on every instance, she must deviate from optimal by at least $\sum_{t=1}^T\|\matx_t - \Kst \matu_t\|_\fro^2 \gtrapprox \dimu T/\epsilon^2$ on the optimal instance;  the $\dimu$ factor here comes from the necessity of exploring all control-input directions. Balancing these terms leads to the final $\Omega(\sqrt{T\dimu \dimx^2})$ lower bound\iftoggle{icml}{ (proven in \Cref{sec:formal_lb}).}{. The formal proof is given in \Cref{sec:formal_lb}.}
  
  %
  \begin{thm}\label{thm:main_lb} Let $c_1,p > 0$ denote universal constants. For $m \in [\dimx]$, define $\nu_m := \sigma_{m}(\Aclst)/\opnorm{\Ru + \Bst^\top \Pst\Bst}$. Then if $\nu_m > 0$, we have 
  \iftoggle{icml}
  {
    \begin{align*}
    \calR_{\Ast,\Bst,T}\left(\eps_T\right) \gtrsim \sqrt{\dimu^2mT} \cdot \frac{1 \wedge \nu_m^2 }{\opnorm{\Pst}^2}, 
    \end{align*}
    where $\eps_T = \sqrt{\dimu^2 m/T}$, provided that $T$ is at least $c_1 (\opnorm{\Pst}^{p}(\dimu m \vee \frac{\dimx^2\Psibst^4 (1\vee\nu_m^{-4})}{m \dimu^2} \vee \dimx \log(1+\dimx \opnorm{\Pst})) $. 
  }
  {
    \begin{align*}
    \calR_{\Ast,\Bst,T}\left(\eps_T\right) \gtrsim \sqrt{\dimu^2mT} \cdot \frac{1 \wedge \nu_m^2 }{\opnorm{\Pst}^2},  \quad \text{where } \eps_T = \sqrt{\dimu^2 m/T},
    \end{align*}
    provided that $T \ge c_1 (\opnorm{\Pst}^{p}(\dimu^2 m \vee \frac{\dimx^2\Psibst^4 (1\vee\nu_m^{-4})}{m \dimu^2} \vee \dimx \log(1+\dimx \opnorm{\Pst})) $. 

  }

\end{thm}
Let us briefly discuss some key features of \Cref{thm:main_lb}.
\begin{itemize}
\item \icmlminpt The only system-dependent parameters appearing in the
  lower bound are the operator norm bounds $\Psibst$ and $\opnorm{\Pst}$, which only depend on the nominal instance. The latter parameter
  is finite whenever the system is stabilizable, and does not
  explicitly depend on the spectral radius or strong stability
  parameters.
\item The lower bound takes $\eps_T \propto T^{-1/2}$,
  so the alternative instances under consideration converge to the
  nominal instance $(\Ast,\Bst)$ as $T \to \infty$.
\item The theorem 
  can be optimized for each instance by tuning the dimension parameter $m \in [\dimx]$: The leading
  $\sqrt{\dimu^{2}mT}$ term is increasing in $m$, while the parameter
  $\nu_m$ scales with $\sigma_{m}(\Aclst)$ and thus is decreasing in
  $m$. The simplest case is when $\sigma_{m}(\Aclst)$ is bounded away
  from $0$ for $m \gtrsim \dimx$; here we
  obtain the optimal $\sqrt{\dimu^2 \dimx T}$ lower bound. In particular, if $\dimu \le \dimx/2$, we can choose $m=\frac{1}{2}\dimx$ to get $\sigma_{m}(\Aclst) \ge \sigma_{\min}(\Ast)$.
\end{itemize}




\subsection{Upper Bound}
\label{ssec:main_ub}

While playing near-optimally prevents the learner from ruling out perturbations of the form Eq.~\eqref{eq:B_perturbation},  she can rule perturbations in orthogonal directions.  Indeed, if $\matu_t \approx \Kst \matx_t$, then $\matx_{t+1} \approx (\Ast + \Bst \Kst)\matx_t + \matw_t$. As a result, the persistent noise process $\matw_t$ allows the learner recover the closed loop dynamics matrix $\Aclst=\Ast + \Bst \Kst$ to Frobenius error $\dimx \eps$ after just $T \gtrapprox 1/\epsilon^2$ steps, regardless of whether she incorporates additional exploration \citep{simchowitz2018learning}. Hence, for perturbations perpendicular to those in Eq.~\eqref{eq:B_perturbation}, the problem closely resembles a setting where $\log T$ is achievable. 

Our main algorithm, \Cref{alg:ce}, \iftoggle{icml}{is detailed in \Cref{sec:upper_bound_main}. It }{}is an $\veps$-greedy scheme that takes advantage of this principle. The full pseudocode and analysis are deferred to \Cref{sec:upper_bound_main}, but we sketch the intuition here. The algorithm takes as input a stabilizing controller $K_0$ and proceeds in epochs $k$ of length $\tau_k = 2^k$. After an initial burn-in period ending with epoch $\ksafe$, the algorithm can ensure the reliability of its synthesized controllers, and uses a (projected) least-squares estimate $(\Ahat_k,\Bhat_k)$ of $(\Ast,\Bst)$ to synthesize a controller $\Khat_k = \Kinf(\Ahat_k,\Bhat_k)$  known as the \emph{certainty equivalent} controller. The learner then selects inputs by adding white Gaussian noise with variance $\sigma_k^2$: $\matu_t = \Khat_t \matx_t + \mathcal{N}(0,\sigma_k^2 I)$. We show that this scheme exploits the rapid estimation along directions orthogonal to those in Eq.~\eqref{eq:B_perturbation}, leading to optimal dimension dependence.

To begin, we show (\Cref{thm:main_perturb_simple}) that the cost of the certainty-equivalent controller is bounded by the estimation error for $\Ahat_k$ and $\Bhat_k$, i.e.
\iftoggle{icml}
{
    \begin{align*}
      &\Jfunc_{\Ast,\Bst}[\Khat_k] - \Jst \\&\lesssim \poly(\opnorm{\Pst}) \cdot (\|\Ahat_k - \Ast\|_{\fro}^2 + \|\Bhat_k - \Bst\|_{\fro}^2),
\end{align*}
}
{
  \begin{align*}
\Jfunc_{\Ast,\Bst}[\Khat_k] - \Jst \lesssim \poly(\opnorm{\Pst}) \cdot (\|\Ahat_k - \Ast\|_{\fro}^2 + \|\Bhat_k - \Bst\|_{\fro}^2),
\end{align*}
}
once $(\Ahat_k,\Bhat_k)$ are sufficiently accurate, as guaranteed by the burn-in period. Through a regret decomposition based on the Hanson-Wright inequality (\Cref{lem:cost_lem}), we next show that the bulk of the algorithm's regret scales as the sum of the suboptimality in the controller for a given epoch, plus the cost of the exploratory noise:
\iftoggle{icml}
{
  $
\sum_{k =\ksafe}^{\log_2 T} \tau_k\left( \Jfunc_{\Ast,\Bst}[\Khat_k] - \Jst\right) + \dimu \tau_k\sigma_k^2 \lessapprox \sum_{k =\ksafe}^{\log_2 T} \tau_k\left( \|\Ahat_k - \Ast\|_{\fro}^2 + \|\Bhat_k - \Bst\|_{\fro}^2\right) + \dimu \tau_k\sigma_k^2.$
}
{
  \begin{align*}
\sum_{k =\ksafe}^{\log_2 T} \tau_k\left( \Jfunc_{\Ast,\Bst}[\Khat_k] - \Jst\right) + \dimu \tau_k\sigma_k^2 \lessapprox \sum_{k =\ksafe}^{\log_2 T} \tau_k\left( \|\Ahat_k - \Ast\|_{\fro}^2 + \|\Bhat_k - \Bst\|_{\fro}^2\right) + \dimu \tau_k\sigma_k^2.  
\end{align*}
}
In the above, we also incur a term of approximately $\sum_{k =\ksafe}^{\log_2 T} \sqrt{(\dimx + \dimu) \tau_k} \lesssim \sqrt{T(\dimx + \dimu)}$, which is lower order than the overall regret of $\sqrt{T \dimx \dimu^2}$. This term arises from the random fluctuations of the costs around their expectation, and crucially, the Hanson-Wright inequality allows us to pay of the \emph{square root} of the dimension.\footnote{The use of the Hanson-Wright crucially leverages independence of the noise process; for general sub-Gaussian martingale noise, an argument based on martingale concentration would mean that the fluctuations contribute $(\dimx + \dimu)\sqrt{T}$ to the regret up to logarithmic factors, yielding an overall regret of $\sqrt{\max\{\dimx,\dimu^2\}\dimx T}$. This is suboptimal regret for $\dimx \gg \dimu^2$, but still an improvement over the $\sqrt{(\dimx + \dimu)^3 T}$-bound of \cite{mania2019certainty}.  It is unclear if one can do better in this setting without improved concentration bounds for quadratic forms of martingale vectors, because it is unclear how an algorithm can ameliorate these random fluctuations. }

Paralleling the lower bound, the analysis crucially relies on the exploratory noise to bound the error in the $\dimx \dimu$-dimensional subspace corresponding to Eq.~\eqref{eq:B_perturbation}, as the error in this subspace grows as $\frac{\dimx \dimu}{\sigma_k^2\tau_k}$. However, for the directions parallel to those in Eq.~\eqref{eq:B_perturbation}, the estimation error is at most $\dimx^2/\tau_k$, and so the total regret is bounded as
\iftoggle{icml}
{
  $
\Regret_{\Ast,\Bst,T}[\Alg] \lessapprox \sum_{k =\ksafe}^{\log_2 T}   \tau_k \left(\frac{\dimx^2}{\tau_k} + \frac{\dimx\dimu}{\tau_k\sigma_k^2}\right) + \dimu \tau_k\sigma_k^2 \approx \dimx^2 \log T +\sum_{k =1 }^{\log_2 T} \frac{\dimx \dimu}{\sigma_k^2} + \dimu \tau_k\sigma_k^2 .$
}
{
  \begin{align*}
\Regret_{\Ast,\Bst,T}[\Alg] \lessapprox \sum_{k =\ksafe}^{\log_2 T}   \tau_k \left(\frac{\dimx^2}{\tau_k} + \frac{\dimx\dimu}{\tau_k\sigma_k^2}\right) + \dimu \tau_k\sigma_k^2 \approx \dimx^2 \log T +\sum_{k =1 }^{\log_2 T} \frac{\dimx \dimu}{\sigma_k^2} + \dimu \tau_k\sigma_k^2 .
\end{align*}
}
\icml{\newline}
Trading off $\sigma_k^2 = \sqrt{\dimx/\tau_k}$ gives regret $\dimx^2 \log T + \sum_{k =1 }^{\log_2 T}\sqrt{\dimx \dimu^2 \tau_k} \approx \dimx^2 \log T + \sqrt{\dimx \dimu^2 T}$. We emphasize that to ensure that the $\dimx^{2}$ term in this bound scales only with $\log{}T$ due to rapid exploration perpendicular to Eq.~\eqref{eq:B_perturbation}, and it is crucial that the algorithm uses doubling epochs to take advantage of this. \iftoggle{icml}{We now state the full guarantee.}{The full guarantee for \Cref{alg:ce} is as follows.}

\begin{thm}\label{thm:main_ub}
  When Algorithm~\ref{alg:ce} is invoked with stabilizing controller $K_0$ and confidence parameter $\delta \in (0,1/T)$, it guarantees that with probability at least $1-\delta$,
  \iftoggle{icml}
  {$\Regret_{T}[\Alg;\Ast,\Bst]$ is bounded as
    \begin{align*}
      &\lesssim \sqrt{ \dimu^2\dimx T \cdot  \Psibst^2\opnorm{\Pst}^{11}\log \tfrac{\opnorm{\Pst}}{\delta}} \\
      &~~+ d^2 \cdot  \calP_0 \Psibst^6 \opnorm{\Pst}^{11} (1+\|K_0\|_{\op}^2)\log\frac{d\Psibst \calP_0}{\delta} \log^2\frac{1}{\delta},
    \end{align*}
  }
  {
    \begin{align*}
         \Regret_{T}[\Alg;\Ast,\Bst] &\lesssim \sqrt{ \dimu^2\dimx T \cdot \Psibst^2\opnorm{\Pst}^{11}\log \frac{1}{\delta}} \\
      &\qquad+ rd^2 \cdot  \calP_0 \Psibst^6 \opnorm{\Pst}^{11} (1+\|K_0\|_{\op}^2)\log\frac{d\Psibst \calP_0}{\delta} \log^2\frac{1}{\delta},
    \end{align*}
  }where $\calP_0 \ldef{} \Jfunc_{\Ast,\Bst}[K_0]/\dimx$ is the normalized cost of $K_0$,  $d = \dimx + \dimu$, and $r = \max\{1, \frac{\dimu}{\dimx}\}$, which is $1$ is the typical setting $\dimu \le \dimx$.
    \end{thm}
Ignoring dependence on problem parameters, the upper bound of \Cref{thm:main_ub} scales
asymptotically as $\sqrt{\dimu^{2}\dimx{}T}$, matching our lower
bound. Like the lower bound, the theorem depends on the instance $(\Ast,\Bst)$ only through the operator norm bounds $\Psibst$ and $\nrm*{\Bst}_{\op}$.  Similar to previous work \citep{dean2018regret,mania2019certainty}, the regret bound has additional dependence on the stabilizing controller $K_0$ through $\nrm*{K_0}_{\op}$ and $\cP_0$, but these parameters only affect the lower-order terms.

\subsection{Consequences for Strongly Stable Systems\label{ssec:main_results_strong}}

To emphasize the dependence on dimension and time horizon in our
results, we now present simplified findings for a special class of
\emph{strongly stable} systems.
\begin{defn}[Strongly Stable System \citep{cohen2018online}]\label{defn:strongly_stable} We say that $\Ast$ is $(\gamma,\kappa)$-strongly stable if there exists a transform $T$ such that $\nrm{T}_{\op}\cdot\nrm{T^{-1}}_{\op} \le \kappa$ and $\|T\Ast T^{-1}\|_{\op} \le 1 - \gamma$. When $\Ast$ is $(\gamma,\kappa)$-strongly stable, we define $\gamstab := \gamma/\kappa^2$.
\end{defn}
For the simplified results in this section we make the following assumption. 
\begin{asm}
  \label{asm:simplified}
  The nominal instance $(\Ast,\Bst)$ is such that $\Ast$ is
  $(\gamma,\kappa)$-strongly stable and  $\opnorm{\Bst} \le
  1$. Furthermore, $\Rx=\Ru = I$. 
\end{asm}
For strongly stable systems under \Cref{asm:simplified}, our main lower bound (\Cref{thm:main_lb}) takes the following particularly simple form.
\begin{cor}[Lower Bound for Strongly Stable
  Systems]\label{cor:small_input} Suppose that \Cref{asm:simplified}
  holds, and that $\dimu \le \frac{1}{2}\dimx$ and $\sigma_{\min}(\Ast) > 0$.\footnote{The assumption $\dimu \le \frac{1}{2}\dimx$ can be replaced with $\dimu \le \alpha \dimx$ for any $\alpha < 1$, and can be removed entirely for special instances. See \Cref{cor:scalar_indentity} in \Cref{ssec:proof_of_Cors} for more details.} Then for any $T \ge (\dimx \dimu + \dimx \log \dimx) \poly(1/\gamstab,1/\sigma_{\min}(\Ast))$, we have
\begin{align*}
\calR_{\Ast,\Bst,T}\left(\veps_T\right)  \gtrsim  \sqrt{\dimu^2\dimx T} \cdot  \sigma_{\min}(\Ast)^2\gamstab^{4},
\end{align*}
where $\veps_T := \sqrt{\dimu^2 \dimx/T}$.
\end{cor}
The upper bound from \Cref{thm:main_ub} takes on a similarly simple form, and is seen to be nearly matching.
\begin{cor}[Upper Bound for Strongly Stable
  Systems]\label{cor:strong_stable_r_ub} Suppose that
  \Cref{asm:simplified} holds. Then Algorithm~\ref{alg:ce} with stabilizing controller $K_0 = 0$ and confidence parameter $\delta \in (0,1/T)$, ensures that probability at least $1-\delta$, 
  \iftoggle{icml}
  {
    $\Regret_{T}[\Alg;\Ast,\Bst]$ is bounded as
    \begin{align*}
      &\lesssim \sqrt{ \dimu^2\dimx T  \cdot \gamstab^{-11}\log \frac{1}{\delta \gamstab } } \\&~~~~+ (\dimx + \dimu)^2 \gamstab^{-12} \log \frac{d}{\delta \gamstab} \log^2 \frac{1}{\delta}.
    \end{align*}
  }
  {
    \begin{align*}
     \Regret_{T}[\Alg;\Ast,\Bst] &\lesssim \sqrt{ \dimu^2\dimx T  \cdot \gamstab^{-11}\log \frac{1}{\delta \gamstab } } + (\dimx + \dimu)^2 \gamstab^{-12} \log \frac{d}{\delta \gamstab} \log^2 \frac{1}{\delta}.
    \end{align*}
  }
       
  \end{cor}
  We observe that the leading $\sqrt{\dimu^2 \dimx T}$ terms in the
  upper and lower bounds differ only by factors polynomial in
  $\gamstab$, as well as a $\sigma_{\min}(\Ast)$ factor incurred by
  the lower bound. The lower order term $(\dimx + \dimu)^2$ in the
  upper bound appears unavoidable, but we leave a complementary lower
  bound for future work. Both corollaries hold because strong stability immediately implies a bound on $\nrm*{\Pst}_{\op}$.  
\begin{proof}[Proof of \Cref{cor:small_input} and \Cref{cor:strong_stable_r_ub}]First, observe that under \Cref{asm:simplified}, $\Psibst \le 1$. Next, note that if $\dimu < \dimx/2$, then for $m = \ceil{\dimx/2}$, $\sigma_{m}(\Aclst) =\sigma_{m}(\Ast + \Bst \Kst) \ge \sigma_{m + \dimu}(\Ast + \Bst \Kst) \ge \sigma_{\min}(\Ast)$. This gives $\nu_m \ge \sigma_{\min}(\Ast)/(1+\opnorm{\Pst})$. Finally, Lemma~\ref{lem:stable_p_bound} (stated and proven in \Cref{sssec:dlyap}) gives $\opnorm{\Pst} \le \gamstab^{-1}$. Plugging these three observations into \Cref{thm:main_lb} and \Cref{thm:main_ub} concludes the proof.
\end{proof}


%% file: body/self_bounding.tex

Both \Cref{thm:main_lb} and \Cref{thm:main_ub} scale only with the
natural system parameter $\nrm*{\Pst}_{\op}$, and avoid explicit
dependence on the spectral radius or strong stability parameters found
in prior work. This is achieved using the \emph{self-bounding ODE} method, a new technique for deriving bounds on perturbations to the $\DARE{}$ solution $\Pinf(A,B)$ and corresponding controller $\Kinf(A,B)$ as the matrices $A$ and $B$ are varied. This method gives a general recipe for establishing perturbation bounds for solutions to implicit equations. It depends only on the norms of the system matrices and $\DARE$ solution $\Pinf(A,B)$, and it applies to all stabilizable systems, even those that are not controllable.

In this section we give an overview of the self-bounding ODE method and use it to prove a simplified version of the main perturbation bound used in our main upper and lower bounds. To state the perturbation bound, we first define the following problem-dependent constants.
\iftoggle{icml}
{
  \begin{align}
 \Csafe(A,B)&= 54\opnorm{\Pinf(A,B)}^5,\quad\text{and}\quad \nonumber\\
 \Cest(A,B)&= 142 \opnorm{\Pinf(A,B)}^{8}.\label{eq:constants}
\end{align}
}
{
  \begin{equation}
 \Csafe(A,B)= 54\opnorm{\Pinf(A,B)}^5,\quad\text{and}\quad
 \Cest(A,B)= 142 \opnorm{\Pinf(A,B)}^{8}.\label{eq:constants}
\end{equation}
}
The parameter $\Csafe(A,B)$ determines the radius of admissible
perturbations, while the parameter $\Cest(A,B)$
determines the quality of controllers synthesized from the resulting perturbation. The main perturbation bound is as follows. 

\begin{restatable}{thm}{mainperturb}\label{thm:main_perturb_simple} Let $(\Ast,\Bst)$ be a stabilizable system. Given an alternate pair of matrices $(\Ahat,\Bhat)$, for each $\circ\in\norms$ define $\epscirc := \max\{\circnorm{\Ahat-\Ast},\circnorm{\Bhat - \Bst}\}$. Then if $\epsop \le 1/\Csafe(\Ast,\Bst)$, 
\begin{enumerate}
  \item $\opnorm{\Pinf(\Ahat,\Bhat)} \lesssim \opnorm{\Pst}$  and $\|\Kst - \Kinf(\Ahat,\Bhat)\|_{\op} \lesssim \frac{1}{\opnorm{\Pst}^{3/2}}$.
\item $\Jfunc_{\Ast,\Bst}[\Kinf(\Ahat,\Bhat)] - \Jfunc^\star_{\Ast,\Bst} \le \Cest(\Ast,\Bst)\epsfro^2$.
\end{enumerate}
\end{restatable}
This theorem is a simplification of a stronger version,
\Cref{thm:main_perturb_app}, stated and proven in
\Cref{app:perturbation_main_results}. Additional perturbation bounds
are detailed in \Cref{app:perturbation_main_results}; notably,
Theorem~\ref{thm:continuity_of_safe set} shows that the condition
$\epsop \le 1/\Csafe(\Ast,\Bst)$ can be replaced by a condition that can be certificated from an approximate estimate of the system.
\iftoggle{icml}{}{
  
}In the remainder of this section, we sketch how to use the
self-bounding ODE method to prove the following slightly more general
version of the first part of \Cref{thm:main_perturb_simple}.
\begin{restatable}{prop}{propmainfirstorder}\label{prop:main_first_order}
  Let $(\Ast,\Bst)$ be a stabilizable system and let $(\Ahat,\Bhat)$ be an alternate pair of matrices. Then, if $u := 8\|\Pst\|_\op^2 \epsilon_\op  < 1$, the pair $(\Ahat,\Bhat)$ is stabilizable and the following bounds hold:
 \begin{enumerate}
  \item $\|\Pinf(\Ahat,\Bhat)\|_{\op} \le (1-u)^{-1/2}\|\Pst\|_{\op}$.
\item For each $\circ\in\norms$, $\|\Kinf(\Ahat,\Bhat) - \Kst\|_{\circ} \le 7(1-u)^{-7/4}\opnorm{\Pst}^{7/2}\,\epscirc$.
\end{enumerate}
\end{restatable}
To begin proving the proposition, set $\delA := \Ahat - \Ast$ and $\delB := \Bhat - \Bst$. We consider a linear curve between the two instances, parameterized by $t \in [0,1]$:
	\begin{align}\label{eq:curve}
	\left(A(t),B(t)\right) = \left(\Ast + t\delA, \Bst + t\delB\right).
	\end{align}
  At each point $t$ for which $(A(t),B(t))$ is stabilizable, the $\DARE$ has a unique solution, 
  which allows us to define associated optimal cost matrices, controllers, and closed-loop dynamics matrices:
  \iftoggle{icml}
  {
    \begin{align}
    &P(t) :=  \Pinf(A(t),B(t)),\quad K(t) := \Kinf(A(t),B(t))\nonumber\\
      &\text{and} \quad\Acl(t)\ldef{}A(t)+B(t)K(t).\label{eq:P_curve}
  \end{align}
  }
  {
      \begin{align}\label{eq:P_curve}
  P(t) :=  \Pinf(A(t),B(t)), \quad K(t) := \Kinf(A(t),B(t)),\quad\text{and}\quad\Acl(t)\ldef{}A(t)+B(t)K(t).
  \end{align}
  }
        Our strategy will be to show that $P(t)$ and $K(t)$ are in fact smooth curves, and then obtain uniform bounds on $\circnorm{P'(t)}$ and $\circnorm{K'(t)}$ over the interval $[0,1]$, yielding perturbation bounds via the mean value theorem. \iftoggle{icml}{To start, }{As a starting point, }we express the derivatives of the $\DARE$ in terms of Lyapunov equations.
        \begin{restatable}[Discrete Lyapunov Equation]{defn}{defdlyap}\label{def:dlyap} Let $X,Y \in \R^{\dimx \times
    \dimx}$ with $Y = Y^\top$ and $\rho(X) < 1$. We let $\calT_X[P] := X^\top P X - P$, and let $\dlyap(X,Y)$ denote the unique PSD solution $\calT_X[P] = Y$.  We let $\dlyap[X] := \dlyap(X,I)$.
\end{restatable}
        The following lemma (proven in \Cref{app:perturbation_derivative_computations}) serves as the basis for our computations, and also establishes the requisite smoothness required to take derivatives.
  \begin{restatable}[Derivative and Smoothness of the $\DARE$]{lem}{computationpprime}\label{lem:computation_p_prime} Let $(A(t),B(t))$ be an analytic curve, and define $\delAcl(t) := A\prm(t) + B\prm(t)\Kinf(A(t),B(t))$. Then for any $t$ such that $(A(t),B(t))$ is stabilizable, the functions $P(u)$ and $K(u)$ are analytic in a neighborhood around $t$, and we have
  \iftoggle{icml}
  {
  $P\prm(u) = \dlyap(\Acl(u), Q_1(u))$, where $Q_1(u) := \Acl(u)^\top P(u) \delAcl (u)  + \delAcl(u)^\top P(u)\Acl(u)$. 
  }
  {
    \begin{align*}
P\prm(u) = \dlyap(\Acl(u), Q_1(u)),\quad \text{ where } Q_1(u) := \Acl(u)^\top P(u) \delAcl (u)  + \delAcl(u)^\top P(u)\Acl(u).
\end{align*}
  }
\end{restatable}
\Cref{lem:computation_p_prime} expresses $P'(t)$ as the solution to an ordinary differential equation. While the lemma guarantees local existence of the derivatives, it is not clear that the entire curve $(A(t),B(t))$, $t \in [0,1]$ is stabilizable. However, since ODEs are locally guaranteed to have solutions, we should only expect trouble when the corresponding ODE becomes ill-defined, i.e. if $P'(t)$ escapes to infinity. We circumvent this issue by observing that $P'(t)$ satisfies the following self-bounding property.
  \begin{restatable}[Bound on First Derivatives]{lem}{lemfirstderbound}\label{lem:first_derivatives_bound} Let $(A(t),B(t))$ be an analytic curve. Then, for all $t$ at which $(A(t),B(t))$ is stabilizable, we have
  \iftoggle{icml}
   {
    $\|P'(t)\|_{\circ}   \le 4\opnorm{P(t)}^3\,\epscirc,$and $\circnorm{K'(t)} \le 7\opnorm{P(t)}^{7/2}\,\epscirc$.
  }
  {
    \begin{align*}
    \|P'(t)\|_{\circ}   \le 4\opnorm{P(t)}^3\,\epscirc,\quad\text{and}\quad \circnorm{K'(t)} \le 7\opnorm{P(t)}^{7/2}\,\epscirc\,.
    \end{align*}
  }

    \end{restatable}
    The bound on $P'(t)$ above follows readily from the expression for $P'(t)$ derived in \Cref{lem:computation_p_prime}, and the bound on $K'(t)$ uses that $K$ is an explicit, analytic function of $P$; see \Cref{app:perturbation_derivative_computations} for a full proof. Intuitively, the self-bounding property states that if $P$ does not escape to infinity, then $P'(t)$ cannot escape either. Since the rate of growth for $P(t)$ is in turn bounded by $P'(t)$, this suggests that there is an interval for $t$ on which $P$ and $P'$ self-regulate one another, ensuring a well-behaved solution. \iftoggle{icml}{}{We proceed to make this intuition formal.}
\subsection{Norm Bounds for Self-Bounding ODEs}
    Informally, the self-bounding ODE method argues that if a vector-valued ODE $y(t)$ satisfies a self-bounding property of the form $\|y'(t)\| \le g(\|y(t)\|)$ wherever it is defined, then the ODE can be compared to a scalar ODE $z'(t) \approx g(z(t))$ with initial condition $z(0) \approx \|y(0)\|$. Specifically, it admits a solution $y(t)$ which is well-defined on an interval roughly as large as that of $z(t)$. We develop the method in a general setting where $y(t)$ (when defined) is the zero of a sufficiently regular function.
    \begin{defn}[Valid Implicit Function]\label{defn:valid_implicit} A function $F(\cdot,\cdot): \R^{m} \times \R^d \to \R^{d}$ is a called a \emph{valid implicit function} with domain $\calU \subseteq \R^d$ if $F$ is continuously differentiable, and if for any continuously differentiable curve $x(t)$ and any $t \in [0,1]$, either (a) $F(x(t),y) = 0$ has no solution $y \in \calU$, or (b) it has a unique solution $y(t) \in \calU$, and there exists an open interval around $t$ and a \iftoggle{icml}{$\mathcal{C}^1$}{continuously differentiable} curve $y(u)$ defined on this interval for which $F(x(u),y(u)) = 0 $. 
    \end{defn}
    This setting captures as a special case the characterization of $P(t)$ from \Cref{lem:computation_p_prime}.
    As a consequence of the lemma, we may take $F=\Fdare$, where, identifying $\Symdx$ as a $\binom{\dimx+1}{2}$-dimensional euclidean space,  $\Fdare: (\R^{\dimx^2} \times \R^{\dimx \dimu}) \times \Symdx\to\Symdx$ is the function whose zero-solution defines the $\DARE$: 
    \iftoggle{icml}
    {
      $\Fdare((A,B),P) :=  A^\top P A - P -  A^\top P B (\Ru + B^\top P B)^{-1}B^\top P A + \Rx.$
    } 
    {
    \begin{align*}
    \Fdare((A,B),P) :=  A^\top P A - P -  A^\top P B (\Ru + B^\top P B)^{-1}B^\top P A + \Rx.
    \end{align*}
    }
    Then $\Fdare$ is a valid implicit function with unique solutions in the set of positive-definite matrices $\calU := \Sympldx$. To proceed, we introduce our self-bounding condition.
    \begin{defn}[Self-bounding]\label{defn:self-bounding} Let $g: \R \to \R_{\ge 0}$ be non-negative and non-decreasing, let $F$ be a valid implicit function with domain $\calU$, and let $\|\cdot\|$ be a norm. For a continuously differentiable curve $x(t)$ defined on $[0,1]$, we say that $F$ is \emph{$(g,\|\cdot\|)$-self bounded} on $x(t)$ if $F(x(0),y) = 0$ has a solution $y \in \calU$ and
    \iftoggle{icml}
    { $\|y'(t)\| \le g(\|y\|)$  for all  $t \in [0,1]$ for which $F(x(t),y)$ $ y \in \calU.$
    }
    {
    \begin{align*}
    \|y'(t)\| \le g(\|y\|) \quad \text{for all } t \in [0,1] \text{ for which } F(x(t),y) \text{ has a solution } y \in \calU.
    \end{align*}
    }
    We call the tuple $(F,\calU,g,\|\cdot\|,x(\cdot))$ a \emph{self-bounding} tuple. 
    \end{defn}
    Lemma~\ref{lem:first_derivatives_bound} shows that $\Fdare$ is $(g,\|\cdot\|_{\op})$-self bounding on the curve the $(A(t),B(t))$ with $g(z) = c z^3$ for $c \propto \epsop$. For functions $g(z)$ with this form we have the following general bound on $\nrm*{y(t)}$.
    \begin{cor}\label{cor:poly_self_bound} Let $(F,\calU,g,\|\cdot\|,x(\cdot))$ be a self-bounding tuple, where $g(z) = cz^p$ for $c > 0$ and $p > 1$. Then, if $\alpha := c(p-1)\|y(0)\|^{p-1} < 1$, there exists a unique continuously differentiable function $y(t) \in \calU$ defined on  $[0,1]$ which satisfies $F(x(t),y(t)) = 0$, and this solution satisfies
    \iftoggle{icml}
    {
      $\forall t \in [0,1]$, $\|y(t)\| \le (1 -\alpha)^{-1/(p-1)}\|y(0)\|$, and $\|y'(t)\| \le c(1-\alpha)^{-p/(p-1)}\|y(0)\|^p$.
    }
    {
      \begin{align*} 
        \forall t \in [0,1], \quad \|y(t)\| \le (1 -\alpha)^{-1/(p-1)}\|y(0)\|, \quad\text{and}\quad \|y'(t)\| \le c(1-\alpha)^{-p/(p-1)}\|y(0)\|^p.
        \end{align*}
    }

\end{cor}
\Cref{cor:poly_self_bound} is a consequence of a similar result for general functions $g$ \iftoggle{icml}{(\Cref{thm:general_valid_implicit}, in \Cref{app:self_bounding})}{(\Cref{thm:general_valid_implicit}), which is stated in \Cref{app:self_bounding}}. The condition on the parameter $\alpha$ directly arises from the requirement that the scalar ODE $w'(u) = c w(u)^3$ has a solution on $[0,1]$.

\paragraph{Finishing the Proof of \Cref{prop:main_first_order}}
\iftoggle{icml}{Finally, }{To close out this section,} we use \Cref{cor:poly_self_bound} to conclude the proof of Proposition~\ref{prop:main_first_order}.\icml{\vspace{-.1in}}\begin{proof}[Proof of Proposition~\ref{prop:main_first_order}]
    Lemma~\ref{lem:first_derivatives_bound} states that for any $t\in\brk*{0,1}$ for which  $(A(t),B(t))$ is stabilizable (i.e.,  $\Fdare([A(t),B(t)],\cdot)$ has a solution), we have the bound
  \begin{align*}
  \opnorm{P'(t)} \le 4\opnorm{P(t)}^3 \epsop.
  \end{align*}
  Applying  \Cref{cor:poly_self_bound} with $p = 2$ and $c = 4\epsop$, we see that if $\alpha := 8\epsop\opnorm{\Pst}^2 < 1$, then $P(t)$ is continuously differentiable on the interval $[0,1]$ and
  \iftoggle{icml}
  {
   $
  \forall t \in [0,1]$, $\opnorm{P(t)} \le \opnorm{\Pst}/\sqrt{1-\alpha}.$
  }
   {
    \begin{align*}
  \forall t \in [0,1], \opnorm{P(t)} \le \opnorm{\Pst}/\sqrt{1-\alpha}.
  \end{align*}
  }By Lemma~\ref{lem:first_derivatives_bound}, $K(t)$ is well defined as well, and satisfies
  \iftoggle{icml}
  {
    $\max_{t \in [0,1]} \|K'(t)\|_{\circ} \le 7\epscirc\max_{t \in [0,1]}\,\opnorm{P(t)}^{7/2} \le (1-\alpha)^{-7/4}\opnorm{\Pst}.$
  }
  {
    \begin{align*}
  \max_{t \in [0,1]} \|K'(t)\|_{\circ} &\le 7\epscirc\max_{t \in [0,1]}\,\opnorm{P(t)}^{7/2} \le (1-\alpha)^{-7/4}\opnorm{\Pst}.
  \end{align*}
  }
  The desired bound on $\|\Kinf(\Ast,\Bst) - \Kinf(\Ahat,\Bhat)\|_{\circ} $ follows from the mean value theorem.
\end{proof}


%% file: body/lower_bound.tex

We now prove the main lower bound, \Cref{thm:main_lb}. The proof follows the plan outlined in \Cref{sec:main_results}: We construct a packing of alternative instances, show that low regret on a given instance implies low estimation error, and then deduce from an information-theoretic argument that this implies high regret an alternative instance. All omitted proofs for intermediate lemmas are given in \Cref{app:lb_proofs}. Recall throughout that we assume $\sigma^2_w=1$.

\subsection{Alternative Instances and Packing Construction}

We construct a packing of alternate instances $(A_e,B_e)$ which take the form $(\Ast + \Kst \Delta_e,\Bst + \Delta_e)$, for appropriately chosen perturbations $\Delta_e$ described shortly. As discussed in \Cref{ssec:main_lb}, this packing is chosen because the learner \emph{cannot} distinguish between alternatives if she commits to playing the optimal policy $\matu_t = \Kst\matx_t$, and must therefore deviate from this policy in order to distinguish between alternatives. We further recall \Cref{lem:lower_bound_dercomp}, which describes how the optimal controllers from these instances varying with the perturbation $\Delta$.
\lowerbounddercomp*

In particular, if $\Acl$ is non-degenerate, then to first order, the Frobenius distance between between the optimal controllers for $\Ast,\Bst$ and the alternatives $(A_e,B_e)$ is $\Omega(\|\Delta\|_{\fro})$.

To obtain the correct dimension dependence, it is essential that the packing is sufficiently large; a single alternative instance will not suffice. Our goal is to make the packing as large as possible while ensuring that if one can recover the optimal controller for a given instance, they can also recover the perturbation $\Delta$.

Let $n  = \dimu$, and let $m \le \dimx$ be the free parameter from the theorem statement. We construct a collection of instances indexed by sign vectors $e \in\espace$. Let $\uvec_{1},\dots,\uvec_n$ denote an eigenbasis basis of $(\Ru + \Bst^\top \Pst \Bst)^{-1}$, and $\vvec_1,\dots,\vvec_m$ denote the first $m$ right-singular vectors of $\Aclst \Pst$. Then for each $e\in\espace$, the corresponding instances is 
\begin{align}
(A_e,B_e) := (\Ast - \Delta_{e}\Kstinf, \Bst + \Delta_e), \quad \text{ where } \Delta_e = \epspack \sum_{i =1}^n\sum_{j=1}^m e_{i,j}\uvec_i\vvec_j^\top. \label{eq:packing_def}
\end{align}
It will be convenient to adopt the shorthand $K_e := \Kinf(A_e,B_e)$, $P_e = \Pinf(A_e,B_e)$ and $\Jfunc_e = \Jfuncopt_{A_e,B_e}$, and $\Psi_e = \max\{1,\opnorm{A_e},\opnorm{B_e}\}$. The following lemma---proven in \Cref{ssec:lem:first_order_approx_quality}---gathers a number of bounds on the error between $(A_e,B_e)$ and $(\Ast,\Bst)$ and their corresponding system parameters. Perhaps most importantly, the lemma shows that to first order, $K_e$ can be approximated using the derivative expression in~\Cref{lem:lower_bound_dercomp}.
\begin{lem}\label{lem:first_order_approx_quality} There exist universal polynomial functions $\frakp_1,\frakp_2$ such that, for any $\epspack \in (0,1)$, if $\epspack^2 \le \frakp_1( \opnorm{\Pst})^{-1}/nm$, the following bounds hold:
\begin{enumerate}
\item\textbf{\emph {Parameter errror:}} $\max\{\fronorm{A_e - \Ast}, \fronorm{B_e - \Bst}\} \le \sqrt{\opnorm{\Pst}}\sqrt{mn}\epspack$.
  \item\textbf{\emph{Boundedness of value functions:}} $\Psi_e \le 2^{1/5}\Mbarst$ and $\opnorm{P_e - \Pst}  \le 2^{1/5}\opnorm{\Pst}$.
  \item\textbf{\emph{Controller error:}} $\|K_e - \Kst\|_{\fro}^2 \le 2\opnorm{\Pst}^{3} mn \epspack^2$.
    \item\textbf{\emph{First-order error:}}  $\|\Kst + \frac{d}{dt}\Kinf(\Ast - t\Delta \Kst, \Bst + t\Delta_e)\big{|}_{t=0}- K_e \|_{\fro}^2 \le \frakp_2(\opnorm{\Pst})^2 (mn)^2 \epspack^4$.

\end{enumerate}
\end{lem}
Notably, item $4$ ensures that the first order approximation in \Cref{lem:lower_bound_dercomp} is accurate for $\epspack$ sufficiently small.

Going forward,  we we choose the polynomials in the above lemma $\frakp_1,\frakp_2$ to satisfy $\frakp_1(x),\frakp_2(x) \ge x$ (without loss of generality). We use that $\opnorm{\Pst} \ge 1$ repeatedly throughout the proof.
\begin{lem}[Lower bound on $\opnorm{\Pst}$]\label{lem:lb_pst} If $\Rx \succeq I$, then $\Pst \succeq I$, and in particular $\opnorm{\Pst} \ge 1$.
\end{lem}
\begin{proof} This is Part 4 of a more general statement, Lemma~\ref{lem:closed_loop_dlyap}, given in \Cref{app:perturbation}.
\end{proof}

Henceforth, we take $\epspack$ sufficiently small so as to satisfy the conditions of \Cref{lem:first_order_approx_quality}.
\begin{asm}[Small $\epspack$] \label{asm:small_epspack}
$\epspack^2 \le \frac{1}{mn}(\frakp_1(\opnorm{\Pst})^{-1} \wedge \frac{1}{20}\frakp_2(\opnorm{\Pst})^{-1})$. 
\end{asm}

\subsection{Low Regret Implies Estimation for Controller}

We now show that if one can achieve low regret on every instance, then one can estimate the infinite-horizon optimal controller $K_e$. Suppressing dependence on $T$, we introduce the shorthand $\SimpleRegret_e[\pi] \ldef \SimpleRegret_{A_e,B_e,T}[\pi]$. Going forward, we restrict ourselves to algorithms whose regret is sufficiently small on every packing instance; the trivial case where this is not satisfied is handled at the end of the proof.
\begin{asm}[Uniform Correctness]\label{asm:uniform_correctness} For all instances $(A_e,B_e)$, the algorithm $\pi$ ensures that $\SimpleRegret_{e}[\pi] \le \frac{T}{6\dimx\opnorm{\Pst}\Mbarst^2} - \epserr$, where $\epserr\ldef{}6\opnorm{\Pst}^3\Mbarst^2$.
 
\end{asm}
We now define an intermediate term which captures which captures the extent to which the control inputs under instance $e$ deviate from those prescribed by the optimal infinite horizon controller $K_e$ on the first $T/2$ rounds:
\begin{align*}
\Kerr_{e}[\pi] := \Exp_{A_e,B_e,\pi}\left[\sum_{t=1}^{T/2}\|\matu_t - K_e \matx_t\|^2\right].
\end{align*}
The following lemma, proven in \Cref{ssec:lem:Kerr_lem}, shows that regret is lower bounded by $\Kerr_e[\pi]$, and hence any algorithm with low regret under this instance must play controls close to $K_e\matx_t$.
\begin{lem}\label{lem:Kerr_lem}There is a universal constant $\ckerr>0$ such that if  Assumptions~\ref{asm:small_epspack} and \ref{asm:uniform_correctness} hold and  $T \ge \ckerr  \opnorm{\Pst}^2\Mbarst^4 $, then
\begin{align*}
\SimpleRegret_e[\pi] \ge \frac{1}{2}\Kerr_e[\pi]  - \epserr.
\end{align*}
\end{lem}
In light of Lemma~\ref{lem:Kerr_lem}, the remainder of the proof will focus on lower bounding the deviation $\Kerr_{e}$.  As a first step, the next lemma---proven in \Cref{ssec:lem:least_squares_lb}---shows that the optimal controller can be estimated well through least squares whenever $\Kerr_e$ is small. More concretely, we consider a least squares estimator which fits a controller using the first half of the algorithm's trajectory. The estimator returns
\begin{equation}
  \label{eq:kls}
\Kls \ldef{} \argmin_{K}\sum_{t=1}^{T/2}\nrm*{\matu_t-K\matx_t}^{2},
\end{equation}
\newcommand{\cmin}{c_{\mathrm{min}}}
when $\sum_{t=1}^{T/2}\matx_t\matx_t^{\trn}\psdgeq{}\cmin{}T\cdot{}I$, and returns $\Kls=0$ otherwise.
\begin{lem}\label{lem:least_squares_lb}
If $T \ge c_0\dimx\log(1+\dimx\opnorm{\Pst})$ and Assumptions~\ref{asm:small_epspack} and \ref{asm:uniform_correctness} hold, and if $\cmin{}$ is chosen to be an appropriate numerical constant, then the least squares estimator \Cref{eq:kls} guarantees
\begin{align*}
\Kerr_e[\pi] \ge \cls T \cdot \Exp_{A_e,B_e,\pi}\left[\|\Kls - K_e\|_\fro^2\right] - 1,
\end{align*}
where $c_0$ and $\cls$ are universal constants.
\end{lem}

Henceforth we take $T$ large enough such that \Cref{lem:Kerr_lem} and \Cref{lem:least_squares_lb} apply.
\begin{asm}\label{asm:T_condition} We have that $T \ge c_0\dimx\log(1+\dimx\opnorm{\Pst})\vee\ckerr  \opnorm{\Pst}^2\Mbarst^4$.
\end{asm}

\subsection{Information-Theoretic Lower Bound for Estimation}
We have established that low regret under the instance $(A_e,B_e)$ requires a small deviation from $ K_e$ in the sense that $\Kerr_e[\pi]$ is small, and have shown in turn that any algorithm with low regret yields an estimator for the optimal controller $K_e$ (Lemma~\ref{lem:least_squares_lb}). We now provide necessary condition for estimating the optimal controller, which will lead to the final tradeoff between regret on the nominal instance and the alternative instance. This condition is stated in terms of a quantity related to $\Kerr_e$:
\begin{align*}
\Ksterr_e[\pi] := \Exp_{A_e,B_e,\pi}\left[\sum_{t=1}^{T/2}\|\matu_t - \Kst\matx_t\|^2\right].
\end{align*}
Both $\Ksterr_e[\pi]$ and $\Kerr_e[\pi]$ concern the behavior of the algorithm under instance $(A_e,B_e)$, but former measures deviation from $\Kst$ (``exploration error'') while the latter measures deviation from the optimal controller $K_e$. Our proof essentially argues the following. Let $(\mate,\mate')$ be a pair of random indices on the hypercube, where $\mate$ is uniform on $\{-1,1\}^{nm}$, and $\mate'$ is obtained by flipping a single, uniformly selected entry of $\mate$. Moroever, let $\Pr_{\mate},\Pr_{\mate'}$ denote the respective laws for our algorithm under these two instances. We show that---because our instances take the form $(\Ast - \Delta \Kst, B + \Delta)$---$\Ksterr_e[\pi]$ captures the $\KL$ divergence between these two instances:
\begin{align*}
\Exp_{\mate}\Ksterr_{\mate}[\pi] \approx \Exp_{\mate,\mate'}\KL(\Pr_{\mate},\Pr_{\mate'}),
\end{align*}
where the expectations are taken with respect to the distribution over $(\mate,\mate')$. In other words, the average error $\Exp_{\mate}\Ksterr_{\mate}[\pi]$ corresponds to the average one-flip KL-divergence between instances. This captures the fact that the instances can only be distinguished by playing controls which deviate from $\matu_t = \Kst \matx_t$.

As a consequence, using a technique based on Assouad's lemma \citep{assouad1983deux} due to \cite{arias2012fundamental}, we prove an information-theoretic lower bound that shows that any algorithm that can recover the index vector $e$ in Hamming distance on every instance must have $\Ksterr_e[\pi]$ is large on some instances. 

As described above, the following lemma concerns the case where the alternative instance index $e$ is drawn uniformly from the hypercube. Let $\Exp_{\mate}$ denote expectation $\mate \unifsim \espace$, and let $\dham(e,e')$ denote the Hamming distance.
\begin{lem}\label{lem:info_lb}  Let $\ehat$ be any estimator depending only on $(\matx_1,\dots,\matx_{T/2})$ and $(\matu_1,\dots,\matu_{T/2})$. Then
\begin{align*}
\text{either}\quad\Exp_{\mate}\Ksterr_{\mate}[\pi] \ge \frac{n}{4\epspack^2}, \quad \text{or} \quad \Exp_{\mate}\Exp_{A_{\mate},B_{\mate},\Alg}\left[\dham(\mate,\ehat)\right]  \ge \frac{nm}{4}.
\end{align*}
\end{lem}
The above lemma is proven in \Cref{ssec:lem:info_lb}.
To apply this result to the least squares estimator $\Kls$, we prove the following lemma (\Cref{ssec:lem_recover_packing}), which shows that any estimator $\wh{K}$ with low Frobenius error relative to $K_e$ can be used to recover $e$ in Hamming distance.
\begin{lem}\label{lem:recover_packing} Let $\ehat_{i,j}(\Khat) \ldef{} \sign(\uvec_i^\top (\Khat - \Kst) \vvec_j)$, and define $\nu_k := \opnorm{\Ru + \Bst^\top \Pst \Bst}/\sigma_{k}(\Aclst)$. Then under Assumption~\ref{asm:small_epspack},
\begin{align*}
 \dham(\ehat_{i,j}(\Khat),e_{i,j}) \le  \frac{2\fronorm{\Khat - K_e}}{\nu_m^2\epspack^2} + \frac{1}{20}nm.
\end{align*}
\end{lem} 
Combining Lemmas~\ref{lem:least_squares_lb},~\ref{lem:info_lb}, and ~\ref{lem:recover_packing}, we arrive at a dichotomy: either the average exploration error $\Ksterr_{e}[\pi]$ is large, or the regret proxy $\Kerr_e[\pi]$ is large.
\begin{cor}\label{cor:lb_dichotomy} Let $\mate \unifsim \espace$. Then if Assumptions~\ref{asm:small_epspack},~\ref{asm:uniform_correctness},and~\ref{asm:T_condition} hold,
\begin{align}
\text{either}\quad\underbrace{\Exp_{\mate}\Ksterr_{\mate}[\pi] \ge \frac{n}{4\epspack^2}}_{\text{\normalfont (sufficient exploration)}}, \quad \text{or} \quad \underbrace{\Exp_{\mate}\Kerr_{\mate}[\pi]   \ge \frac{\cls}{10} T  nm \nu_m^2\epspack^2 - \epsls}_{(\text{\normalfont large deviation from optimal)}}. \label{eq:suff_exploration_prop}
\end{align}
\end{cor}
\begin{proof} Let $\ehat = \ehat(\Kls)$, where $\ehat$ is the estimator from Lemma~\ref{lem:recover_packing}, and $\Kls$ is as defined in Lemma~\ref{lem:least_squares_lb}. Since this estimator only depends on $\matx_{1},\dots,\matx_{T/2}$ and $\matu_1,\dots,\matu_{T/2}$, we see that if the first condition in \Cref{eq:suff_exploration_prop} $\text{(sufficient exploration)}$ fails, then by Lemma~\ref{lem:info_lb}, we have $\Exp_{\mate}\Exp_{A_{\mate},B_{\mate},\Alg}\left[\dham(\ehat,\mate)\right]  \ge \frac{nm}{4}  = \frac{nm}{5} + \frac{nm}{20}$. Thus, by Lemma~\ref{lem:recover_packing}, we have $\frac{2\Exp_{\mate}\Exp_{A_{\mate},B_{\mate},\Alg}\fronorm{\Khat - K_{\mate}}}{\nu_m^2\epspack^2} \ge \frac{nm}{5}$, yielding $\Exp_{\mate}\Exp_{A_{\mate},B_{\mate},\Alg}\fronorm{\Khat - K_{\mate}}^2 \ge \frac{1}{10}nm\nu_m^2\epspack^2$. The bound now follows from Lemma~\ref{lem:least_squares_lb}.
\end{proof}
\subsection{Completing the Proof}

To conclude the proof, we show (\Cref{ssec:lem:deviation_from_Kst}) that $\Exp_{\mate}\Kerr_{\mate} \approx \Exp_{\mate}\Ksterr_{\mate}$, so that the final bound follows by setting $\epspack^2 \approx \sqrt{1/mT}$. 
\begin{lem}\label{lem:deviation_from_Kst} Under Assumptions~\ref{asm:small_epspack} and~\ref{asm:uniform_correctness}, we have
$\Exp_{\mate}\Ksterr_{\mate}[\pi] \le 2\Exp_{\mate}\Kerr_{\mate}[\pi] + 4nmT\opnorm{\Pst}^4\epspack^2$.
\end{lem}
Combining Lemma~\ref{lem:deviation_from_Kst} with \Cref{cor:lb_dichotomy}, we have
\begin{align*}
\max_e \Kerr_{e}[\pi] \ge \Exp_{\mate}\Kerr_{\mate}[\pi] \ge \left(\frac{n}{8\epspack^2} - 2nmT\opnorm{\Pst}^4\epspack^2\right) \wedge  \frac{\cls}{10} T  nm \nu_m^2\epspack^2.
\end{align*}
Setting $\epspack^2 = \frac{1}{32\opnorm{\Pst}^2\sqrt{mT}}$ and substituting in $n = \dimu$, we find that as long as $T$ is large enough such that Assumptions~\ref{asm:small_epspack}-\ref{asm:T_condition} hold,
\begin{align*}
\max_e \Kerr_{e}[\pi] &\approxgeq \dimu\sqrt{mT}/\opnorm{\Pst}^2 \wedge \sqrt{\dimu^2 mT}\nu_m^2/\opnorm{\Pst}^2 - 1\\
&\approxgeq (1 \wedge \nu_m^2)\sqrt{m\dimu^2T}/\opnorm{\Pst}^2) - 1.
\end{align*}
Thus, by Lemma~\ref{lem:least_squares_lb}, we have that for a sufficiently small numerical constant $\clb$ (which we choose to have value at most $1$ without loss of generality),
\begin{align*}
\max_{e} \SimpleRegret_e[\pi] &\ge 2\clb \frac{(1 \wedge \nu_m^2)\sqrt{n^2mT}}{\opnorm{\Pst}^2} - \frac{1}{2} - \epserr \ge 2\clb \frac{(1 \wedge \nu_m^2)\sqrt{n^2mT}}{\opnorm{\Pst}^2} - 7\dimx\opnorm{\Pst}^3\Mbarst^2.
\end{align*}
It follows that once
\begin{align}\label{eq:T_condition_proof}
T \ge c_1 \left(\opnorm{\Pst}^{p}(n^2 m \vee \frac{\dimx^2\Mbarst^4 (1\vee\nu_m^{-4})}{m n^2} \vee \dimx \log(1+\dimx \opnorm{\Pst})\right),
\end{align} 
where $c_1$ and $p$ sufficiently numerical constants, Assumptions~\ref{asm:small_epspack} and~\ref{asm:T_condition} are indeed satisfied, so we have
\begin{align*}
\max_e\SimpleRegret_{e}[\pi] \ge  \clb \frac{(1 \wedge \nu_m^2)\sqrt{\dimu^2mT}}{\opnorm{\Pst}^2} := \mathcal{R}.
\end{align*}
We now justify \Cref{asm:uniform_correctness}. Suppose the assumption fails, i.e. for some instance $e$ the algorithm has $\SimpleRegret_{e}[\pi] \ge \frac{T}{6\Mbarst^2\dimx} - \epserr$. Then since $\clb \le 1$ and $\opnorm{\Pst} \ge 1$, we see that if $\sqrt{T} \ge 12\Mbarst^2\dimx/\sqrt{mn^2}$, then $\SimpleRegret_{e}[\pi] \ge 2\mathcal{R} - \epserr \ge \mathcal{R}$. By taking $c_1$ sufficiently large, we see that whenever \Cref{eq:T_condition_proof} holds, we have $\SimpleRegret_{e}[\pi] \ge 2\mathcal{R} - \epserr \ge \mathcal{R}$ as desired.

To conclude, we verify that the construction is consistent with the scale parameter $\eps_T$ from the theorem statement:
\begin{align*}
\fronorm{A_e - \Ast}^2 \vee \fronorm{B_e - \Bst}^2  &\overset{(i)}{\le} nm\epspack^2\opnorm{\Pst} \overset{(ii)}{\le} n\sqrt{m/T}\leq{}\eps_{T},
\end{align*}
where $(i)$ follows by Lemma~\ref{lem:first_order_approx_quality}, and $(ii)$ follows by plugging in our choice for $\epspack$. $\qed$


%% file: body/upper_bound.tex

We now formally describe our main algorithm, \Cref{alg:ce}, and prove that it attains
the upper bound in \Cref{thm:main_ub}. The algorithm is a variant of
certainty equivalent control with continual $\veps$-greedy
exploration. In line with previous work
\citep{dean2017sample,dean2018regret,cohen2019learning,mania2019certainty},
the algorithm takes as input a controller $K_0$ that is guaranteed to
stabilize the system but otherwise may be arbitrarily suboptimal
relative to $\Kst$. The algorithm proceeds in epochs of doubling length. At the
beginning of epoch $k$, the algorithm uses an ordinary least squares
subroutine (\Cref{alg:ols}) to form an estimate $(\Ahat_k, \Bhat_k)$ for the system dynamics using data collected in
the previous epoch. The algorithm then checks whether the estimate is
sufficiently close to $(\Ast,\Bst)$ for the perturbation
bounds developed in \Cref{thm:main_perturb_simple} take effect; such
closeness guarantees that the optimal controller for
$(\Ahat_k,\Bhat_k)$ stabilizes the system and has low regret. If the
test fails, the algorithm falls back on the stabilizing controller
$K_0$ for the remainder of the epoch, adding exploratory
noise with constant scale. Otherwise, if the test succeeds, the
algorithm forms the certainty equivalent controller
$\wh{K}_k\ldef{}\Kinf(\Ahat_k,\Bhat_k)$ and plays this for the
remainder of the epoch, adding exploratory noise whose scale is
carefully chosen to balance exploration and exploitation.

\paragraph{Preliminaries}Before beginning the proof, let us first give some additional
definitions and notation. We adopt the shorthand $d
:= \dimx + \dimu$ and define $\kfin=\ceil*{\log_2T}$. For every controllers $K$ for which  $(A + BK)$ is
stable, we define $\Pinf(K;A,B) := \dlyap(A+B K, \Rx +
K^\top \Ru K$. It is a standard fact (see
e.g. Lemma~\ref{lem:P_bounds_lowner}) that such controllers have
$\Jfunc_{A,B}[K] = \tr(\Pinf(K;A,B))$.

We make will make heavy use of the following system parameters for the controllers used within
\Cref{alg:ce}:
    \begin{alignat*}{4}
      &P_k := \Pinf(\Khat_k;\Ast,\Bst), \quad &&\calP_0 := \frac{\Jfunc_0}{\dimx} \le \opnorm{\Pinf(K_0;\Ast,\Bst)},\\
& \Jfunc_k := \Jfunc_{\Ast,\Bst}[\Khat_k], \quad &&\Jfunc_0 := J_{\Ast,\Bst}\brk{K_0},\\
      &\Aclk := \Ast + \Bst \Khat_k,\quad &&\Aclnot := \Ast + \Bst K_0.
    \end{alignat*}

\subsection{Proof}

\begin{algorithm}[t!]
\textbf{Input:} Stabilizing controller $K_0$,  confidence parameter $\delta$. \\
\textbf{Initialize:} $\basin \gets \sffalse$.\\

Play $\matu_1 \sim \calN(0, I)$.\\
\For{$k = 2,3,\dots$}
{
  Let $\tau_k \leftarrow 2^k$.\\
  \algcommentbig{OLS estimator and covariance matrix using samples $\tau_{k-1},\dots,\tau_{k}-1$. See Algorithm~\ref{alg:ols}.}\\
  Set  $(\Ahat_k,\Bhat_k,\matLam_k) \leftarrow \OLS(k)$. \\
  \If{ $\basin = \sffalse$ }
  {
    $\Conf_k \leftarrow   6 \lambda_{\min}(\matLam_k)^{-1} \prn*{d \log 5 + \log \prn*{4k^2\det(3(\matLam_k)/\delta}}$ (infinite if $\matLam_k \not\succ 0$). \\ 
    \If{ $\matLam_k \succeq I$ and $1/\Conf_k \ge 9  \Csafe(\Ahat_k,\Bhat_k)^2$ }
    {
      $\basin \leftarrow \sftrue$, \,$\ksafe \leftarrow k$.\\
      $\Ballsafe,\sigmain^2 \leftarrow
      \SafeRoundInit(\Ahat_k,\Bhat_k,\Conf_k,\delta)$. \algcomment{Confidence
        ball (\Cref{alg:saferoundinit}).}
      
    }
    \textbf{else for }$t = \tau_k,\dots,2\tau_k - 1$, play $\matu_t = K_0\matx_t + \matg_t$, where $\matg_t \sim \calN(0,I)$.
  }
  \Else
  {
    Let $(\Atil_k,\Btil_k)$ denote the euclidean projection of
    $(\Ahat_k,\Bhat_k)$  onto $\Ballsafe$.\\
    $\Khat_k \gets \Kinf(\Atil_k,\Btil_k)$.\\
    \For{$t = \tau_k,\dots,2\tau_k-1$}
    {
        Play $\matu_t = \Khat_k\matx_t + \sigma_{k}\matg_t$, where $\matg_t \sim \calN(0, I)$, and $\sigma_k^2 := \min\{1,\sigmain^2\tau_{k}^{-1/2}\}$.
    }
  }
}
\caption{Certainty Equivalent Control with Continual Exploration}\label{alg:ce}
\end{algorithm}

  We begin the proof by showing that the initial estimation phase (in
  which the algorithm uses the stabilizing controller $K_0$) ensures
  that various regularity conditions hold for the epochs $k \ge
  \ksafe$ (in which the algorithm uses the certainty-equivalent controller). One such regularity condition bounds the $\Hinfty$-norm, which describes the \emph{worst-case} response of a system to perturbations. We recall the following definition from \Cref{app:perturbation}:
  \begin{restatable}[$\Hinfty$ norm]{defn}{defnhinf}\label{defn:Hinf_norm} For any stable $\tilde{A} \in \R^{\dimx^2}$ (e.g. $A + B\Kinf(A,B)$), we define $\Hinf{\tilde{A}} := \sup_{z \in \C: |z| = 1}\|(zI - \tilde{A})^{-1}\|_{\op}$. 
\end{restatable}

  The following result is proved in \Cref{ssec:lem_perturb_correct}.
  \begin{lem}[Correctness of Perturbations]\label{lem:perturb_correct} On the event 
  \begin{align*}
  \Esafe :=  \left\{\left\| \begin{bmatrix} \Ahat_{\ksafe} - \Ast
        \mid \Bhat_{\ksafe} - \Bst \end{bmatrix}\right\|_{\op}^2  \le
    \Conf_{\ksafe}\right\},\end{align*} the following bounds hold for all $k \ge \ksafe$:
  \begin{enumerate}
    \item $\Jfunc_k - \Jst \le \Cest(\Ast,\Bst)\left(\|\Ahat - \Ast\|_{\fro}^2 + \|\Bhat - \Bst\|_{\fro}^2\right) \lesssim \opnorm{\Pst}^8\left(\|\Ahat - \Ast\|_{\fro}^2 + \|\Bhat - \Bst\|_{\fro}^2\right) $. 
    \item $\Jfunc_k \lesssim \Jst$, and $\opnorm{P_k} \lesssim \opnorm{\Pst}$.
    \item $\opnorm{\Khat_k}^2 \le \frac{21}{20}\opnorm{\Pst}$.
    \item $\Hinf{\Aclk} \lesssim \Hinf{\Aclst} \lesssim \opnorm{\Pst}^{3/2}$. 
    \item $\Aclk^\top \dlyap[\Aclst] \Aclk \preceq (1 - \frac{1}{2}\opnorm{\dlyap[\Aclst]}^{-1})$, where $I \preceq \dlyap[\Aclst] \preceq \Pst$, where we recall the shorthand $\dlyap[\Aclst] = \dlyap(\Aclst,I)$.
    \item $\sigmain^2 \eqsim \sqrt{\dimx}\opnorm{\Pst}^{9/2}\Psibst \sqrt{\log \frac{\opnorm{\Pst}}{\delta}}$.
  \end{enumerate}
  \end{lem}
We will verify at the end of the proof that $\Esafe$ indeed holds with
high probability. We remark that Part 5 of the above lemma plays a
role similar to that of ``sequential strong stability'' in
\citet[Definition 2]{cohen2019learning}. By using $\dlyap[\Aclst]$ as a common Lyapunov function, we remove the complications involved in applying sequential strong stability. 

Building on this result, we provide (\Cref{ssec:safe_regret_decomp}) a decomposition of
the algorithm's regret which holds conditioned on $\Esafe$.     
  \begin{lem}[Regret Decomposition on Safe
    Rounds]\label{lem:safe_regret_decomp}
    There is an event $\Ereg$ which holds with probability at least
    $1-\frac{\delta}{8}$ such that, on $\Ereg \cap \Esafe$, following
    bound holds 
  \begin{align}
      \sum_{t=\tau_{\ksafe}}^{T}(\matx_t^\top \Rx \matx_t + \matu_t^\top \Ru \matu_t - \Jst) &\lesssim \sum_{k=\ksafe}^{\kfin} \tau_k(J_k - \Jst) + \log T \max_{k \le \log_T}\|\matx_{\tau_k}\|_2^2\label{eq:first_line_decomp}\\
      &\quad+\sqrt{T}\left( \dimu\sigmain^2 \Psibst^2\opnorm{\Pst}) + \sqrt{d \log(1/\delta)}\opnorm{\Pst}^4\right)\nonumber\\
      &+ \log^2 \frac{1}{\delta}(1+\sqrt{d}\sigmain^2\Psibst^2) \opnorm{\Pst}^4.\nonumber
      \end{align}
    \end{lem}
    Let us unpack the terms that arise in
    \Cref{eq:first_line_decomp}. The term $\sum_{k=\ksafe}^{\kfin}
    \tau_k (\Jfunc_k - \Jst)$ captures the suboptimality of the
    controlers $\Khat_k$ selected at each epoch. We bound this term by
    using that, in light of Lemma~\ref{lem:safe_regret_decomp}, we have
    $\Jfunc_k - \Jst \propto \|\Ahat_k - \Ast\|_{\fro}^2 + \|\Bhat_k -
    \Bst\|_{\fro}^2$. The next term, $\log T \cdot{}\max_{k \le
      \log_T}\|\matx_{\tau_k}\|^2$, is of lower order, and roughly
    captures the penalty for switching controllers at each epoch. The
    term proportional to $\sqrt{T}$ captures both the penalty for
    injecting exploratory noise into the system (which incurs a
    dependence on $\dimu$), as well as random fluctuations in the cost
    coming from the underlying noise process. Finally, the term on the
    last line of the display is also of lower order
    ($\poly(\log{}T)$). To proceed, we show that the norms
    $\nrm*{\matx_{\tau_k}}$ appearing in the second term are
    well-behaved. 
\begin{lem}\label{lem:matx_bound}  There is an event $\Ebound$ which
  holds with probability at least $1-\frac{\delta}{8}$ such that, conditioned on $\Esafe \cap \Ebound$,
\begin{align*}
  \nrm*{\matx_{\tau_k}} \le  \sqrt{\matx_{\tau_k}^\top \dlyap[\Aclst]\matx_{\tau_k}} &\lesssim \sqrt{\Psibst \Jfunc_0\log(1/\delta)}\opnorm{\Pst}^{3/2},\quad \forall k \ge \ksafe.
\end{align*}
\end{lem}
This bound is quite crude, but is sufficient for our purposes. We give
a concise proof (Appendix \ref{ssec:matx_bound}) using that in light
of Lemma~\ref{lem:perturb_correct}, $\dlyap[\Aclst]$ acts as a Lyapunov function for all the
systems $\Aclk$ conditioned on $\Esafe$.

To bound the error terms $\cJ_k-\cJ_{\star}$ appearing in
\Cref{eq:first_line_decomp} we prove
(\Cref{ssec:lem:estimation_bound}) the following bound, which ensures
the correctness of the estimators $(\Ahat_k,\Bhat_k)$ once
$k\geq{}\ksafe$.
\mscomment{this is the lemma}
\begin{lem}\label{lem:estimation_bound}Define $\tauls := d \max\{1,\dimu/\dimx\}\left(\opnorm{\Pst}^3\calP_0 +  \opnorm{\Pst}^{11} \Psibst^6 \right)  \log \frac{d \opnorm{\Pst}}{\delta}$. The, 
  There is an event $\Els$, which holds with probability at least $1-\delta/8$,
  such that conditioned on $\Els \cap \Esafe \cap \Ebound$,
\begin{align*}\|\Ahat_k - \Ast\|_\fro^2 + \|\Bhat_k - \Bst\|_\fro^2  \lesssim\frac{\dimx \dimu\|\Pst\|_{\op}^2}{\sigmain^2 \tau_k^{1/2}}\log(\tfrac{1}{\delta}) + \|\Pst\|_{\op}^3\frac{\dimx^2}{\tau_k}\log( \tfrac{1}{\delta})^2, \quad \forall k: c\tauls \le \tau_k \le T,
\end{align*}
where $c>0$ is a universal constant.
\end{lem}

    \begin{algorithm}[t!]
    \textbf{Input: }Examples
    $\matx_{\tau_{k}-1},\ldots,\matx_{\tau_k}$, $\matu_{\tau_{k-1}},\ldots,\matu_{\tau_{k}-1}$.\\
  \textbf{Return} $(\Ahat_k,\Bhat_k,\matLam_k)$, where 
  \begin{align*}
  \begin{bmatrix} \Ahat_k & \Bhat_k \end{bmatrix} \leftarrow \left(\sum_{t = \tau_{k-1}}^{\tau_k-1}\matx_{t+1}\begin{bmatrix}\matx_t \\ \matu_t\end{bmatrix}^\top\right)\matLam_k^{\dagger}, \quad\text{and}\quad \matLam_k \gets \sum_{t = \tau_{k-1}}^{\tau_k-1}(\matx_t,\matu_t)(\matx_t,\matu_t)^{\top}.
  \end{align*}
  \caption{$\OLS(k)$\label{alg:ols}}
  \end{algorithm}

  \begin{algorithm}[t!]
\textbf{Input:} Stabilizable pair $(\Ahat,\Bhat,\Conf,\delta)$. \\
\textbf{Return} $\Ballsafe \ldef \calB_{\op}(\Conf;\Ahat,\Bhat)$ and
      $\sigmain^2 \ldef \sqrt{\dimx}\opnorm{\Pinf(\Ahat,\Bhat)}^{9/2}\max\{1,\opnorm{\Bhat}\} \sqrt{\log \frac{\opnorm{\Pinf(\Ahat,\Bhat)}}{\delta}}$.
\caption{$\SafeRoundInit(\Ahat,\Bhat,\Conf,\delta)$}\label{alg:saferoundinit}
\end{algorithm}
We now put all of these pieces together to prove the final regret
bound. Henceforth, we condition on the event $\Esafe \cap \Ebound \cap
\Ereg \cap \Els$. To begin, consider the sum of errors
$\cJ_k-\cJ_{\star}$ in \Cref{eq:first_line_decomp}. We apply
Lemma~\ref{lem:xt_bound} followed by the bound on $\Jfunc_k \lesssim
\Jst$ from Lemma~\ref{lem:perturb_correct}, which yields
\begin{align*}
&\sum_{k=\ksafe}^{\kfin} \tau_k (\Jfunc_k - \Jst) + \log T \max_{k \le \log_T}\|\matx_{\tau_k}\|^2 \\
&\le \sum_{k > \tauls} \tau_k (\Jfunc_k - \Jst) +  \Jst \sum_{k:\tau_k \le c \tauls} \tau_k +  \sqrt{\Psibst \Jfunc_0\log(1/\delta)}\opnorm{\Pst}^{3/2} \log T \\
&\lesssim \left\{\sum_{k > \tauls} \tau_k (\Jfunc_k - \Jst)\right\}   + \tauls \Jst + \sqrt{\Psibst \Jfunc_0\log(1/\delta)}\opnorm{\Pst}^{3/2} \log T \\
&\lesssim \left\{\sum_{k > \tauls} \tau_k (\Jfunc_k - \Jst)\right\}   + \dimx\opnorm{\Pst}\tauls  \log  \frac{1}{\delta},
\end{align*}
where the last line uses that $\delta\leq1/T$ to combine the lower-order
terms in the line preceding it. Next, using the bound $\Jfunc_k - \Jst
\lesssim \opnorm{\Pst}^8\left(\fronorm{\Ast-\Ahat}^2 +
  \fronorm{\Bst-\Bhat}^2\right)$ from Lemma~\ref{lem:perturb_correct}
followed by the error bound in Lemma~\ref{lem:estimation_bound}, we have
\begin{align*}
\sum_{k > \tauls} \tau_k (\Jfunc_k - \Jst) &\le \opnorm{\Pst}^8\sum_{k > \tauls} \frac{\dimu\dimx}{\sigmain^2\sqrt{\tau_k}}\opnorm{\Pst}^2\log\frac{1}{\delta} + \frac{\dimx^2}{\tau_k}\opnorm{\Pst}^{3}\log^2\frac{1}{\delta}\\
&\lesssim \frac{\dimu\dimx\sqrt{T}}{\sigmain^2}\opnorm{\Pst}^{10}\log\frac{1}{\delta} +\ \underbrace{ \dimx^2\opnorm{\Pst}^{3}\log^2\frac{1}{\delta} \log T}_{\lesssim \dimx\tauls  \log^2  \frac{1}{\delta}},
\end{align*}
where again we use $\log T \lesssim \log(1/\delta)$. Combining the
computations so far shows that
\begin{align*}
\sum_{k=\ksafe}^{\kfin} \tau_k (\Jfunc_k - \Jst) + \log T \max_{k \le \log_T}\|\matx_{\tau_k}\|^2
&\le \frac{\dimu\dimx\sqrt{T}}{\sigmain^2}\opnorm{\Pst}^{10}\log\frac{1}{\delta} +\dimx\tauls  \log^2  \frac{1}{\delta}.
\end{align*}
Hence, on $\Esafe \cap \Ebound \cap \Ereg \cap \Els$m the regret in
the episodes $k \ge \ksafe$ decomposes into a component scaling with
$\sqrt{T}$ and a component scaling with $\log{}T$:
 \begin{align*}
   &\sum_{t=\tau_{\ksafe}}^{T}(\matx_t^\top \Rx \matx_t + \matu_t^\top \Ru \matu_t - \Jst) \\
      &\qquad\lesssim \underbrace{\sqrt{T}\left(\dimu \sigmain^2 \Psibst^2\opnorm{\Pst} + \sqrt{d \log(1/\delta)}\opnorm{\Pst}^4 + \frac{\dimu\dimx}{\sigmain^2}\opnorm{\Pst}^{10}\log\frac{1}{\delta} \right)}_{\text{$\sqrt{T}$-component)}}\\
      &\qquad\qquad+\underbrace{(1+\sqrt{d}\sigmain^2\Psibst^2) \opnorm{\Pst}^4 \log^2 \frac{1}{\delta}  +  \dimx\tauls  \log^2  \frac{1}{\delta}}_{(\text{($\mathrm{poly}(\log{}T)$-component)}}.
      \end{align*}
      Using that $\sigmain^2 \eqsim
      \sqrt{\dimx}\opnorm{\Pst}^{9/2}\Psibst \sqrt{\log
        \frac{\opnorm{\Pst}}{\delta}}$ (\Cref{lem:perturb_correct})
      and recalling that $d = \dimx + \dimu$, we upper bound these terms as
      \begin{align*}
        &\text{($\sqrt{T}$-component)} \lesssim  \sqrt{T \dimu^2\dimx\Psibst^2\opnorm{\Pst}^{11} \log \frac{1}{\delta}}, \\
        &\text{($\mathrm{poly}(\log{}T)$-component)} \lesssim \dimx\tauls  \log^2  \frac{1}{\delta}.
      \end{align*}
      We conclude that conditioned on $\Esafe \cap \Ebound \cap \Ereg \cap \Els$,
      \begin{align}
      \sum_{t=\tau_{\ksafe}}^{T}(\matx_t^\top \Rx \matx_t + \matu_t^\top \Ru \matu_t - \Jst) \lesssim \sqrt{T \dimu^2\dimx\Psibst^2\opnorm{\Pst}^{11} \log \frac{1}{\delta}} +\dimx\tauls  \log^2  \frac{1}{\delta}. \label{regret_on_safe}
      \end{align}
    To finish the proof, we (a) verify that $\Esafe$ indeed holds with
    high probability, and (b) bound the regret contribution of the initial rounds (proof given in \Cref{ssec:lem:initial_phase}).
  \begin{lem}\label{lem:initial_phase} The event $\Esafe$ holds with probability $1 - \frac{\delta}{2}$, and the following event $\calE_{\mathrm{reg,init}}$ holds with probability $1 - \frac{\delta}{8}$: 
  \begin{align*}
  \sum_{t=1}^{\tau_{\ksafe}-1}\matx_{t,0}^\top \Rx \matx_{t,0} + \matu_{t,0}^\top \Ru \matu_{t,0} 
  &\lesssim \calP_0 d^2\Psibst^2 \opnorm{\Pst}^{10} (1+\|K_0\|_{\op}^2)\log \frac{d\Psibst^2 \calP_0}{\delta}\log\tfrac{1}{\delta}.
  \end{align*}
  \end{lem}
  Thus, $\Esafe \cap \Ebound \cap \Ereg \cap
  \Els\calE_{\mathrm{reg,init}} $ holds with total probability at
  least $1-\delta$, and conditioned on this event Lemma~\ref{lem:initial_phase} and \Cref{regret_on_safe} imply
   \begin{align*}
     \Regret_{T}[\Alg;\Ast,\Bst] &= \sum_{t=1}^{T}(\matx_t^\top \Rx \matx_t + \matu_t^\top \Ru \matu_t - \Jst) \\&\lesssim \sqrt{T \dimu^2\dimx\Psibst^2\opnorm{\Pst}^{11}\log \frac{1}{\delta} } \\
      &\qquad+ \calP_0 d^2\Psibst^2 \opnorm{\Pst}^{10} (1+\|K_0\|_{\op}^2)\log \frac{d\Psibst^2 \calP_0}{\delta}\log\tfrac{1}{\delta} + \dimx\tauls  \log^2  \frac{1}{\delta}
    \end{align*}
    Recalling that $\tauls := d \max\{1,\dimu/\dimx\} \left(\opnorm{\Pst}^3\calP_0 +
      \opnorm{\Pst}^{11} \Psibst^6 \right)  \log \frac{d
      \opnorm{\Pst}}{\delta} $,  that $\calP_0, \opnorm{\Pst}\Psibst \ge
    1$, and that $d\opnorm{\Pst} \le d\Jst \le d
    \Jfunc_0 = d^2\calP_0$, we move to a simplified upper bound:
    \begin{align*}
     \Regret_{T}[\Alg;\Ast,\Bst] &\lesssim \sqrt{T \dimu^2\dimx\Psibst^2\opnorm{\Pst}^{11}\log \frac{1}{\delta} } \\
      &+  d^2 r \calP_0 \Psibst^6 \opnorm{\Pst}^{11} (1+\|K_0\|_{\op}^2)\log \frac{d^2\Psibst^2 \calP_0}{\delta} \log^2 \frac{1}{\delta}, 
    \end{align*}
    where $r = \max\{1, \frac{\dimu}{\dimx}\}$. 
    Since the square of $d\Psibst$ inside the logarithm contributes only a constant factor, we may remove it in the final bound. Moreover, since  This concludes the proof.

    \qed


%% file: body/conclusion.tex

We have established that the asymptotically optimal regret for the
online LQR problem is $\wt{\Theta}(\sqrt{\dimu^{2}\dimx{}T})$, and
  that this rate is attained by $\veps$-greedy exploration. We are
  hopeful that the our new analysis techniques, especially our perturbation bounds, will find
  broader use within the non-asymptotic theory of control and
  beyond. Going forward our work raises a number of interesting conceptual
  questions. Are there broader classes of ``easy''
  reinforcement learning problems beyond LQR for which naive
  exploration attains optimal sample complexity, or is LQR a fluke?
  Conversely, is
  there a more demanding (eg, robust) version of the LQR
  problem for which more sophisticated exploration techniques such as
  robust synthesis \citep{dean2018regret} or optimism in the face of
  uncertainty \citep{abbasi2011regret,cohen2019learning} are required
  to attain optimal regret? On the purely technical side, recall that
  while our upper and lower bound match in terms of dependence on $\dimu$, $\dimx$,
  and $T$, they differ in their polynomial dependence on
  $\nrm*{\Pst}_{\op}$. Does closing this gap require new algorithmic
  techniques, or will a better analysis suffice? 


%% file: body/acks.tex

\paragraph{Acknowledgements} Max Simchowitz is generously supported by an Open Philanthropy graduate student fellowship. Dylan Foster acknowledges the support of NSF TRIPODS award \#1740751.

%% file: appendix/app_org_notation.tex

	\subsection{Notation}

	\begin{center}
	\begin{longtable}{| l | l |}
	\hline
	\textbf{Notation} & \textbf{Definition}  \\
	\hline
	$T$ & problem horizon\\
	$\dimx,\dimu$ & state/input dimension\\
	$\matx_t,\matu_t$ & state/input at time $t$\\
	$\matw_t$ & noise at time $t$\\
	$\Rx,\Ru$ & control costs\\
	$\Regret_{A,B,T}[\pi]$ & Regret of a policy (as a random variable)\\
	$\SimpleRegret_{A,B,T}[\pi]$ & Expected Regret of a policy\\
	$\calR_{\Ast,\Bst,T}(\eps)$ &  $\min_{\pi}\max_{A,B}\left\{\SimpleRegret_{A,B,T}[\pi]:\fronorm{A - \Ast}^2 \vee \fronorm{B - \Bst}^2  \le \epsilon\right\}.$\\
	\hline
	$\Pinf(A,B)$ & Solution to the $\DARE$\\
	$\Kinf(A,B)$ & Optimal Controller for $\DARE$\\
	$\Jfunc_{A,B}[K]$ & Infinite horizon control cost of $K$ on instance $(A,B)$\\
	$\Hinf{A}$ & $\max_{z \in \C: |z|=1}\|(zI - A)^{-1}\|_{\op}$\\
	$\calB_{\op}(\epsilon;A_0,B_0)$& $\crl{(A,B)\mid{}\nrm*{A-A_0}_{\op}\vee\nrm*{B-B_0}_{\op}\leq{}\eps}$\\
	\hline
	$\dlyap(X,Y)$ & \makecell[l]{Solves $\calT_X[P] = Y$, where $\calT_X[P] := X^\top P X - X$.\\
	Requires $\rho(X) < 1$, $Y = Y^\top$.\\
	Given by $\sum_{i \ge 0}(X^i)^\top Y X^i.$} \\
	\hline
	\textbf{System parameters} & \\
	\hline
	$(\Ast,\Bst)$ & \makecell[l]{
	\textbf{Upper bound:} Ground truth for upper bound.\\
	\textbf{Lower bound:} Nominal instance for local minimax complexity.}\\
	\hline
	$\Pst$ & $\Pinf(\Ast,\Bst)$\\
	$\Kst$ & $\Kinf(\Ast,\Bst)$\\
	$\Aclst$ & $\Ast + \Bst \Kst$\\
	$\Jfuncopt$ & $\Jfuncopt_{\Ast,\Bst} := \min_{K}\Jfunc_{\Ast,\Bst}[K]= \Jfunc_{\Ast,\Bst}[\Kst]$\\
	$\Mbarst$ & $\max\{1,\opnorm{\Ast},\opnorm{\Bst}\}$\\$\Psibst$ &$\max\{1,\opnorm{\Bst}\}$\\
	\hline
	\end{longtable}
      \end{center}
      \subsection{Organization of the Appendices}
        The appendix is divided into \iftoggle{icml}{three}{two}
        parts. Part~\ref{part:technical_tools} establishes the main technical tools used throughout the upper and lower bounds. \Cref{app:perturbation} describes and proves our main perturbation bounds, deferring additional proof details to \Cref{app:perturbation_supporting}. Appendix~\ref{app:self_bounding} proves guarantees for the Self-Bounding ODE method, summarized in Corollary~\ref{cor:poly_self_bound}, as well as a slightly more general statement for generic self-bounding relations, \Cref{thm:general_valid_implicit}. This part of the appendix concludes with \Cref{app:ols}, which describes a set of tools for analyzing ordinary least squares estimation, which we use in the proofs of both our upper and lower bounds.

	\iftoggle{icml}
	{
	\Cref{part:lb} provides the proof of our lower bound, \Cref{thm:main_lb}. \Cref{sec:formal_lb} presents a complete proof in terms of numerous constituent lemmas, and \Cref{app:lb_proofs} proves these supporting lemmas. \Cref{part:ub}  mirror the structure of \Cref{part:lb}, with \Cref{sec:upper_bound_main} presenting formal pseudocode for our algortithm and a proof of the upper bound, and \Cref{app:ub_proofs} verifying the relevant constituent lemmas from  \Cref{sec:upper_bound_main}.
	}
	{
		\Cref{part:ub_lb_details} provides omitted details from the proofs of our main results. Specifically, \Cref{app:lb_proofs} proves the constituent lemmas for the lower bound from \Cref{sec:formal_lb}, and \Cref{app:ub_proofs} does the same for the proof of the upper bound given in \Cref{sec:upper_bound_main}. 
              }

\icml{\newpage}              
              

%% file: appendix/smoothness.tex

\paragraph{Preliminaries}
Throughout, we shall use extensively the $\dlyap$ operator, which we recall here.
\defdlyap*

 We shall  need to describe the ``$P$''-matrix analogue of the functional $\Jfunc$.
\begin{defn} Suppose that $(\Ast + \Bst K)$ is stable. We define $\Pinf(K;\Ast,\Bst) := \dlyap(\Ast+\Bst K, \Rx + K^\top \Ru K$.
\end{defn}It is a standard fact (see e.g. Lemma~\ref{lem:P_bounds_lowner}) that $\Jfunc_{\Ast,\Bst}[K] = \tr(\Pinf(K;\Ast,\Bst))$ whenever $\Ast + \Bst K$ is stable. We also recall the definition of the $\Hinfty$-norm.
\iftoggle{icml}
{
  \begin{restatable}[$\Hinfty$ norm]{defn}{defnhinf}\label{defn:Hinf_norm} For any stable $\tilde{A} \in \R^{\dimx^2}$ (e.g. $A + B\Kinf(A,B)$), we define $\Hinf{\tilde{A}} := \sup_{z \in \C: |z| = 1}\|(zI - \tilde{A})^{-1}\|_{\op}$. 
\end{restatable}
}
{
  \defnhinf*
}

\textbf{Organization of \Cref{app:perturbation}} The remainder of this appendix is organized as follows. \Cref{app:perturbation_main_results} states our main perturbation upper bounds, and provides proofs in terms of various supporting propositions. \Cref{ssec:derivative_computations} walks the reader through the relevant computations of various derivatives. \Cref{appssec:technical_tools} states numerous technical tools which we use in the proofs of our main perturbation bounds, and finally \Cref{ssec:supporting_perturbation} proves the supporting propositions leveraged in \Cref{app:perturbation_main_results}. Many supporting proofs are deferred to \Cref{app:perturbation_supporting}.

\subsection{Main Results}
\label{app:perturbation_main_results}

\subsubsection{Main Perturbation Upper Bound}
 Recall $\Csafe(\Ast,\Bst) = 54\opnorm{\Pst}^5$, and $\Cest(\Ast,\Bst)
 = 142 \opnorm{\Pst}^{8}$. We state a strengthening of our main
 perturbation bound from the main text
 (\Cref{thm:main_perturb_simple}) here.
\begin{restatable}{thm}{mainperturb}\label{thm:main_perturb_app} Let $(\Ast,\Bst)$ be a stabilizable system. Given an alternate pair of matrices $(\Ahat,\Bhat)$, for each $\circ\in\norms$ define $\epscirc := \max\{\circnorm{\Ahat-\Ast},\circnorm{\Bhat - \Bst}\}$. Then if $\epsop \le 1/\Csafe(\Ast,\Bst)$, 
\begin{enumerate}
  \item $\opnorm{\Pinf(\Ahat,\Bhat)} \le 1.0835 \opnorm{\Pst}$ and 
  $\|\Bst(\Kst - \Kinf(\Ahat,\Bhat))\|_{2} < \frac{1}{5 \opnorm{\Pst}^{3/2}}$.
\item $\Jfunc_{\Ast,\Bst}[\Kinf(\Ahat,\Bhat)] - \Jfunc^\star_{\Ast,\Bst} \le \Cest(\Ast,\Bst)\epsfro^2$.
\item $\|\Pinf(\Kinf(\Ahat,\Bhat);\Ast,\Bst) - \Pst\|_{\op} \le \Cest(\Ast,\Bst)\epsop^2$.
  \item Moreover, $\Pinf(\Kinf(\Ahat,\Bhat);\Ast,\Bst) \preceq (21/20)\Pst$.
\end{enumerate}
\end{restatable}

\begin{proof} 
 Throughout, we use $\Pst \succeq I$ (see Lemma~\ref{lem:lb_pst}). This theorem requires two consituent results. First, we have a perturbation bound for $\Pinf$ and $\Kinf$, which refines Proposition~\ref{prop:main_first_order}, and is proven in Section~\ref{sssec:prop:app_main_first_order}.

\begin{prop}\label{prop:app_main_first_order}  Let $(\Ast,\Bst)$ be a stabilizable system, and define the $\DARE$ solution $\Pst := \Pinf(\Ast,\Bst)$ and controller $\Kst = \Kinf(\Ast,\Bst)$. Given an alternate pair of matrices $(\Ahat,\Bhat)$, define for norms $\circ \in \{\op,\fro\}$ the error $\epsilon_{\circ} := \max\{\|\Ast - \Ahat\|_{\circ}, \|\Bst - \Bhat\|_{\circ}\}$. Then, if $\alpha := 8\|\Pst\|_\op^2 \epsilon_\op  < 1$, the pair $(\Ahat,\Bhat)$ is stabilizable, and 
\begin{align*}
&\|\Pinf(\Ahat,\Bhat)\|_{\op} \le (1-\alpha)^{-1/2}\|\Pst\|_{\op},\\
&\|\Ru^{1/2}(\Kinf(\Ahat,\Bhat) - \Kst)\|_{\circ} \le 7(1-\alpha)^{-7/4}\opnorm{\Pst}^{7/2}\,\epscirc,\\
&\|\Bst(\Kinf(\Ahat,\Bhat) - \Kst)\|_{\circ} \le 8(1-\alpha)^{-7/4}\opnorm{\Pst}^{7/2}\,\epscirc.
\end{align*}
In addition, if $\epsop \le 32\opnorm{\Pst}^3$, then
\begin{align*}
\|\Pst^{1/2}B(\Kinf(\Ahat,\Bhat) - \Kst)\|_{\circ} \le 9 (1-\alpha)^{-7/4}\opnorm{\Pst}^{7/2}\epscirc.
\end{align*}
\end{prop}

Next, we have a perturbation bound for the $\Jfunc$-functional as the controller $K$-is varied. The proof is deferred to Section~\ref{ssec:prop:Jfunc_bound}.
\begin{prop}\label{prop:Jfunc_bound} Fix any controller $K$ satisfying $\|\Bst(K - \Kst)\|_2 \le 1/5\opnorm{\Pst}^{3/2}$. Then,  
\begin{align*}
\Jfunc_{\Ast,\Bst}[K] - \Jfunc_{\Ast,\Bst}  &\le \opnorm{\Pst} \max\{\fronorm{K - \Kst}^2,\fronorm{\Pst^{1/2}\Bst(K - \Kst)}^2\},\\
 \|\Pinf(K;\Ast,\Bst) - \Pinf(\Ast,\Bst)\|_{\op} &\le \opnorm{\Pst} \max\{\opnorm{K - \Kst}^2,\opnorm{\Pst^{1/2}\Bst(K - \Kst)}^2\}.
\end{align*}
\end{prop}

 Now, observe that $\epsop \le 1/54\opnorm{\Pst}^{5} < 1/8\opnorm{\Pst}^2$ and $\alpha = 8\opnorm{\Pst}^2\epsop$, Proposition~\ref{prop:app_main_first_order} gives that 
 \begin{align*}
 \opnorm{\Pinf(\Ahat,\Bhat)} \le \opnorm{\Pst}/\sqrt{1-8/54}\le 1.0835\opnorm{\Pst},
 \end{align*}
 and that
\begin{align*}
5\opnorm{\Pst}^{3/2}\cdot \|\Bst(\Kinf(\Ahat,\Bhat) - \Kst)\|_{\op} &\le 8(1-\alpha)^{-7/4}\opnorm{\Pst}^{7/2}\,\epsop\\
&\le 40(1-\alpha)^{-7/4}\opnorm{\Pst}^{5}\,\epsop\\
&\le 40/54 \cdot (1 - 8/54)^{-7/4} < 1.
\end{align*}
Hence, for such $\epsop$, we find from Proposition~\ref{prop:Jfunc_bound} followed by Proposition~\ref{app:perturbation_first_derivatives_bounds} that 
\begin{align*}
\Jfunc_{\Ast,\Bst}[K] - \Jfunc_{\Ast,\Bst}  &\le \opnorm{\Pst} \max\{\fronorm{\Ru^{1/2}(K - \Kst)}^2,\fronorm{\Pst^{1/2}\Bst(K - \Kst)}^2\}\\
&\le 81 \opnorm{\Pst}^8 (1-\alpha)^{-7/2} \epsfro^{2}\\
  &\le 142 \opnorm{\Pst}^8 \epsfro^{2},
\end{align*}
and similarly, using $\opnorm{\Pst} \ge 1$,
\begin{align*}
 \|\Pinf(K;\Ast,\Bst) - \Pinf(\Ast,\Bst)\|_{\op} &\le 142 \opnorm{\Pst}^8 \epsop^{2} \le \frac{1}{20},\end{align*}
 yielding $\Pinf(K;\Ast,\Bst) \preceq (1+ \frac{1}{20})\Pst$ as $\Pst \succeq I$.

\end{proof}

\subsubsection{Perturbation of $\Hinfty$ Norm and Lyapunov Functions}
Next, we establish perturbation bounds on the $\Hinfty$ norm of the closed loop system, and show that all perturbed closed loop systems share a common Lyapunov function.
\begin{restatable}{thm}{hinfperturb}\label{thm:hinf_perturbation}Let $\Ast,\Bst$ be stabilizable, and let $(\Ahat,\Bhat)$ satisfy the conditions of Theorem~\ref{thm:main_perturb_app}, with $\Rx \succeq I$, and $\Ru = I$. Define $\Aclst := \Ast + \Bst\Kst$, and given $(\Ahat,\Bhat) \in \calB_{\op}(\epsilon_{\star},\Ast,\Bst)$, define 
  and $\Aclhat := \Ast + \Bst\Kinf(\Ahat,\Bhat)$. Then, 
 \begin{enumerate}
  \item $I \preceq \dlyap[\Aclst] \preceq \Pst$.
  \item $\Hinf{\Aclhat} \le 2\Hinf{\Aclnot} \le 4\opnorm{\dlyap[\Aclst]}^{3/2} \le  4\opnorm{\Pst}^{3/2}$.
  \item $\Aclhat^\top \cdot \dlyap[\Aclst] \cdot \Aclhat  \preceq (1 - \frac{1}{2}\opnorm{\dlyap[\Aclst}^{-1})\,\dlyap[\Aclst]\, \preceq (1 - \frac{1}{2}\opnorm{\Pst}^{-1})\,\dlyap[\Aclst]$. 
 \end{enumerate}
\end{restatable}

\noindent\emph{Proof of Part $1$.}
We can directly verify $\dlyap[\Aclst] \succeq I$ from the definition, and  $\dlyap[\Aclst] \preceq \Pst$ by Lemma~\ref{lem:closed_loop_dlyap}.

\noindent\emph{Proof of Part $2$.}
We use a general-purpose perturbation bound for the $\Hinfty$ norm, proved in \ref{ssec:prop:Hinfty_bound}.
\begin{prop}[$\Hinfty$ Bounds]\label{prop:Hinfty_bound} Fix $u \in (0,1)$, and matrixes $\Asafe, A_1 \in \R^{\dimx^2}$ with $\Asafe$ stable. Then if $\|A_1 - \Asafe\| \le \frac{\alpha}{\Hinf{\Asafe}}$, $\Hinf{A_1} \le \frac{1}{1-\alpha}\Hinf{\Asafe}$.
\end{prop}

From Part 1 of Theorem~\ref{thm:main_perturb_app},
\begin{align}
\|\Aclst - \Aclhat\|_{\op} \le \|\Bst(\Kinf(\Ast,\Bst) - \Kinf(\Ahat,\Bhat))\|_{\op} < \frac{1}{5\opnorm{\Pst}^{3/2}}. \label{eq:Acl_diff}
\end{align}
By Lemma~\ref{lem:Hinf_bound} followed by Lemma~\ref{lem:closed_loop_dlyap}, we have that
\begin{align*}
\Hinf{\Aclst}  \le 2\|\dlyap[\Aclst]\|_{\op}^{3/2} \le 2\|\Pst\|_{\op}^{3/2}.
\end{align*}
Therefore, since $\opnorm\Pst \ge 1$, we have
\begin{align*}
\|\Aclst - \Aclhat\|_{\op}  < \frac{1}{(5/2)\Hinf{\Aclst}} \le \frac{1}{2\Hinf{\Aclst}}.
\end{align*}
Propostion~\ref{prop:Hinfty_bound} then implies that $\Hinf{\Aclhat}
\le 2\Hinf{\Aclst}$. Moreover, by Lemma~\ref{lem:Hinf_bound}, we can
upper bound this in term by $4\opnorm{\dlyap[\Aclst]}^{3/2} \le
4\opnorm{\Pst}^{3/2}$.

\noindent\emph{Proof of Part $3$.}
Here, we use a perturbation bound which we prove from first principles, without the self-bounding ODE method (proved in \Cref{ssec:prop:lyapunov_perturbation}).
\begin{prop}\label{prop:lyapunov_perturbation} Suppose that $A$ is a stable matrix, and suppose that $\Ahat$ satisfies
\begin{align*}
\|\Ahat - A\|_{\op} \le \frac{1}{4}\min\left\{\frac{1}{\|\dlyap[A]\|_{\op}\|A\|_{\op}}, \|\dlyap[A]\|_{\op}^{-1/2}\right\}, 
\end{align*}
Then, $\Ahat^\top \dlyap[A] \Ahat \preceq (1-\frac{1}{2}\opnorm{\dlyap[A]}^{-1})\cdot \dlyap[A]$.
\end{prop}

By Lemma~\ref{lem:helpful_norm_bounds}, we have $\opnorm{\Aclst} \le \opnorm{\Pst}^{1/2}$. Since $\opnorm{\dlyap[\Aclst]} \le \opnorm{\Pst}$, combining with Eq.~\ref{eq:Acl_diff} gives
\begin{align*}
\|\Aclst - \Aclhat\|_{\op}  < \frac{1}{5\opnorm{\Pst}^{3/2}} < \frac{1}{4\opnorm{\dlyap[\Aclst] } \opnorm{\Aclst}}.
\end{align*}
Similarly, we have $\|\Aclst - \Aclhat\|_{\op} < \frac{1}{\opnorm{\dlyap[\Aclst]}^{1/2}}$, which means that, in particular, $\Aclst,\Aclhat$ satisfy the conditions for $A,\Ahat$ in \Cref{prop:lyapunov_perturbation}. This means that $\Aclhat^\top \dlyap[\Aclst]\Aclhat \preceq (1-\frac{1}{2}\opnorm{\dlyap[\Aclst]}^{-1})\cdot \dlyap[\Aclst]  (1-\frac{1}{2}\opnorm{\Pst}^{-1})$. The last inequality follows from Part $1$. 
\qedsymbol

\subsubsection{Continuity of the Safe Set }
We show that the size of the so-called ``safe'' set is continuous in nearby instances. This allows us to use an instance $(A_0,B_0)$ to guage whether the perturbed system $(\Ahat,\Bhat)$ is sufficiently close to $(\Ast,\Bst)$ to ensure correctness of the perturbation bounds.
\begin{thm}\label{thm:continuity_of_safe set}
Let $(A_0,B_0)$ be a stabilizable system. Then, for any pair of systems $(\Ahat,\Bhat),(\Ast,\Bst) \in \calB_{\op}(\frac{1}{3\Csafe(A_0,B_0)},A_0,B_0)$ is stabilizable, and satisfies $\max\{\opnorm{\Ast - \Ahat},\opnorm{\Bhat - \Bst}\} \le 1/\Csafe(\Ast,\Bst)$. Moreover, $\opnorm{\Pinf(\Ast,\Bst)} \le 1.0835\opnorm{\Pinf(A_0,B_0)} $.
\end{thm}
\begin{proof} Let $\epsilon_0 := \max\{\opnorm{A_0 - \Ahat},\opnorm{B_0 - \Bst}\} \le 1/\Csafe(A_0,B_0)$. Applying Theorem~\ref{thm:main_perturb_app} Part 1 with $(\Ahat,\Bhat) \leftarrow(\Ast,\Bst)$ and $(\Ast,\Bst) \leftarrow (A_0,B_0)$, we have $\opnorm{\Pinf(\Ast,\Bst)} \le 1.0835\opnorm{\Pinf(A_0,B_0)} $. Hence, $\Csafe(\Ast,\Bst) \le 1.5\Csafe(A_0,B_0)$. Hence $(\Ahat,\Bhat),(\Ast,\Bst) \in \calB_{\op}(\frac{2}{3\Csafe(A_0,B_0)},\Ast,\Bst) \subseteq \calB_{\op}(\frac{(2)\cdot (1.5)}{3\Csafe(\Ast,\Bst)},\Ast,\Bst) $, which means by triangle inequality that $\max\{\opnorm{\Ahat - \Ast},\opnorm{\Bhat - \Bst}\} \le 1/\Csafe(\Ast,\Bst)$. 
\end{proof}

\subsubsection{Quality of First-Order Taylor Approximation}
We bound the error of the first-order taylor expression in the following theorem.
\begin{thm}\label{thm:quality_of_taylor_approx} There exists universal constants $c,p > 0$ such that the following holds. Let $(\Ast,\Bst)$ be stabilizable, and let $\epscirc := \max\{\circnorm{\Ahat - \Ast},\circnorm{\Bhat - \Bst}\}$, and suppose that $\epsop \le 1/\Csafe(\Ast,\Bst)$ and $\Rx \succeq I$, $\Ru = I$. Let
\begin{align*}
K' := \frac{d}{dt}\Kinf(\Ast + t(\Ast -\Ahat),\Bst + t(\Bst - \Bhat)).
\end{align*}
Then, $\circnorm{\Kinf(\Ahat,\Bhat) - (\Kst + K')} \le c \opnorm{\Pst}^p \epsop^2\epscirc^2$. 
\end{thm}
\begin{proof}
Consider the curve $K(t) = \Kinf(A(t),B(t))$ for $A(t) = (1-t)\Ast + t\Ahat $ and $B(t) = (1-t)\Bst + t\Bhat$. By Theorem~\ref{thm:main_perturb_app}, the curve $A(t),B(t)$  for $t\in [0,1]$ consists of all stabilizable matrices with $\opnorm{\Pinf(A(t),B(t))} \lesssim \opnorm{\Pst}$. By Lemma~\ref{lem:computation_p_prime}, the curve $K(t)$ is analytic on $[0,1]$. Moreover, from Lemma~\ref{lem:computation_k_dprime} below, we have $\circnorm{K''(t)} \le c_0\opnorm{P(t)}^p \epsop^2\epscirc^2 \le c\opnorm{\Pst}^p \epsop^2\epscirc^2 $ for universal constants $c_0,p$. The bound now follows by Taylor's theorem. 
\end{proof}


\subsection{Key Derivative Computations \label{ssec:derivative_computations}}

  In the following computations, let $\delA = \Ahat - \Ast$ and $\delA := \Bhat - \Bst$. We recall $\epscirc := \max\{\circnorm{\delA},\circnorm{\delB}\}$. 

  We consider derivatives allow curves $(A(t),B(t)) = (\Ast + t\delA,\Bst + t\delB)$, and associated functions $P(t) := \Pinf(A(t),B(t))$ and $\Kinf(A(t),B(t))$ defined at stabilizes $A(t),B(t)$. All proofs are given in Section~\ref{ssec:derivative_computation_proofs}. 

  We begin by recalling the derivative computation from the main text, which also establishes local smoothness of $K(t)$ and $P(t)$.

  \computationpprime*

  Note that the above lemma allows for general analytic curves $(A(t),B(t))$. For our purposes, we restrict to linear curves given a above. For $K'$, we have the following computation

  \begin{lem}[Computation of $K\prm$]\label{lem:computation_k_prime}
  The first derivative of the optimal controller can be expressed as
  \begin{align}
  &K\prm = -(\Ru + B^\top P B)^{-1}\left(\delB^\top P \Acl + B^\top P(\delAcl) + B^\top P\prm \Acl\right).\label{eq:bpa_prime}
  \end{align}
  \end{lem}

  Of importance to our lower bound is the setting where the perturbations are of the form $(\delA,\delB) := (\Delta \Kst, \Delta)$. In this case, the expression for the derivative of $K$ simplifies considerably. we recall the following from the main text
  \lowerbounddercomp*
  \begin{proof}  Observe that for the perturbation in question, $\delAcl(0) = \Delta\Kst - \Delta K(0) = \Delta\Kst - \Delta\Kst = 0$. By Lemma~\ref{lem:computation_p_prime} and the fact that $\dlyap(X,0) = 0$, we have that $P'(0) = 0$. Thus, the term $B^\top P(\delAcl) + B^\top P\prm \Acl$ in Eq.~\eqref{eq:bpa_prime}  is $0$ at $t = 0$. The result follows. 
  \end{proof}

\subsubsection{Bounds on the Derivatives}

Here, we state bounds on the various derivatives. Recall $\epscirc := \max\{\circnorm{\delA},\circnorm{\delB}\}$. These bounds are established in Sections~\ref{app:perturbation_first_derivatives_bounds} and~\ref{app:perturbation_second_derivatives_bounds}, respectively.

\lemfirstderbound*
In fact, it will be more useful to prove the following related bound.
\begin{restatable}{lem}{Bkprime}\label{lem:Bkprime}$\circnorm{\Ru^{1/2}K'} \vee \circnorm{P^{1/2}BK'}\vee \circnorm{BK'} \le  7\opnorm{P}^{7/2}\epscirc$.
\end{restatable}

For our lower bounds, we shall also use a second-order derivative bound
\begin{restatable}[Bound on $K\dprm$]{lem}{kdprime}
\label{lem:computation_k_dprime}  If $\epscirc = \max\{\circnorm{\Ast - \Ahat},\circnorm{\Bst - \Bhat}\}$ and $K(t) = \Kinf(A(t),B(t))$ for $A(t) = (1-t)\Ast + t\Ahat $ and $B(t) = (1-t)\Bst + t\Bhat$, that at any $t$ at which $(A(t),B(t))$ is stabilizable,
\begin{align*}
\circnorm{K\dprm(t)} \le \poly(\opnorm{P(t)})\epsop\epscirc.
\end{align*}
\end{restatable}

\subsection{Main Control Theory Tools\label{appssec:technical_tools}}

\subsubsection{Properties of the $\dlyap$ Operator\label{sssec:dlyap}}  We  begin by describing relevant facts about  $\dlyap$ operator. The first is standard 
(see e.g. \citep{bof2018lyapunov,BoydNotes}), and gives a closed-form expression for the function.
\begin{lem} Let $Y = Y^\top$ and $\rho(X) < 1$. Then $\calT_{X}(Y) := X^\top Y X - Y$ is an invertible map from $\Symdx \to \Symdx$, and 
  \begin{equation}
\label{eq:dlyap_series}
\dlyap(X,Y) = \calT_X^{-1}(Y) = \sum_{k = 0}^{\infty} (X^{\trn})^kYX^k.
\end{equation}
\end{lem}
Next, we show that $\dlyap$ is order-preserving in the following sense.
\begin{lem}[Elementary $\dlyap$ bounds]\label{lem:closed_loop_dlyap} The following bounds hold
\begin{enumerate}
\item If $Y\preceq Z$ and $\Asafe$ is stable, then $ \dlyap(\Asafe,X) \preceq \dlyap(\Asafe,Y)$.
\item $Y \succeq 0$ and $\Asafe$ is stable, $\dlyap(\Asafe,Y) \succeq Y$.
\item  Suppose $\Rx \succeq I$, and let $A+BK$ is stable. Then, 
 \begin{align*}
\pm\, \dlyap(A+BK,Y) \preceq \dlyap(A+BK,I) \|Y\|_{\op} \preceq \|Y\|_{\op}\cdot P_{\infty}[K;A,B].
\end{align*} 
\item When $\Rx \succeq I$, $\dlyap[A+BK] \preceq P_{\infty}[K;A,B]$, and $I \preceq \dlyap[A+B\Kinf(A,B)] \preceq P_{\infty}(A,B)$.
\item If $\Asafe$ is stable, $\|\dlyap[\Asafe]\|_{\op} = \|\dlyap[\Asafe^\top]\|_{\op}$.
\end{enumerate}
\end{lem}

Next, we give a standard identity which relates the cost functions $J$ to the $\dlyap$ operator. 
\begin{lem}[PSD bounds on $P$]\label{lem:P_bounds_lowner} Let $(\Ast,\Bst)$
  be a stabilizable system, and let $\Ast+\Bst K$ be stable. Set $\Kst = \Kinf(\Ast,\Bst)$. Then, 
\begin{align*}
 P_{\infty}[K;\Ast,\Bst]  &\succeq  \Pinf(\Ast,\Bst) = \Pinf(\Kst;\Ast,\Bst).
\end{align*}
Moreover, we have $\Jfuncopt_{\Ast,\Bst}[K] = \tr( P_{\infty}[K;A,B])$, and in particular,
$\Jfuncopt_{\Ast,\Bst} = \Jfunc_{\Ast,\Bst}[\Kst] = \tr(\Pinf(\Ast,\Bst))$. As a conseuqnece, if $\Rx \succeq I$, then $\Jfuncopt_{\Ast,\Bst}[K] \ge \Jfuncopt_{\Ast,\Bst} \ge \dimx$ by \Cref{lem:closed_loop_dlyap} part 4.
\end{lem}
The following is a consequence of the above lemmas, and is useful for deriving interpretable corollaries of our main results.
\begin{restatable}{lem}{stablep}\label{lem:stable_p_bound}
        Suppose that $\Rx = I$. If $\Ast$ is
        $(\gamma,\kappa)$-strongly stable, then $\opnorm{\Pst} \le
        \gamstab^{-1}$ and  $\frac{1}{\dimx}\Jfunc_{\Ast,\Bst}[0] \le \gamstab^{-1}$. More generally, if $(\Ast + \Bst K)$ is $(\gamma,\kappa)$-strongly stable, then $\opnorm{\Pst} \le  \gamstab^{-1}(1 + \opnorm{K}^2)$. 
      \end{restatable}
      \begin{proof}[Proof of Lemma~\ref{lem:stable_p_bound}]
      By considering the controller $K = 0$, Lemma~\ref{lem:P_bounds_lowner} implies $\Pst \preceq \dlyap[\Ast, \Rx]$ and $\Jfunc_{\Ast,\Bst}[0] = \tr(\dlyap[\Ast, \Rx]) \le \frac{\opnorm{\dlyap[\Ast, \Rx]}}{\dimx}$. If $\Rx = I$, and we can bound
      \begin{align*}
      \opnorm{\dlyap[\Ast, I]} \le \sum_{t \ge 0}\opnorm{A^t}^2.
      \end{align*}
      If  there exists a transform $T$ with $\sigma_{\max}(T)/\sigma_{\min}(T) \le \kappa$ such that $\|T\Ast T^{-1}\|_{\op} \le 1 - \gamma$, then $\opnorm{A^t}\le \kappa (1-\gamma)$. Hence, $\opnorm{\Pst} \le  \opnorm{\dlyap[\Ast, I]} \le \kappa^2\sum_{t \ge 0}\gamma^2 \le \frac{\kappa^2}{1-(1-\gamma)^2} \le \frac{\kappa^2}{1 - (1-\gamma)} = \kappa^2\gamma^{-1}$. More generally, we have that $\Pst \preceq \dlyap[\Ast + \Bst K, \Rx + K^\top \Ru K]$ for $\Rx, \Ru = I$, $\Rx + K^\top \Ru K \preceq (1+\opnorm{K}^2)I$, and the bound follows by invoking Lemma~\ref{lem:closed_loop_dlyap}.
      \end{proof}

\subsubsection{Helpful Norm Bounds}
\begin{lem}[Helpful norm bounds]\label{lem:helpful_norm_bounds} 
Let $(\Ast,\Bst)$ be given, with $\Pst = \Pinf(\Ast,\Bst)$, $\Kst = \Kinf(\Ast,\Bst)$, and $\Aclst = \Ast+\Bst \Kst$. If $\Rx \succeq I$, $\Ru = I$, then the  following bounds hold:
\begin{enumerate}
  \item $\Pst \succeq I$, so that $\|\Pst^{-1}\|_2 \le 1$, and $\opnorm{\Pst} \ge 1$.
  \item $\opnorm{\Kst}^2 \le \opnorm{\Pst}$ and $\opnorm{\Aclst}^2 \le \opnorm{\Pst}$. 
  \item More generally, if $(\Ast + \Bst K)$ is stable, $K^\top K \preceq \Pinf(K;\Ast,\Bst) = \dlyap(\Ast + \Bst K, \Rx + K^\top \Ru K)$.
\end{enumerate}
\end{lem}

\subsubsection{Bounds on $\Pinf(K;\Ast,\Bst)$ and
  $\Jfunc_{\Ast,\Bst}[K]$} We now state a variant of a result due to
\cite{fazel2018global}, which bounds the effect of perturbations on $\Pinf(K;\Ast,\Bst) - \Pinf(\Ast,\Bst)$.
\begin{lem}[Generalization of Lemma 12 of \cite{fazel2018global}\label{lem:performance_diff}, see also Eq 3.2 in \cite{ran1988existence}] Let $K$ be an arbitrary static controller which stabilizes $\Ast,\Bst$. Then, 
\begin{align*}
\Pinf(K;\Ast,\Bst) - \Pinf(\Ast,\Bst) &= \dlyap( \Ast + \Bst K, (K - \Kst)^\top (\Ru + \Bst^\top \Pst \Bst)(K - \Kst)).
\end{align*}
\end{lem} 
As a consequence of Lemma~\ref{lem:performance_diff} and~\ref{lem:closed_loop_dlyap}, we have the following corollary.
\newcommand{\Sigmaadj}{\Sigma^{\mathrm{adj}}}
\begin{cor}\label{cor:value_subopt}
Let $K$ be any arbitrary static controller which stabilities
$\Ast,\Bst$, and suppose $\Ru = I$. Define the
\emph{adjoint\footnote{Note that the canonical state covariance
    matrix $\Sigma_{\Ast,\Bst}[K]$ is given by $\dlyap((\Ast + \Bst K)^{\top},I)$. By Lemma~\ref{lem:P_bounds_lowner}, we have that $\opnorm{\Sigma_{\Ast,\Bst}[K]} = \opnorm{\Sigmaadj_{\Ast,\Bst}[K]}$ } as $\Sigmaadj_{\Ast,\Bst}[K] := \dlyap(\Ast + \Bst K,I)$} covariance matrix. Then,
\begin{align*}
\Jfunc_{\Ast,\Bst}[K] - \Jfunc_{\Ast,\Bst} &\le  \|\Sigmaadj_{\Ast,\Bst}[K]\|_{\op} \max\{\|\Ru^{1/2}(K - \Kst)\|_{\fro}^2,\|\Pst^{1/2}B(K - \Kst)\|_{\fro}^2\},\\
\|\Pinf(K;\Ast,\Bst) - \Pinf(\Ast,\Bst)\|_{\op} &\le  \|\Sigmaadj_{\Ast,\Bst}[K]\|_{\op} \max\{\|\Ru^{1/2}(K - \Kst)\|_{\op}^2,\|\Pst^{1/2}B(K - \Kst)\|_{\op}^2\}.
\end{align*}
\end{cor}

\subsubsection{Linear Lyapunov Theory} We now state a classical result
in Lyapunov theory (see, e.g. \cite{BoydNotes}). Recall the notation $\dlyap[A] := \dlyap(A,I)$.
\begin{lem}\label{lem:basic_lyap} For any $x \in \R^\dimx$ and stable $A$, we have $\dlyap[A]\succeq I$ and 
\begin{align*}
A^\top \dlyap[A] A \preceq (1 - \opnorm{\dlyap[A]}^{-1})\cdot \dlyap[A].
\end{align*}
\end{lem}
\begin{lem}\label{lem:Hinf_dlyap_bound}For any stable $A$, $\Hinf{A} \le 2\opnorm{\dlyap[A]}^{3/2}$\label{lem:Hinf_bound}. More generally, suppose that $P \succeq I$ is a matrix satisfying $(Ax)^\top \dlyap[A] (Ax) \le (1 - \rho) x^\top P x$. Then, 
\begin{align*}
\Hinf{A} \le \sum_{t \ge 0}\|A^t\|_2 \le2\frac{\sqrt{\opnorm{P}}}{\rho}.
\end{align*}
\end{lem}

\subsection{Proofs for Supporting Perturbation Upper Bounds \label{ssec:supporting_perturbation}}

\subsubsection{Proof of Proposition~\ref{prop:app_main_first_order}\label{sssec:prop:app_main_first_order}}
The proof is analogous to that of Proposition~\ref{prop:main_first_order}, except we also apply the derivative bound on $\Ru^{1/2}K'(t)$ $BK'(t)$ from Lemma~\ref{lem:Bkprime}. That bound also gives
\begin{align*}
\|\Ru^{1/2} K'(t)\|_{\circ} &\le  7\opnorm{P(t)}^{7/2}\epscirc\\
\|\Bst K'(t)\|_{\circ} &= \epsop \|K'(t)\|_{\circ} + \|B(t) K'(t)\|_{\circ} \le (1+\epsop)7\opnorm{P(t)}^{7/2}\epscirc \le 8\opnorm{P(t)}^{7/2}\epscirc,
\end{align*}
so that the desired bound follow by the mean value theorem.

Moreover, we have
\begin{align*}
\|\Pst^{1/2}\Bst K'(t)\|_{\circ} \le \|\Pst^{1/2}P(t)^{-1/2}\|_{\op}\|P\Bst K'(t)\|_{\circ} \le \|\Pst^{1/2}P(t)^{-1/2}\|_{\op} 8\opnorm{P(t)}^{7/2}\epscirc.,
\end{align*}
which translates to a bound of 
\begin{align*}
\|\Pst^{1/2}B(\Kinf(\Ahat,\Bhat) - \Kst)\|_{\circ} \le \max_{t \in [0,1]} \|\Pst^{1/2}P(t)^{-1/2}\|_{\op} \le 7\opnorm{P(t)}^{7/2}\epscirc.
\end{align*}
Finally, by the mean value theorem, we can bound for $\epsop \le 1/32\opnorm{\Pst}^3$ and $\alpha = 8\epsop\opnorm{\Pst}^2 \le 1/4$,
\begin{align*}
\|P(t) - \Pst\|_{\op} &\le 4 \max_{t \in[0,1]} \opnorm{P(t)}^3\epsop \le 4\opnorm{\Pst}^3(1-\alpha)^{-3/2} \le \frac{1}{8}(4/3)^{3/2}.
\end{align*}
Since $P(t) \succeq I$, then, this implies that for $t \in [0,1]$, $P(t) \succeq (1 - \frac{1}{8}(4/3)^{3/2}) \Pst$, yielding $\|\Pst^{1/2}P(t)^{-1/2}\|_{\op} \le \sqrt{(1 - \frac{1}{8}(4/3)^{3/2})} \le 9/8$. Hence, for this $\epsop$, we have
\begin{align*}
\|\Pst^{1/2}\Bst(\Kinf(\Ahat,\Bhat) - \Kst)\|_{\circ} \le 9 \opnorm{P(t)}^{7/2}\epscirc.
\end{align*}

\qed

\subsubsection{Proof of Proposition~\ref{prop:Jfunc_bound}\label{ssec:prop:Jfunc_bound}}

  Since the final bound we derive does not depend on the control basis, we may assume without loss of generality that $\Ru = I$. 
  Recall the steady state covariance matrix $\Sigmaadj_{\Ast,\Bst}[\Kst] := \dlyap(\Ast + \Bst K,I)$. We shall prove the following lemma.
  \begin{lem}\label{lem:Sigma_perturb}
  Suppose that $\|\Bst(K - \Kst)\|_2 \le 1/5\opnorm{\|\Sigmaadj_{\Ast,\Bst}[\Kst]\|}^{3/2}$, then $\|\Sigmaadj_{\Ast,\Bst}[K]\| \le 2\|\Sigmaadj_{\Ast,\Bst}[\Kst]\|$.
  \end{lem}
  Note that a similar result was given by Lemma 16 \citep{fazel2018global}; we give our proof using the self-bounding ODE method to demonstrate the generality of its scope, and to avoid dependence on system matrices. Noting that $\opnorm{\Sigmaadj_{\Ast,\Bst}[\Kst]} \le \opnorm{\Pst}$ as verified above, it is enough that $\|\Bst(K - \Kst)\|_2 \le 1/5\opnorm{\Pst}$ to ensure that $\|\Sigmaadj_{\Ast,\Bst}[K]\| \le 2\Pst$.  When this holds, we have by Corollary~\ref{cor:value_subopt}, we have (assuming $\Ru = I$)
  \begin{align*}
  \Jfunc_{\Ast,\Bst}[K] - \Jfunc_{\Ast,\Bst} &\le \opnorm{\Sigmaadj_{\Ast,\Bst}[K]}  \max\{\fronorm{\Ru^{1/2}(K - \Kst)}^2, \fronorm{\Pst^{1/2}\Bst (K - \Kst)}^2\} \\
  &\le \opnorm{\Pst}  \max\{\fronorm{\Ru^{1/2}(K - \Kst)}^2, \fronorm{\Pst^{1/2}\Bst (K - \Kst)}^2\}, \\
  \opnorm{\Pinf(K;\Ast,\Bst) - \Pst}  &\le \opnorm{\Pst}  \max\{\opnorm{\Ru(K - \Kst)}^2, \opnorm{\Pst^{1/2}\Bst (K - \Kst)}^2\},
  \end{align*}
   as needed.
  \begin{proof}[Proof of Lemma~\ref{lem:Sigma_perturb}]

\newcommand{\Ktil}{\widetilde{K}}
We shall now use the self-bounding machinery developed above to bound $\Sigmaadj_{\Ast,\Bst}[K]$. Introduce the straight curve $\Ktil(t) := \Kst + t\delK$, where $\delK = K - \Kst$, and where the $\tilde{(\cdot)}$ is to avoid confusion with the curve $K(t) = \Kinf(A(t),B(t))$.  Let $\Sigma(t) = \dlyap(\Ast + \Bst \Ktil(t), I)$, so that $\Sigma(0) = \Sigmaadj_{\Ast,\Bst}[\Kst]$ and $\Sigma(1) = \Sigmaadj_{\Ast,\Bst}[K]$.

 By the definition of $\dlyap$, we have that at all $t$ for which $K(t)$ stabilizes $\Ast,\Bst$, 
\begin{align*}
\Sigma(t) = (\Ast + \Bst \Ktil(t))^\top \Sigma(t)(\Ast + \Bst \Ktil(t)) + I.
\end{align*}
We shall now prove that $\Sigma(t)$ satisfies a self-bounding relation analogous to \Cref{prop:main_first_order}.
\begin{claim} For all $t \in [0,1]$ for which $\Sigma(t)$ is defined, $\|\Sigma'(t)\|_\op \le 2\|\Sigma(t)\|^{5/2} \|\Bst \delK\|_\op$. 
\end{claim}
\begin{proof}
  Taking a derivative with respect to $\Sigma$, we have
  \begin{align*}
  \Sigma'(t) = (\Ast + \Bst \Ktil(t))^\top \Sigma'(t)(\Ast + \Bst \Ktil(t)) + Q_\Sigma(t),
  \end{align*}
  where $Q_{\Sigma}(t) = (\Bst \delK)^\top \Sigma(t)(\Ast + \Bst \Ktil(t)) + (\Ast + \Bst \Ktil(t))^\top \Sigma(t)\Bst \delK$. Thus, we can render 
  \begin{align*}
  \Sigma'(t) = \dlyap(\Ast+\Bst\Ktil(t), Q_\Sigma(t)).
  \end{align*}
  By an argument analogus to \Cref{lem:closed_loop_dlyap}, we have $\pm \Sigma'(t) \preceq \|Q_{\Sigma}(t)\|\dlyap(\Ast+\Bst\Ktil(t),I) = \|Q_{\Sigma}(t)\| \Sigma(t)$, yielding the self-bounding relation
  \begin{align*}
  \|\Sigma'(t)\|_\op \le \|Q_{\Sigma}(t)\|_\op \|\Sigma(t)\|_\op.
  \end{align*}
  Moreover, we can bound for $t \in [0,1]$
  \begin{align*}
  \|Q_{\Sigma}(t)\|_{\op} &\le 2\|\Sigma(t)\|_\op\|\Bst\delK\|_\op  \opnorm{\Ast+\Bst\Ktil(t)} \\
  &\le 2\|\Sigma(t)\|_\op^{3/2}\|\Bst\delK\|_\op,
  \end{align*}
  where we use thaat  $\opnorm{\Ast+\Bst\Ktil(t)}^2 \le \opnorm{\dlyap(\Ast+\Bst\Ktil(t),I)} = \opnorm{\Sigma(t)}$. 
\end{proof}

We check explicitly that $\Sigma(t)$ corresponds to the solution of a valid implicit function with domain $\calU := \{\Sigmaadj: \Sigmaadj > 0\}$ (using the more general second condition that ensures that $t \mapsto \Sigma(t)$ is a continuously differentiable funciton, which follows from the form of  $\dlyap$). Applying Corollary~\ref{cor:poly_self_bound} with $p = 5/2$ and $c = 2\opnorm{\Bst \delK}$,  this yields that if $\alpha = (p-1)c \opnorm{\Sigma(0)}^{3/2} = 3\opnorm{\Sigma(0)}^{3/2}\opnorm{\Bst \delK} < 1$, then, $\opnorm{\Sigma(1)} \le (1 - u)^{-2/3} \opnorm{\Sigma(0)}$. In particular, if $\opnorm{\Bst \delK} \le 1/5\opnorm{\Sigma(0)}^{3/2}$, then we can show $\opnorm{\Sigma(1)} \le  2\opnorm{\Sigma(0)}$.
\end{proof}
\subsubsection{Proof of Proposition~\ref{prop:Hinfty_bound}\label{ssec:prop:Hinfty_bound}}

Introduce the curve $A(t) = \Asafe + t\delA$, where $\delA = A_1 - \Asafe$,  define $Y_z(t) := (zI - A(t))^{-1}$. Then, $\Hinf{A(t)} = \sup_{z \in \Torus}\|Y_z(t)\|_2$. Let us now use the self-bounding method to bound $\|Y_z(t)\|$. We can observe that
\begin{align*}
Y_z'(t) = (zI - A(t))^{-1} \delA (zI - A(t))^{-1},
\end{align*}
so that $\|Y_z'(t)\|_2 \le \|Y_z(t)\|_2^2 \|\delA\|.
$. Since $Y_z(t)$ corresponds to the zeros of the valid implicit function $F)z(A,Y) = Y \cdot (zI - A(t)) - I$, Theorem~\ref{thm:general_valid_implicit} implies that, if $\|\delA\| \le \frac{u}{\Hinf{\Asafe}} =  \min_{z \in \Torus}\frac{u}{\|Y_z(0)\|_2}$, then we have $\|Y_z(1)\| \le \frac{1}{1-\alpha}\|Y_z(0)\|$ for all $z \in \Torus$. Hence,
\begin{align*}
\Hinf{A_1} = \max_{z \in \Torus} \|Y_z(1)\| \le \max_{z \in \Torus}\|Y_z(0)\| = \frac{1}{1-\alpha}\Hinf{\Asafe},
\end{align*}
as needed.\qed

\subsubsection{Proof
  of~\Cref{prop:lyapunov_perturbation}\label{ssec:prop:lyapunov_perturbation}}
Observe that we have
  \begin{align*}
  &(\Ahat x)^\top \dlyap[A] (\Ahat x \\
  &\le (Ax)^\top \dlyap[A](Ax) + x^\top (\Ahat - A)^\top\dlyap[A] A x + x^\top (\Ahat - A)^\top\dlyap[A] (\Ahat - A) x\\
  &\le (1-\opnorm{\dlyap[A]}^{-1})\cdot x^\top \dlyap[A]x + \|x\|_2^2 \left(\|\Ahat - A\|_{\op}\|A\|_{\op} + \|\Ahat - A\|_{\op}^2\right) \|\dlyap[A]\|_{\op})\\
  &\le (1-\opnorm{\dlyap[A]}^{-1} + \left(\|\Ahat - A\|_{\op}\|A\|_{\op} + \|\Ahat - A\|_{\op}^2\right) \|\dlyap[A]\|_{\op})) \cdot x^\top \dlyap[A]x,
  \end{align*}
  where we used that $\dlyap[A] \succeq I$. In particular, if 
  \begin{align*}
  \|\Ahat - A\|_{\op} \le \frac{1}{4}\min\crl*{\frac{1}{\|A\|_{\op}\|\dlyap[A]\|_{\op}}, \|\dlyap[A]\|_{\op}^{-1/2}}, 
  \end{align*}
  then, the above is at most, $(1-\frac{1}{2}\opnorm{\dlyap[A]}^{-1})\cdot x^\top \dlyap[A]x$. \qed




%% file: appendix/derivative_comp.tex

\section{Supporting Proofs for \Cref{app:perturbation}}
\label{app:perturbation_supporting}

\subsection{Proofs for Main Technical Tools (Section~\ref{appssec:technical_tools}) \label{ssec:dylap_proofs}}
	We begin with the following lemma, which follows from a standard computation.
	\subsubsection{Proof of Lemma~\ref{lem:closed_loop_dlyap} \label{appsssec:closed_loop_dlyap}}
	  \begin{proof}
	  Let $\rho(A_0)<1$, and so from
	  \eqref{eq:dlyap_series} we have that for any $Z$ with $ Y \preceq Z $ that
	  \begin{align*}
	  \dlyap(A_0,Y) = \sum_{k = 0}^{\infty} (A_0^k)^\top YA_0^k &\preceq \sum_{k = 0}^{\infty} (A_0^{k})^\top Z(A_0^k).
	  \end{align*}
	  Second, if $Y \succeq 0$, $\sum_{k = 0}^{\infty} (A_0^k)^\top Y(A_0^k)\dlyap(A_0,Y)\succeq Y$. 

	  The third statement is a direct consequence of the first. Moreover, since $I \preceq \Rx \preceq \Rx + K^\top \Ru K$, taking $Z = \|Y\| (\Rx + K^\top \Ru K)$ yields the fourth inequality. 

	  For the last statement, let $\|x\|_2 = 1$. Then, we have
	  \begin{align*}
	  x^\top \dlyap[A_0] x &= \sum_{k = 0}^{\infty} x^\top(A_0^{\top})^k (A_0^k) x \,= \sum_{k = 0}^{\infty} \tr(A_0^{k} xx^\top (A_0^k)^\top)\\
	  &= \dlyap(A_0,xx^\top) \\
	  &\preceq \|xx^\top\|_{\op}\|\dlyap(A+BK,I)\|_{\op} I = \|\dlyap(A+BK,I)\|_{\op}I.
	  \end{align*}
	\end{proof}

	\subsubsection{Proof of Lemma~\ref{lem:P_bounds_lowner} \label{appsssec:lowner_proof}}
	We begin with the following lemma, whose proof is a straightforward computation.
	\begin{lem}\label{lem:value_form} Let $\Ast,\Bst$ be stabilizable. For a controler $K$ such that $\Ast + \Bst K$ is stable, we define the \emph{value function}
	\begin{align*}
	V^K(x) := \sum_{t=0}^\infty \cost(x_t^{K,x},Kx_t^{K,x}), \quad \text{where } x_0^{K,x} = x,\quad\text{and}\quad x_{t}^{K,x} = (\Ast + \Bst K)x^{K,x}_{t-1}. 
	\end{align*}
	We then have $x_{t}^{K,x} = (\Ast + \Bst K)^tx$, $\sum_{t=0}^{\infty} (x_{t}^{K,x})^\top Yx_t^{K,x} = \dlyap(\Ast + \Bst K,Y)$, and in particular,
	\begin{align*}
	V^K(x) = x^\top\dlyap(\Ast+\Bst K,\Rx+K^\top \Ru K)x = x^\top \Pinf(K)x,
	\end{align*}
	\end{lem}
	We now prove Lemma~\ref{lem:P_bounds_lowner}.

	\begin{proof}
	Introduce the shorthand $\Pst = \Pinf(\Ast,\Bst)$, $\Pinf(K) = \Pinf(K;\Ast,\Bst)$.  
	and in particular, $V^{\Kinf}(x) =x^\top \Pinf(x).$ It is well known that $x^\top \Pinf x  = V^{\Kinf}(x)$ and that
	$V^{\Kinf}(x) = \inf_K V^K(x) \leq{}V^{K}(x)$ \citep{bertsekas2005dynamic}. Hence, $\Pinf(K) \succeq \Pinf$. Finally, observe that by using that $A+BK$ is stable, we have
	\begin{align*}
	\Jfunc_{A,B}[L] &=  \lim_{t\to\infty}\frac{1}{t}\sum_{i=1}^{t}\Exp_{A,B,K}[\matx_i^\top \Rx \matx_i +
	\matu_i^\top\Ru \matu_i]\\
	&=\lim_{t\to\infty}\Exp_{A,B,K}[\matx_t^\top \Rx \matx_t +
	\matu_t^\top\Ru \matu_t]  \\
	&=  \tr\prn*{\sum_{s=0}^{\infty}\prn*{(A+BK)^{\top{}}}^{s}(\Rx + K^\top\Ru K)(A+BK)^{s}},\\
	&= \tr(\Pinf(K)).
	\end{align*}
	 The identity for $\Pinf$ is the
	special case where  $K = \Kinf$. 
	\end{proof}

\subsubsection{Proof of Lemma~\ref{lem:helpful_norm_bounds}}
\begin{proof}
We address each bound in succession. 
\begin{enumerate}
  \item $\sigma_{\min}(\Pst) \ge 1$ by Lemma~\ref{lem:closed_loop_dlyap}. 
  \item We have that 
  \begin{align*}
  \Pinf = \dlyap(\Ast+\Bst\Kst,\Rx + \Kst^\top \Ru \Kst) \succeq \Rx + \Kst^\top \Ru \Kst \succeq \Kst^\top \Kst,
  \end{align*} since $\Ru \succeq I$ and $\Rx \succeq I$. Moreover, we have that
  \begin{align*}
  \Pinf = \dlyap(\Ast+\Bst\Kst\Rx + \Kst^\top \Ru \Kst) &= \sum_{t=0}^t ((\Ast+\Bst\Kst)^{\top})^t (\Rx + \Kst^\top \Ru \Kst)  (\Ast+\Bst\Kst)^t\\
  &\succeq \sum_{t=0}^t ((\Ast+\Bst\Kst)^{\top})^t (\Ast+\Bst\Kst)^t\\
  &\succeq (\Ast+\Bst\Kst)^{\top}(\Ast+\Bst\Kst).
  \end{align*}
\end{enumerate}
\end{proof}

	\subsubsection{Proof of Lemma~\ref{lem:performance_diff}}

	\begin{proof} The first inequality is precisely Lemmas 12 in \cite{fazel2018global}. In light of Lemma~\ref{lem:value_form}, it suffices to show that
	\begin{align*}
	V^K(x) - V^{\Kst}(x) = x^\top \dlyap(\Ast + B\Kst, (K - \Kst)^\top(\Ru + \Bst^\top \Pst \Bst) (K - \Kst))
	\end{align*}
	Lemma 10 in \cite{fazel2018global} implies (noting $E_{\Kst} = 0$ for $E_K$ defined therein) that
	\begin{align*}
	V^K(x) - V^{\Kst}(x) &= \sum_{t=0}^{\top} (x_t^{K,x})^\top   (K - \Kst)^\top(\Ru + \Bst^\top \Pst \Bst)(K- \Kst)  x_t^{K,x}\\
	&= \dlyap( \Ast + \Bst K,  (K - \Kst)^\top(\Ru + \Bst^\top \Pst \Bst)(K- \Kst)),
	\end{align*}
	where the second inequality uses Lemmas 12 in \cite{fazel2018global}. 

	\end{proof}
	    
	\subsubsection{Proof of Lemma~\ref{lem:Hinf_dlyap_bound}}
	Let us prove the more general claim.
	\begin{align*}
	\Hinf{A} \le \sum_{t=0}^{\infty} \|A^i\|_{\op} &= \sum_{t=0}^{\infty} \sqrt{\|(A^i)^\top (A^i) \|_{\op}}\\
	&\le \sum_{t=0}^{\infty} \sqrt{\frac{1}{\sigma_{\min}(\dlyap[A])} \|(A^i)^\top P  (A^i) \|_{\op}}\\
	&\le \sum_{t=0}^{\infty} \frac{1}{\sigma_{\min}(P)}\sqrt{(1 - \rho)^{i}\|P\|_{\op}}\\
	&\le \|P\|_{\op}^{1/2}\sum_{t=0}^{\infty} \sqrt{1 - \rho}^{i} \quad \text{ since } P \succeq I \\
	&= \|P\|_{\op}^{1/2} \frac{1}{1-\sqrt{1 - \rho}} \\
	&\le \|P\|_{\op}^{1/2}  \frac{1+\sqrt{1 - \rho}}{1-(1 - \rho}\\
	&\le 2\|P\|_{\op}^{1/2}/\rho.
	\end{align*}

\subsection{Derivative
  Computations \label{ssec:derivative_computation_proofs}}
\label{app:perturbation_derivative_computations}

\subsubsection{Proof of Lemma~\ref{lem:computation_p_prime}}
	Recall the function
	\begin{align*}
    \Fdare([A,B],P) =  A^\top P A - P -  A^\top P B (\Ru + B^\top P B)^{-1}B^\top P A + \Rx.
    \end{align*}
    Let us compute the differentiable of this map. To keep notation, let us suppress the dependence of the $A,B$ arguments on $t$.  We have that
    \begin{align*}
    \Differential  \Fdare[\rmd P, \rmd t] \big{|}_{A(t),B(t),P} &= \Differential (A^\top P A) - \Differential P + \Differential (A^\top P B) \cdot (\Ru + B^\top P B)^{-1}B^\top P A  \\
    &-(A^\top P B)(\Ru + B^\top P B)^{-1} \cdot (B^\top P A)\Differential \\
	&~~~~ - (A^\top P B) \cdot \Differential ((\Ru + B^\top P B)^{-1}) \cdot B^\top P A\\
	&= \Differential (A^\top P A)+ \Differential(A^\top P B) \cdot K + K^\top \cdot \Differential(B^\top P A)\\ 
	&\qquad- (A^\top P B) \cdot \Differential((\Ru + B^\top P B)^{-1})\cdot B^\top P A,
	\end{align*}
	where for compactness, we substituted in the formula
	\begin{align}\label{eq:K_tp}
	K = K(t,P) = - (\Ru + B(t)^\top P B(t))^{-1}B(t)^\top P A(t).
	\end{align}
	
    Recall that for a symmetric matrix, we have
        ($(X^{-1})'=-X^{-1}X'X^{-1}$). Thus, substituting in the definition of $K$, we
        can write the last term in the expression above as
	\begin{align*}
	&-(A^\top P B) \cdot \Differential ((\Ru + B^\top P B))\prm B^\top P A \\
	&= (A^\top P B)  (\Ru + B^\top P B)^{-1}(\Ru + B^\top P B) \Differential(\Ru + B^\top P B)^{-1} B^\top P A\\
	&= K^{\top} \Differential(\Ru + B^\top P B) K.
	\end{align*}
	Hence, gathering terms, we have
	\begin{align*}
	\Differential   \Fdare[\rmd P, \rmd t] \big{|}_{A(t),B(t),P} = \Differential(A^\top P A) - \Differential (P) + \Differential(A^\top P B) K + K^\top\cdot \Differential (B^\top P A) + K^{\top} \Differential(\Ru + B^\top P B) K.
	\end{align*}

	Let us now adopt shorthand $(\cdot)\prm := \frac{d}{dt}(\cdot)$. Expanding the derivatives using the product rule, we then have
	\begin{align*}
	 \Differential  \Fdare[\rmd P, \rmd t] \big{|}_{A(t),B(t),P} &= A^\top \Differential P A  - \Differential(P) + A^\top \Differential  B K + K^\top B^\top \Differential  A + K^{\top} B^\top \cdot \Differential  P \cdot B K\\
	&\qquad+ A\prmtop P A + A^\top PA\prm   + A\prmtop BK + (BK)^\top PA\prm \\
	&\qquad+ A^\top P(B\prm K) + (B\prm K)^\top P A + (B\prm K)^{\top}PBK + (BK)^\top P(B\prm K).
	\end{align*}
        Grouping terms, this is equal to 
	\begin{align*}
	 \Differential  \Fdare[\rmd P, \rmd t] \big{|}_{A(t),B(t),P} &= (A+BK)^\top \rmd P (A+BK)  - \rmd P \big{|}_{A(t),B(t),P}\\
	&\qquad+ A\prmtop P (A+BK) + (A+BK)^\top PA\prm  \big{|}_{A(t),B(t),P} \\
	&\qquad+ (B\prm K)\prmtop P (A+BK) + (A+BK)^\top P (B\prm K)\big{|}_{A(t),B(t),P} \\
	&= (A+BK)^\top \cdot \rmd P \cdot (A+BK) - \rmd P \big{|}_{A(t),B(t),P}\\
	&\qquad+ \underbrace{(A\prm(t) + B\prm(t) K)^\top P (A(t)+B(t)K) + (A(t)+B(t)K) P (A\prm(t) + B\prm(t) K)}_{:=Q_1(t,P),\, \text{ and } K = (t,P) \text{ as in Eq.~\eqref{eq:K_tp}}}\\
	&= \calT_{A(t)+B(t)K}[\rmd P] + Q(t,P)\rmd t.
	\end{align*}
	In particular, if $\Fdare([A(t),B(t)],P) = 0$, then for $K(t,p)$ as in Eq.~\eqref{eq:K_tp}, the matrix $A(t)+B(t)K(t,P)$ is stable. Hence, 
	$\calT_{A(t)+B(t)K(t,P)}[\cdot]$ is invertible on $\Symd$. Moreover, since the second term has no-explicit depending on $\rmd P$, we find that $(\rmd P,\rmd t) \mapsto \Differential  \Fdare[\rmd P, \rmd t] \big{|}_{A(t),B(t),P}$ is full-rank, with zero solution 
	\begin{align*}
	\rmd P = \calT_{A(t)+B(t)K(t,P)}^{-1}[Q_1(t,P)\rmd t] = \dlyap(A(t)+B(t)K(t,P),Q_1(t,P)).
	\end{align*}
	By the implicit function theorem, this implies that there if $\Fdare([A(t),B(t)],P) = 0$, then there exists a neighborhood around $t$ on which the function $u \mapsto P(u)$ is analytic (recall $\Fdare$ is analytic), and $\Fdare([A(u),B(u)],P(u)) = 0$ on this neighborhood.By the above display then, we have $P'(u) = \dlyap(A(u)+B(u)K(u),Q_1(u))$, where $Q_1(u) \leftarrow Q_1(u,P(u))$ and $K(t) \leftarrow K(u,P(u))$ are specializations of the above to the curve $u \mapsto P(u)$.

	\qed

\subsubsection{Computation of $K'$ (Lemma~\ref{lem:computation_k_prime})}

 	Throughout, we suppress dependence on $t$, and the computations are understood to hold only at those $t$ for which $(A(t),B(t))$ is stabilizable.

\begin{proof}
Note that we can take derivatives freely by Lemma~\ref{lem:computation_p_prime}. Invoking the product rule and the identity
($(X^{-1})'=-X^{-1}X'X^{-1}$), 	
	\begin{align*}
	K\prm &= (\Ru + B^\top P B)^{-1}\cdot(\Ru + B^\top P B)\prm \cdot (\Ru + B^\top P B)^{-1}B^\top P A - (\Ru + B^\top P B)^{-1}\cdot (B^\top P A)\prm\\
	&= -(\Ru + B^\top P B)^{-1}(\Ru + B^\top P B)\prm  \cdot K  - (\Ru + B^\top P B)^{-1}(B^\top P A)\prm\\
	&= -(\Ru + B^\top P B)^{-1}\left((\Ru + B^\top P B)\prm K  + (B^\top P A)\prm\right).
	\end{align*}
We simplify the expression inside the parentheses as
	\begin{align*}
	(\Ru + B^\top P B)\prm K  + (B^\top P A)\prm &= B\prmtop P (A+BK) + B^\top P(A\prm + B\prm K) + B^\top P\prm (A+BK)\\
	&= B\prmtop P \Acl + B^\top P(\delAcl) + B^\top P\prm \Acl.
	\end{align*}
Since $B'=\Delta_B$, this yields the result.
	\end{proof}

\subsubsection{Computation of $P\dprm$}
	
	Again, suppress dependence on $t$. We compute $P\dprm$, which Lemma~\ref{lem:computation_p_prime} ensures exists whenever $(A(t),B(t))$ is stabilizable.

  \begin{lem}[Computation of $P\dprm$]\label{lem:computation_p_dprime}
  The second derivative of the optimal cost matrix has the form
  \begin{align*}
  P\dprm &= \dlyap(\Acl,Q_2),
  \end{align*}
  where $Q_2 := \Acl\prmtop P\prm\Acl + \Acl^{\top}P\prm \Acl\prm + Q_1'$ is a symmetric matrix defined in terms of 
  \begin{align*}
   Q_1' := \Acl\prmtop
        P(\delAcl) +\Acl^{\trn}P\prm\delAcl+ \Acl^\top P(B\prm K\prm) +
        (B\prm K\prm)^\top P\Acl + \delAcl^\top P\prm\Acl + \delAcl^\top P\Acl\prm\,.
  \end{align*}
  \end{lem}
\begin{proof}
Applying the product rule to the expression for $P'$ from
\Cref{lem:computation_p_prime}, we have
	\begin{align*}
	P\dprm &= \Acl^\top P\dprm\Acl + \Acl\prmtop P\prm\Acl + \Acl^{\top}P\prm \Acl\prm\\
	&\qquad+\Acl\prmtop P\delAcl + \Acl^{\trn}P\prm\delAcl +
          \Acl^\top P(\delAcl)\prm + (\delAcl)\prmtop P\Acl +
          \delAcl^\top P'\Acl + \delAcl^\top P\Acl\prm\\
	&= \dlyap(\Acl,Q_2), 
	\end{align*}

	where $Q_2 \ldef \Acl\prmtop P\prm\Acl + \Acl^{\top}P\prm \Acl\prm
        +\Acl\prmtop P\delAcl +\Acl^{\trn}P\prm\delAcl+ \Acl^\top P(\delAcl)\prm +
        (\delAcl)\prmtop P\Acl + \delAcl^\top P'\Acl + \delAcl^\top P\Acl\prm$. We conclude
        by observing that $(\delAcl)\prm = (A\dprm + B\dprm K + B\prm
        K\prm) = B\prm K\prm$, since $A$ and $B$ are linear in $t$.
	\end{proof}

\subsection{Norm Bounds for  Derivatives\label{ssec:perturbation_derivative_bounds}}
\subsubsection{Norm bounds for First Derivatives}
\label{app:perturbation_first_derivatives_bounds}

In this section, we work through obtaining concrete bounds on the derivatives of $P(t),K(t)$ using the expressions derived in the previous section. As above, we assume that $\Ru \succeq I$ and $\Rx \succeq I$. We state some more bounds that will be of use to use.

\begin{lem}[Norm-Bounds for Derivative Quantities]\label{lem:helpful_norm_bounds_two} 
Let $(\Ast,\Bst)$ be given, with $\Pst = \Pinf(\Ast,\Bst)$, $\Kst = \Kinf(\Ast,\Bst)$, and $\Aclst = \Ast+\Bst \Kst$. If $\Ru,\Rx \succeq I$, then the  following bounds hold:
\begin{enumerate}
  \item Let $R_0:=\Ru + \Bst^\top \Pst \Bst$. Then for any $X,Y \in \{\Bst,\Pst^{1/2}\Bst,\Ru^{1/2},I\}$, $\opnorm{XR_0^{-1}Y^\top} \le 1$.
  \item For $\circ \in \{\op,\fro\}$, we have $\circnorm{\delAcl} \le 2\opnorm{P}^{1/2}\epscirc$.
\end{enumerate}
\end{lem}
\begin{proof} 
First, we have that $\opnorm{XR_0^{-1}Y^\top} \le \opnorm{XR_0^{-1/2}}\opnorm{YR_0^{-1/2}} \le \sqrt{\opnorm{XR_0^{-1}X^\top}\opnorm{YR_0^{-1}Y^\top}}$. Since $\Ru, P \succeq I$, we can verify that $XX^\top,YY^\top \preceq R_0$, which means that $\opnorm{XR_0^{-1}X^\top},\opnorm{YR_0^{-1}Y^\top} \le 1$.

Second, for $\circnorm{\cdot}$ denoting either the operator or Frobenius norm, we bound $\circnorm{\delAcl} = \circnorm{\delA + \delB K} \le \circnorm{\delA} + \circnorm{\delB}\opnorm{K} = \epscirc(1 + \opnorm{K}) \le 2\sqrt{\opnorm{P}}\epscirc$. 
\end{proof}


\subsubsection{Proof of Lemma~\ref{lem:first_derivatives_bound} and~\Cref{lem:Bkprime} \label{app:perturbation_second_derivatives_bounds}}

Recall that $P\prm = \dlyap(\Acl,Q_1)$, where $Q_1 := \Acl^\top P(\delAcl) + (\delAcl)^\top P\Acl$. Hence, using \Cref{lem:closed_loop_dlyap} with $\Rx \succeq I$, followed by Lemmas~\ref{lem:helpful_norm_bounds} and~\ref{lem:helpful_norm_bounds_two}, we can bound
\begin{align*}
\circnorm{P\prm} &= \circnorm{\dlyap(\Acl,Q_1)} \\
&\le \|P\|_\op\circnorm{Q_1} \le 2\|P\|^2_\op\|\Acl\|_\op\circnorm{\delAcl}\\
&\le 2\|P\|_{\op}^2\cdot \|P\|_{\op}^{1/2} \cdot 2\opnorm{P}^{1/2}\epscirc = 4\opnorm{P}^3.
\end{align*}

Next, recall from \Cref{lem:computation_k_prime} that we have the identity 
\begin{align*}
K\prm = -R_0^{-1} \left(\delB^\top P \Acl + B^\top P(\delAcl) + B^\top P\prm \Acl\right),
\end{align*}
where $R_0 := \Ru + B^\top P B$.
Next bound each of the three terms that arise. Again using $\|R_0^{-1}\|_{\op} \le 1$ and $\opnorm{\Acl} \le \opnorm{P}^{1/2}$ (Lemma~\ref{lem:helpful_norm_bounds}), we have 
\begin{align*}
\circnorm{R_0^{-1} \delB P \Acl } \le \opnorm{P}^{3/2}\epscirc.
\end{align*}
Next, since $\|R_0^{-1}B^\top P^{1/2}\|_{\op} \le 1$ (Lemma~\ref{lem:helpful_norm_bounds_two}), we have
\begin{align*}
\circnorm{(\Ru + B^\top P B)^{-1} \left( B^\top P(\delAcl) + B^\top P\prm \Acl\right)} &\le \opnorm{P}^{1/2}\opnorm{\delAcl} + \opnorm{P^{-1/2}}\opnorm{P'}\opnorm{\Acl}\\
&\le 2\opnorm{P}\epscirc + \opnorm{P^{-1/2}}\opnorm{P'}\opnorm{P}^{1/2} \\
&\le 2\opnorm{P}\epscirc + 4\opnorm{P}^{7/2}\epscirc.
\end{align*}
where the second to last line uses Lemma~\ref{lem:helpful_norm_bounds}, and the last line uses $\opnorm{P^{-1/2}} \le 1$, as well as $\opnorm{P'} \le 4\opnorm{P}^3$. Putting the bounds together, we have $\circnorm{K'} \le 7\opnorm{P}^{7/2}\epscirc$.

\qed.

We also restate and prove an an analogous bound that pre-conditions $K'(t)$ by appropriate matrices. 
\Bkprime*

\begin{proof} The bound is analogous to the bound on $K'$ from Lemma~\ref{lem:first_derivatives_bound}, but now uses right multiplication of $R_0^{-1}$ which adresses left-multiplication by $B,P^{1/2}B,\Ru^{1/2}$.
\end{proof}

\subsubsection{Norm Bounds for Second Derivatives}
\label{app:perturbation_second_derivative}
Next, we turn to bounding $P\dprm$ and $K\dprm$. We shall need some
intermediate lemmas. Let us bound the intermediate term $\Acl\prm$
\begin{lem}
\label{lem:acl_prime_bound} It holds that
 $\max\{\circnorm{\delAcl},\circnorm{\Acl\prm}\} \le9 \|P\|_2^{7/2}\epscirc$, and $\circnorm{\delAcl'} \le \epscirc\epsop \opnorm{P}^{7/2}$.
\end{lem}
\begin{proof} $\Acl\prm = \delAcl + BK\prm$. From \Cref{lem:helpful_norm_bounds_two}, $\circnorm{\delAcl} \le 2\sqrt{\opnorm{P}}\epscirc$. Moreover, from Lemma~\ref{lem:Bkprime}, $ \circnorm{BK\prm} \le 7\opnorm{P}^{7/2}\epscirc$. Thus, $\circnorm{\Acl\prm} \le 9\opnorm{P}^{7/2}\epscirc$. The second bound uses $\delAcl' = \delB K'$, and the same bound on $\circnorm{K'}$. 
\end{proof}
Next, we bound the norm of $P\dprm$.
\begin{lem}  We have the bound $\circnorm{P\dprm} \leq{}  \poly(\opnorm{\Pst})\epsop\epscirc$.
\end{lem}
\begin{proof}
Recall that $P\dprm = \dlyap(\Acl,Q_2)$, where
\begin{align*}
Q_2    &= \Acl\prmtop P\prm\Acl + \Acl^{\top}P\prm \Acl\prm \\
    &~~~~+\Acl\prmtop
      P(\delAcl) +\Acl^{\trn}P\prm\delAcl+ \Acl^\top P(B\prm K\prm) +
      (B\prm K\prm)^\top P\Acl + \delAcl^\top P\prm\Acl + \delAcl^\top P\Acl\prm\,.
\end{align*}
Hence, $\opnorm{P\dprm}\le \opnorm{P}\opnorm{Q_2}$. We upper bound the
norm of $Q_2$ by
\begin{align*}
\circnorm{Q_2} &\le
                 2\left(\circnorm{\Acl\prm}\opnorm{P\prm}\opnorm{\Acl}
                 + \circnorm{\Acl\prm}\opnorm{\delAcl}\opnorm{P} +
                 \circnorm{B\prm} \opnorm{K\prm}\opnorm{P\Acl} 
                 + \opnorm{\Acl}\opnorm{P\prm}\circnorm{\delAcl}
\right).
\end{align*}
Using 
\Cref{lem:helpful_norm_bounds} and \Cref{lem:first_derivatives_bound}, one can show that 
\begin{align*}
\circnorm{P''} \le \poly(\opnorm{\Pst})\epsop\epscirc.
\end{align*}
\end{proof}

\begin{proof}[Proof of Lemma~\ref{lem:computation_k_dprime} ]
From Lemma~\ref{lem:computation_k_prime}, we have that
 \begin{align}
  &K\prm = -(\Ru + B^\top P B)^{-1}\left(\delB^\top P \Acl + B^\top P(\delAcl) + B^\top P\prm \Acl\right).\label{eq:bpa_prime}
  \end{align}
  Denote $Q_3 := \delB^\top P \Acl + B^\top P(\delAcl) + B^\top P\prm \Acl$, and $R_0 := \Ru + B^\top P B$. Then, we have
  \begin{align*}
  K\dprm &= R_0^{-1} Q_3'(t) + R_0^{-1} (\Ru + B^\top P B)'  R_0^{-1} Q_3(t)\\
  &= R_0^{-1} Q_3'(t) + R_0^{-1} (\Ru + B^\top P B)' K'.
  \end{align*}
  Lets first handle the term $R_0^{-1} Q_3'(t)$. From Lemma~\ref{lem:helpful_norm_bounds_two}, we have that $\opnorm{R_0^{-1}} \le 1, \opnorm{R_0^{-1}B}$. Thus,
  \begin{align*}
  \circnorm{R_0^{-1} Q_3'(t)} &\le \opnorm{R_0^{-1}}\circnorm{\delB^\top P \Acl)'} + \opnorm{R_0^{-1}B}\circnorm{(P\prm \Acl)'}\\
  &\le \circnorm{(\delB^\top P \Acl)'} + \circnorm{(P\prm \Acl)'}\\
   &\le \circnorm{\delB}\opnorm{P}{\Acl'} \opnorm{P'}\opnorm{\Acl} + \circnorm{P''}\opnorm{\Acl} + \circnorm{P'}\opnorm{\Acl'}\\
   &\le \poly(\opnorm{\Pst})\epsop\epscirc,
  \end{align*}
  where we invoke the derivative computations above. Similarly, we can show that
  \begin{align*}
  \circnorm{R_0^{-1} (\Ru + B^\top P B)' K'}  &\le \circnorm{ R_0^{-1} \delB P B K' } +  \circnorm{ R_0^{-1} B^\top P' B K' } + \circnorm{R_0^{-1}B P \delB K'}\\
  &\le (\opnorm{R_0^{-1}}\opnorm{P}\epsop + \opnorm{R_0^{-1}B^\top}\opnorm{P'}) \circnorm{BK'} +  \opnorm{R_0^{-1}B }\opnorm{P}\epscirc\opnorm{K'}\\
  &\le (\opnorm{P}\epsop + \opnorm{P'}) \circnorm{BK'} +  \opnorm{P}\epscirc\opnorm{K'} \le \poly(\opnorm{P})\epscirc\epsop.
  \end{align*}
 
\end{proof}


%% file: appendix/app_self_bounding.tex

We begin by stating \Cref{thm:general_valid_implicit}, which provides a generic guarantee for self-bounding ODES (Definition~\ref{defn:self-bounding}).
\begin{thm}\label{thm:general_valid_implicit} Let $(F,\calU,g,\|\cdot\|,x(\cdot))$ be a self-bounding tuple. Suppose that for some $\eta > 0$, $h(\cdot)$ satisfies $h(z) \ge g(z) + \eta$ for all $z \ge \|y(0)\|$, and that  the scalar ODE
\begin{align*}
w(0) = \|y(0)\| + \eta ,\quad w'(t) = h(w(t) 
\end{align*}
has a continuously differentiable solution on $[0,1]$. Then, there exists a unique continuously differentiable function $y(t) \in \calU$ defined on  $[0,1]$ which satisfies $F(x(t),y(t)) = 0$, and this solution satisfies $\|y(t)\| \le w(t) \le w(1)$, $\|y'(t)\| \le g(w(t)) \le g(w(1))$ for all $t \in [0,1]$.
\end{thm}
We shall prove the above theorem, and then derive \Cref{cor:poly_self_bound} as a consequence. We begin the proof of this theorem with a simple scalar comparison inequality.
\begin{lem}[Scalar Comparison Inequalities for Curves]\label{lem:comparison_stuff} Suppose that $x(t),w(t)$ are continuously differentiable curves defined on $[0,u)$. Suppose further that, for a function $f(\cdot,\cdot)$, $x\prm(t) = f(x(t),t)$, and that $w\prm(t) = g(x(t))$. In addition, suppose
\begin{enumerate}
  \item $w(0) > x(0)$
\item $g(\cdot) \ge 0$ 
\item For $t \in [0,u)$ such that $x(t) \ge w(0)$,  $g(x(t)) > f(x(t),t)$.
\end{enumerate} 
Then, $x(t) < w(t)$ for $t \in [0,u)$.
\end{lem}
\begin{proof} Define $\delta(t) = w(t) - x(t)$. Since $\delta(0) > 0 $, there exists an $s > 0$ such that $\delta(t)$ for $t \in [0,s)$.  Choose the maximal such $s := \sup\{t:\delta(t') \ge 0,\, \forall t' < t\}$, and suppose for the sake of contradiction that $s < u$. Then, by continuity, $\delta(s) = 0$, and therefore $\delta\prm(s) = g(w(s)) - f(x(s),s) = g(x(s)) - f(x(s),s)$, since $x(s)  = w(u)$ for $\delta(s) = 0$.

Next, note that since $g(\cdot) \ge 0$, $w(t)$ is non-decreasing on $[0,u)$, and thus $w(s) \ge w(0) $ for all $s \in [0,u)$. Since $x(s) = w(s)$ at $s$, we have $x(s) \ge w(0)$ as well. Thus, $\delta\prm(s) = g(x(s)) - f(x(s),s) > 0$, by the assumption of the lemma. Hence, for an $\epsilon > 0$ sufficiently small, $\delta(s - \epsilon) < \delta(s) = 0$. This contradicts the fact that of $\delta(t') = 0$ for all $t' < s$. 
\end{proof}

Next, we extend the above scalar comparison inequality to a comparison inequality between scalar ODEs, and vector ODEs.
\begin{lem}[Norm Comparison for Vector ODE]\label{lem:norm_comparison}  Let $\|\cdot\|$ denote an arbitrary norm. Suppose that $v(t) \in \R^d$ is a continuously differentiable curve defined on $[0,u)$ such that $\|v'(t)\| \le g(\|v(t)\|)$ for a non-decreasing function $g$. Fix $\eta > 0$, and let $h(z)$ denote a function such that $h(z) \ge \max\{0,g(z) + \eta\}$ for all $z \ge \|v(0)\| $. Then, if the ODE 
\begin{align*}
w(0) = \|v(0)\| + \eta, \quad w'(t)= h(w(t))
\end{align*}
has a continuously differentiable solution defined on $[0,u)$, then $\|v(t)\| \le w(t)$ for all $t \in [0,u)$
\end{lem}
\begin{proof}[Proof of \Cref{lem:norm_comparison}] 
The main challenge is that $\|\cdot\|$ may be non-smooth. We circumvent this with a Gaussian approximation. Let $c_Z := \Exp_{Z \sim \calN(0,I)}[\|Z\|]$, and for every $\eta > 0$, and define $\Psi_{\eta}(v) := \Exp_{Z \sim \calN(0,I)}[\|v + \frac{\eta}{2c_Z} Z\|]$. Defining $c_{Z}:= \Exp_{Z \sim \calN(0,1)}[\|Z\|]$. Moreover, we can see that $0 \le \|v\| \le \Psi_{\eta}(v) \le \|v\| + \eta/2 < \|v\| + \eta$ by Jensen's inequality and the triangle inequality. Consider the curve $x(t)$
\begin{align*}
x(t) = \Psi_{\eta}(v(t)), t \in [0,u),
\end{align*}
Note then that the curve satisfies
\begin{align*}
x(0) = \Psi_{\eta}(v(0)), \quad \text{and} \quad x\prm(t)) =  f(t,x(t))  = \frac{\rmd}{\rmd t}(\Psi_{\eta}(v(t))),
\end{align*}
where $f(t,x(t))$ does not depend implicitly on $x(t)$, but only on $t$ through the function $t \mapsto v(t)$. 

Now, let $g$ be a monotone function satisfying $\|v'(t)\| \le g(\|v(t)\|)$, and let $h$ be the assumed function satifying $h(z) \ge g(z) + \eta$ for all $z \ge \|v(0)\| + \eta$. We define the associated ODE
\begin{align*}
w(0) = \|v(0)\| + \eta, \quad w\prm(t) = h(w(t)),
\end{align*}
which we assume is also defined on $[0,u)$. We would like to show that $w(t) > x(t)$ for $t \in [0,u)$. To this end, we would like to verify the conditions of \Cref{lem:comparison_stuff}. First, we have $w(0) = \|v(0)\| + \eta > \Psi_{\eta}(v(0)) = x(0)$, by above application of the triangle inequality. 

For the second condition, we have
\begin{align*}
f(t,x(t)) &:= \frac{\rmd}{\rmd t}(\Psi_{\eta}(v(t))) = \frac{\rmd}{\rmd t}\Exp_{Z \sim \calN(0,I)}\left[\left\|v + \frac{\eta}{2c_Z} Z\right\|\right] \\
&\le  \Exp_{Z \sim \calN(0,1)}\left[\left\|v\prm(t) +  \frac{\eta}{2c_Z} Z\right\|\right] \\
&=  \Psi_{\eta}(v'(t)) < \|v'(t)\| + \eta\\
&\le  g(\|v(t)\|) + \eta \quad \text{ (since $g$ satisfies $\|v'(t)\| \le g(\|v(t)\|)$ )}\\
&\le  g(\Psi_{\eta}(v(t))  + \eta \quad \text{ (since $g$ is monotone)}\\
&=  g(x(t)) + \eta.
\end{align*}
Now, if $h(z) \ge g(z) + \eta$ for any $z\ge w(0)$, then, we see that, for any $t \in [0,u)$ such that $x(t) \ge w(0)$, we have $f(t,x(t))  \le h(x(t))$. \Cref{lem:comparison_stuff} therefore implies that $x(t) \le w(t)$ for $t \in [0,u)$. But $x(t) = \Psi_{\eta}(v(t)) \ge \|v(t)\|$. 
\end{proof}

Let us now prove the general guarantee for self-bounding functions.
\begin{proof}[Proof of Theorem~\ref{thm:general_valid_implicit}]
Observe that by the valid-function assumption and the assumption that $F(x(0),y(0))$ has a solution, there exists some interval $[0,u)$ on which a solution $y(t)$ to $F(x(t),y(t)) = 0$ exists. Let $u$ denote the maximal  value of $u \le 2$ for which this holds. 

First, let us bound $\|y(t)\|$ for $t \in \calI := [0,u) \cap [0,1]$. By assumption, there is a function $h(z) \ge g(z) + \eta$, where $g(z)$ is non-negative and non-decreasing, such that the scalar ODE $w'(t) = h(w(t))$ has a solution on $[0,1]$ with $w(0) = \|y(0)\| + \eta$. By Lemma~\ref{lem:norm_comparison}, we then that $\|y(t)\| \le w(t)$ on $\calI$. Moreover, since $w'(t) \ge 0$ since $h$ is non-negative, we have $\|y(t)\| \le w(t) \le w(1)$ on $\calI$.

We conclude by showing that $\calI = [0,1]$.  Suppose for the sake of contradiction that $\calI \ne [0,1]$. Then $u \in (0,1]$. Moreover, by Definition~\ref{defn:valid_implicit}, $F(x(u),\cdot) = 0$ has no solution, since otherwise, $y(t)$ would be defined on $[0,u+\epsilon)$ for some $\epsilon > 0$, contradicting the maximality of $u$. Therefore, to contradict our hypothesis $\calI \ne [0,1]$, it suffices to show that $F(x(u),\cdot) = 0$ has a solution. To this end, define 
\begin{align*}
\widetilde{y}(s) := \int_{0}^s y\prm(t)dt,
\end{align*}
which is well defined and continuous for $s \in [0,u)$, since $y\prm(s)$ is continuously differentiable on this interval. Moreover,  $\|y\prm(t)\| \le g(y(t)) \le g(w(t)) \le g(w(1))$ on $[0,u)$ since $y(t) \le w(t) \le w(1)$. Therefore, $y'(t)$ is uniformly bound on $[0,u)$, so that $\widetilde{y}(u) = \lim_{s \to u}\widetilde{y}(s)$ is well-defined at $u$, and in fact continuous on $[0,u]$. 

Since $\widetilde{y}(s)$ is continuous on $[0,u]$, and since $F(\cdot,\cdot)$ and $x(s)$ are continuous, $\lim_{s \to u} F(x(s),\widetilde{y}(s)) = F(x(u),\widetilde{y}(u))$. But by the fundamental theorem of Calculus, we see that $\widetilde{y}(s) = y(s)$ for $s \in [0,u)$, so that $F(x(u),\widetilde{y}(s)) = F(A(s),B(s),y(s)) = 0$ for $s \in [0,u)$. Thus, $\lim_{s \to u} F(A(s),B(s),\widetilde{y}(s)) = 0$, and hence $F(x(u),\widetilde{y}(u)) = 0$. This shows that $F(x(u),\cdot) = 0$ has a solution, as needed.
\end{proof}

We now prove the corollary for the specific function form $g(z) = cz^p$.
\begin{proof}[Proof of Corollary~\ref{cor:poly_self_bound}]

Fix $\eta > 0$ to be selected later. By assumption, we have
\begin{align*}
\|y\prm(t)\| \le g(\|y(t)\|), \quad g(z) = cz^{p}.
\end{align*}
Moreover, for an $\eta > 0$ to be selected, and for $z \ge \|y(0)\|$, we have 
\begin{align*}
g(z) + \eta \le \underbrace{(1 + \frac{\eta}{c\|y(0)\|^{p}}}_{:= c_{\eta}})z^p := h(z).
\end{align*}
Now, consider the ODE
\begin{align*}
w_{\eta}\prm(t) = h(w_{\eta}(t)), \quad w_{\eta}(0) = \opnorm{y(0)} + \eta.
\end{align*}
Let us show that, for $\eta$ sufficiently small, this ODE exists on $[0,1]$. Indeed, the solution to this the ODE is
\begin{align*}
\frac{1}{(p-1)w_{\eta}^{p-1}(0)} - \frac{1}{(p-1)w_{\eta}^{p-1}(t)} = c_{\eta}t.
\end{align*}
So that a continuously differentiable solution $w_{\eta}(t)$ exists for $t \in [0,1]$ as long as 
\begin{align}
c_{\eta} < \frac{1}{(p-1)w_{\eta}^{p-1}(0)}=\frac{1}{(p-1)(\|y(0)\| + \eta)^{p-1}},\label{eq:soln_defined}
\end{align}
and the solution is given by 
\begin{align*}
w_{\eta}(t) = \left(\frac{1}{(\|y(0)\| + \eta)^{p-1}} - (p-1)c_{\eta}t\right)^{-1/(p-1)}.
\end{align*}
In particular, if $c < \frac{1}{(p-1)(\|y(0)\|)^{p-1}}$, then since $\lim_{\eta \to 0} c_{\eta} = c$, there exists an $\eta_0 > 0$ sufficiently small so that the condition in~\eqref{eq:soln_defined} the above display holds for all $\eta \in (0,\eta_0)$. Therefore, by Theorem~\ref{thm:general_valid_implicit},
\begin{align*}
\max_{t \in [0,1]}\|y(t)\| \le \inf_{\eta \in (0,\eta_0)} w_{\eta}(t) = \left(\frac{1}{\|y(0)\|^{p-1}} - ct\right)^{-(p-1)} \le \left(\frac{1}{\|y(0)\|^{p-1}} - c(p-1)\right)^{-1/(p-1)}.
\end{align*}
In particular, when $\alpha = c(p-1)\|y(0)\|_{}^{p-1} < 1$,  then 
\begin{align*}
\max_{t \in [0,1]}\|y(t)\| \le (1-\alpha)^{-1/(p-1)}\|y(0)\|.
\end{align*}
Hence, for all $t \in [0,1]$, we have that $\nrm*{y(t)} \le  c(1-\alpha)^{-p/(p-1)}\|y(0)\|$.
\end{proof}


%% file: appendix/ols_appendix.tex

\subsection{Ordinary Least Squares Tools}
\label{app:ols}
In what follows, we develop a general toolkit for analyzing the performance of ordinary least squares. 
\begin{restatable}[Martingale Least Square Setup]{defn}{defmartls}\label{def:setup_mart_ls} First, let $\{\matz_t\}_{t \ge 1} \in (\R^{d})^\N$ and $\{\maty_t\}_{t \ge 1} \in (\R^{m})^\N$ denote sequences of random vectors adapted to a filtration $\{\calF_{t}\}_{t \ge 0}$. We define the empirical covariance matrix $\matLam := \sum_{t=1}^T \matz_t\matz_t^\top$, we define the least-squares estimator 
\begin{align*}
\Thetahat_T := \left(\sum_{t=1}^T\maty_t\matz_t\right)\left(\sum_{t=1}^T\matz_t \matz_t^\top\right)^{\dagger}
\end{align*}
Lastly, we assume that the sequence $\mate_t := \maty_t - \Thetast \matz_t \in \R^m$ is $\sigma^2$-sub-Gaussian conditioned on $\calF_{t-1}$.
\end{restatable}

 We begin with a standard self-normalized tail bound (cf. \cite{abbasi2011improved}).
\begin{lem}[Self-Normalized Tail Bound]\label{lem:self_normalized} Suppose that $\{\mate_t\}_{t \ge 1} \in \R^{\N}$ is a scalar $\calF_{t}$-adapted sequence such that $\mate_t \mid \calF_{t-1}$ is $\sigma^2$ sub-Gaussian. Fix a matrix $V_0 \succeq 0$. Then with probability $1-\delta$,
\begin{align*}
\left\|\sum_{t=1}^T\matx_t\mate_t \right\|_{(V_0+\matLam)^{-1}}^2 \le 2\sigma^2\log \left\{\frac{1}{\delta}\det(V_0^{-1/2}(V_0 + \matLam)V_0^{-1/2})\right\}.
\end{align*}
\end{lem}
As a corollary, we have the following Frobenius norm bound for regression, proved in \Cref{sec:proof_lem_frob_Ls}. 
\begin{lem}[Frobenius Norm Least Squares, Coarse Bound]\label{lem:frob_ls} In the martingale least squares setting of \Cref{def:setup_mart_ls}, we have
\begin{align*}
\Pr\left[ \left\{\|\Thetahat_T - \Thetast \|_\fro^2 \ge 3m \lambda_{\min}(\matLam)^{-1} \log \left\{\tfrac{m\det(3\Lambda_0^{-1/2}(\matLam)\Lambda_0^{-1/2}}{\delta}\right\}\right\} \cap \{\matLam \succeq \Lambda_0\}\right] \le \delta,
\end{align*}
and 
\begin{align}
\Pr\left[ \left\{\|\Thetahat_T - \Thetast \|_{\op}^2 \ge 6 \lambda_{\min}(\matLam)^{-1}  (d \log 5 + \log \left\{\tfrac{\det(3\Lambda_0^{-1/2}(\matLam)\Lambda_0^{-1/2}}{\delta}\right\}\right\} \cap \{\matLam \succeq \Lambda_0\}\right] \le \delta. \label{eq:op_norm_bound}
\end{align}
\end{lem}

Unfortunately, this tail bound will lead to a dimension dependence of $\Omega(d)$, which may be suboptimal if $\matLam$ has eigenvalues of varying magnitude. Instead, we opt for a related bound that pays for Rayleigh quotients between $\Lambda_0$ and $\matLam$, proved in \Cref{sec:proof_lem_frob_Ls}.
\begin{restatable}[Two-Scale OLS Estimate]{lem}{lemtwoscale}\label{lem:two_scale_self_normalized_bound} Consider martingale least squares setting of \Cref{def:setup_mart_ls}.  Let $\ProjMat \in \R^{d\times d}$ denote a projection matrix onto a subspace of dimension $p$, and fix an orthonormal basis of $\R^d$, $v_{1},\dots,v_d$, such that $v_{1},\dots,v_p$ form an orthonormal basis for the range of $\ProjMat$. Further, fix positive constants $0 < \lambda_1 \le \nu \le \lambda_2$ such that $\nu \le \sqrt{\lambda_1 \lambda_2/2}$, and define the event 
\begin{align*}
\calE := \left\{\matLam \succeq  \lambda_1 \ProjMat +  \lambda_2 (I-\ProjMat) \right\} \cap \left\{\|\ProjMat \matLam (1-\ProjMat)\|_{\op} \le \nu\right\}
\end{align*} 
Then, with probability $1 - \delta$, if $\calE$ holds, 
\begin{align*}
\|\Thetahat_T - \Thetast \|_\fro^2 \le  \frac{12 mp \matkappa_1 }{\lambda_1}  \log \frac{3 md \matkappa_1}{ \delta} + \left(\frac{\nu}{\lambda_1}\right)^2 \cdot \frac{48m(d-p)\matkappa_2 }{\lambda_2}\log \frac{3 md\matkappa_2}{ \delta}. 
\end{align*}
where $\matkappa_1:= \max_{1 \le j \le p} v_j^\top \matLam v_j / \lambda_1$ and $\matkappa_2 := \max_{p+1 \le j \le d} v_j^\top \matLam v_j / \lambda_2$. 
\end{restatable}
In the best case, the cross term bound $\nu$ is taken to be equal to $\lambda_1$, in which case the bound reads 
\begin{align*}
\|\Thetahat_T - \Thetast \|_\fro^2 \le  \BigOhTil{\frac{ mp \matkappa_1 }{\lambda_1} + \frac{m(d-p)\matkappa_2 }{\lambda_2}} = \BigOhTil{m \tr \left(\lambda_1 \ProjMat +  \lambda_2 (I-\ProjMat)\right)^{-1} },
\end{align*}
which is what would typically expect from linear regression which independent samples. The main obstable in proving \Cref{lem:two_scale_self_normalized_bound} is that two PSD matrices $A \preceq B$ do not necessarily satisfy $A^2 \preceq B^2$; hence, we require a more careful argument which incurs dependence on the term $\|\ProjMat \matLam (1-\ProjMat)\|_{\op} $, which is upper bounded by the parameter $\nu$ in the lemma statement.

\begin{lem}[Covariance Lower Bound] \label{lem:covariance_lb} Suppose that $\matz_t\mid \calF_{t-1} \sim \mathcal{N}(\matzbar_t, \Sigma_t )$, where $\matzbar$ and $\Sigma_t \in \R^d$ are $\calF_{t-1}$-measurable and $\Sigma_t \succeq \Sigma \succ 0$. Let $\calE$ be any event for which $\matLambar_T := \Exp[\matLam\I(\calE) ]$ satisfies $\tr(\matLambar_T) \le T J$ for some $J \ge 0$. Then, for 
\begin{align*}
T \ge \frac{2000}{9}\left(2d\log \tfrac{100}{3} + d \log \frac{J}{\lambda_{\min}(\Sigma)}\right),
\end{align*}
it holds that, for $\Lambda_0 := \frac{9T}{1600}\Sigma $
\begin{align*}
\Pr\left[\left\{\matLam \not\succeq \frac{9T}{1600}\Sigma \right\}\cap \calE\right] \le  2\exp\left( - \tfrac{9}{2000(d+1)}T \right).
\end{align*}
\end{lem}


\subsection{Basic Concentration Bounds}

Here we state some useful concentration bounds for Gaussian distributions.
\begin{lem}[Proposition 1.1 in \cite{hsu2012tail}] Let $\matg \sim \calN(0,I_d)$ be an isotropic Gaussian vector, and let $A$ be a symmetric matrix. Then,
\begin{align*}
\Pr\left[ \left|\matg^\top A \matg - \tr(A)\right| > 2t^{1/2}\|A\|_{\fro} + 2t\|A\|_{\op}\right] \le 2 e^{-t}. 
\end{align*}
\end{lem}
By replacing the Frobenius and operator norms in the above inequality with the Hilbert-Schmidt norm, we obtain the following corollary.
\begin{cor}\label{cor:crude_HS} Let $A \succeq 0$, and let $\matg \sim \calN(0,I_d)$. Then, with probability $1-\delta$ for any $\delta < 1/e$,
\begin{align*}
\matg^\top A \matg  \lesssim \tr(A)\log \frac{1}{\delta} = \Exp[\matg^{\top}A \matg_t]\log \frac{1}{\delta}.
\end{align*}
\end{cor}

\subsection{Proofs from \Cref{app:ols}}
\subsubsection{Proof of Lemma~\ref{lem:frob_ls} \label{sec:proof_lem_frob_Ls}}
	We assume without loss of generality that $\sigma^2 = 1$. Let $\mate_t = \maty_t - \Thetast\matz_t$. Let $\matX \in \R^{Td}$ denote the matrix whose rows are $\mate_t$, and $\matE^{(i)} \in \R^T$ denote the vector $(\mate_{1,i},\dots,\mate_{T,i})$, where $\mate_{t,i}$ is the $i$-th coordinate of $\mate_t$. Let $\lambda_{\min} := \lambda_{\min}(\Lambda_0)$. Then,
	\begin{align*}
	\|\Thetahat_T - \Thetast \|_\fro^2 &= \sum_{i=1}^m \left\|\matLam^{-1} \matX^\top \matE^{(i)} \right\|_{2}^2 \\
	&\le \lambda_{\min}(\matLam)^{-1}\sum_{i=1}^m \left\|\matLam^{-1/2} \matX^\top \matE^{(i)} \right\|_{2}^2 \\
	&= \sum_{i=1}^m \lambda_{\min}(\matLam)^{-1}\left\|\matX^\top\matE^{(i)} \right\|_{\matLam^{-1} }^2 \\
	&\le \frac{3}{2}\sum_{i=1}^m \lambda_{\min}(\matLam)^{-1}\left\|\matX^\top\matE^{(i)} \right\|_{(\matLam + \frac{1}{2}\Lambda_0)^{-1} }^2,
	\end{align*}
	where the last line holds for $\Lambda_0 \preceq \matLam$. Invoking \Cref{lem:self_normalized}, we have that with probability at least $1-\delta$, it holds for any fixed $i \in [m]$ that
	\begin{align*}
	\left\|\matX^\top\matE^{(i)} \right\|_{(\matLam + \Lambda_0)^{-1} }^2 \le 2 \log  \left\{\frac{1}{\delta}\det((\frac{\Lambda_0}{2})^{-1/2}(\frac{\Lambda_0}{2} + \matLam)(\frac{\Lambda_0}{2})^{-1/2})\right\}.
	\end{align*}
	Since $\Lambda_0 \preceq \matLam$, we have $\frac{\Lambda_0}{2} + \matLam \le \frac{3}{2}\matLam$, when the above can be bounded by
	\begin{align*}
	\left\|\matX^\top\matE^{(i)} \right\|_{(\matLam + \Lambda_0)^{-1} }^2 \le 2 \log  \left\{\frac{1}{\delta}\det(3(\Lambda_0)^{-1/2}\matLam\Lambda_0^{-1/2})\right\}.
	\end{align*}
	Union bounding over $i \in [m]$ and summing the bound concludes.

\subsection{ Proof of \Cref{lem:two_scale_self_normalized_bound} \label{sec:lem_two_scale}}
	We begin with a linear algebraic lemma lower bounding the square of a PSD matrix in the Lowner order.

		\begin{lem}\label{lem:matrix_sq_lb}
			Let $X = \begin{bmatrix} X_{11} & X_{12} \\
			X_{12}^\top & X_{22}\end{bmatrix} \succ 0$. Then, for any parameter $\alpha > 0$,
			\begin{align*}
			X^2  &\succeq \begin{bmatrix} ( 1- \alpha_1) X_{11}^2  + (1 - \alpha_2^{-1}) X_{12}X_{12}^\top & 0\\
			0 & (1 - \alpha_2) X_{22}^2 +(1 - \alpha_1^{-1}) X_{12}^\top X_{12} \end{bmatrix}
			\end{align*}
		\end{lem}
		\begin{proof}[Proof of \Cref{lem:matrix_sq_lb}]
			
			We begin by expanding
			\begin{align*}
			\begin{bmatrix} X_{11} & X_{12} \\
			X_{12}^\top & X_{22}\end{bmatrix}^2 &= \begin{bmatrix} X_{11}^2 + X_{12}X_{12}^\top & X_{11}X_{12} + X_{12}X_{22} \\
			X_{22} X_{12}^\top + X_{12}^\top X_{11}& X_{22}^2 + X_{12}^\top X_{12}\end{bmatrix}\\
			& = \begin{bmatrix} X_{11}^2 + X_{12}X_{12}^\top & 0\\
			0 & X_{22}^2 + X_{12}^\top X_{12}\end{bmatrix} 
			+ \begin{bmatrix} 0 & X_{11} X_{12} \\
			X_{12}^\top X_{11} & 0 
			\end{bmatrix} + \begin{bmatrix} 0 &  X_{12} X_{22} \\
			 X_{22}X_{12}^\top & 0 
			\end{bmatrix} .
			\end{align*}
			Now, for any vector $v = (v_1,v_2)$, and any $\alpha_1 > 0$, we have 
			\begin{align*}
			\left\langle v, \begin{bmatrix} 0 & X_{11} X_{12} \\
			X_{12}^\top X_{11}
			\end{bmatrix} v \right \rangle  &= 2v_1^\top X_{11} X_{12} v_2 \\
			&\ge -  2\|v_1^\top X_{11}\| \|X_{12} v_2\| \\
			&= -  2 \cdot \alpha_1^{1/2}\|v_1^\top X_{11}\| \cdot \alpha_1^{-1/2}\|X_{12} v_2\|\\ 
			&\ge -\alpha_1 \|v_1^\top X_{11}\|^2 - \alpha_1^{-1} \|X_{12} v_2\|^2\\
			&= \left\langle v, \begin{bmatrix} - \alpha_1 X_{11}^2 & 0\\
			0 & - \alpha_1^{-1} X_{12}^\top X_{12}
			\end{bmatrix} v \right\rangle.
			\end{align*}
			Similarly, for any $v$ and $\alpha_2 > 0$, we have
			\begin{align*}
			\left\langle v, \begin{bmatrix} 0 & X_{12} X_{22} \\
			X_{22} X_{12}^\top
			\end{bmatrix} v \right \rangle 
			&\ge \left\langle v, \begin{bmatrix} - \alpha_2^{-1}  X_{12} X_{12}^\top & 0\\
			0 & - \alpha_2 X_{22}^2
			\end{bmatrix} v \right\rangle.
			\end{align*}
			Thus, for any $\alpha_1,\alpha_2 > 0$, 
			\begin{align*}
			X^2  &\succeq \begin{bmatrix} ( 1- \alpha_1) X_{11}^2  + (1 - \alpha_2^{-1}) X_{12}X_{12}^\top & 0\\
			0 & (1 - \alpha_2) X_{22}^2 +(1 - \alpha_1^{-1}) X_{12}^\top X_{12} \end{bmatrix}
			\end{align*}
		\end{proof}

	The important case of \Cref{lem:matrix_sq_lb} is when the matrix $X$ in question can be lower bounded in terms of the weighted sum of two complementary projection matrices.
		\begin{lem}\label{lem:proj_matrix_cond} 
			Let $\ProjMat$ be an orthogonal projection matrix, let $X \succ 0$, and suppose that there exist positive constants $\lambda_1 \le \nu \le \lambda_2$ be such that   $X \succeq \lambda_1 (I-\ProjMat) + \lambda_2 \ProjMat$, and $\|\ProjMat X (1-\ProjMat)\|_{\op} \le \nu$. Then, if $\nu \le \sqrt{\lambda_1 \lambda_2}/2$, 
			\begin{align*}
			X^2 \succeq \frac{1}{4} \lambda_{1}^2 \ProjMat + \frac{\lambda_{1}^2 \lambda_2^2}{16\nu^2}  (I - \ProjMat) 
			\end{align*}
		\end{lem}
		\begin{proof}[Proof of \Cref{lem:proj_matrix_cond}]

			Denote the number of rows/columns of $X$ by $d$. By an orthonormal change of basis, we may assume that $\ProjMat$ is the projection onto the first $k = \dim(\range(\ProjMat))$ cannonical basis vectors. Writing $X$ and $\ProjMat$ in this basis we have
			\begin{align}
			X = \begin{bmatrix} X_{11} & X_{12} \\
			X_{12}^\top & X_{22}\end{bmatrix} \succeq \lambda_1 (I-\ProjMat) + \lambda_2 \ProjMat =  \begin{bmatrix} \lambda_1 I_{d -p} & 0 \\
			0 &  \lambda_2 I_k \end{bmatrix}. \label{eq:X_in_basis}
			\end{align}
			It suffices to show that, in this basis
			\begin{align}
			X^2  &\succeq \begin{bmatrix} \frac{\lambda_1^2}{4} I_p & 0\\
			0 & \frac{\lambda_2^2\lambda_1^2}{16\nu^2}  I_{d-p} \end{bmatrix}. \label{eq:X_Squared_tws}
			\end{align}
			From \Cref{eq:X_in_basis}, $X_{11} \succeq \lambda_1 I_{p} $, $X_{22} \succeq \lambda_2 I_{p-d}$, and $\|X_{12}\|_{\op} = \|\ProjMat X (1-\ProjMat)\|_{\op}$. Hence, hence the parameters $\nu \ge \lambda_1$ in the lemma satisfies $\nu \ge \|X_{12}\|_{\op}$. Thus, 
			\begin{align*}
			X^2  &\succeq \begin{bmatrix} \left\{( 1- \alpha_1)  + (1 - \alpha_2^{-1}) (\nu/\lambda_1)^2 \right\} \cdot \lambda_1^2 I_p & 0\\
			0 & \left\{(1 - \alpha_2)  +(1 - \alpha_1^{-1}) (\nu/\lambda_2)^2\right\} \cdot \lambda_2^2 I_{d-p} \end{bmatrix}.
			\end{align*}
			Set $\alpha_1 = \frac{1}{2}$, and take  $\alpha_2$ to satisfy $(1 - \alpha_2^{-1})(\nu/\lambda_1)^2 = -1/4$. Then, we have the following string of implicitations
			\begin{align*}
			1 - \alpha_2^{-1} &= -\lambda_1^2/4\nu^2 \quad \implies \quad
			\alpha_2^{-1} = \frac{4 \nu^2 + \lambda_1^2}{\nu^2} \quad \implies\\
			\alpha_2 &= \frac{4\nu^2}{4\nu^2 + \lambda_1^2} \quad \implies \quad
			1 - \alpha_2 = \frac{\lambda_1^2}{4 \nu^2 + \lambda_1^2} \ge \frac{\lambda_1^2}{8 \nu^2},
			\end{align*} 
			where in the last line we use $\nu \ge \lambda_1$. For this choice, and using $\nu \le \sqrt{\lambda_1 \lambda_2}/2$,
			\begin{align*}
			X^2  &\succeq \begin{bmatrix} \frac{\lambda_1^2}{4} I_p & 0\\
			0 & \{\frac{\lambda_1^2}{8 \nu^2}  - \frac{\nu^2}{\lambda_2^2}\} \cdot \lambda_2^2 I_{d-p} \end{bmatrix} \succeq \begin{bmatrix} \frac{\lambda_1^2}{4} I_p & 0\\
			0 & \frac{\lambda_2^2\lambda_1^2}{16\nu^2}  I_{d-p} \end{bmatrix}, 
			\end{align*}
			as needed.
		\end{proof}

	We are now in a position to prove our desired lemma.
		\begin{proof}[Concluding the proof of \Cref{lem:two_scale_self_normalized_bound}]
			
			Let's establish the following notation:
			\begin{itemize}
				\item We assume without loss of generality that $\sigma^2 = 1$. 
				\item Let $\mate_t = \maty_t - \Thetast\matz_t$. Let $\matE \in \R^{Td}$ denote the matrix whose rows are $\mate_t$, and $\matE^{(i)} \in \R^T$ denote the vector $(\mate_{1,i},\dots,\mate_{T,i})$, where $\mate_{t,i}$ is the $i$-th coordinate of $\mate_t$.
				\item Recall that $v_{1,\dots,p}$ denote a basis for the range of the projection operator $\ProjMat$, and $v_{p+1,\dots,d}$ complete this basis to form an orthonormal basis for $\R^d$.
			\end{itemize}

			 By the assumptions that $\matLam \succeq \lambda_1 (I-\ProjMat) + \lambda_2 \ProjMat$, and $\nu \ge \lambda_1 \vee \|\ProjMat \matLam (1-\ProjMat)\|$, and $\nu \le \sqrt{\lambda_1 \lambda_2}/2$, \Cref{lem:proj_matrix_cond} implies that 
			\begin{align*}
			\matLam^2  \succeq \frac{1}{4} \lambda_{1}^2 \ProjMat + \frac{\lambda_{1}^2 \lambda_2^2}{16\nu^2}  (I - \ProjMat) 
			\end{align*} 
			so after inversion, 
			\begin{align*}
			(\matX_T^\top \matX_T)^{-2} \preceq 4\lambda_{-1}^2 \ProjMat + \frac{16\nu^2}{\lambda_{1}^2 \lambda_2^2}  (I - \ProjMat).
			\end{align*}
			Hence, we can render
			\begin{align*}
			\|\Thetahat_T - \Thetast \|_\fro^2 &= \sum_{i=1}^m \left\|(\matX^\top \matX)^{-1} \matX^\top \matE^{(i)} \right\|_{2}^2 \\
			&= \sum_{i=1}^m \left \langle \matX^\top \matE^{(i)} , (\matX^\top \matX)^{-2} \matX^\top \matE^{(i)} \right \rangle \\
			&\le \sum_{i=1}^m \left \langle \matX^\top \matE^{(i)} , \left(4\lambda_{1}^{-2} \ProjMat + \frac{16\nu^2}{\lambda_{1}^2 \lambda_2^2}  (I - \ProjMat)\right) \matX^\top \matE^{(i)} \right \rangle \\
			&= \sum_{i=1}^m 4\lambda_{1}^{-2}\|\ProjMat \matX^\top \matE^{(i)}\|^2 +  \frac{16\nu^2}{\lambda_{1}^2 \lambda_2^2}  \|(I - \ProjMat) \matX^\top \matE^{(i)}\|^2 \\  
			&= \sum_{i=1}^m  \frac{4}{\lambda_1^2} \left(\sum_{j=1}^p  \left\langle v_j, \matX^\top \matE^{(i)} \right\rangle^2\right) + \frac{16\nu^2}{\lambda_{1}^2 \lambda_2^2} \left(\sum_{j=p+1}^d   \left\langle v_j ,\matX^\top \matE^{(i)}\right\rangle^2\right)  
			\end{align*}
			\newcommand{\symv}{\boldsymbol{v}}
			For an index $j$, let $\lambda[j]$ equal $\lambda_1$ if $j \le p$, and $\lambda_2$ if $ p + 1 \le d \le d $, and define the vector $\matX_j = \matX v_j$. Then, $\left\langle v_j, \matX^\top \matE^{(i)} \right\rangle^2$ can be bounded as  as 
			\begin{align*}
			\left\langle v_j, \matX^\top \matE^{(i)} \right\rangle^2 &= \|\matX_j^\top \matE^{(i)}\|_2^2\\
			&= \|\matX_j\|^{2} \cdot \frac{\left\| \matX_j^\top \matE^{(i)}\right\|_2^2}{\|\matX_j\|_2^{2}}\\
			&\le \|\matX_j\|^{2} \cdot \frac{3}{2} \frac{\left\| \matX_j^\top \matE^{(i)}\right\|_2^2}{\|\matX_j\|_2^2 + \frac{1}{2}\lambda[j]},
			\end{align*}
			where in the last inequality we use that $\|\matX_j\|_2^2 = v_j^\top \matLam v_j \ge v_j^\top (\lambda_1 (1-\ProjMat) + \lambda_2 \ProjMat) v_j = \lambda[j]$. 

			By the scalar-valued  self normalized tail inequality~\Cref{lem:self_normalized},  it holds with probability $1-\delta$ that 
			\begin{align*}
			\left\langle v_j, \matX^\top \matE^{(i)} \right\rangle^2 &\le   3\|\matX_j\|_2^2 \log \frac{\frac{1}{2}\lambda[j] + \|\matX_j\|_2^2}{\frac{1}{2}\lambda_j[j]\delta}\\
			&\le   3\lambda[j] \matkappa[j] \log \frac{\frac{1}{2}\lambda[j] + \lambda[j] \matkappa[j]}{\frac{1}{2}\lambda_j[j]\delta}\\
			&\le   3\lambda[j] \matkappa[j] \log \frac{3 \matkappa[j]}{\delta}\\
			&= \begin{cases} 3\lambda_1 \matkappa_1  \log \frac{3 \matkappa_1}{\delta} & j \le d \\
			3\lambda_2 \matkappa_2  \log \frac{3 \matkappa_2}{\delta} & j \ge d \\
			\end{cases}
			\end{align*}
			where we set $\matkappa[j] := \frac{v_j^\top \matLam v_j}{\lambda[j]} \ge 1 $, and note that $\matkappa_1 := \max_{1 \le j \le d} \matkappa[j]$ and $\matkappa_2 := \max_{j > p} \matkappa[j]$. Hence, taking a union bound over all $dm$ coordinates
			\begin{align*}
			& \|\Thetahat_T - \Thetast \|_\fro^2\\
			 &\quad\le \sum_{i=1}^m  \frac{4}{\lambda_1^2} \left(\sum_{j=1}^p 3\lambda_1 \matkappa_1 \log \frac{3 dm\matkappa_1}{ \delta}\right) + \frac{16\nu^2}{\lambda_{1}^2 \lambda_2^2} \left(\sum_{j=p+1}^d   3\lambda_2 \matkappa_2 \log \frac{3 dm\matkappa_2}{ \delta}\right)  \\
			&\quad\le  \frac{12 mp \matkappa_1 }{\lambda_1}  \log \frac{3 md \matkappa_1}{ \delta} + \left(\frac{\nu}{\lambda_1}\right)^2 \cdot \frac{48m(p-d)\matkappa_2 }{\lambda_2}\log \frac{3 md\matkappa_2}{ \delta}. 
			\end{align*}
		\end{proof}
\subsubsection{Proof of Lemma~\ref{lem:covariance_lb}}

	By the  the
  Paley-Zygmund inequality (specifically, the variant in \citet[Equation
  3.12]{simchowitz2018learning}), one can easily show that the sequence $(\matz_t)$ satisfies the $(1,\Sigma,\frac{3}{10})$-block martingale small ball property \citep[Definition 2.1]{simchowitz2018learning}. Then, for any matrix $\Lambda_+ \succeq 0$, \citet[Section D.2]{simchowitz2018learning} (correcting the section for a lost normalization factor of T) shows that
  \begin{align}
  \Pr\left[ \left\{\matLam \not\succeq \frac{T}{16}(\frac{3}{10})^2 \Sigma \right\} \cap \{\matLam \preceq \Lambda_+\}\right] &\le \exp\prn*{ - \frac{1}{10}T(\frac{3}{10})^2 + 2d\log(\frac{100}{3}) + \log\det \Lambda_+(T\Sigma)^{-1}} \nonumber\\
  &\le \exp\left( - \frac{9T}{1000} + 2d\log(\frac{100}{3}) + d \log \frac{\opnorm{\Lambda_+}}{T\lambda_{\min}(\Sigma)})\right). \label{some_eq}
  \end{align}
  Now, notice that if we select $\Lambda_+ = \frac{\tr(\matLambar_T)}{\delta} I$, the bound $\opnorm{\Lambda} \le \tr(\Lambda)$ for $\Lambda \succeq 0$ and an application of Markov's inequality show that, $\Pr[\left\{\matLam \not\preceq \Lambda_+\right\} \cap \calE] \le \delta $. Hence, we have 
  \begin{align*}
  \Pr\left[ \matLam \not\succeq \frac{T}{16}(\frac{3}{10})^2\right] &\le  \inf_{\delta > 0} \exp\left( - \frac{9T}{1000} + 2d\log(\frac{100}{3}) + d \log \frac{\tr(\matLambar_T)}{T\lambda_{\min}(\Sigma)\delta})\right) + \delta\\
  &\le  \inf_{\delta > 0} \delta^{-d}\exp\left( - \frac{9T}{1000} + 2d\log(\frac{100}{3}) + d \log \frac{\tr(\matLambar_T)}{T\lambda_{\min}(\Sigma)\delta})\right) + \delta.
  \end{align*}
  Note that balancing $a \delta^{-d}  = \delta$ selects $\delta = a^{1/d+1}$, giving that the above is at most
  \begin{align*}
  2\exp\left( - \frac{1}{d+1}\left(\frac{9T}{1000} - 2d\log(\frac{100}{3}) -  d \log \frac{\tr(\matLambar_T)}{T\lambda_{\min}(\Sigma)})\right)\right).
  \end{align*}
  We conclude by bounding $\tr(\matLambar_T) \le JT$ by assumption and applying some elementary algebra.


%% file: appendix/lb_appendix.tex

\subsection{Proof of \Cref{lem:first_order_approx_quality}\label{ssec:lem:first_order_approx_quality}}
  Observe that
  \begin{align*}
  \max\{\opnorm{A_e - \Ast}, \opnorm{B_e - \Bst}\} &\le \max\{\fronorm{A_e - \Ast}, \fronorm{B_e - \Bst}\} \\
  &\le \max\{\opnorm{\Kst}^{1/2},1\}\fronorm{\Delta}\\
  &\le \sqrt{mn}\epspack \max\{\opnorm{\Kst}^{1/2},1\} \le \sqrt{\opnorm{\Pst}}\sqrt{mn}\epspack,
  \end{align*}
  where the last inequality is by
  Lemma~\ref{lem:helpful_norm_bounds}. This prove the first point of
  the lemma. Next, if $\epspack^2 \le \frac{1}{\opnorm{\Pst}}\Csafe^2(\Ast,\Bst)/nm$, then,
  \begin{align*}
  \max\{\opnorm{A_e - \Ast}, \opnorm{B_e - \Bst}\} \le \frac{1}{\Csafe(\Ast,\Bst)} \le  (1 - 2^{1/5}),
  \end{align*}
  which implies $\Psi_e \le 2^{1/5}\max\{1,\opnorm{\Ast},\opnorm{\Bst}\}$. Moreover Theorem~\ref{thm:main_perturb_app} yields
  \begin{align*} 
  \opnorm{P_e - \Pst} \le 1.085 \opnorm{\Pst} \le 2^{1/5}\opnorm{\Pst}.
  \end{align*}
  
  For the next point, \Cref{thm:quality_of_taylor_approx} bounds the error of the Taylor approximation, and 
  implies that for some polynomial $\frakp$,
   \begin{align*}
  &\|-(\Ru + \Bst^\top \Pst \Bst)^{-1} \cdot \Delta^\top \Pst \Aclst +  \Kst - K_e\|_{\fro}^2\\
   &\le \frakp(\opnorm{\Pst}) \max\{\opnorm{A_e - \Ast}, \opnorm{B_e - \Bst}\}^2 \max\{\fronorm{A_e - \Ast}, \fronorm{B_e - \Bst}\} ^2\\
  &\le \frakp(\opnorm{\Pst})\max\{\fronorm{A_e - \Ast}, \fronorm{B_e - \Bst}\}^4\\
  &\le (nm)^2\underbrace{\opnorm{\Pst}\frakp(\opnorm{\Pst})}_{:=\frakp_2(\frakp(\opnorm{\Pst}))}\epspack^4.
  \end{align*}
  Finally, point $4$ follows by bounding 
  \begin{align*}
  \|\Kst - K_e\|_{\fro} &\le  \|-(\Ru + \Bst^\top \Pst \Bst)^{-1} \cdot \Delta^\top \Pst \Aclst +  \Kst - K_e\|_{\fro} + \fronorm{(\Ru + \Bst^\top \Pst \Bst)^{-1} \cdot \Delta^\top \Pst \Aclst}\\
  &\le  mn\epspack^2\frakp_2(\opnorm{\Pst}) + \fronorm{(\Ru + \Bst^\top \Pst \Bst)^{-1} \cdot \Delta^\top \Pst \Aclst}\\
  &\le  mn\epspack^2\frakp_2(\opnorm{\Pst}) + \fronorm{\Delta}\opnorm{\Pst}{\opnorm{\Aclst}}\\
  &\le  mn\epspack^2\frakp_2(\opnorm{\Pst}) + \fronorm{\Delta}\opnorm{\Pst}^{3/2} \tag*{(Lemma~\ref{lem:helpful_norm_bounds})}\\
  &\le  mn\epspack^2\frakp_2(\opnorm{\Pst}) + \epspack\sqrt{mn}\opnorm{\Pst}^{3/2} .
  \end{align*}
  By taking $\epspack^2 \le 1/mn\poly(\opnorm{\Pst})$, the expression
  above can be made to be at most $2mn\epspack^2\opnorm{\Pst}^3$.

\subsection{Proof of~\Cref{lem:Kerr_lem}\label{ssec:lem:Kerr_lem}}

  Our strategy is to relate $\Regret_T[\pi;A_e,B_e] $ and $ \Kerr_e[\pi]$ to the benchmark inducted by following the true optimal policy $\pist = \pist(A,B)$ which minimizes $\Exp_{A_e,B_e,\pi}\sum_{t=1}^T \cost(\matx_t,\matu_t)]$ over all possible policies $\pi$.

  To begin, consider an arbitrary stabilizable system $(A,B)$. Let $\Kinf := \Kinf(A,B)$ and $\Pinf = \Pinf(A,B)$. For $T$ fixed and a control policy $\pi$, let 
  \begin{align*}
  \Kerr[\pi] := \Exp_{A,B,\pi}\left[\sum_{t=1}^{T/2} \|\matu_t - \Kinf(A,B)\matx_t\|_2^2\right].
  \end{align*}
  We define the $Q$-functions and value functions associated with the
  LQR problem as follows.
  \begin{align*}
  \QftT(x,u) := \Exp_{A,B,\pist}\left[\sum_{s=t}^T \cost(\matx_s,\matu_s) \mid \matx_t = x, \matu_t = u\right] \quad \quad \VftT(x) := \inf_{u} \QftT(x,u),
  \end{align*}
  where $\Exp_{A,B,\pist(A,B)}[\cdot \mid \matx_t = x,\matu_t = u]$ denotes that the state at time $t$ is $\matx_t = x$, inputs is $\matu_t = u$, and all future inputs are according to the policy $\pi_{\star}(A,B)$. Note then that $\pist$ always perscribes the action $\matu_t := \argmin \QftT(\matx_t,u)$ at time $t$. We can now characterize the form of the $\QftT$ and $\pist$ using the following lemma.

  \begin{lem}[Optimal Finite-Horizon Controllers \citep{bertsekas2005dynamic}]\label{lem:opt_finite_terms} Define the elements 
  \begin{align*}
  P_{t+1} &:= \Rx + A^{\trn}P_{t}{}A - A^{\trn}P_t B\Sigma_t^{-1}B^{\trn}P_tA,\\
  \Sigma_{t+1}&:= \Ru + B^\top P_{t} B,\\
  K_{t+1}&:= -\Sigma_{t+1}^{-1}B^\top P_{t} A,
  \end{align*}
  with the convention that $P_{0}=\Rx$. Then, $\VftT(x) = x^\top P_{T-t} x$, and $\QftT(x,u) - \VftT(x) = \| u - K_{T-t} \,x \|_{\Sigma_{T-t}}^2$, and $(\pist)_{t;T}(\matx_t) = K_{T-t} \,\matx_t$. 
  \end{lem}
  For completeness, we prove the lemma in Section~\ref{sssec:lem:opt_finite_terms}. Having defined the true optimal policy, we that the regret is lower bounded as follows.
  \begin{lem}\label{lem:learn_Kinf_Delt} Fix a system $A,B$, and suppose that $\Regret_T[\pi;A,B] \le T\,\Jfuncopt_{A,B}$. Then, 
  \begin{align*}
  \Regret_T[\pi;A,B] \ge \frac{1}{2}
  \Kerr[\pi] - \Jfuncopt_{A,B} \left(2T\,\left(\max_{t \ge T/2}\Delt\right) + \sum_{t \ge 0}\Delt\right),
  \end{align*}
  where we define the errors $\Delt := \|\Sigma_{t}^\top(\Kinf-  K_{t})\Rx^{-1/2}\|_{2}^2$.  
  \end{lem}
  \begin{proof}[Proof of Lemma~\ref{lem:learn_Kinf_Delt}]

  We compare both the cost under $\pi$ and the cost under a comparator to $\Vf_{1;T}(0)$, the value of the optimal policy starting at $\matx_1 = 0$.
  \begin{align*}
  \Regret_{T}[\pi;A,B] &= \Exp_{A,B,\pi}\left[\sum_{t=1}^T\cost_t(\matx_t,\matu_t)\right] - \Vf_{1;T}(0) - (T\Jfuncopt_{A,B}{\Kinf}- \Vf_{1;T}(0))\\
  &\ge \Exp_{A,B,\pi}\left[\sum_{t=1}^T\cost_t(\matx_t,\matu_t)\right] - \Vf_{1;T}(0) - (T\Exp_{A,B,\Kinf}\left[\sum_{t=1}^T\cost_t(\matx_t,\matu_t)\right]- \Vf_{1;T}(0)),
  \end{align*}
  where we use the fact that the infinite horizon regret induced by $\Kst$ on a finite time horizon $T$ is upper bounded by $T$-times the infinite horizon cost (this can be verified by direct computation). 

  Next, we use the performance difference lemma, which states that for any policy $\pi'$,
  \begin{align*}
  \Exp_{A,B,\pi'}\left[\sum_{t=1}^T\cost_t(\matx_t,\matu_t)\right] - \Vf_{0;T}(0)  &= \sum_{t=1}^T\Exp_{A,B,\pi'}[ \Qf_{t;T}(\matx_t, \matu_t) - \Vf_{t;T}(\matx_t)]\\
  &= \sum_{t=1}^T\Exp_{A,B,\pi'}\left[ \|\matu_t - K_{T-t} \matx_t\|^2_{\Sigma_{T-t}}\right]\tag*{(\Cref{lem:opt_finite_terms})}.
  \end{align*}
  Therefore, 
  \begin{align*}
  \Regret_{T}[\pi;A,B] &=  \underbrace{\sum_{t=1}^T\Exp_{A,B,\pi}\left[ \|\matu_t -  K_{T-t} \matx_t\|^2_{\Sigma_{T-t}}\right]}_{\text{(policy suboptimality)}} - \underbrace{\sum_{t=1}^T\Exp_{A,B,\Kinf}\left[ \|\matu_t -  K_{T-t} \matx_t\|^2_{\Sigma_{T-t}}\right]}_{\text{(comparator suboptimality)}}.
  \end{align*}
  \paragraph{Comparator Suboptimality.}
  We begin with two claims.
  \begin{claim}\label{claim:Del} $ \|(\Kinf- K_{T-t}) \matx_t\|^2_{\Sigma_{T-t}} \le \DelTmt \cdot \matx_t^\top \Rx\matx$.
  \end{claim}
  \begin{proof}
  We have that
  \begin{align*}
  \left\|\left(\Kinf-  K_{T-t}\right) \matx_t\right\|^2_{\Sigma_{T-t}} &=  \left\|\Sigma_{T-t}^{1/2}(\Kinf-  K_{T-t}) \Rx^{-1/2} \Rx^{1/2}\matx_t\right\|^2 \\
  &\le \left\|\Sigma_{T-t}^{1/2}(\Kinf-  K_{T-t}) \Rx^{-1/2}\right\|_{2}^2 \left\|\Rx^{1/2}\matx_t\right\|_2^2 :=  \eta_{T-t} \cdot \matx_t^\top \Rx\matx_t.
  \end{align*}
  \end{proof}
  \begin{claim} $\Exp_{A,B,\Kinf}[\sum_{t=1}^T \matx_t^\top \Rx \matx_t] \le T\Jfuncopt_{A,B}$.
  \end{claim}
  \begin{proof} We have
  \begin{align*}
  \Exp_{A,B,\Kinf}[\sum_{t=1}^T \matx_t^\top \Rx \matx_t] &\le \Exp_{A,B,\Kinf}[\sum_{t=1}^T \matx_t^\top (\Rx + \Kinf^\top \Ru \Kinf)\matx_t ] \\
  &= \tr(\sum_{t=1}^T \sum_{s=0}^t((A + B\Kinf)^s)^\top (\Rx + \Kinf^\top \Ru \Kinf)((A + B\Kinf)^s))\\
  &\le \tr(\sum_{t=1}^T \sum_{s=0}^{\infty}((A + B\Kinf)^s)^\top (\Rx + \Kinf^\top \Ru \Kinf)((A + B\Kinf)^s))\\
  &= T\tr( \Pinf(A,B)) = T\Jfuncopt_{A,B},
  \end{align*}
  where the last equalities are by Lemma~\ref{lem:P_bounds_lowner}.
  \end{proof}
  Invoking these two claims, we have
  \begin{align*}
  \sum_{t=1}^T\Exp_{A,B,\Kinf}\left[ \|\matu_t - \KTmt \matx_t\|^2_{\Sigmat}\right] \le \sum_{t=1}^T\Exp_{A,B,\Kinf}\Delt\left[ \matx_t^\top \Rx \matx_t\right] 
    \le \Jfuncopt_{A,B}\sum_{t=1}^T\eta_{T-t} \le \Jfuncopt_{A,B} \sum_{t = 0}^{\infty}\eta_{t}.
  \end{align*}
  \paragraph{Policy Suboptimality.}
  We first make the following claim.
  \begin{claim}
    \label{claim:give_take} Let $(\calX,\langle \cdot, \cdot \rangle_{\calX})$ denote an inner product space with induced norm $\|\cdot\|_{\calX}$. Then for any $x,y \in \calX$, $\|x+y\|_{\calX}^2 \ge \frac{1}{2}\|x\|_{\calX}^2 - \|y\|_{\calX}^2$.
  \end{claim}
  \begin{proof} We have $\|x+y\|_{\calX}^2 = \|x\|_{\calX}^2 + \|y\|_{\calX}^2 + 2\langle x, y\rangle_{\calX}$. Note that, for any $\alpha > 0$, we have $|2\langle x, y\rangle_{\calX}| = |2\langle  \alpha^{1/2} x, \alpha^{-1/2} y\rangle_{\calX}|  \le \alpha\|x\|_{\calX}^2 + \alpha^{-1}\|y\|_{\calX}^2$. Setting $\alpha = \frac{1}{2}$, we have $|2\langle x, y\rangle_{\calX}| \le \frac{1}{2}\|x\|_{\calX}^2 + 2\|y\|_{\calX}^2$.  Hence $ \|x+y\|_{\calX}^2 = \|x\|_{\calX}^2 + \|y\|_{\calX}^2 + 2\langle x, y\rangle_{\calX} \ge \|x\|_{\calX}^2 + \|y\|_{\calX}^2 - (\frac{1}{2}\|x\|_{\calX}^2 + 2\|y\|_{\calX}^2) =\frac{1}{2}\|x\|_{\calX}^2 - \|y\|_{\calX}^2 $. 
  \end{proof}
  We can now lower bound
  \begin{align*}
   &\Exp_{A,B,\pi}\left[\sum_{t=1}^T\|\matu_t - K_{T-t} \matx_t\|^2_{\Sigma_{T-t}} \right] \\
   &\ge \Exp_{A,B,\pi}\left[\sum_{t=1}^{T/2}\|\matu_t -  K_{T-t} \matx_t\|^2_{\Sigma_{T-t}} \right]\\
  &\ge \Exp_{A,B,\pi}\left[\sum_{t=1}^{T/2}\frac{1}{2}\|\matu_t - \Kinf\matx_t\|^2_{\Sigma_{T-t}} - \|(K_{T-t} - \Kinf) \matx_t\|^2_{\Sigma_{T-t}}\right] \tag*{(\Cref{claim:give_take})}\\
  &\ge \Exp_{A,B,\pi}\left[\sum_{t=1}^{T/2}\frac{1}{2}\|\matu_t - \Kinf\matx_t\|^2_{2} - \|( K_{T-t} - \Kinf) \matx_t\|^2_{\Sigma_{T-t}}\right] \tag*{($\Sigma_{T-t} \succeq \Ru \succeq I$)}.\\
  \intertext{The expression above is equal to}
  &= \frac{1}{2}\Kerr[\pi] - \sum_{t=1}^{T/2} \Exp_{A,B,\pi}\left[\|( K_{T-t} - \Kinf) \matx_t\|^2_{\Sigma_{T-t}}\right]\\
  &\ge \frac{1}{2}\Kerr[\pi] - \sum_{t=1}^{T/2}\eta_{t}\Exp_{A,B,\pi}\left[\matx_t^\top \Rx \matx_t \right] \tag*{(\Cref{claim:Del})}\\
  &\ge \frac{1}{2}\Kerr[\pi] -2\max_{t=1}^{T/2}\eta_{T-t} \Jfuncopt_{A,B} \tag*{(\Cref{claim:xcost_ub})} \\
  &\ge \frac{1}{2}\Kerr[\pi] -2\Jfuncopt_{A,B}\max_{t\ge T/2}\eta_{t}  \tag*{(\Cref{claim:xcost_ub})},
  \end{align*}
  where the last inequality uses the following claim.
  \begin{claim}\label{claim:xcost_ub} If $\Regret_{A,B,T}[\pi] \le T\Jfuncopt_{A,B}$ (in particular, under \Cref{asm:uniform_correctness}), then for any $\tau \le T$ and $Q \preceq \Rx$, we have $\sum_{t=1}^{\tau}\Exp_{A,B,\pi}\left[\matx_t^\top Q \matx_t \right] \le 2T\Jfuncopt_{A,B}$.
  \end{claim}
  The claim is stated for an arbitrary matrix $Q$ so that it can be specialized where necessary. 
  \begin{proof} We have $\sum_{t=1}^{\tau}\Exp_{A,B,\pi}\left[\matx_t^\top Q \matx_t \right] \le \sum_{t=1}^{T}\Exp_{A,B,\pi}\left[\matx_t^\top \Rx \matx_t \right] \le \sum_{t=1}^{T}\Exp_{A,B,\pi}\left[\matx_t^\top \Rx \matx_t + \matu_t^\top \Ru \matu\right] = \Regret_{A,B,T}[\pi] + T\Jfuncopt_{A,B} \le 2T\Jfuncopt_{A,B}$.
  \end{proof}
  Combining the comparator suboptimality and policy suboptimality bounds completes the proof of \Cref{lem:learn_Kinf_Delt}.
  \end{proof}
  The next lemma shows that the error sequence $\Delt$ has geometric decrease.
  \begin{lem}[Bound on $\Delt$]\label{lem:delta_bound} Let $(A,B)$ be statibilzable. Then, for $\eta_t$ defined above, we have
  \begin{align*}
  \eta_{t}\leq{}\prn*{1+\tfrac{1}{\nu}}^{-t},\quad
    \text{where}\quad \nu=2\opnorm{\Pinf(A,B)}\Psi(A,B)^2.
  \end{align*}
  \end{lem}

  \begin{proof}[Proof of \Cref{lem:delta_bound}]
    Since $\Rx \succeq I$,
    \begin{align*}
    \eta_{t} \leq{}
    \frac{1}{\eigmin(\Rx)}\nrm*{\Sigma_{t;T}^{1/2}\prn*{\Kinf- K_{T-t}}}_{\op}^{2} \le \nrm*{\Sigma_{t;T}^{1/2}\prn*{\Kinf- K_{T-t}}}_{\op}^{2}.
    \end{align*}
    Next, observe that from \Cref{lem:opt_finite_terms} we have
    \begin{align*}
      \nrm*{\Sigma_{T-t}^{1/2}\prn*{\Kinf- K_{T-t}}}_{\op}^{2}
      = \sup_{\nrm*{x}\leq{}1}\nrm*{\prn*{\Kinf- K_{T-t}}x}_{\Sigma_{T-t}}^{2}&=\sup_{\nrm*{x}\leq{}1}\brk*{\Qf_{t;T}(x,\Kinf{}x)-\Vf_{T;t}(x)}.
      \end{align*}
      Since  $\Qf_{t;T}(x,\Kinf{}x)$ is a finite horizon Q-function for a stationary process with non-negative rewards, we have $\Qf_{t;T}(x,u) \le \Qf_{\infty}(x,u)$. Therefore, the above is
      \begin{align*}
       &\leq{}\sup_{\nrm*{x}\leq{}1}\brk*{\Qf_{\infty}(x,\Kinf{}x)-\Vf_{T;t}(x)}\\
       &=\sup_{\nrm*{x}\leq{}1}\brk*{\Vf_{\infty}(x)-\Vf_{T-t}(x)}\\
       &=\sup_{\nrm*{x}\leq{}1}\brk*{x^{\trn}P_{\infty}x-x^{\trn}\Pt{}x}\\
       &=\nrm*{P_{\infty}-P_{T-t}}_{\op},
    \end{align*}
    where we use that $\Pinf$ is the value function for the infinite horizon process (\citet[Proposition 1]{lincoln2006relaxing}). By reparametrizing, we have verified that
    \begin{align*}
    \eta_t \le \nrm*{P_{\infty}-P_{t}}_{\op},
    \end{align*}

    To conclude, we apply \Cref{lem:ricatti_recursion}, which implies that $\nrm*{\Pinf-\Pt}_{2}\leq{}\prn*{1+\frac{1}{\nu}}^{-(T-t+1)}$, where $\nu$ is as in the lemma statement: 
  \begin{lem}[\cite{dean2018regret}, Lemma E.6]
  \label{lem:ricatti_recursion}
  Consider the Riccati recursion
  \[
  P_{t+1}=\Rx+A^{\trn}P_{t}A - A^{\trn}BP_t(\Ru+B^{\trn}P_tB)^{-1}B^{\trn}P_tA,
  \]
  where $\Rx$ and $\Ru$ are positive definite and $P_0=0$. When $P_{\infty}$ is the unique solution of the DARE, we have
  \begin{equation}
    \label{eq:ricatti_convergence}
    \nrm*{P_t-\Pinf}_{\op}\leq{} \|\Pinf\|_{\op}\prn*{1+\frac{1}{\nu}}^{-t},
  \end{equation}
  where $\nu=2\nrm*{\Pinf}_{\op}\cdot
  \prn*{\frac{\nrm*{A}_{\op}^{2}}{\eigmin\prn*{\Rx}}\vee
    \frac{\nrm*{B}_{\op}^{2}}{\eigmin\prn*{\Ru}}}$.\footnote{The bound
    stated in \cite{dean2018regret} is sightly incorrect in that it is missing a factor of  $\|\Pinf\|_{\op}$. The reader can verify the correctness of our statement by examining \citet[Proposition 1]{lincoln2006relaxing}.}
  \end{lem}
   We can take $\nu \le 2\opnorm{\Pinf}\max\{\opnorm{A}^2,\opnorm{B}^2\}\le 2\Psi(A,B)\opnorm{\Pinf}^2$, as $\Rx,\Ru \succeq I$.
  \end{proof}
  We can now conclude the proof of \Cref{lem:Kerr_lem}.
  \begin{proof}[Proof of \Cref{lem:Kerr_lem}]
   From Lemma~\ref{lem:learn_Kinf_Delt}, we have the lower bound
  \begin{align*}
  \Regret_T[\pi;A_e,B_e] \ge \frac{1}{2}\Kerr_{e}[\pi] - \Jfunc_e \left(2T\,\left(\max_{t \ge T/2}\Delt\right) + \sum_{t \ge 0}\Delt\right),
  \end{align*}
  Recall $\nu := 2\opnorm{P_e}\Psi_e^2$ from~\Cref{lem:delta_bound}, and that $\Delt \le \opnorm{P_e}(1+\nu^{-1})^t \le \exp( - t/\nu)$. Therefore
  \begin{align*}
  \sum_{t \ge 0}\eta_t \le 2\opnorm{P_e}^2\Psi_e^2,
  \end{align*}
  Hence, if $T \ge 2\nu \log(2T)$, we have that 
  \begin{align*}
  2T\,\left(\max_{t \ge T/2}\Delt\right) \le \opnorm{P_e} \le \opnorm{P_e}^2,
  \end{align*}
  where we use $\Pinf(\cdot,\cdot) \succeq I$ (Lemma~\ref{lem:closed_loop_dlyap}). Hence, for such $T$,
  \begin{align*}
  \Regret_T[\pi;A,B] &\ge \frac{1}{2}\Kerr_{e}[\pi] - 3\opnorm{P_e}^2\Psi_e^2 
  \\
  &\ge \frac{1}{2}\Kerr_{e}[\pi]- J_e3\opnorm{P_e}^2\Psi_e^2
  \\
  &\ge \frac{1}{2}\Kerr_{e}[\pi]- \dimx\underbrace{3\opnorm{P_e}^3\Psi_e^2}_{=\gamma_e}.
  \end{align*}
  Since $\nu \ge 1$, the condition $T \ge 2\nu \log(2T)$ holds as long as $T \ge c' \nu^2 = c' \opnorm{P_e}^2\Psi_e^4 c'$. Reparametrizing in terms of $\Pst,\Mbarst$ in view of Lemma~\ref{lem:first_order_approx_quality} concludes the proof.
  \end{proof}

  \subsubsection{Proof of Lemma~\ref{lem:opt_finite_terms} \label{sssec:lem:opt_finite_terms}}

    We first recall a standard expression for the value function \citet[Section 4.1]{bertsekas2005dynamic}:
    \begin{align*}
    \Vf_t(x) = \nrm*{x}_{P_{T-t}}^{2} + \sum_{s=t+1}^{T}\tr\prn*{P_{T-s}}.
    \end{align*}
    To obtain the expression for the $\Qf_t$, we have 
    \begin{align*}
    \Qf_t(x,u) &= \cost(x,u) + \Exp_{\matw_t}\brk*{\Vf_{t+1}(Ax + Bu + \matw_t)}\\
    &= \cost(x,u) + (Ax + Bu)^\top P_{T-(t+1)} (Ax + Bu) + \Exp[\matw_t^\top P_{T-(t+1)} \matw_t] + \sum_{s=t+2}^{T}\tr\prn*{P_{T-s}} \\
    &= \cost(x,u) +  (Ax + Bu)^\top P_{T-(t+1)} (Ax + Bu) + \sum_{s=t+1}^{T}\tr\prn*{P_{T-s}}.
    \end{align*}
    Note that $\Vf_t(x) := \min_u \Qf_t(x,u)$, and $\Qf_t(x,u)$ is a quadratic function. We can compute 
    \begin{align*}
    \argmin_{u} \Qf_t(x,u) &= \argmin_u \cost(x,u) +  (Ax + Bu)^\top P_{T-(t+1)} (Ax + Bu) + \beta_t\\
    &= \argmin_u x^\top \Rx x + u^\top \Ru u  +  (Ax + Bu)^\top P_{T-(t+1)} (Ax + Bu) + \beta_t\\
    &= \argmin_u u^\top (\Ru + B^\top \Ptpl{} B) u  +  2 u^{\trn}B^\top P_{T-(t+1)} Ax\\
    &= -(\Ru + B^\top P_{T-(t+1)} B)^{-1}B^\top \Ptpl{}Ax\\
                           &= K_{t;T}x.
    \end{align*}
    Moreover, since $\Qf_t(x,u)$ is quadratic in $u$ with quadratic form $\Sigma_{T-t}=(\Ru + B^\top P_{T-(t+1)} B)$ and the gradient $\grad{}_u\Qf_t(x,u)$ vanishes at the minimizer $u=K_{t;T}x$, we have
    \begin{align*}
    \Qf_t(x,u) - \Vf_t(x) = \Qf_t(x,u) - \Qf_t(x,K_{t;T}x) = \|u -  K_{t;T}x\|^2_{\Sigma_{T-t}}.
    \end{align*}
    \qed

\subsection{Proof of \Cref{lem:least_squares_lb}\label{ssec:lem:least_squares_lb}}
\label{ssec:least_squares_lb}
  Again, let us begin proving the lemma for an arbitrary stabilizable $(A,B)$, and then specialize to the packing instances $(A_e,B_e)$. For a fixed policy $\pi$, and let all probabilities and expectations be under the law $\Pr_{\pi;A,B}$. Our strategy follows from \cite{arias2012fundamental}. Let $\Kinf = \Kinf(A,B)$ and let $\matdel_t := \matu_t - \Kinf\matx_t$, and note that $\Kerr[\pi] = \Exp_{A,B,\pi}[\sum_{t=1}^{\frac{T}{2}}\|\matdel_t\|_2^2]$. Define the covariance matrix
  \begin{align*}
  \matLam_{\frac{T}{2}} := \sum_{t=1}^{\frac{T}{2}}\matx_t\matx_t^{\top}.
  \end{align*}
  For some constant $c > 0$ to be chosen at the end of the proof, consider a `thresholded'' least squares estimator defined as follows:
  \begin{align*}
  \Kls &:= \I\left\{\matLam_{\frac{T}{2}} \succeq c\frac{T}{2} I\right\} \cdot \left(\sum_{t=1}^{\frac{T}{2}} \matu_t\matx_t^\top\right)\matLam_{\frac{T}{2}}^{-1}\\
  &=  \I\left\{\matLam_{\frac{T}{2}} \succeq c\frac{T}{2} I\right\} \left(\sum_{t=1}^{\frac{T}{2}} \matdel_t\matx_t^\top\right)\matLam_{\frac{T}{2}}^{-1} +\I\left\{\matLam_{\frac{T}{2}} \succeq c\frac{T}{2} I\right\} \left(\sum_{t=1}^{\frac{T}{2}} \Kinf\matx_t\matx_t^\top\right)\matLam_{\frac{T}{2}}^{-1}\\
  &= \I\left\{\matLam_{\frac{T}{2}} \succeq c\frac{T}{2} I\right\}\Kinf + \I\left\{\matLam_{\frac{T}{2}} \succeq c\frac{T}{2} I\right\} \left(\sum_{t=1}^{\frac{T}{2}} \matdel_t\matx_t^\top\right)\matLam_{\frac{T}{2}}^{-1}.
  \end{align*}
  Hence, introducing the matrices $\matX := \begin{bmatrix} \matx_1 \mid \matx_2 \mid \dots \matx_{\frac{T}{2}} \end{bmatrix}$, and $\matDel := \begin{bmatrix} \matdel_1 \mid \matdel_2 \mid \dots \matdel_{\frac{T}{2}} \end{bmatrix}$,
  \begin{align*}
  \Exp\left[\|\Kls - \Kinf\|_{\fro}^2\right] &= \|\Kinf\|_{\fro}^2\Pr\left[\matLam_{\frac{T}{2}} \not\succeq c\frac{T}{2} I\right] + \Exp\left[\I\left\{\matLam_{\frac{T}{2}} \succeq c\frac{T}{2} I\right\} \left\|\left(\sum_{t=1}^{\frac{T}{2}} \cdot \matdel_t\matx_t^\top\right)\matLam_{\frac{T}{2}}^{-1}\right\|_{F}^2\right]\\\\
  &= \|\Kinf\|_{\fro}^2\Pr\left[\matLam_{\frac{T}{2}} \not\succeq c\frac{T}{2} I\right] + \Exp\left[\I\left\{\matX\matX^{\top} \succeq c\frac{T}{2} I\right\} \left\|(\matDel\matX^{\top})(\matX\matX^\top)^{-1}\right\|_{F}^2\right]\\
  &= \|\Kinf\|_{\fro}^2\Pr\left[\matLam_{\frac{T}{2}} \not\succeq c\frac{T}{2} I\right] + \Exp\left[\I\left\{\matX\matX^{\top} \succeq c\frac{T}{2} I\right\} \left\|\matDel\matX^{\dagger}\right\|_{F}^2\right]\\
  &\le \|\Kinf\|_{\fro}^2\Pr\left[\matLam_{\frac{T}{2}} \not\succeq c\frac{T}{2} I\right] + \Exp\left[\I\left\{\matX\matX^{\top} \succeq c\frac{T}{2} I\right\} \left\|\matDel\right\|_{F}^2 \left\|\matX^{\dagger}\right\|_{2}^2\right]\\
  &\le \|\Kinf\|_{\fro}^2\Pr\left[\matLam_{\frac{T}{2}} \not\succeq c\frac{T}{2} I\right] + \frac{1}{c\frac{T}{2}}\Exp\left[ \left\|\matDel\right\|_{F}^2 \right]\\
  &= \|\Kinf\|_{\fro}^2\Pr\left[\matLam_{\frac{T}{2}} \not\succeq c\frac{T}{2} I\right] + \frac{2}{cT}\Kerr[\pi]\\
  &\le \Jinf(A,B)\Pr\left[\matLam_{\frac{T}{2}} \not\succeq c\frac{T}{2} I\right] + \frac{2}{cT}\Kerr[\pi].
  \end{align*}
  where the second-to-last line follows since $\left\|\matDel\right\|_{F}^2 = \sum_{t=1}^{\frac{T}{2}}\|\matdel_t\|_2^2$, and the last line uses by Lemma~\ref{lem:helpful_norm_bounds} which bounds $\Kinf^\top \Kinf \preceq \Pinf(A,B)$, so that $\fronorm{\Kinf}^2 \le \tr(\Pinf(A,B)) = \Jfuncopt_{A,B}$.

  In order to conclude the proof, we need to select show that, for some constant $c$ sufficiently small and $\frac{T}{2}$ sufficiently large, $\Pr\left[\matLam_{\frac{T}{2}} \not\succeq c\frac{T}{2} I\right]$ is neglible. Let us now apply Lemma~\ref{lem:covariance_lb}. Let $\calF_{t}$ denote the filtration generated by $(\matx_s,\matu_s)_{s \le t}$ and $\matu_{t+1}$. Observe that $\matx_t \mid \calF_{t-1} \sim \calN(\overline{\matx}_t, I)$, where $\overline{\matx}_t$ is $\calF_{t-1}$-measurable. 

  Let us now specialize to an instance $(A,B) = (A_e,B_e)$. We can then bound
\begin{align*}
\Exp[\tr(\sum_{t=1}^{T/2}\matx_t \matx_t^\top )] \le \Exp\left[\sum_{t=1}^{T}\matx_t^\top \Rx \matx_t + \matu_t^\top \Ru \matu\right] \le 2\Jfuncopt_{A,B}T, 
\end{align*}
by the \Cref{asm:uniform_correctness} and Claim~\ref{claim:xcost_ub}. 
 Hence, $\tr(\Exp[\matLam_{T/2}]) \le \frac{T}{2} \cdot (4\Jfuncopt_{A_e,B_e})$. Therefore, if
\begin{align*}
T \ge \frac{2000}{9}\left(2\dimx\log \tfrac{100}{3} + \dimx \log 4\Jfuncopt_{A_e,B_e}\right),
\end{align*}
we have that
\begin{align*}
\Pr\left[\matLam_{T/2} \not\succeq \frac{9(T/2)}{1600} \right] \le  2\exp\left( - \tfrac{9}{2000(\dimx+1)}T \right) .
\end{align*}
In particular, there exists a universal constants $c,\cls$ such that (recalling$\Jfuncopt_{A_e,B_e} = \tr(\Pinf(A_e,B_e)$ )
\begin{align*}
T \ge c\dimx \log(1+\dimx \opnorm{\Pinf(A_e,B_e)}) \ge c\dimx\log(1+\Jfuncopt_{A_e,B_e}).
\end{align*}
then for a universal constant $\cls$, we have
\begin{align*}
\Exp\left[\|\Kls - \Kinf\|_F^2\right] \le \epsls + \frac{1}{T\cls}\Kerr[\pi].  
\end{align*}
Moreover, for $(A_e,B_e)$, we can upper bound $\opnorm{P_e} \lesssim  \opnorm{\Pst}$ (and amend $c$ accordingly) using Lemma~\ref{lem:first_order_approx_quality}, concluding the proof.

\subsection{Proof of Lemma~\ref{lem:info_lb}\label{ssec:lem:info_lb}}
  Let $\tau = T/2$. Recall that our packing consists of systems $(A_e,B_e)$ indexed by sign-vectors $e \in \espace$:
  \begin{align*}
  (A_e,B_e) := (\Ast - \Delta_{e}\Kst, \Bst + \Delta_e), \quad \text{ where } \Delta_e = \epsilon \sum_{i =1}^n\sum_{j=1}^m e_{i,j}\uvec_i\vvec_j^\top.
  \end{align*}
  To keep notation compact, let $\vecq := (i,j)$ denote a stand-in for the double indices $(i,j)$, with $\vecq_1 = i$ and $\vecq_2 = j$. Given an indexing vector $e \in \espace$, $e^c$ denote the vector consisting of coordinates of $e$ other than $(\vecq_1,\vecq_2)$. For $a \in \{\minone,1\}$, we set
  \begin{align*}
  \Delta_{a,\vecq,e_{\vecq}^c} := \epsilon\left(ae_{\vecq_1,\vecq_2}+ \sum_{\vecq' \ne \vecq} e_{\vecq'_1,\vecq'_2}\uvec_{\vecq'_1}\vvec_{\vecq'_2}^\top\right).
  \end{align*}
  and define $A_{a,\vecq,e_{\vecq}^c}$, $B_{a,\vecq,e_{\vecq}^c}$ analogously, let $\Pr_{a,\vecq,e_{\vecq}^c}$ denote the law of the first $\tau = T/2$ rounds under $\Pr_{A_{a,\vecq,e_{\vecq}^c},B_{a,\vecq,e_{\vecq}^c},\Alg}\left[\cdot\right]$. 

  We now consider an indexing vector $\mate$ drawn uniformly from $\espace$. We will then let $\Pr_{a,\vecq}$ denote the law $\Pr_{\mate_{\vecq}^c}[\Pr_{a,\vecq,\mate_{\vecq}^c}]$, maginalizing over the entries $\mate_{\vecq}^c$. Our proof now follows from the argument in \cite{arias2012fundamental}. We note then that, for any $\vecq$ and any $\ehat$ that depends only on the first $\tau = T/2$ time steps, we can bound
  \begin{align*}
  \Exp_{\mate}\Exp_{A_{\mate},B_{\mate}}\left[\left|\mate_{\vecq} - \ehat_{\vecq}\right|\right] &= \Exp_{\mate_q \unifsim \{\minone,1\}}\Exp_{\mate_{q}^c}\Exp_{A_{\mate_q,\vecq,\mate^c_q},B_{\mate_q,\vecq,\mate^c_q}}\left[\left|\mate_{\vecq} - \ehat_{\vecq}\right|\right] \\
  &= \Exp_{\mate_q \unifsim \{\minone,1\}}\Exp_{\mate_q,q}\left[\left|\mate_{\vecq} - \ehat_{\vecq}\right|\right] 
  \ge \frac{1}{2}\left(1 - \TV\left(\Pr_{\minone,\vecq},\Pr_{1,\vecq}\right)\right).
  \end{align*}
  Hence, by Cauchy Schwarz,
  \begin{align*}
  \Exp\left[\sum_{\vecq} |\mate_{\vecq} - \ehat_{\vecq}|\right] &\ge \sum_{\vecq}\frac{1}{2}(1 - \TV(\Pr_{0,\vecq},\Pr_{1,\vecq}))\\
   &\ge \frac{nm}{2}\sum_{\vecq}\left(1 - \sqrt{\frac{1}{nm}\sum_{\vecq}\TV(\Pr_{0,\vecq},\Pr_{1,\vecq})^2}\right).
  \end{align*}
  Moreover, by Jensen's ineqality followed by a symmetrized Pinsker's ienqualty,
  \begin{align*}
  \TV(\Pr_{0,\vecq},\Pr_{1,\vecq})^2 &\le \Exp_{\mate_{\vecq}^c}\left[\TV(\Pr_{-1,\vecq,\mate_{\vecq}^c},\Pr_{1,\vecq,\mate_{\vecq}^c})^2\right]\\
  &\le \frac{1}{2}\Exp_{\mate_{\vecq}^c}\left[\frac{\KL(\Pr_{-1,\vecq,\mate_{\vecq}^c},\Pr_{1,\vecq,\mate_{\vecq}^c})}{2} + \frac{\KL(\Pr_{1,\vecq,\mate_{\vecq}^c},\Pr_{\minone,\vecq,\mate_{\vecq}^c})}{2}\right].
  \end{align*}
  We now require the following lemma to compute the relevant $\KL$-divergences, which we prove below.
  \begin{lem}\label{lem:KL} Let $\Delta^{(0)},\Delta^{(1)} \in \R^{\dimx
      \times \dimu}$, $\tau \in \N$, and let $\Ast,\Bst$ be the nominal
    systems defined above. For $i \in\{0,1\}$, let $\Pr_{i}$ denote the law of the first  $\tau$ iterates under $\Pr_{\Ast - \Delta^{(i)} \Kst,\Bst +\Delta^{(i)},\Alg}[\cdot]$. Then,
  \begin{align*}
  &\KL(\Pr_{0},\Pr_{1}) = \frac{1}{2}\tr\left(\left(\Delta^{(0)} - \Delta^{(1)}\right)\Lambda_{\tau}(\Delta^{(0)})\left(\Delta^{(0)} - \Delta^{(1)}\right)^\top\right).
  \end{align*}
  where we have defined the matrix
  \begin{align*}
  &\Lambda_{\tau}(\Delta) := 
  \Exp_{\Ast - \Delta \Kst,\Bst +\Delta,\Alg}\left[\sum_{t=1}^{\tau}\left(\matu_t - \Kst \matx_t\right)\left(\matu_t - \Kst \matx_t\right)^{\top}\right].
  \end{align*}
  \end{lem}
  We can now compute 
  \begin{align*}
  \KL(\Pr_{\textmin1,\vecq,\mate_{\vecq}^c},\Pr_{1,\vecq,\mate_{\vecq}^c}) &=  \tr((\Delta_{1,\vecq,e_{\vecq}^c} - \Delta_{\textmin1,\vecq,e_{\vecq}^c})^\top \Lambda_{\tau}(\Delta_{-1,\vecq,e_{\vecq}^c}) (\Delta_{1,\vecq,e_{\vecq}^c} - \Delta_{\textmin1,\vecq,e_{\vecq}^c}) )\\
  &= 2\epsilon^2 \tr(u_{\vecq_1}w_{\vecq_2}^\top \Lambda_{\tau}(\Delta_{\textmin1,\vecq,e_{\vecq}^c}) w_{\vecq_2}u_{\vecq_1}^\top) \\
  &= 2\epsilon^2 w_{\vecq_2}^\top \Lambda_{\tau}(\Delta_{\textmin1,\vecq,e_{\vecq}^c}) w_{\vecq_2}.
  \end{align*}
  Hence, we have
  \begin{align*}
  \TV(\Pr_{0,\vecq},\Pr_{1,\vecq}))^2 &\le  \frac{1}{2}\cdot 2\epsilon^2 w_{\vecq_2}^\top\left(\Exp_{\mate_{\vecq}^c}\left[\frac{ \Lambda_T(\Delta_{0,\vecq,e_{\vecq}^c}) +  \Lambda_T(\Delta_{1,\vecq,e_{\vecq}^c})}{2}\right]\right)w_{\vecq_2}\\
  &= \epsilon^2 w_{\vecq_2}^\top\left(\Exp_{\mate}\left[ \Lambda_{\tau}(\Delta_{\mate})\right]\right)w_{\vecq_2}.
  \end{align*}
  Hence, since $\{w_j\}$ for an orthonormal basis,
  \begin{align*}
  \sum_{\vecq} \TV(\Pr_{0,\vecq},\Pr_{1,\vecq}))^2 &= \sum_{\vecq}\epsilon w_{\vecq_2}^\top\left(\Exp_{\mate}\left[ \Lambda_{\tau}(\Delta_{\mate})\right]\right)w_{\vecq_2}\\
  &= \epsilon\sum_{i=1}^m \sum_{j=1}^nw_j^\top\left(\Exp_{\mate}\left[ \Lambda_{\tau}(\Delta_{\mate})\right]\right)w_j\\
  &\le m\epsilon^2 \tr(\Exp_{\mate}[\Lambda_{\tau}(\Delta_{\mate}])).
  \end{align*}
  We simplify further as
  \begin{align*}
  \tr(\Exp_{\mate}[\Lambda_T(\Delta_{\mate}]) &= \Exp_{\mate}[\tr(\Lambda_T(\Delta_{\mate})]\\
  &= \Exp_{\mate}\left[\tr\left(\Exp_{\Ast - \Delta_{\mate} \Kst,\Bst +\Delta_{\mate},\Alg}\left[\sum_{t=1}^{\tau}\left(\matu_t - \Kst \matx_t\right)\left(\matu_t - \Kst \matx_t\right)^{\top}\right]\right)\right]\\
  &= \Exp_{\mate}\left[\Exp_{A_{\mate},B_{\mate},\Alg}\left[\sum_{t=1}^{\tau}\tr\left(\left(\matu_t - \Kst \matx_t\right)\left(\matu_t - \Kst \matx_t\right)^{\top}\right)\right]\right]\\
  &= \Exp_{\mate}\left[\Exp_{A_{\mate},B_{\mate},\Alg}\left[\sum_{t=1}^{\tau}\|\matu_t - \Kst \matx_t\|^2\right]\right].
  \end{align*}
  Therefore, we conclude
  \begin{align*}
  \Exp\left[\sum_{\vecq} |\mate_{\vecq} - \matehat_{\vecq}|\right] \ge \frac{nm}{2}\sum_{\vecq}\left(1 - \sqrt{\frac{ \epsilon^2 }{n}\Exp_{\mate}\left[\Exp_{A_{\mate},B_{\mate},\Alg}\left[\sum_{t=1}^{\tau}\|\matu_t - \Kst \matx_t\|^2\right]\right]}\right).
  \end{align*}
  This concludes the proof of the proposition. \qed

  \subsubsection{Proof of Lemma~\ref{lem:KL}}
  By convexity of $\KL$ and Jensen's inequality, one can see that the $\KL$ under a randomized algorithm $\Alg_{\mathrm{rand}}$ is upper bounded by the largest $\KL$ divergence attained by one of the deterministic algorithms corresponding to a realization of its random seeds. Hence, we may assume without loss of generality that $\Alg$ is deterministic.

  By first conditioning the performance of $\Alg$ on its random seed, then integrating the $\KL$ combu
  Note that by  we may assume that $\Alg$ is deterministic. Let $\calF_{t-1}$ denote the filtration generated by $(\matx_{1:t-1},\matu_{1:t-1})$. 
  \begin{align*}
  \KL(\Pr_{0},\Pr_{1}) =\sum_{t=1}^{\tau} \Exp_{A(\Delta^{(0)}),B(\Delta^{(0)}),\Alg}[\KL( \Pr_{0}(\matx_t,\matu_t \mid \calF_{t-1}),\Pr_{\Delta_2,T}(\matx_t,\matu_t \mid \calF_{t-1})],
  \end{align*}
  where $\Pr_{0}(\matx_t,\matu_t \mid \calF_{t-1})$ denotes the conditional probability law.  Note that $\matu_t$ is deterministic given $\calF_{t-1}$. Moreover, $\matx_t \mid \calF_{t-1}$ has the distribution of $\calN( (A-\Delta^{(i)}\Kst)\matx_t + (B+\Delta^{(i)})\matu_t, I)$ under $\Pr_{i}(\cdot \mid \calF_{t-1})$. Hence, using the standard formula for Gaussian KL, 
  \begin{align*}
  &\KL( \Pr_{i}(\matx_t,\matu_t \mid \calF_{t-1}),\Pr_{i}(\matx_t,\matu_t \mid \calF_{t-1})) \\
  &= \frac{1}{2}\|(A-\Delta^{(0)}\Kst)\matx_t + (B+\Delta^{(0)}) - (A-\Delta^{(1)}\Kst)\matx_t + (B+\Delta^{(1)}))\|_2^2\\
  &= \frac{1}{2}\|(\Delta^{(0)} - \Delta^{0})(\matu_t - \Kst \matx_t)\|_2^2\\
  &= \frac{1}{2}\tr((\Delta^{(0)} - \Delta^{1})^\top(\matu_t - \Kst \matx_t)(\matu_t - \Kst \matx_t)^\top(\Delta^{(0)} - \Delta^{1})).
  \end{align*}
  The lemma now follows from summing from $t=1,\dots,\tau$ and taking expectations.

\subsection{Proof of Lemma~\ref{lem:recover_packing}\label{ssec:lem_recover_packing}}
  We have $\I(e_{i,j} \ne \ehat_{i,j}(\Khat)) = \I( e_{i,j}\ehat_{i,j}(\Khat) \ne 1) = \I( e_{i,j} \uvec_i^\top (\Khat - \Kst) \vvec_j \le 0)$. Define the Taylor approximation error matrix $\Delta_{2,e} := \Kst - (\Ru + \Bst^\top\Pst\Bst)^{-1}(\Delta_e \Aclst \Pst) - K_e$. We then have
  \begin{align*}
  e_{i,j} \uvec_i^\top (\Khat - \Kst) \vvec_j &\ge  e_{i,j} \uvec_i^\top (K_e - \Kst) \vvec_j  - |\uvec_i^\top (\Khat - \Kst) \vvec_j|\\
  &\ge  e_{i,j} \uvec_i^\top (\Ru + \Bst^\top\Pst\Bst)^{-1}(\Delta_e \Aclst \Pst) \vvec_j   -  |\uvec_i^\top \Delta_{2,e}\vvec_j| -   |\uvec_i^\top (\Khat - \Kst) \vvec_j|\\
  &=  e_{i,j} \underbrace{\frac{\sigma_j(\Aclst \Pst)}{\sigma_i(\Ru + \Bst^\top\Pst\Bst)}}_{\le \nu_m} \uvec_i^\top \Delta_e \vvec_j   -  \left(|\uvec_i^\top \Delta_{2,e}\vvec_j| +   |\uvec_i^\top (\Khat - \Kst) \vvec_j|\right),
  \end{align*}
  where we use the definition of $\uvec_i$ and $\vvec_j$, as less as $\sigma_j(\Aclst \Pst) \ge \sigma_j(\Aclst)$ since $\Pst \succeq I$. Since $\{\uvec_{i'}\}$ and $\{\vvec_{j'}\}$ form an orthornomal basis, we have $\uvec_i^\top \Delta_e \vvec_j  = \uvec_i^\top \sum_{i'=1}^n\sum_{j'=1}^m(\epspack e_{i',j'}\uvec_{i'}\vvec_{j'}^\top) \vvec_j  = \epspack e_{i,j} $. Hence,
  \begin{align*}
  e_{i,j} \uvec_i^\top (\Khat - \Kst) \vvec_j &\ge \nu_m\epspack  - \left(|\uvec_i^\top \Delta_{2,e}\vvec_j| +   |\uvec_i^\top (\Khat - \Kst) \vvec_j|\right).
  \end{align*}
It follows that for any $u \in (0,1)$,
  \begin{align*}
  \I\left( e_{i,j} \uvec_i^\top (\Khat - \Kst) \vvec_j \le 0\right) &\le 
  \I\left(|\uvec_i^\top (\Khat - \Kst) \vvec_j| \ge \sqrt{u}\nu_m \epspack\right )  +  \I\left(|\uvec_i^\top \Delta_{2,e} \vvec_j| \ge (1-\sqrt{u})\nu_m \epspack\right ) \\
  &\le \frac{|\uvec_i^\top (\Khat - \Kst) \vvec_j|^2}{u\nu_m^2 \epspack^2} +\frac{|\uvec_i^\top \Delta_{2,e} \vvec_j|}{(1-\sqrt{u})^2\nu_m^2 \epspack^2}.
  \end{align*}
  Since $\uvec_i,\vvec_j$ form an orthonormal basis, we have
  \begin{align*}
  \dham(e_{i,j},\ehat_{i,j}(\Khat)) &= \sum_{i=1}^n\sum_{j=1}^m\I\left( e_{i,j} \uvec_i^\top (\Khat - \Kst) \vvec_j \le 0\right) \\
  &\le \frac{\fronorm{\Khat - K}}{u\nu_m^2\epspack^2} + \frac{\fronorm{\Delta_{2,e}}^2}{(1-\sqrt{u})^2\nu_m^2\epspack^2}.
  \end{align*}
  Finally, since $\fronorm{\Delta_{2,e}}^2 \le (nm)^2 \epspack^4 \frakp_2(\opnorm{\Pst})^2$ by Lemma~\ref{lem:first_order_approx_quality}, we have that for $u = 1/\sqrt{2}$ and for $\epspack^2 \le  \frac{1}{20nm}\frakp_2(\opnorm{\Pst}) \le \frac{1}{nm}(1-1/\sqrt{2})\sqrt{20}/\frakp_2(\opnorm{\Pst}) $ that the above is at most
  \begin{align*}
  \dham(e_{i,j},\ehat_{i,j}(\Khat)) \le  \frac{2\fronorm{\Khat - K}}{\nu_m^2\epspack^2} - \frac{nm}{20}.
  \end{align*} 
  \qed

\subsection{Proof of Lemma~\ref{lem:deviation_from_Kst}\label{ssec:lem:deviation_from_Kst}}

Introduce the shorthand $\Kerr_{e} := \Kerr_{T/2}[\pi;A_e,B_e]$. We then have

\begin{align}
\Exp_{\mate}\Ksterr_{\mate}[\pi] &=\Exp_{\mate}\Exp_{A_{\mate},B_{\mate},\pi}\left[\sum_{t=1}^{T/2}\|\matx_t - \Kst \matu_t\|^2\right] \nonumber\\
&\le 2\Exp_{\mate}\left[\Exp_{A_{\mate},B_{\mate},\pi}\left[\sum_{t=1}^{T/2}\|\matx_t - \Kinfmate \matu_t\|^2 +\|(\Kinfmate  - \Kst)\matx_t \|^2\right]\right]\nonumber\\
&= 2\Exp_{\mate}\Kerr_{\mate}[\pi]+2\Exp_{\mate}\tr\left((\Kinfmate  - \Kst)^\top\Exp_{A_{\mate},B_{\mate},\pi}\left[\sum_{t=1}^{T/2}\matx_t \matx_t^\top\right](\Kinfmate  - \Kstinf)\right)\nonumber\\
&\le 2\Exp_{\mate}\Kerr_{\mate}[\pi] + 2\left(\max_{e}\|\Ke  - \Kstinf\|_{\fro}^2\right) \cdot \Exp_{\mate}\left\| \Exp_{A_{\mate},B_{\mate},\pi}\left[\sum_{t=1}^{T/2}\matx_t \matx_t^\top\right]\right\|_{\op}\nonumber\\
&\le 2\Exp_{\mate}\Kerr_{\mate}[\pi] + 4nm\opnorm{\Pst}^3\epspack^2  \cdot \Exp_{\mate}\left\| \Exp_{A_{\mate},B_{\mate},\pi}\left[\sum_{t=1}^{T/2}\matx_t \matx_t^\top\right]\right\|_{\op} \label{eq:kerr_kerrst_compare},
\end{align} 
where the last inequality uses Lemma~\ref{lem:first_order_approx_quality}. 

\begin{lem}\label{lem:op_norm_bound} Suppose $\epsilon$ is sufficiently small. Given matrices $A_e,B_e$ and optimal controller $\Ke$,
\begin{align*}
\left\| \Exp_{A_{e},B_{e},\pi}\left[\sum_{t=1}^{T/2}\matx_t \matx_t^\top\right]\right\|_{\op} &\le (3/2)T\opnorm{P_e} + 2J_e\opnorm{B_e}^2.\Kerr_e[\pi] \\
&\le 2T\opnorm{\Pst} + 3\Jst\Mbarst^2.\Kerr_e[\pi],
\end{align*}
where the last inequality uses Lemma~\ref{lem:first_order_approx_quality}.
\end{lem}
In particular, note that by Assumption~\ref{asm:uniform_correctness} and Lemma~\ref{lem:Kerr_lem}, we have the bound 
\begin{align*}
\Exp_{\mate}[\Kerr_{\mate}[\pi]] \le 2\Exp_{\mate}\SimpleRegret_{\mate}[\pi] + \epserr \le 2\epserr T \le \frac{T}{3\dimx\Mbarst^3}.
\end{align*}
Then, noting $\Jst \le \dimx\opnorm{\Pst}$, we can bound $\Exp_{\mate}\left\| \Exp_{A_{\mate},B_{\mate},\pi}\left[\sum_{t=1}^{T/2}\matx_t \matx_t^\top\right]\right\|_{\op} \le 3T\opnorm{\Pst}$. Combining with Eq.~\eqref{eq:kerr_kerrst_compare}, we have
\begin{align*}
\Exp_{\mate}\Ksterr_{\mate}[\pi] \le 2\Exp_{\mate}\Kerr_{\mate}[\pi] + 4nmT\opnorm{\Pst}^4\epspack^2. 
\end{align*}
\qed

\subsubsection{Proof of \Cref{lem:op_norm_bound}}
  Let $\matx_t$ denote the sequence induced by playing the algorithm $\pi$. Recalling the notation $\matdel_t = \matu_t - B_e \Ke \matx_{t}$, we then have
  \begin{align}
  \matx_t = A_e \matx_{t-1} + \matu_t + \matw_t = (A_e + B_e \Ke)\matx_{t-1} + B_e\matdel_t + \matw_t. \label{eq:matx_t_del_seq}
  \end{align}
  We further define the comparison sequence
  \begin{align}
  \matxbar_t := (A_e + B_e \Ke)\matxbar_{t-1} + \matw_t \label{eq:matxbar_seq}
  \end{align}
 in which we play the optimal infinite-horizon inputs for $(A_e,B_e)$. As shorthand, let $\Exp_e[\cdot] :=
        \Exp_{A_e,B_e,\pi}[\cdot]$, and recall that  $\Kerr_{e} := \Kerr_{T/2}[\pi;A_e,B_e]$. We can bound the desired operator
        norm of the algorithms 
  \begin{align*}
  \left\|\Exp_e\left[\sum_{t=1}^{T/2}\matx_t\matx_t^\top\right]\right\|_{\op} \le \left\|\Exp_e\left[\sum_{t=1}^{T/2}\matxbar_t\matxbar_t^\top\right]\right\|_{\op}  + \left\|\Exp_e\left[\sum_{t=1}^{T/2}\matxbar_t\matxbar_t^\top - \matx_t\matx_t^\top\right]\right\|_{\op}.
  \end{align*}
  It therefore suffices to establish the bounds
  \begin{align}
  &\left\|\Exp_e\left[\sum_{t=1}^{T/2}\matxbar_t\matxbar_t^\top\right]\right\|_{\op} \le T \opnorm{P_e}\label{eq:main_expectation_op} \\
  &\left\|\Exp_e\left[\sum_{t=1}^{T/2}\matxbar_t\matxbar_t^\top - \matx_t\matx_t^\top\right]\right\|_{\op} \le \frac{1}{2}T\opnorm{P_e} + 2J_e \opnorm{B_e}^2 \Kerr_e. \label{eq:error_expectation_op}
  \end{align}
  Let us first prove \Cref{eq:main_expectation_op}. We can compute
  \begin{align*}
  \left\|\Exp_e\left[\sum_{t=1}^{T/2}\matxbar_t\matxbar_t^\top\right]\right\| &\le \left\|\sum_{t=1}^{T/2}\sum_{s=0}^{-1} (A_e + B_e \Ke)^{s}\left((A_e + B_e \Ke)^{s}\right)^\top\right\|_{\op}\\
  &\le \frac{T}{2}\opnorm{\dlyap((A_e + B_e \Ke)^\top, I)} =\frac{T}{2}\opnorm{\dlyap((A_e + B_e \Ke)^, I)} \le \frac{T}{2}\opnorm{P_e},
  \end{align*}
  where the last two steps are by \Cref{lem:closed_loop_dlyap}.

  Next, we prove \Cref{eq:error_expectation_op}. By Jensen's inequality, the triangle inequality, and Cauchy-Schwarz, we can bound
  \begin{align*}
  \left\|\Exp_e\left[\sum_{t=1}^{T/2}\matxbar_t\matxbar_t^\top - \matx_t\matx_t^\top\right]\right\|_{\op} &\le \Exp_e\left[\sum_{t=1}^T\|\matxbar_t\matxbar_t^\top - \matx_t\matx_t^\top\|_{\op}\right]\\
  &\le \Exp_e\left[\sum_{t=1}^{T/2}2\|\matxbar_t-\matx_t\|\|\matxbar_t\| + \|\matxbar_t-\matx_t\|^2\right]\\
  &\le 2\sqrt{\Exp_e\left[\sum_{t=1}^{T/2}\|\matxbar_t\|^2\right]}\sqrt{\Exp_e\left[\sum_{t=1}^{T/2}\|\matxbar_t-\matx_t\|^2\right]} + \Exp_e\left[\sum_{t=1}^{T/2}\|\matxbar_t-\matx_t\|^2\right].
  \end{align*}
   From \Cref{eq:matx_t_del_seq,eq:matxbar_seq}, we have that 
  \begin{align*}
  \matxbar_t - \matx_t &= (A_e + B_e \Ke)\matxbar_{t-1} + \matw_t  - \left((A_e + B_e \Ke)\matx_{t-1} + B_e\matdel_t + \matw_t\right)\\
   &= (A_e + B_e \Ke)(\matxbar_{t-1} -\matx_{t-1}) - B_e\matdel_t\\
   &= -\sum_{s=1}^{t}(A_e + B_e \Ke)^{t-s}B_e\matdel_s.
  \end{align*}
  Therefore, we have that
  \begin{align*}
  \sum_{t=1}^{T/2}\|\matxbar_t - \matx_t\|_2^2 &\le \sum_{t=1}^{T/2}\sum_{s=1}^{t}\|A_e + B_e \Ke)^{t-s}B_e\matdel_s\|_2^2\\
  &\le\sum_{t=1}^{T/2}\matdel_t^\top \left(B_e^\top \sum_{s=0}^{\infty}(A_e + B_e \Ke)^{s\top}(A_e + B_e \Ke)^{s}\right) B_e \matdel_t^\top\\
  &= \sum_{t=1}^{T/2}\matdel_t^\top (B_e^\top  \dlyap(A_e + B_e \Ke ,I)B_e \matdel_t\\
  &\le \opnorm{B_e}^2\opnorm{P_e}\sum_{t=1}^{T/2}\|\matdel_t\|_2^2,
  \end{align*}
  where we use Lemma~\ref{lem:closed_loop_dlyap} in the last inequality. Taking expectations, we have
  \begin{align*}
  \sum_{t=1}^{T/2}\|\matxbar_t - \matx_t\|_2^2 \le \opnorm{B_e}^2\opnorm{P_e}\Kerr_{e}.
  \end{align*}
  This yields 
  \begin{align*}
  \left\|\Exp_e\left[\sum_{t=1}^{T/2}\matxbar_t\matxbar_t^\top - \matx_t\matx_t^\top\right]\right\|_{\op} 
  &\le  2\sqrt{\Exp_e\left[\sum_{t=1}^{T/2}\|\matxbar_t\|^2\right] \opnorm{B_e}^2\opnorm{P_e}} \Kerr_e + \opnorm{B_e}^2\opnorm{P_e}\Kerr_e\\
  &\le 2\sqrt{T/2 \cdot J_e \opnorm{B_e}^2\opnorm{P_e}\Kerr_e}  + \opnorm{B_e}^2\opnorm{P_e}\Kerr_e\\
  &= \sqrt{2T \cdot J_e \opnorm{B_e}^2\opnorm{P_e}\Kerr_e}  + \opnorm{B_e}^2\opnorm{P_e}\Kerr_e,
  \end{align*}
  where use the bound that $\sum_{t=1}^{T/2}\Exp[\|\matxbar_t\|^2] \le (T/2)J_e$ using similar arguments to Lemma~\ref{claim:xcost_ub}. 
  The above can be bounded by 
  \begin{align*}
  &\le \frac{1}{2}T\opnorm{P_e} + J_e \opnorm{B_e}^2 \Kerr_e + + \opnorm{B_e}^2\opnorm{P_e}\Kerr_e\\
  &\le\frac{1}{2} T\opnorm{P_e} + 2J_e \opnorm{B_e}^2 \Kerr_e,
  \end{align*}
  since $J_e = \tr(P_e)$.
\subsection{Additional Corollaries of \Cref{thm:main_lb}}
 \label{ssec:proof_of_Cors}
%
For scaled identity systems, we can remove the requirement that $\dimu \le (1 - \Omega(1))\dimu$.
\begin{cor}[Scaled Idenity System]\label{cor:scalar_indentity}Suppose that $\Ast = (1-\gamma) I$ for $\gamma \in (0,1)$, that $\Bst = U^\top$ where $U$ has orthonormal columns, and $\Rx,\Ru = I$. Then, for $T \ge c_1 \gamma^{-p}\left(\dimu\dimx \vee \frac{\dimx(1-\gamma)^{-4}}{ \dimu^2}\right) \vee c_1\dimx \log(1+\dimx \gamma^{-1})$, 
\begin{align*}
\calR_{\Ast,\Bst,T}\left(\sqrt{\dimu^2 \dimx/T}\right) \gtrsim \gamma^{-4}(1-\gamma)^2\sqrt{\dimu^2\dimx T}.
\end{align*}
\end{cor}
\begin{proof}[Proof of Corollary~\ref{cor:scalar_indentity}] By the
  same arguments as in \Cref{cor:small_input}, we have $\Mbarst \le 1$
  and $\opnorm{\Pst} \le \gamma^{-1}$. To conclude, let us lower bound
  $\sigma_{\min}(\Aclst) \gtrsim 1-\gamma$, which yields $\nu_{\dimx}
  \gtrsim \frac{1-\gamma}{\gamma}$. Reparameterize $a = (1-\gamma)$. Then for $\Ast = a I$ and $\Bst = U^\top$. Then, we can see that the $\DARE$  decouples into scalar along the columns of $U$ and their orthogonal complement. That is, if $p,k$ is the solution to
\begin{align}
(1-a^2)p &= -p^2 a^2(1 + p )^{-1} + 1,\quad k = -(1+p)^{-1}pa,\label{eq:scalar_dare}
\end{align}
then $\Aclst = (\Ast - k UU^\top ) = (a-k)UU^\top + a(I - UU^\top)$, so that 
\begin{align*}
\sigma_{\min}(\Aclst) \ge \min\{a,a-k\} = \min\{ a, \frac{a}{1+p}\} = \frac{a}{1+p}.
\end{align*} 
To conclude, we  solve~\eqref{eq:scalar_dare} and show that $p$ is bounded above by a universal consant. For scalar $(a,b)$, the solution to the DARE is
\begin{align*}
(1-a^2)p +p^2 (1-a^2)&= - p^2 a^2 + (1 +  p ) \quad \text{and thus} \quad -a^2p + p^2 -1  = 0.\\
\end{align*}
The solution $p$ is then given by
\begin{align*}
p = \frac{a^2 \pm \sqrt{a^4 + 4 }}{2} \le \frac{1 + \sqrt{5}}{2}, 
\end{align*}
as needed.
\end{proof}




%% file: appendix/appendix_upper_help.tex

\subsection{Proof of Lemma~\ref{lem:perturb_correct}  (Correctness of Perturbations) \label{ssec:lem_perturb_correct}}

  On the event $\Esafe$ of Lemma~\ref{lem:initial_phase}, the condition defining $\ksafe$ yields 
  \begin{align*}
    \left\| \begin{bmatrix} \Ahat_{\ksafe} - \Ast  \mid \Bhat_{\ksafe} - \Bst \end{bmatrix}\right\|_\op^2  \le \Conf_{\ksafe}  \le 1/3\Csafe(\Ahat_{\ksafe},\Bhat_{\ksafe}).
    \end{align*}
    By the continuity of $\Csafe$ given by Theorem~\ref{thm:continuity_of_safe set}, we then have that, for any $(\Ahat,\Bhat) \in \Bsafe$ ,
    \begin{align*}
     \left\| \begin{bmatrix} \Ahat - \Ast  \mid \Bhat - \Bst \end{bmatrix}\right\|_\op^2  \le \Csafe(\Ast,\Bst).
    \end{align*}
    In particular, the projection step ensures that the above holds for any $(\Ahat_k,\Bhat_k)$. Let us now go point by point.  Theorem~\ref{thm:main_perturb_app} then implies that
    \begin{enumerate}
     \item $P_k \preceq \frac{21}{20}\Pst$, and thus $J_k \lesssim \Jst$.
      \item $\Jfunc_k - \Jst = \Jfunc_{\Ast,\Bst}[\Kinf(\Ahat_k,\Bhat_k)] - \Jfunc^\star_{\Ast,\Bst} \le \Cest(\Ast,\Bst)\epsfro^2$.
    \item By Lemma~\ref{lem:helpful_norm_bounds}, 
    \begin{align*}\opnorm{\underbrace{\Kinf(\Ahat_k,\Bhat_k)}_{:= \Khat_k}}^2 \le \opnorm{\dlyap(\Ast + \Bst \Khat_k, \Rx + \Khat_k^\top \Ru \Khat_k)} = \opnorm{P_k} \le \frac{21}{20}\opnorm{\Pst}.
    \end{align*}
  \end{enumerate}
    The next two points of the lemma follow from Theorem~\ref{thm:hinf_perturbation}. 

    For the last point, recall that \[\sigmain^2 = \sqrt{\dimx}\opnorm{\Pinf(\Ahat_{\ksafe},\Bhat_{\ksafe})}^{9/2}\max\{1,\opnorm{\Bhat_{\ksafe}}\} \sqrt{\log \frac{\opnorm{\Pinf(\Ahat_{\ksafe},\Bhat_{\ksafe})}}{\delta}}.\] Since $\Conf_{\ksafe} \lesssim 1$, we have $\max\{1,\opnorm{\Bhat_{\ksafe}}\} \eqsim \Psibst$. Let us show $\opnorm{\Pst} \eqsim \opnorm{\Pinf(\Ahat_{\ksafe},\Bhat_{\ksafe})}$. By Lemma~\ref{lem:P_bounds_lowner}, $\Pinf(\Ahat_{\ksafe},\Bhat_{\ksafe}) \preceq P_{\ksafe}$, which is $\precsim \Pst$ by point 1 of this lemma. On the other hand, $\opnorm{\Pst} \lesssim \opnorm{\Pinf(\Ahat_{\ksafe},\Bhat_{\ksafe})}$ by Theorem~\ref{thm:continuity_of_safe set}.

\subsection{Proof of Main Regret Decomposition (Lemma~\ref{lem:safe_regret_decomp}\label{ssec:safe_regret_decomp})}
    We establish Lemma~\ref{lem:safe_regret_decomp} by establishing a more general regret decomposition for arbitrary feedback controllers $K$, noise-input variances $\sigma_u$, and control costs $R_1,R_u$. This will allow us to reuse the same computations for similar calculations in the initial estimation phase (Lemma~\ref{lem:initial_phase}), and for covariance matrix upper bounds as well.
    \begin{defn}[Control Evolution Distribution]\label{defn:control_law} We define the law $\calD(K,\sigma_u, x_1)$ to denote the law of the following dynamical system evolutation: $\matx_1 = x_1$, and for $t \ge 2$, the system evolves according to the following distribution:
    \begin{align}
    \matx_{t} = \Ast\matx_{t-1} + \matw_t, \quad \matu_t = K\matx_t + \sigma_u \matg_t,\label{eq:K_system}
    \end{align}
    where $\matw_t \sim \calN(0,I_{\dimx})$ and $\matg_t \sim \calN(0,I_{\dimu})$.
    \end{defn}

      We begin with the following characterization, proven in \Cref{app:proof_lem_cost_lem}, of the quadratic forms that will arise in our regret bounds. Note that we use arbitrary cost matrices $R_1,R_2\succeq{}0$.
        \begin{lem}\label{lem:cost_lem} Let $K$ be a stabilizing controller, and let $(\matx_t,\matu_t)_{t\ge 1}$ denote the linear dynamical system described by the evolution of the law $\calD(K,\sigma_u,x_1)$.  For cost matrices $R_1,R_2 \succeq 0$, define the random variable
        \begin{align*}
        \Cost(R_1,R_2;x_1,t,\sigma_u) := \sum_{s=1}^t \matx_t^\top R_1 \matx_t + \matu_t^\top R_2 \matu_t = \matgbar^{\top}\Lamgbar\,\matgbar + x_1^\top \Lamxone \,x_1 + 2\matgbar^{\top}\Lamcross x_1.
        \end{align*}
        Further, define $R_K = R_1 + K^\top R_2 K$, $A_K = \Ast + \Bst K$, $P_K = \dlyap(A_K,R_K)$, and $J_K := \tr(P_K)$. 
        \begin{enumerate}
        \item In expectation, we have
        \begin{align*}
        \Exp[ \Cost(R_1,R_2;x_1,t,\sigma_u)] &\le t J_K + 2\sigma_u^2 t\dimu \left(\opnorm{R_2} + \opnorm{\Bst}^2\opnorm{P_K}\right) + x_1^\top P_K x_1
        \end{align*}
        \item Set $\deff := \min\{\dimu,\rank(R_1) + \rank(R_2)\}$.
        With a probability $1 - \delta$, we have 
        \begin{align*}
        \Cost(R_1,R_2;x_1,t,\sigma_u) &\le t J_K + 2\sigma_u^2 \deff t\left(\opnorm{R_2}+ \opnorm{\Bst}^2\opnorm{P_K}\right)\\
        &+ \BigOh{\sqrt{d t \log \tfrac{1}{\delta}} + \log\tfrac{1}{\delta}}\left((1+\sigma_u^2\|\Bst\|_{\op}^2) \|R_K\|_{\op}\Hinf{A_{K}}^2 + \sigma_u^2 \|R_2\|_{\op}^2\right)   \\
        &+ 2x_1^\top P_K x_1.
        \end{align*}
        \item More crudely, we can also bound, with probability $1 - \delta$,
        \begin{align}
       \Cost(R_1,R_2;x_1,t,\sigma_u) &\lesssim t\log \frac{1}{\delta} \left(J_K + 2\sigma_u^2 \deff \left(\opnorm{R_2} + \opnorm{\Bst}^2\opnorm{P_K}\right)\right) + 2x_1^\top P_K x_1.\label{eq:crude_Q_form_bound}
       \end{align}
      \end{enumerate}
      \end{lem}

      \newcommand{\costnoisek}{\Cost_{\mathrm{noise},k}}
      \newcommand{\costconck}{\Cost_{\mathrm{conc},k}}
      
      Let us now apply the above lemma to our present setting. For $k \ge \ksafe$, define the terms
      \begin{align*}
      \costnoisek &:= \dimu \left(\opnorm{\Ru} + \opnorm{\Bst}^2\opnorm{P_k}\right)\\
      \costconck &:= \left((1+\sigma_k^2\|\Bst\|_{\op}^2) \|\Rx + \Khat_k^\top \Ru \Khat_k\|_{\op}\Hinf{\Aclk}^2\right) + \sigma_k^2\opnorm{\Ru}.
      \end{align*}
      By Lemma~\ref{lem:cost_lem} and the fact $\Jst \le J_k$, 
      \begin{align*}
      \sum_{t=\tau_{\ksafe}}^{T}(\matx_t^\top \Rx \matx_t + \matu_t^\top \Ru \matu_t - \Jst) &\lesssim \sum_{k=\ksafe}^{\kfin} \tau_k(J_k - \Jst) + \tau_k \sigma^2_k \costnoisek \\
      &+\sum_{k=\ksafe}^{\kfin} (\sqrt{\tau_k d \log(1/\delta)} + \log(1/\delta) \costconck  + \sum_{k=\ksafe}^{\kfin}\matx_{\tau_k}^\top P_k \matx_{\tau_k}.
      \end{align*}
      Let us first bound the $\costnoisek$-terms. Since $1 \le \opnorm{P_k} \lesssim \opnorm{\Pst}$ on event $\Esafe$ (Lemma~\ref{lem:perturb_correct}) and $\opnorm{\Ru} = 1$, we have
      \begin{align*}
      \costnoisek &\le \dimu (\opnorm{\Ru} +\opnorm{\Bst}^2\opnorm{P_k}) \lesssim  \dimu \Psibst^2\opnorm{\Pst}.
      \end{align*}
      Since $\sigma_k^2 \le \sigmain^2 \tau_{k}^{-1/2}$ and $\opnorm{P_k} \lesssim \opnorm{\Pst}$, we then obtain
      \begin{align*}
      \sum_{k=\ksafe}^{\kfin}  \tau_k \sigma^2_k \costnoisek \lesssim  \sqrt{T}\dimu  \sigmain^2 \Psibst^2\opnorm{\Pst}.
      \end{align*}
      Next, let us bound
      \begin{align*}
      \costconck := \left((1+\sigma_k^2\|\Bst\|_{\op}^2) \|\Rx + \Khat_k^\top \Ru \Khat_k\|_{\op}\Hinf{\Aclk}^2\right) + \sigma_k^2 \opnorm{\Ru}.
      \end{align*}
      Observe that $\Rx + \Khat_k^\top \Ru \Khat_k \preceq \dlyap[\Aclk, \Rx + \Khat_k^\top \Ru \Khat_k] = P_k$. On the good event $\Esafe$, we have $\opnorm{P_k} \lesssim \opnorm{\Pst}$, $\Hinf{\Aclk} \lesssim \Hinf{\Aclst} \le \opnorm{\Pst}^{3/2}$ (Lemma~\ref{lem:perturb_correct}),  and  by definition. $\|\Bst\|_{\op}^2 \le\Psibst^2$. Thus, the above is at most (again taking $\Ru = I$)
      \begin{align*}
      \costconck \lesssim   \opnorm{\Pst}^4  + \sigma_k^2 \left(\opnorm{\Ru} + \Psibst^2 \opnorm{\Pst}^4\right) \le \opnorm{\Pst}^4\left( 1+ \Psibst^2 \sigma_k^2\right) .
      \end{align*}
      Therefore, 
      \begin{align*}
      \sum_{k=\ksafe}^{\kfin} (\sqrt{\tau_k d \log(1/\delta)} + \log(1/\delta)) \costconck &\lesssim \sqrt{T d \log(1/\delta)}\opnorm{\Pst}^4 + \log(T)\log(1/\delta)\opnorm{\Pst}^4\\
      &\qquad+\sigmain^2 \log(T)\log(1/\delta)\sqrt{d}\Psibst^2 \opnorm{\Pst}^4.\\
      &\le \sqrt{T d \log(1/\delta)}\opnorm{\Pst}^4  + \log^2 \frac{1}{\delta}(1+\sqrt{d}\sigmain^2\Psibst^2) \opnorm{\Pst}^4.
      \end{align*}
      where we use $\log(T) \le \log(1/\delta)$. 
      Finally, we have the bound
      \begin{align*}
      \sum_{k=\ksafe}^{\kfin}\matx_{\tau_k}^\top P_k \matx_{\tau_k} &\lesssim \log T \max_{k \le \log T} \matx_{\tau_k}^\top P_k \matx_{\tau_k}\\
      &\le \log T \max_{k \le \log T} \|\matx_{\tau_k}\|_2^2 \opnorm{P_k}\\
       &\lesssim \log T \max_{k \le \log T} \|\matx_{\tau_k}\|_2^2 \opnorm{\Pst}.
      \end{align*}
      Hence, putting things together, we have
      \begin{align*}
      \sum_{t=\tau_{\ksafe}}^{T}(\matx_t^\top \Rx \matx_t + \matu_t^\top \Ru \matu_t - \Jst) &\lesssim \sum_{k=\ksafe}^{\kfin} \tau_k(J_k - \Jst) + \log T \max_{k \le \log_T}\|\matx_{\tau_k}\|_2^2\\
      &\quad+\sqrt{T}\left(\dimu  \sigmain^2 \Psibst^2\opnorm{\Pst}) + \sqrt{d \log(1/\delta)}\opnorm{\Pst}^4\right)\\
      &+ \log^2 \frac{1}{\delta}(1+\sqrt{d}\sigmain^2\Psibst^2) \opnorm{\Pst}^4.
      \end{align*}
      Reparameterizing $\delta \leftarrow \frac{\delta}{6 T}$ and taking a union bound preserves the above inequality up to constants (since $\log T \le \log \frac{1}{\delta}$), and reduces the failure probability across all episodes to $\delta/6$.

\subsection{Bounding the States: Lemma~\ref{lem:matx_bound}\label{ssec:matx_bound}}

  Since $\dlyap[\Aclst] := \dlyap(\Aclst,I) \succeq I$, we bound the right hand side of 
  \begin{align*}
  \|\matx_{\tau_k}\|_{2}^2 \le \matx_{\tau_k}^\top \dlyap[\Aclst]\matx_{\tau_k}.
  \end{align*}
  To bound the right handside, we manipulate the following quantity.
  \begin{align*}
  \left\|\dlyap[\Aclst]^{1/2}\prod_{i=j}^{k-1} A_{\mathrm{cl},i}\right\|_{\op} = \sqrt{ \left\|\left(\prod_{i=j}^{k-1} A_{\mathrm{cl},i}^{\tau_i}\right)^\top \dlyap[\Aclst] \left(\prod_{i=j}^{k-1} A_{\mathrm{cl},i}^{\tau_i}\right)\right\|_{\op}}.
  \end{align*}
  By Lemma~\ref{lem:perturb_correct}, we have that for all $i \ge \ksafe$ on event $\Esafe$,
  \begin{align*} 
  A_{\mathrm{cl},i}^\top\dlyap[\Aclst]A_{\mathrm{cl},i} \preceq (1 - \frac{1}{2\opnorm{\dlyap[\Aclst]}})\dlyap[\Aclst]
  \end{align*} 
  This yields that 
  \begin{align*}
  \left\|\dlyap[\Aclst]^{1/2}\prod_{i=j}^{k-1} A_{\mathrm{cl},i}\right\|_{\op} &\le \sqrt{ \prn*{\prod_{i=j}^{k-1} \prn*{1 - \frac{1}{2\opnorm{\dlyap[\Aclst]}}}^{\tau_i}}^2 \opnorm{\dlyap[\Aclst]} }\\
  &=  \left(1 - \frac{1}{2\opnorm{\dlyap[\Aclst]}}\right)^{\sum_{i=j}^{k-1}\tau_i} \sqrt{\opnorm{\dlyap[\Aclst]} }\\
  &=  \left(1 - \frac{1}{2\opnorm{\dlyap[\Aclst]}}\right)^{\tau_{k-1}} \sqrt{\opnorm{\dlyap[\Aclst]} }.
  \end{align*}

  For $k > \ksafe$, define the vector $\mate_k := \matx_{\tau_k} - A_{\mathrm{cl},k-1}^{\tau_{k-1}}\matx_{\tau_{k-1}}$. Now, we can write
  \begin{align*}
  \matx_{\tau_{k}} &= \mate_k + A_{\mathrm{cl},k-1}^{\tau_{k-1}}\matx_{\tau_{k-1}}\\
  &= \mate_k + A_{\mathrm{cl},k-1}^{\tau_{k-1}}\left(\mate_{k-1} + A_{\mathrm{cl},k-2}^{\tau_{k-2}}\matx_{\tau_{k-2}}\right)\\
  &= \sum_{j = \ksafe+1}^k \left(\prod_{i=j}^{k-1} A_{\mathrm{cl},i}^{\tau_i}\right) \mate_j + \left(\prod_{i=\ksafe}^{k-1} A_{\mathrm{cl},i}^{\tau_i}\right)\matx_{\tau_{\ksafe}}.
  \end{align*}
  Thus, 
   \begin{align*}
  &\|\dlyap[\Aclst]^{1/2}\matx_{\tau_{k}} \|_2 \lesssim \max_{\ksafe \le j \le k}\|\mate_j\|_{2} \sqrt{\opnorm{\dlyap[\Aclst]} } \left(1 + k \left(1 - \frac{1}{2\opnorm{\dlyap[\Aclst]}}\right)^{\tau_{k-1}}\right)\\
   &\qquad + \|\matx_{\tau_{\ksafe}}\|_{2} \sqrt{\opnorm{\dlyap[\Aclst]} }.\\
   &\qquad \lesssim\max_{\ksafe \le j \le k} \|\mate_j\|_{2} \sqrt{\opnorm{\dlyap[\Aclst]}}  \|\dlyap[\Aclst]\|_{\op} + \|\matx_{\tau_{\ksafe}}\|_{2} \sqrt{\opnorm{\dlyap[\Aclst]}} \\
   &\qquad =\max_{\ksafe \le j \le k} \|\mate_j\|_{2}  \|\dlyap[\Aclst]\|_{\op}^{3/2} + \|\matx_{\tau_{\ksafe}}\|_{2} \sqrt{\opnorm{\dlyap[\Aclst]}} .
  \end{align*}
  where above we have used the inequality $\max_{k \ge 1} k(1 - \rho)^k \lesssim \frac{1}{\rho}$, together with $\tau_k - 1 \ge k$. 
  We begin with the following technical claim:
  \begin{lem}\label{lem:xt_bound} Let $\matx_t$ denote the $t$-th iterate from the control law $\calD(K,x_1,\sigma_u)$ (\Cref{defn:control_law}). Then,  with probability at least $1-\delta$,
    \begin{align*}
    \|\matx_t - A_K^{t-1}x_1\| \le \BigOh{\sqrt{J_K ( 1+ \sigma_{u}^2\|\Bst\|_2^2) \log \frac{1}{\delta}}}.
    \end{align*}
    \end{lem}
  Since $\delta < 1/T$, $\sigma^2_k \le 1$,  and $\Jfunc_k \lesssim \Jst$, a union bound and reparametrization of $\delta$ implies that, the following holds with probability $1 - \delta/8$:
  \begin{align*}
  &\forall k \le \ksafe: \tau_k \le T, \quad \|\mate_k\|_2 \lesssim \alpha_1 := \sqrt{\Jbar \Psibst^2 \log \frac{1}{\delta} }
  &\|\matx_{\tau_{\ksafe}}\|_{2} \le \alpha_0 \sqrt{\Jfunc_0 \Psibst^2 \log \frac{1}{\delta}} 
  \end{align*}
  Concluding, this implies the crude bound
  \begin{align*}
  \|\matx_{\tau_{k}} \|_2 \le  \sqrt{\matx_{\tau_k}^\top \dlyap[\Aclst]\matx_{\tau_k}} &\lesssim  \sqrt{\opnorm{\dlyap[\Aclst]}}(\alpha_0 + \alpha_1 \opnorm{\opnorm{\dlyap[\Aclst]}}) \le  \opnorm{\Pst}^{3/2} (\alpha_0 + \alpha_1) \\
  &\lesssim \sqrt{\Psibst(\Jfunc_0 + \Jst)\log(1/\delta)}\opnorm{\Pst}^{3/2}\\
  &\lesssim \sqrt{\Psibst \Jfunc_0\log(1/\delta)}\opnorm{\Pst}^{3/2}.
  \end{align*}
    \qed

\subsection{Proof of Estimation Bound (Lemma~\ref{lem:estimation_bound}) \label{ssec:lem:estimation_bound}}

  Our strategy is two invoke the two-scale estimation bound \Cref{lem:two_scale_self_normalized_bound}, which we restate here
  \lemtwoscale*

  To instantiate the lemma, consider the following orthogonal projection operators:
  \begin{restatable}[Round-wise projections]{defn}{roundwiseproj} Recall $d = \dimx + \dimu$. Given $v \in \R^d$, let $v = (v^x,v^u)$ denote its decomposition along the $x$ and $u$ directions. For a given round $k \ge \ksafe$, let 
  \begin{itemize}
    \item $\calV_k := \{v \in \R^{d}: v^x + \Khat_k v^u = 0\}$, and let $\calV_k^{\perp}$ denotes it orthogonal complement.
    \item $\ProjMat_k$ denote the orthogonal projection onto $\calV_k$, and let $\ProjMat_{k}^{\perp} := (I-\ProjMat_k)$ denote the projection on $\calV_{k}^{\perp}$.
    \item  Let  $v_{k,1},\dots,v_{k,\dimu}$ denote an eigenbasis of $\ProjMat_k$, and $v_{k,\dimu +1},\dots,v_{k,d}$ denote an eigenbasis of $\ProjMat_k^\perp$. In particular, note that for any $0 < c_1 < c_2 $, $v_{k,1},\dots,v_{k,d}$ is an eigenbasis (in descending order) for $c_1 \ProjMat_k + c_2 \ProjMat_k^\perp$.
  \end{itemize}
  \end{restatable}
  To apply \Cref{lem:two_scale_self_normalized_bound}, take 
  \begin{align}
  &m \gets \dimx, \quad p \gets \dimu, \quad d \gets \dimx + \dimu, \quad v_{i} \gets v_{i,k}, \quad \matLam \gets \matLam_k , \quad \ProjMat \gets \ProjMat_k \label{eq:subs1}
  \end{align}
  Hence, on the event
  \newcommand{\Ecov}{\mathcal{E}_{\mathrm{cov}}}
  \begin{align}
  \Ecov(\nu,\alpha_1,\alpha_2,\lambda_1,\lambda_2) := \left\{\begin{matrix} \matLam_k \succeq \lambda_1 \ProjMat_k + \lambda_2 \ProjMat \\
 \max_{1 \le i \le \dimu} v_{i,k}^\top \matLam_k v_{i,k} \le \kappa_1 \lambda_1\\
  \max_{\dimu +1 \le i \le d} v_{i,k}^\top \matLam_k v_{i,k} \le \kappa_2 \lambda_2\\
\nrm*{\ProjMat_k \matLam_k (I - \ProjMat_k)}_{\op} \le \nu\\
  \end{matrix} \right\}, \quad \nu \le \sqrt{\lambda_1\lambda_2/2}, \label{eq:Ecov_event}
  \end{align}
  then, it holds that
  \begin{align}
  \|\Ast - \Ahat\|_{\fro}^2 + \|\Bst - \Bhat\|_{\fro}^2 \lesssim \frac{\dimx \dimu\kappa_1}{\lambda_1}\log(\kappa_1 d/\delta) + (\nu/\lambda_1)^2\frac{\dimx^2 \kappa_2}{\lambda_2}\log(\matkappa_1 d/\delta), \text{w.p. } 1 - \delta/24  \label{eq:ls_conclusion}
  \end{align}
  The first step in our bound will be to lower bound the relevant, centered covariances, thereby 
  \begin{restatable}[Round-wise covariance lower bound]{lem}{roundwisecov}\label{lem:step_covariance_lb_round_k} Let $k \ge \ksafe+1$, at let $t \in\{\tau_k,\dots,\tau_{k+1}-1\}$. Then, on $\Esafe$. If $\sigma_{k}^2$ satisfies $\sigma_{k}^2 \le \frac{1}{6.2\|\Pst\|_{\op}}$, we have that
  \begin{align*}
  \Exp\left[(\matz_t - \Exp[\matz_t \mid \calF_{t-1}])(\matz_t - \Exp[\matz_t \mid \calF_{t-1}])^{\top}\right] \succeq \Gamma_k := \frac{\sigma_{k}^2}{6.2\|\Pst\|_{\op}}\ProjMat_k + \frac{1}{2}\ProjMat_k^{\perp}.
  \end{align*}
\end{restatable}
See Section~\ref{sssec:step_covariance_lb_round_k} for the proof.
  We now convert the above bound into a L\"{o}wner lower bound, then conclude by giving an upper bound on $\tau_{\ksafe}$. We rely on the following guarantee $\matLam_k$. To state the bound, we introduce a following shorthand which allows us to abbreviate statements of the form ``$k$ such that $\tau_k$ is sufficiently large''.
  \begin{defn}\label{def:gtrsimst}
  We say that the condition $f(x) \gtrsimst g(x)$ is met if it holds that $f \ge Cg$ for a sufficiently large, but unspecified universal constant $C$. 
  \end{defn}
  Note that $\gtrsimst$ differs from $\gtrsim$ in the following respect: $f(x) \gtrsimst g(x)$ may require a large constant say $C = 1000$; in constract, $f(x) \gtrsim g(x)$ holds whenever $f(x) \ge c g(x)$ for any constant, even say $c = 1/1000$. We also define 
  \begin{align}
\tauls :=  d\max\{1, \tfrac{ \dimu}{\dimx}\}\left(\opnorm{\Pst}^3\calP_0 +  \opnorm{\Pst}^{11} \Psibst^6 \right)  \log \frac{d \opnorm{\Pst}}{\delta}
\end{align}

  With this setup in place,
  \begin{lem}\label{lem:covar_est_upper_lower}. The following bounds hold simultaneously with probability $1 - \delta/24$, if $\Ebound \cap \Esafe$ holds:
  \begin{enumerate}
    \item  For  $\tau_k \gtrsimst \sqrt{\log(d/\delta)}$, we have
    \begin{align*}
    \max_{i \in \{1,\dots,\dimu\}} v_{k,i}^\top \matLam_k v_{k,i} \lesssim \tau_k\sigma_k^2
    \end{align*}
    \item Suppose   that $\tau_k \ge \opnorm{\Pst}^3\Jfunc_0 \vee \Psibst^4 \sigmain^4$. Then,
    \begin{align*}
    \max_{i \in [d]} v_{k,i}^\top \matLam_k v_{k,i} \lesssim \tau_k \opnorm{\Pst} \log(d/\delta).
    \end{align*}
    \item If $\tau_k \gtrsimst \tauls$, then the above two conditionds hold,   $\sigma_k^2$ sastisfies the conditions of Lemma~\ref{lem:step_covariance_lb_round_k}, and $\matLam_k \succeq c \tau_k \Gamma_k$ for some universal constant $c > 0$.
   \end{enumerate}
  \end{lem}
  The proof is defered to \Cref{ssec:lem:initial_covariance_lb}. Next, we \mscomment{text and sign post} \Cref{sec:lem_proj_mat_bound}

\begin{lem}\label{lem:proj_mat_bound} Suppose that $\tau_{k} \ge \tauls$. Then, with probability $1 - \delta/8$, the following upper bound holds on the events $\Esafe \cap \Ebound$:
\begin{align*}
\nrm*{\ProjMat_k \matLam_k (I - \ProjMat_k)}_{\op} &\le \nu_k := C_{\nu}\tau_k \sigma_k^2
\end{align*}
where $C_{\nu}$ is a universal constant. 
\end{lem}
In light of the above lemmas, we see that the event $\Ecov$ in \Cref{eq:Ecov_event} holds with probability $1-2\delta/24$ with the substitutions  
\begin{align*}
&\kappa_1 \lesssim \|\Pst\|_{\op}, \quad \kappa_2 \lesssim \opnorm{\Pst} \log(d/\delta)\\
&\nu_k \gets C_{\nu} \tau_k \sigma_k^2, \quad \lambda_1 \gets \frac{\tau_k\sigma_{k}^2}{6.2\|\Pst\|_{\op}} \quad \lambda_2 = \frac{1}{2}\tau_k
\end{align*}
One can verify that for $\tau_k \gtrsim \tauls$, we have $\nu_k \le \sqrt{\lambda_1\lambda_2/2}$. Hence, applying \Cref{eq:ls_conclusion}, it holds with with probability $1 - \delta/8$ (on $\Esafe \cap \Ebound$), it holds that
\begin{align}
  &\|\Ast - \Ahat\|_{\fro}^2 + \|\Bst - \Bhat\|_{\fro}^2 \nonumber\\
  &\qquad \lesssim\frac{(\dimx \dimu)\|\Pst\|_{\op}^2}{\tau_k \sigma_k^2}\log(d\|\Pst\|_{\op} /\delta) + \|\Pst\|_{\op}^3\frac{\dimx^2 }{\tau_k}\log( d\|\Pst\|_{\op} /\delta)^2, \nonumber\\
  &\qquad \lesssim\frac{\dimx \dimu\|\Pst\|_{\op}^2}{\sigmain^2 \tau_k^{1/2}}\log(\tfrac{d\|\Pst\|_{\op}}{\delta}) + \|\Pst\|_{\op}^3\frac{\dimx^2 }{\tau_k}\log( \tfrac{d\|\Pst\|_{\op}}{\delta})^2,
  \end{align}
  where the last line uses $\tau_k \sigma_k^2 = \sigmain^2\sqrt{\tau_k}$ for $\tau_k \gtrsimst \tauls$ (see \Cref{lem:perturb_correct}) and $1/\delta \ge T \ge \tau_k \gtrsimst \tauls$. 
  \qed

  \subsubsection{Proof of \Cref{lem:step_covariance_lb_round_k}\label{sssec:step_covariance_lb_round_k}}
      \newcommand{\vperp}{v_{\perp}}
    \newcommand{\vpar}{v_{\parallel}}
Define
    \begin{align*}
    \Sigma_k = \begin{bmatrix} I_{\dimx}  \\
     \Khat_k^{\top}  \\
    \end{bmatrix}\begin{bmatrix} I_{\dimx}  & \Khat_k  \\
    \end{bmatrix} + \begin{bmatrix} 0  & 0 \\
    0 &   \sigma_{k}^2 I
    \end{bmatrix}
    &= \begin{bmatrix} I_{\dimx}  & \Khat_k  \\
    \Khat^{\top} & \Khat_k^\top \Khat_k +  \sigma_{k}^2 I
    \end{bmatrix}.
    \end{align*}
    We see that for $k \ge \boldknot + 1$ and $t \ge \tau_k$, we have that  $\matz_{t} \mid \calF_{t-1} \sim \calN(\matzbar_{t},\Sigma_{k})$, where $\matzbar_{t}$ and $\Sigma_{t}$ are $\calF_{t-1}$-measurable. Our goal will now be to lower bound $\Sigma_k \succsim \gamma_1 \ProjMat_k + \gamma_2 \ProjMat_k^{\perp}$. To this end, let $v \in \calS^{d-1}$, and write $v = \ProjMat_k v + \ProjMat_k^{\perp} v := \vpar + \vperp$. Observe then that
    \begin{align*}
    \left|\vperp^{\top}\Sigma_k \vpar\right| \le \left|\vperp^{\top}\begin{bmatrix} I_{\dimx}  \\
     \Khat_k^{\top}  \\
    \end{bmatrix}\begin{bmatrix} I_{\dimx}  & \Khat_k  \\
    \end{bmatrix} \vpar\right|
     + \left|\vperp^{\top}\begin{bmatrix} 0  & 0 \\
    0 &   \sigma_{k}^2 I
    \end{bmatrix}\vpar\right| = \left|\vperp^{\top}\begin{bmatrix} 0  & 0 \\
    0 &   \sigma_{k}^2 I
    \end{bmatrix}\vpar\right| \le \sigma_{k}^2\|\vperp\|\|\vpar\|,
    \end{align*}
    where we use that $\begin{bmatrix} I_{\dimx}  & \Khat_k  
    \end{bmatrix}\vpar = 0$. On the other hand, since $\vperp \in \mathsf{null}\left(\begin{bmatrix} I_{\dimx}  & \Khat_k  
    \end{bmatrix}\right)^{\perp}$, we have
    \begin{align*}
    \vperp^\top\Sigma_{k}\vperp &\ge \vperp^\top\begin{bmatrix} I_{\dimx}  \\
     \Khat_k^{\top}  \\
    \end{bmatrix}\begin{bmatrix} I_{\dimx}  & \Khat_k  \\
    \end{bmatrix}\vperp \\
    &= \|\begin{bmatrix} I_{\dimx}  & \Khat_k  
    \end{bmatrix} \vperp\|^2 \\
    &\ge \|\vperp\|^2 \sigma_{\dimx}\left(\begin{bmatrix} I_{\dimx}  & \Khat_k  
    \end{bmatrix} \right)^2 =  \|\vperp\|^2 \lambda_{\min}(I_{\dimx} + \Khat_k^\top \Khat_k) \ge  \|\vperp\|^2 .
    \end{align*}
    We can therefore bound, for any $\alpha > 0$,
    \begin{align*}
      v^{\top} \Sigma_k v &=  \vperp^\top\Sigma_k \vperp + 2\vperp^{\top}\Sigma_k \vpar + \vpar^{\top}\Sigma_k \vpar\\
    &\ge  \|\vperp\|^2 - 2\sigma_{k}^2\|\vperp\|\vpar\| + \lambda_{\min}(\Sigma_k)\|\vperp\|^2\\
    &\ge  \|\vperp\|^2 - \sigma_{k}^2(\alpha\|\vpar\|^2 + \frac{1}{\alpha} \|\vpar\|^2) + \lambda_{\min}(\Sigma_k)\|\vperp\|^2.
    \end{align*}
    Taking $\alpha = \lambda_{\min}(\Sigma_k)/2\sigma_{k}^2$, we have
    \begin{align*}
    v^{\top} \Sigma_k v  \ge\|\vperp\|^2\underbrace{(1 - \sigma_{k}^2 \cdot \frac{2\sigma_{k}^2}{\lambda_{\min}(\Sigma_k)})}_{:=\gamma_1} + \underbrace{\frac{1}{2}\lambda_{\min}(\Sigma_k)}_{:=\gamma_2}\|\vperp\|^2.
    \end{align*}
    Hence, we have show that, for $\gamma_1,\gamma_2$ defined in the above display,
    $\Sigma_k \succeq \gamma_1 \ProjMat^{\perp}_{k} + \gamma_2 \ProjMat_{k}$. Let us now lower bound each of these quantities. From \citet[Lemma F.6]{dean2018regret}, since $\|\Khat_k\|^2 \lesssim \opnorm{\Pst}$ (Lemma~\ref{lem:perturb_correct}), and $\opnorm{\Pst} \ge 1 \ge \sigma_k^2$, 
    \begin{align*}
    \lambda_{\min}(\Sigma_{k}) &\ge \sigma_{k}^2 \min\left\{\frac{1}{2},\frac{1}{2\|\Khat_k\|_{\op}^2 + \sigma_{k}^2}\right\} \ge \sigma_{k}^2\min\left\{\frac{1}{2},\frac{1}{2.1\|\Pst\|_{\op} + \sigma_{k}^2}\right\} \\
    &\ge \sigma_{k}^2\min\left\{\frac{1}{2},\frac{1}{3.1\|\Pst\|_{\op}}\right\} = \frac{\sigma_{k}^2}{6.2\|\Pst\|_{\op}}.
    \end{align*}
    Hence, for $\sigma_k^2 \le \frac{1}{6.2\|\Pst\|_{\op}}$, we have $\gamma_1 \ge \frac{1}{2}$, and $\gamma_2 \ge \frac{\sigma_{k}^2}{3.1\|\Pst\|_{\op}}.$

  \subsubsection{Proof of \Cref{lem:initial_covariance_lb}\label{ssec:lem:initial_covariance_lb}}
  \begin{proof} All union bounds will be absorbed into $\delta$ factors, as $\delta \le 1/T$ and $T \ge d$.
  We decompose $v_{k,i} = v_{k,i}^x +v_{k,i}^u$ along its $x$ and $u$ coordinate. It suffices to show that each bound holds individually with probability $1-\delta$ for a fixed $i$, and $k$, since the union bound over $k$ can be absorbed into the $\delta$ factor (as $\delta \le 1/T$), and dimension addressed by reparametrizing $\delta \leftarrow \delta/d$.

  \paragraph{Point 1:} For $i \in \{1,\dots,\dimu\}$, $v_{k,i}$ lies in the vector space $\calV_k$. Therefore
  \begin{align*}
  v_{k,i}^\top \matLam_k v_{k,i} &= \sum_{t= \tau_k}^{2\tau_k - 1} v_{k,i}^\top \begin{bmatrix}\matx_t \\ \matu_t\end{bmatrix}\begin{bmatrix}\matx_t \\ \matu_t\end{bmatrix}^\top v_{k,i}\\
  &= \sum_{t= \tau_k}^{2\tau_k - 1} v_{k,i}^\top \begin{bmatrix}\matx_t \\ \Khat_k \matx_t +\sigma_{k}\matg_t\end{bmatrix}\begin{bmatrix}\matx_t \\ \Khat_k \matx_t +\sigma_{k}\matg_t\end{bmatrix}^\top v_{k,i}\\
  &=\sum_{t= \tau_k}^{2\tau_k - 1} v_i^\top \begin{bmatrix} \\ \sigma_{k}\matg_t\end{bmatrix}\begin{bmatrix}\matx_t \\  \matx_t +\sigma_{k}\matg_t\end{bmatrix}^\top v_{k,i}\\
  &= \sigma_k^2 \sum_{t= \tau_k}^{2\tau_k - 1} \langle v_{k,i}^u, \matg_t\rangle^2 \sim \|v_{k,i}^u\|_2^2 \sigma_k^2  \cdot \chi^2(\tau_k).
  \end{align*}
  By standard $\chi^2$-concentration, the above is $\lesssim \tau_k \|v_{k,i}^u\|_2^2 \sigma_k^2 \le \tau_k\sigma_k^2$ for $\tau_k \ge \sqrt{\log(1/\delta)}$. 

  \paragraph{Point 2:} For arbitrary $i$, set $R_1 := v_{k,i}^x(v_{k,i}^x)^\top$and $R_2 :=  v_{k,i}^u(v_{k,u}^{u})^{\trn}$. 

  \begin{align*}
  v_{k,i}^\top \matLam_k v_{k,i} &= \sum_{t= \tau_k}^{2\tau_k - 1} v_{k,i}^\top \begin{bmatrix}\matx_t \\ \matu_t\end{bmatrix}\begin{bmatrix}\matx_t \\ \matu_t\end{bmatrix}^\top v_{k,i} \le 2\sum_{t= \tau_k}^{2\tau_k - 1} \matx_t^\top R_1 \matx_1 + \matu_t^\top R_2 \matu_t.
  \end{align*}
   Thus, Lemma~\ref{lem:cost_lem} ensures that, with probability $1 - \delta$, we have that for the matrix $P := \dlyap(\Ast +\Bst \Khat_k, R_1 + \Khat_k^\top R_2 \Khat_k)$,
  \begin{align*}
  v_{k,i}^\top \matLam_k v_{k,i} \lesssim \tau_k\log \frac{1}{\delta} \left(\tr(P) + 2\sigma_k^2 \deff \left(\opnorm{R_2} + \opnorm{\Bst}^2\opnorm{P}\right)\right) + \opnorm{P}\|\matx_{\tau_{k}}\|^2,
  \end{align*}
  where $\deff \le \rank(R_1) + \rank(R_2) = 2$. Since $R_1 \preceq I\preceq \Rx$ and $R_2 \preceq I =  \Ru$, we have $R_1 + \Khat_k^\top R_2 \Khat_k \preceq \Rx + \Khat_k^\top\Ru \Khat_k$, and thus (by Lemma~\ref{lem:closed_loop_dlyap}),  $P = \dlyap(\Ast +\Bst \Khat_k, R_1 + \Khat_k^\top R_2 \Khat_k) \preceq \dlyap(\Ast +\Bst \Khat_k, \Rx + \Khat_k^\top\Ru \Khat_k) = P_k$. Moerover Lemma~\ref{lem:perturb_correct}, we get $\opnorm{P_k} \lesssim \opnorm{\Pst}$. Moreover, since $P$ can be shown to have rank at most $2$, $\tr(P) \lesssim \opnorm{\Pst}$. Finally,  $\|x_{\tau_{k}}\|_2 \le \sqrt{\Jfunc_0\log(1/\delta)}\opnorm{\Pst}^{3/2}$ from Lemma~\ref{ssec:matx_bound}, 
  \begin{align*}
  v_{k,i}^\top \matLam_k v_{k,i} \lesssim \tau_k\log \frac{1}{\delta} \left(\opnorm{\Pst} + \sigma_k^2 \left(1 + \opnorm{\Bst}^2\opnorm{\Pst}\right)\right) + \Jfunc_0\log(1/\delta)\opnorm{\Pst}^4\\
  \lesssim \tau_k\log \frac{1}{\delta} \left(\opnorm{\Pst} + \sigma_k^2 \opnorm{\Pst}\Psibst^2 \right) + \Jfunc_0\log(1/\delta)\opnorm{\Pst}^4.
  \end{align*}
  In particular, if $\tau_k \ge \Jfunc_0 \opnorm{\Pst}^3$, and $ \sigma_k \le 1/\Psibst^2$ (for which it suffices $\tau_k \le \sigmain^4 \Psibst^4$), we have
  \begin{align*}
  v_{k,i}^\top \matLam_k v_{k,i} \lesssim \tau_k \opnorm{\Pst} \log(1/\delta).
  \end{align*}

  \paragraph{Point 3:} Suppose now that $\tau_k \ge \opnorm{\Pst}^3\Jfunc_0 \log(1/\delta) \vee \Psibst^4 \sigmain^4$. Then, by using the expectation bound statement of Lemma~\ref{lem:cost_lem}, and summing over inidices $i$, we have
  \begin{align*}
  \Exp[\tr(\matLam_k) \cap \Ebound \cap \Esafe] \lesssim d\tau_k \opnorm{\Pst}.
  \end{align*}
  Hence, by Lemma~\ref{lem:covariance_lb}, if 
  \begin{align*}
  \tau_k \gtrsimst d \log \left\{\frac{d\opnorm{\Pst}}{\sigma_{\min}(\Gamma_k)}\right\}.
  \end{align*}
  then with probability $1-e^{-\tau_k/d}$ on $\Ebound \cap \Esafe$, we have that $\matLam_k \succsim \tau_k \Gamma_k$. Note that since $\tau_k \ge  \opnorm{\Pst}^3\Jfunc_0 \log(1/\delta) \ge d\log(1/\delta)$ (since $\Jfunc_0 \ge d$ and $\opnorm{\Pst} \ge 1$), we have also $1-e^{-\tau_k/d} \ge 1 - \delta$.

  Now, if in addition $\tau_k \gtrsimst \sigmain^4\opnorm{\Pst}^2$, Lemma~\ref{sssec:step_covariance_lb_round_k} entails that $\Gamma_k \succsim \frac{\sigma_k^2}{\opnorm{\Pst}} = \frac{\sigmain^2}{\sqrt{\tau_k}\opnorm{\Pst}}$ (note that for such $k$, $\sigmain^2/\sqrt{\tau_k} \le 1$). With a few simplifications, we see then that if
  \begin{align*}
  \tau_k \gtrsimst d \log \tau_k + d\log \left\{\frac{ d\opnorm{\Pst}}{1 \wedge \sigmain^2}\right\} \vee  \opnorm{\Pst}^3\Jfunc_0 \log(1/\delta) \vee  \sigmain^4 (\max\{\Psibst^4,\opnorm{\Pst}^2\}),
  \end{align*}
  then with probability $1-\BigOh{\delta}$, $\matLam_k \succsim \tau_k \Gamma_k$. Since $\tau_k \gtrsimst d \log\tau_k$ for $\tau_k \gtrsim d \log d$, we need simply $\tau_k \gtrsimst  d\log \left\{\frac{ d\opnorm{\Pst}}{1 \wedge \sigmain^2}\right\} \vee  \opnorm{\Pst}^3\Jfunc_0 \log(1/\delta) \vee  \sigmain^4 (\max\{\Psibst^4,\opnorm{\Pst}^2\})$ to ensure  $\matLam_k \succsim \tau_k \Gamma_k$ with probability $1-\BigOh{\delta}$. Shrinking $\delta$ by a constant reduces the failure probability to $1 - \delta$. Lastly, using $\sigmain^2 \ge 1$ by definition, and $\sigmain^2 \lesssim \sqrt{\dimx}\opnorm{\Pst}^{9/2}\Psibst \sqrt{\log \frac{\opnorm{\Pst}}{\delta}} $ by Lemma~\ref{lem:perturb_correct}, we can bound
  \begin{align*}
  &d\log \left\{\frac{ d\opnorm{\Pst}}{1 \wedge \sigmain^2}\right\} \vee  \opnorm{\Pst}^3\Jfunc_0 \log(1/\delta) \vee  \sigmain^4 (\Psibst^4 \vee \opnorm{\Pst}^2) \\
  &\lesssim d\log d\opnorm{\Pst} + \opnorm{\Pst}^3\Jfunc_0 \log(1/\delta) + \dimx \opnorm{\Pst}^{9} \Psibst^2 \log \frac{\opnorm{\Pst}}{\delta} (\Psibst^4 \vee \opnorm{\Pst}^2)\\
  &\lesssim  d \left(\opnorm{\Pst}^3\calP_0 +  \opnorm{\Pst}^{11} \Psibst^6 \right)  \log \frac{d \opnorm{\Pst}}{\delta} := \tauls.
  \end{align*}
  where in the last line we use $\Psibst,\opnorm{\Pst} \ge 1$, $\dimx \le d$, and $\calP_0 = \Jfunc_0/\dimx$. 
  \end{proof}
\subsubsection{Proof of \Cref{lem:proj_mat_bound} \label{sec:lem_proj_mat_bound}}

    \newcommand{\matxring}{\mathring{\matx}}
    \newcommand{\matxch}{\check{\matx}}
    Consider a fixed epoch $k$, and for $t\in \{\tau_{k},\tau_{k} + 1, \dots,\tau_{k+1}-1\}$, introduce the contribution from all previous steps:
    \begin{align*}
    \matxring_t := \Exp[\matx_t \mid \calF_{\tau_{k}}] = \Aclk^{t-\tau_{k}}\matx_{\tau_{k}}l
    \end{align*} 
     where $(\calF_{t})$ is the filtration generated by $\matx_{1:t-1},\matw_{1:t-1},\matu_{1:t-1}$, and where the equality follows from direct computation. Further, introduce the difference  
     \begin{align}
     \matxch_t := \matx_t - \matxring_t, \label{eq:matxch}
     \end{align} 
     The important observation here is
     \begin{obs}\label{obs:matxch_linear} $\matxch_t$ is a linear form in the jointly standard-normal Gaussian vector $\matgbar := (\matw_t,\matg_t : \tau_k \le t \le \tau_{k+1}-1)$. 
     \end{obs} We then have 
    \begin{align*}
    \ProjMat_k \matLam_k (I - \ProjMat_k) &= \sum_{t = \tau_{k}}^{\tau_{k+1}-1} \ProjMat_k \begin{bmatrix}\matx_t  \\\matu_t
    \end{bmatrix} \begin{bmatrix}\matx_t  \\\matu_t
    \end{bmatrix}^\top (I - \ProjMat_k)\\
    &= \sum_{t = \tau_{k}}^{\tau_{k+1}-1} \ProjMat_k \begin{bmatrix}\matx_t \\\Khat_k \matx_t + \sigma_k \matg_t
    \end{bmatrix} \begin{bmatrix}\matx_t \\\Khat_k \matx_t + \sigma_k \matg_t
    \end{bmatrix}^\top  (I - \ProjMat_k)\\
    &\overset{(i)}{=} \sum_{t = \tau_{k}}^{\tau_{k+1}-1} \ProjMat_k \begin{bmatrix} 0 \\  \sigma_k \matg_t
    \end{bmatrix} \begin{bmatrix}\matx_t \\\Khat_k \matx_t + \sigma_k \matg_t
    \end{bmatrix}^\top  (I - \ProjMat_k)\\
    &\overset{(ii)}{=}  \sum_{t = \tau_{k}}^{\tau_{k+1}-1} \ProjMat_k \begin{bmatrix} 0 \\  \sigma_k \matg_t
    \end{bmatrix} \begin{bmatrix} \matxch_t \\  \Khat_k \matxch_t \end{bmatrix}^\top  (I - \ProjMat_k) + 
    \sum_{t = \tau_{k}}^{\tau_{k+1}-1} \ProjMat_k \begin{bmatrix} 0 \\  \sigma_k \matg_t
    \end{bmatrix} \begin{bmatrix} \matxring_t \\  \Khat_k \matxring_t \end{bmatrix}^\top  (I - \ProjMat_k)
    \\\
    &\quad+ \sigma_k^2\sum_{t = \tau_{k}}^{\tau_{k+1}-1} \ProjMat_k \begin{bmatrix} 0 & 0 \\
     0 & \matg_t \matg_t^\top \end{bmatrix} (I - \ProjMat_k),
    \end{align*}
    where here $(i)$ follows since $\ProjMat_k$ annihilates the subspace $(x,u) : u = Kx$, and where $(ii)$ decomposes using linearity. Since $\|\ProjMat_k\|_{\op} \vee \|I - \ProjMat_k\|_{\op} \le 1$, we can upper bound, we bound the operator norm of $\nrm*{\ProjMat_k \matLam_k (I - \ProjMat_k)}_{\op}$ by three terms:
    \newcommand{\term}{\mathsf{term}}

    \begin{multline}
    \nrm*{\ProjMat_k \matLam_k (I - \ProjMat_k)}_{\op} \le  \\
     \sigma_k \underbrace{\nrm*{\sum_{t = \tau_{k}}^{\tau_{k+1}-1} \matg_t
     \begin{bmatrix} \matxch_t\mid  \Khat_k \matxch_t \end{bmatrix} }_{\op}}_{\term_1} + \sigma_k  \underbrace{ \nrm*{\sum_{t = \tau_{k}}^{\tau_{k+1}-1}  \matg_t \begin{bmatrix} \matxring_t\mid  \Khat_k \matxring_t \end{bmatrix} }_{\op}}_{\term_2} + \sigma_k^2  \underbrace{ \nrm*{\sum_{t = \tau_{k}}^{\tau_{k+1}-1} \matg_t \matg_t^\top }_{\op}}_{\term_3}. \label{eq:projmat_op_bound}
    \end{multline}
    Let us bound each term in order;

    \paragraph{Bounding $\term_1$}
    \newcommand{\Tleft}{\mathcal{T}_{\mathrm{left}}}
    \newcommand{\Tright}{\mathcal{T}_{\mathrm{right}}}
    \begin{defn}[Relevant $\epsilon$-nets]\label{defn:relevant_epsilon_nets}
    Let $\Tleft$ be an $1/4$-net of the unit ball in $\R^\dimu$, and let $\Tright$ be a $1/4$-net of the unit ball in the subsapce $\R^{\dimx + \dimu}$ spanned by $(I - \ProjMat_k)$, that is, spanned by vectors of the form $(x,\Khat_k x)$. We use $w$ to denote elements of $\Tleft$, and $v = (v_1,v_2)$ to denote elements of this net $\Tright$, partitioned into the $\R^{\dimx}$ and $\R^{\dimu}$ components.
    \end{defn}

    A standard covering argument entails that 
    \begin{align*}
    \term_1 \lesssim \max_{w \in \Tleft, v = (v_1,v_2) \in \Tright} \term_1(w,v)\quad \text{where }  \term_1(w,v) := \sum_{t = \tau_{k}}^{\tau_{k+1}-1} \langle w, \matg_t\rangle
    \langle (\matxch_t,\Khat_k \matxch_t), v \rangle.
    \end{align*}
    From \Cref{obs:matxch_linear}, $\term_1(w,v)$ is a quadratic form in a Gaussian vector $\matgbar = ((\matw_t,\matg_t) : \tau_{k} \le t \le \tau_{k+1} - 1)$ encoding both input noise $\matg_t$ and process noise $\matw_t$, and thus may express that for some matrix $Q(w,v)$, which without loss of generality we can choose to be symmetric, for which 
    \begin{align*}
      \term_1(w,v) = \matgbar^\top Q(w,v)\matgbar
    \end{align*}
    When clear from context, we abridge $Q = Q(w,v)$. 

    We shall apply the Hanson Wright inequality, which requires bounding $\tr(Q)$, $\|Q\|_{\fro}$, and $\|Q\|_{\op}$. We first see that $\tr(Q) = 0$, since $\Exp[\term_1(w,v)] = 0$, as the noise $\matg_t$ is independent of $\matxch_t$. Moreover, we can bound $\|Q\|_{\fro} \le \sqrt{\rank(Q)}\|Q\|_{\op}$. To bound $\rank(Q)$:
    \begin{claim}\label{claim:Q_rank} $\rank(Q) \le \tau_{k}$, and thus $\|Q\|_{\fro} \le \sqrt{\tau_{k}}\|Q\|_{\op}$
    \end{claim}
    \begin{proof}[Proof of \Cref{claim:Q_rank}] Let $X_{t} : \R^{\dim(\matgbar)} \to  \R^{\dimu}$ denote the projection matrix yielding $X_t \matgbar = \matg_t$, and let  $Y_t := \R^{\dim(\matgbar)} \to \R^{\dimx + \dimu}$ denote the linear operator $Y_t \matgbar = (\matxch_t,\Khat_k \matxch_t)$; these are indeed linear by \Cref{obs:matxch_linear}, namely that $\matxch_t = \matx_t - \matxring_t$ is linear in the noise vector $\matgbar$.

    Let $\bar{X} := X_t$ denote the matrix whose rows are $X_{\tau_{k}},\dots, X_{\tau_{k+1} - 1}$, and let $\bar{Y}$ be analogous. Then, 
    \begin{align*}
    \matgbar^\top Q(w,v)\matgbar &= \term_1(w,v) \\
    &= \sum_{t = \tau_{k}}^{\tau_{k+1}-1} \matg_t^\top w v^\top
     (\matxch_t,\Khat_k \matxch_t), v  = \sum_{t=\tau_{k}}^{\tau_{k+1}} \matgbar^\top X_t^\top wv^\top Y_t \matgbar \\
     &= \matgbar\bar{X}\Diag_{\tau_{k}(wv^\top)}\bar{Y}\matgbar,
    \end{align*}
    where $\Diag_{\tau_{k}(wv^\top)}$ denotes the block-diagonal matrix consisting of $\tau_k$ blocks of the rank-one terms $wv^\top$. Hence, the quadratic form $Q$ can be rendered explicitly as $Q = \bar{X}\Diag_{\tau_{k}}(wv^\top)\bar{Y}$, which has rank at most $\rank(\Diag_{\tau_{k}}(wv^\top)) = \tau_{k} \cdot \rank(wv^\top) = \tau_{k}$. 
    \end{proof}
    \newcommand{\Qbar}{\bar{Q}}
    Next, we bound  $\|Q\|_{\op}$. 
    \begin{claim} $\opnorm{Q} \lesssim \opnorm{\Pst}^{5/4}$
    \end{claim}
    \begin{proof} We construct a matrix $\Qbar \succeq 0$ such that $-\Qbar \preceq Q \preceq \Qbar$. It follows that $\|Q\|_{\op} \le \Qbar$
    Let us contrinue by constructing $\Qbar$. For any parameter $\alpha > 0$, 
    \begin{align*}
    \term_1(w,v) := \sum_{t = \tau_{k}}^{\tau_{k+1}-1} \langle w, \matg_t\rangle
    \langle (\matxch_t,\Khat_k \matxch_t), v \rangle &\le \sum_{t = \tau_{k}}^{\tau_{k+1}-1} \frac{\alpha}{2} \matg_t^\top ww^\top \matg_t + \frac{1}{2\alpha}(\matxch_t,\Khat_k \matxch_t)^\top vv^\top(\matxch_t,\Khat_k \matxch_t)^\top \\
    &\overset{(i)}{\le} \frac{\alpha}{2}\underbrace{\sum_{t = \tau_{k}}^{\tau_{k+1}-1} \matg_t^\top ww^\top \matg_t}_{:= \matgbar^\top \Qbar_1(w,v) \matgbar} + \frac{1}{\alpha} \underbrace{\sum_{t = \tau_{k}}^{\tau_{k+1}-1}  \matxch_t^\top (v_1 v_1^\top + \Khat_k v_2 v_2 \Khat_k^\top)\matxch_t}_{:= \matgbar^\top \Qbar_2(w,v) \matgbar} \\
    &\overset{(ii)}{:=} \matgbar^\top \Qbar(w,v) \matgbar
    \end{align*}
    where in $(i)$, we use the decompose $v = (v_1,v_2)$ as in \Cref{defn:relevant_epsilon_nets}, and use the elementary inequality $vv^\top \preceq \begin{bmatrix} v_1 v_1^\top & 0 \\ 0 & v_2 v_2^\top \end{bmatrix}$. In $(ii)$, we again invoke that $\matxch_t$ is linear in $\matgbar$ (\Cref{obs:matxch_linear}).

    Since the above inequality holds for any realization of the noise,  we find that $Q(w,v) \preceq   \Qbar(w,v)$, and a similar manipulation shows that $Q(w,v) \succeq - \Qbar(w,v)$. Noting that $\|Q(w,v)\|_{\op} \le \|\Qbar(w,v)\|_{\op} = \| \frac{\alpha}{2}\Qbar_1(w,v) + \frac{1}{\alpha}\Qbar_2(w,v)\|_{\op} \le \frac{\alpha}{2}\|\Qbar_{1}(w,v)\|_{\op} +\frac{1}{\alpha}  \|\Qbar_2(w,v)\|_{\op}$, we bound $\|\Qbar_{1}(w,v)\|_{\op}$ and $\|\Qbar_2(w,v)\|_{\op}$. We have first that, deterministically,
    \begin{align*}
    \matgbar^\top \Qbar_1(w,v) \matgbar = \sum_{t = \tau_{k}}^{\tau_{k+1}-1} \matg_t^\top ww^\top \matg_t \le \sum_{t = \tau_{k}}^{\tau_{k+1}-1}\|\matg_t\|^2 \le \sum_{t = \tau_{k}}^{\tau_{k+1}-1}\|\matg_t\|^2 + \|\matw_t\|^2 = \|\matgbar\|^2,
    \end{align*}
    so that $\|\Qbar_1(w,v)\|_{\op} \le 1$. Moreover, by applying \Cref{lem:quadratic_forms_cost} with $x_1 = 0$, $R_2 = 0$, $\sigma_u = \sigma_k$,  $R_1 = v_1 v_1^\top + \Khat_k v_2 v_2 \Khat_k^\top$, we can bound 
    \begin{align*}
     \|\Qbar_2(w,v)\|_{\op} &\lesssim (1 + \sigma_k^2 \|\Bst\|_{\op}^2)\|\Ast + \Bst \Khat_k\|_{\Hinfty}\|R_1\|_{\op}\\
     &\lesssim (1 + \sigma_k^2 \|\Bst\|_{\op}^2)\|\Ast + \Bst \Khat_k\|_{\Hinfty}(1 + \|\Khat_k\|_{\op}^2).
    \end{align*}
    In particular, for $\sigma_k \le 1/\Psibst^2$, we have $(1 + \sigma_k^2 \|\Bst\|_{\op}^2) \lesssim 1$, and and nder the safe event $\Esafe$, we have $\|\Khat_k\|^2_{\op} \le \frac{21}{20}\opnorm{\Pst}$  and $\|\Ast + \Bst \Khat_k\|_{\Hinfty} := \|\Aclk\|_{\Hinfty} \le \opnorm{\Pst}^{3/2}$. Since $\opnorm{\Pst} \ge 1$, this allows us to bound 
    \begin{align*}
     \|\Qbar_2(w,v)\|_{\op} \lesssim \opnorm{\Pst}^{5/2},
    \end{align*}
    provided again that $\sigma_k \le \Psibst$. Hence, setting $\alpha = \opnorm{\Pst}^{5/4}$, 
    \begin{align*}
    \|Q(w,v)\|_{\op} \le \frac{\alpha}{2} \|\Qbar_1(w,v)\|_{\op} + \frac{1}{\alpha} \|\Qbar_2(w,v)\|_{\op} \lesssim \frac{\alpha}{2} + \frac{1}{\alpha}\opnorm{\Pst^{5/2}} \lesssim \opnorm{\Pst}^{5/4}. 
    \end{align*}
    \end{proof}
    Gathering the past two claims, we have $\|Q\|_{\op} \lesssim \opnorm{\Pst}^{5/4}$ and $\|Q\|_{\fro} \lesssim \sqrt{\tau_{k}}\opnorm{\Pst}^{5/4}$. Hence, the Hanson-Wright Inequality implies that, for some universal constant $c>0$,
    \begin{align*}
    \Pr\left[ \term_1(w,v) \ge c\opnorm{\Pst}^{5/4}\left(\sqrt{\tau_{k}} \log(1/\delta_0) + \log(1/\delta_0) \right) \right] \le \delta_0.
    \end{align*}
    Finally, taking a union bound over $w \in \Tleft$ and $v \in \Tright$, and noting that we may choose $\log |\Tleft||\Tright| \lesssim \dimx + \dimu$, we have for $\tau_{k+1} \ge \dimx + \dimu + \log(1/\delta_0)$ and another universal constant $c$ that
    \begin{align*}
    \Pr\left[ \sup_{v \in \Tleft,w \in \Tright} \term_1(w,v) \ge c\opnorm{\Pst}^{5/4}\sqrt{\tau_{k} (\dimx + \dimu + \log(1/\delta_0)} \right] \le \delta.
    \end{align*}
    Finally, by a standard bound on  the cardinality of the minimal-size nets $\Tleft$ and $\Tright$ (see e.g. \citet[Section 4.2]{vershynin2018high}), this implies that for another constant $c$,
    \begin{align*}
    \Pr\left[ \term_1 \ge c\opnorm{\Pst}^{5/4}\sqrt{\tau_{k} (\dimx + \dimu + \log(1/\delta_0)} \right] \le \delta_0.
    \end{align*}

    \paragraph{Bounding $\term_2$}
    Again, we use a covering argument. with $\Tleft$ be an $1/4$-net of the unit ball in $\R^\dimu$, and let $\Tright$ be a $1/4$-net of the unit ball in the subspace $\R^{\dimx + \dimu}$ spanned by $(I - \ProjMat_k)$. Similarly, use 
    \begin{align*}
    \term_2 \lesssim \max_{w \in \Tleft, v = (v_1,v_2) \in \Tright} \term_2(w,v)\quad \text{where }  \term_2(w,v) &:= \sum_{t = \tau_{k}}^{\tau_{k+1}-1} \langle w, \matg_t\rangle
    \langle (\matxring_t,\Khat_k \matxring_t), v \rangle
    \end{align*}

    Since $\matxring_t =  \Aclk^{t-\tau_{k}}\matx_{\tau_{k}}$, we see that $(\matg_t : t \ge \tau_{k})$ are independent of the terms $\langle (\matxring_t,\Khat_k \matxring_t), v \rangle$. Since $\|v\| = 1$,  $\langle w, \matg_t \rangle$ are then independent standard Gaussian random vectors. Hence, by standard Gaussian tail bound, 
    \begin{align*}
    \term_2(w,v)  \le \sqrt{2 \sum_{t=\tau_{k}}^t 
    \langle (\matxring_t,\Khat_k \matxring_t), v \rangle^2 \log(1/\delta_0)}, \text{ w.p. } 1 - \delta_0).
    \end{align*}
    We can  then bound
    \begin{align*}
    \sum_{t=\tau_{k}}^t 
    \langle (\matxring_t,\Khat_k \matxring_t), v \rangle^2 &= v^\top\begin{bmatrix} I \\ \Khat_k \end{bmatrix}^\top \sum_{t=\tau_{k}}^{\tau_{k+1}-1}  \matxring_t \matxring_t^\top \begin{bmatrix} I \\ \Khat_k \end{bmatrix}v \le \nrm*{\sum_{t=\tau_{k}}^{\tau_{k+1}-1}  \matxring_t \matxring_t^\top}_{\op} \|v\|^2 \nrm*{\begin{bmatrix} I \\ \Khat_k \end{bmatrix}}^2 \\
    &= (1 + \|\Khat_k\|^2) \nrm*{\sum_{t=\tau_{k}}^{\tau_{k+1}-1}  \matxring_t \matxring_t^\top}_{\op} \le (1 + \|\Khat_k\|^2)\tr\left(\sum_{t=\tau_{k}}^{\tau_{k+1}-1}  \matxring_t \matxring_t^\top\right).
    \end{align*}
    Furthermore, since $\matxring_t = \Aclk^{t-\tau_{k+1}}\matx_{\tau_{k}}$, the linearity of trace yields 
    \begin{align*}
    \tr\left(\sum_{t=\tau_{k}}^{\tau_{k+1}-1} \matxring_t \matxring_t^\top\right) &= \matx_{\tau_{k}}^\top\left(\sum_{i= 0}^{\tau_{k+1} - \tau_k}  (\Aclk^i)^\top (\Aclk^i)\right)\matx_{\tau_{k}}^\top\\
    &\le \matx_{\tau_{k}}^\top\left(\sum_{i = 0}^{\infty}  (\Aclk^i)^\top (\Aclk^i)\right)\matx_{\tau_{k}}^\top\\
    &= \matx_{\tau_{k}}^\top \dlyap(\Aclk,I)\matx_{\tau_{k}} \le (\sqrt{\Psibst \Jfunc_0\log(1/\delta)}\opnorm{\Pst}^{3/2})^2 \text{ on event } \Ebound,
    \end{align*}
    where the above uses \Cref{lem:matx_bound}. Moreover, on the the event $\Esafe$, from \Cref{lem:perturb_correct},  $\|P_k\|_{\op} \lesssim \|\Pst\|_{\op}$ and $\opnorm{\Khat_k}^2 \lesssim \opnorm{\Pst}$ Hence, on $\Ebound \cap \Esafe$, 
    \begin{align*}
    \tr(\sum_{t=\tau_{k}}^{\tau_{k+1}-1} \matxring_t \matxring_t^\top)^{1/2}  \lesssim \sqrt{\Psibst \Jfunc_0\log(1/\delta)}\opnorm{\Pst}^{3/2}, \quad (1 + \|\Khat_k\|^2_{\op})^{1/2} \lesssim \opnorm{\Pst}^{1/2}
    \end{align*}
    Hence, combining the above bounds
    \begin{align*}
    \left(\sum_{t=\tau_{k}}^t 
    \langle (\matxring_t,\Khat_k \matxring_t), v \rangle^2\right)^{1/2} \lesssim \sqrt{\Psibst \Jfunc_0\log(1/\delta)}\opnorm{\Pst}^{2}
    \end{align*}
    and thus, on $\Esafe \cap \Ebound$, the following holds with probability $1 - \delta_0$:
    \begin{align*}
    \term_2(w,v)  \le \sqrt{\Psibst \Jfunc_0\log(1/\delta) \log(1/\delta_0)}\opnorm{\Pst}^{2}
    \end{align*}
    Again, using a union bound over a standard cardinality bound on $\Tleft$ and $\Tright$, we obtain that with probability $1 - \delta_0$
    \begin{align*}
    \term_2 \lesssim  \max_{w \in \Tleft, v = (v_1,v_2) \in \Tright} \term_2(w,v)  \le \sqrt{\Psibst \Jfunc_0\log(1/\delta) (\dimx + \dimu + \log(1/\delta_0) )}\opnorm{\Pst}^{2}. 
    \end{align*}
    \paragraph{Bounding $\term_3$}
    Recall that
    \begin{align*}
    \term_3 = \nrm*{\sum_{t = \tau_{k}}^{\tau_{k+1}-1} \matg_t \matg_t^\top}_{\op}
    \end{align*}
    Observe that $\matg_t$ are independent and isotropic Gaussian random vectors. 
    Noting that $\tau_{k+1} = 2\tau_k$, a standard operator norm bound for Gaussian matrices, \citet[Theorem 5.39]{vershynin2010introduction}
     yields that, with probability $ 1 - \delta_0$, 
     \begin{align*}
     \term_3 \lesssim (\sqrt{\tau_k} + \sqrt{\dimu + \log(1/\delta_0)})^2
     \end{align*}
     In particular, for $\tau_k \ge \dimu + \log(1/\delta_0)$, $\term_3 \lesssim \sqrt{\tau_k}$. 
    \paragraph{Combining the bounds} Combining the above bounds (and shrinking $\delta_0$ to absord constants into union bounds as necessary), it holds that for $\tau_k \ge \dimu + \dimx + \log(1/\delta_0)$, we have that with probability $1 -\delta_0$, 
    \begin{align*}
    \term_1 &\lesssim \opnorm{\Pst}^{5/4}\sqrt{\tau_{k} (\dimx + \dimu + \log(1/\delta_0))} \\
    \term_2 &\lesssim \sqrt{\Psibst \Jfunc_0\log(1/\delta) (\dimx + \dimu + \log(1/\delta_0) )}\opnorm{\Pst}^{2}. \\
    \term_3 &\lesssim\tau_k
    \end{align*}
    In particular, if in addition it holds that $\tau_{k} \ge \Psibst \Jfunc_0 \opnorm{\Pst}^{3/2}\log(1/\delta)$, we find
    \begin{align*}
    \term_1 + \term_2 &\lesssim \opnorm{\Pst}^{5/4}\sqrt{\tau_{k} (\dimx + \dimu + \log(1/\delta_0))}.
    \end{align*}
    Hence, from \Cref{eq:projmat_op_bound}, under the above conditions on $\tau_k$, and on $\Esafe \cap \Ebound$, we have that, with probability $ 1- \delta_0$,
    \begin{align*}
    \nrm*{\ProjMat_k \matLam_k (I - \ProjMat_k)}_{\op} &\le  \sigma_k(\term_1 + \term_2) + \sigma_k^2 \term_2 \\
    &\lesssim  \sigma_k\opnorm{\Pst}^{5/4}\sqrt{\tau_{k} (\dimx + \dimu + \log(1/\delta_0))} +  \sigma_k^2 \tau_k.
    \end{align*}
    Setting $\delta_0 = \delta/16$ and absorbing gives
    \begin{align*}
    \nrm*{\ProjMat_k \matLam_k (I - \ProjMat_k)}_{\op}
    &\lesssim  \left(\sqrt{\tau_{k}\sigma_k^2 \left(\opnorm{\Pst}^{5/2}(d+ \log(1/\delta)) +\sigma_k^2 \tau_k \right) } \right),
    \end{align*}
    Note that our stipulated conditions on $\tau_k$ hold for $\tau_k \gtrsimst \tauls$. Moreover, when this holds, we have
    \begin{align*}
    \sigma_k^2 \tau_k \ge \min\{\tau_k, \sigmain^2 \tau_k^{-1/2}\} \ge \min\{\tauls, \sqrt{\dimx \tauls} \} \ge \opnorm{\Pst}^{5/2}(d+ \log(1/\delta)) 
    \end{align*}
    so that
    \begin{align*}
    \nrm*{\ProjMat_k \matLam_k (I - \ProjMat_k)}_{\op} &\lesssim  \tau_k \sigma_k^2
    \end{align*}
    \qed

\subsection{Proof of Lemma~\ref{lem:initial_phase} $(k < \ksafe)$ \label{ssec:lem:initial_phase}}

  We analyze the rounds $k < \ksafe$, which correspond to the rounds before the least-squares procedure produces a sufficiently close approximation to $(\Ast,\Bst)$ that we can safely implement certainty equivalent control. 

  In order to avoid directly conditioning on events $\{\ksafe \le (\dots)\}$, let us define the sequence $\matz_{t,0} := (\matx_{t,0},\matu_{t,0})$ on the same probability space as $(\matx_t,\matu_t)$ to denote the system driven by the same noise $\matw_t$, and with the same random perturbations $\matg_t$, but where the evolution is with respect to the dynamics
  \begin{align*}
  \matx_{t,0} = \Ast + \Bst \matu_{t,0} & \quad \matu_{t,0} = K_0\matx_{t,0} + \matg_t,
  \end{align*}
  that is, the dynamics defined by the distribution $\calD(K_0,\sigma_u^2 = 1,x_1  = 0)$. Observe that, for any $t < \tau_{\ksafe}$, it holds that $\matx_{t,0} = \matx_t$ and $\matu_{t,0} = \matu_t$, so it will suffice to reason about this sequence.

  \paragraph{Proof that $\Esafe$ holds}
    As above, to reason rigorously about probabilities, we introduce $\Ahat_{k,0}$, $\Bhat_{k,0}$ as the OLS estimators on the $\matz_{k,0} := (\matx_{k,0},\matu_{k,0})$ sequence, and define the covariance matrix
    \begin{align*}
    \matLam_{k,0} := \sum_{t=\tau_k}^{2\tau_k - 1} \matz_{k,0}\matz_{k,0}^{\top}.
    \end{align*}
    We also define the induced confidence term:
    \begin{align*}
    \Conf_{k,0}  = 6 \lambda_{\min}(\matLam_{k,0})^{-1}  \prn*{d \log 5 + \log \left\{\tfrac{4k^2\det(3(\matLam_{k,0})}{\delta}\right\}}.
    \end{align*}

    \newcommand{\Econf}{\mathcal{E}_{\mathrm{conf}}}
    \begin{lem} The following event holds with probability $ 1- \delta$:
    \begin{align*}
    \Econf := \left\{\forall k \le \ksafe \text{ with } \matLam_{k} \succeq I, \quad \left\| \begin{bmatrix} \Ahat_{k} - \Ast  \mid \Bhat_{k} - \Bst \end{bmatrix}\right\|_2^2  \le \Conf_{k}\right\}.
    \end{align*}
    \end{lem}
    \begin{proof} Applying \eqref{eq:op_norm_bound} in \Cref{lem:frob_ls} with $\Lambda_0 = I$, we see that for any fixed $k$ for which $\matLam_{k,0} \succeq I$, $\Conf_{k,0}$ is a valid $\delta/4k^2$-confidence interval; that is $\opnorm{[\Ast - \Ahat_{k,0}\mid \Bst - \Bhat_{k,0}]} \le \Conf_{k,0}$. By a union bound, the confidence intervals are valid with probability $1-\delta/2$, simultaneously. Since the the sequence $\matx_{t,0}$ coincides with $\matx_t$ for $t \le \tau_{\ksafe}$, and $\matu_{t,0}$ with $\matu_t$ for $t \le \tau_{\ksafe} - 1$, we see that $\Conf_{k,0} = \Conf_k$ for all $k \le \ksafe$.
    \end{proof}

  \paragraph{Proof of Regret Bound}

  We begin with the following regret bound.
  \begin{lem}\label{lem:initial_round_bounds}  For $\delta < 1/T$, the following hold with probability $1-\delta$,
  \begin{align*}
  \sum_{t=1}^{\tau_{\ksafe}-1}\matx_{t,0}^\top \Rx \matx_{t,0} + \matu_{t,0}^\top \Ru \matu_{t,0} &\lesssim d\tau_{\ksafe} \Psibst^2\calP_0 \log(\frac{1}{\delta}).
  \end{align*}
  \end{lem}
  \begin{proof}  
 It suffices to show that  the $(\matx_{t,0},\matu_{t,0})$ sequences satisfies the following bound:
  \begin{align*}
  \sum_{t=1}^{\tau_{k_0}-1}\matx_{t,0}^\top \Rx \matx_{t,0} + \matu_{t,0}^\top \Ru \matu_{t,0} &\lesssim \tau_{k_0}\left(\Jfunc_0(1+\|\Bst\|_{\op}^2) + \tr(\Ru)\right)\log(\frac{1}{\delta}),
  \end{align*}
  where the inequality suffices since $\calP_0 \ge 1$ (indeed, $\Jfunc_0 \ge \Jst \ge d$ by Lemma~\ref{lem:P_bounds_lowner}), and thus $\Jfunc_0(1+\|\Bst\|_{\op}^2) + \tr(\Ru) = \dimx \calP_0(1+\|\Bst\|_{\op}^2) + \dimu \le d\calP_0\Mbarst$.

 For the second, we have from \Cref{lem:quadratic_forms_cost} and the fact that $\matx_1 = 0$ that there is a Gaussian quadratic form $\matgbar^\top \Lamgbar\,\matgbar$ which is equal to $\sum_{t=1}^{\tau_{k_0}-1}\matx_{t,0}^\top \Rx \matx_{t,0}$, and where $\tr(\Lamgbar) \le \tau_{k_0}\left(\Jfunc_0(1+\|\Bst\|_{\op}^2) + \tr(\Ru)\right)$. The second bound now follows from the crude statement of Hanson Wright in Corollary~\ref{cor:crude_HS}. The last statement follows by a union bound, noting that we need to bound over $\kmax = \log_2T \le T \le 1/\delta$, rounds, and absorbing constants.
  \end{proof}

    We conclude by arguing an upper bound on $\tau_{\ksafe}$. We rely on the following guarantee.
    \begin{lem}\label{lem:epsafe_sufficient} Suppose $\Esafe$ holds. Then for all $k < \ksafe$ for which $\matLam_k \succeq I$, we must have that $\Conf_k \gtrsim \epsafe$, where $\epsafe = \opnorm{\Pst}^{-10}$.  
    \end{lem}
    \begin{proof} For all $k < \ksafe$ for which $\matLam_k \succeq I$, we must have that $\Conf_k > 1/\Csafe(\Ahat_k,\Bhat_k)\}$. If $\Conf_k \le c/\Csafe(\Ast,\Bst)^2$ for a sufficiently small $c$, then the same perturbation argument as in Theorem~\ref{thm:continuity_of_safe set} entails that we have $\Conf_k \le 1/9\Csafe(\Ahat_k,\Bhat_k)^2$, yielding a contradiction. Finally, we subsitute in $\Csafe(\Ast,\Bst)^2 \lesssim \opnorm{\Pst}^{10}$ by \Cref{eq:constants}.
    \end{proof}
    Recall that we say $f \gtrsimst f$ if ``$f \ge C g$'' for a sufficiently large constant $C$ (\Cref{def:gtrsimst}). In light of the above lemma, Part 2 will follow as soon as we can show that, for any $\epsilon \in (0,1)$,, 
    \begin{align}\label{eq:tau_k_safe_wts}
    \text{if } \tau_k \gtrsimst \frac{d(1+ \|K_0\|_{\op}^2)}{\epsilon}\log \frac{\Psibst^2 \Jfunc_0}{\delta} , \quad \text{then } \Conf_{k,0} \le \epsilon, \text{ and } \matLam_{k,0} \succeq I\, \text{ w.p. } 1 - \BigOh{\delta}.
    \end{align}
    We begin with a lower bound the matrices $\matLam_{k,0}$:
    \begin{lem}\label{lem:initial_covariance_lb} for a sufficiently large constant $C$. Finally, set $\tau_{\min} =  d \log (1+ \Psibst \Jfunc_0)$. Then, for any $k$ such that $\tau_{k} \gtrsimst \tau_{\min} \vee d \log(\frac{1}{\delta})$, , it holds that 
    \begin{align*}
    \Exp[\tr(\matLam_{k,0})] \lesssim \Psibst^2 \Jfunc_0\tau_k, \quad \Pr\left[\lambda_{\min}(\matLam_{k,0}) \gtrsimst \frac{\tau_{k}}{1+\|K_0\|_2^2} \right] \le  \delta.
    \end{align*}
    \end{lem}
    The bound above is proven in Section~\ref{sssec:proof_lemma_initial_cov_lb}. We can now verify Eq.~\eqref{eq:tau_k_safe_wts},  concluding the proof of Part 2.
    \begin{proof}[Proof of Eq.~\eqref{eq:tau_k_safe_wts}] Suppose that $k$ is such that $\tau_{k} \gtrsimst \tau_{\min} \vee d \log(\frac{1}{\delta})$. Then, by the above lemma, and using $\det(cX) = c^d\det(X)$ for $X \in \R^{d \times d}$, we have, with probability $1-\BigOh{\delta}$,
    \begin{align*}
    \Conf_{k,0} \lesssim \frac{1+\|K_0\|_2^2}{\tau_{k}}  (d + \log \frac{k^2}{\delta} + \log \det((\matLam_{k,0}))\\
    \le \frac{1+\|K_0\|_2^2}{\tau_{k}}  (d + \log \frac{k^2}{\delta} +d\log \tr((\matLam_{k,0})),
    \end{align*}
    where we use that $X\succeq 0$, we have $\log \det (X) = \sum_{i=1}^d \log \lambda_i(X) \le d\log \tr(X)$. By Markov's inequality, we have with probability $1-\delta$ that $ \tr((\matLam_{k,0}) \le \Exp[ \tr((\matLam_{k,0})]/\delta \lesssim \Psibst^2 \Jfunc_0\tau_k \le \Psibst^2 \Jfunc_0/\delta$, since $\tau_k \le   T \le 1/\delta$. Hence, with some elementary operators, we can bound
    \begin{align*}
    \Conf_{k,0} \lesssim \frac{d}{\tau_k} \log \frac{\Psibst^2 \Jfunc_0}{\delta}. 
    \end{align*}
    Hence, for $\tau_k \gtrsimst \frac{d}{\epsilon} \log \frac{\Psibst^2 \Jfunc_0}{\delta}$, we have with probability $1-\delta$ that we have $\Conf_k \le \epsilon$. 
    \end{proof}

\subsubsection{Proof of Lemma~\ref{lem:initial_covariance_lb}\label{sssec:proof_lemma_initial_cov_lb}}
  \begin{enumerate}
  \item We first need to argue a lower bound on matrices $\Sigma_t$ such that that $\matz_{t,0} \mid \calF_{t-1} \sim \calN(\matzbar_{t,0},\Sigma_{t,0})$, where $\matzbar_{t,0},\Sigma_{t,0}$ are $\calF_{t-1}$ measurable. It is straightforward to show that 
  \begin{align*}
  \Sigma_{t,0} = \begin{bmatrix} I & K_0 \\
  K_0^{\top} & K_0^\top K_0 +  I
  \end{bmatrix},
  \end{align*}
  which by \citet[Lemma F.6]{dean2018regret}, has least singular value bounded below as
  \begin{align*}
  \lambda_{\min}(\Sigma_{t,0}) \ge \min\left\{\frac{1}{2}, \frac{1}{1 + 2\|K_0\|^2}\right\} \ge \frac{1}{2 + 2\|K_0\|^2}.
  \end{align*}
  \item Next, we need an upper bound on 
  \begin{align*}
  \Exp[\tr(\matLam_{k,0})] &= \Exp[\sum_{t=\tau_{k-1}}^{\tau_{k} - 1}\|\matx_t\|^2 + \|\matz_t\|^2] \\
  &\le \Exp[\sum_{t=1}^{\tau_{k} - 1}\|\matx_t\|^2 + \|\matz_t\|^2] \\
  &\le \tau_k(1+\|\Bst\|^2)\tr(\dlyap(A_{K_0},I + K_{0}^\top K_0 )) + \tr\tau_k(I)\\
  &\le 2\tau_k(1+\|\Bst\|^2)\tr(\dlyap(A_{K_0},I + K_{0}^\top K_0 ))\\
  &\le 2\tau_k(1+\|\Bst\|^2)J_{K_0} = 4(\tau_k - \tau_{k-1})(1+\|\Bst\|^2)J_{K_0} \\
  &\le 4(\tau_k - \tau_{k-1})(1+\Psibst^2)J_{K_0} \lesssim \tau_k \Psibst^2J_{K_0},
  \end{align*}
  where we use that $I \preceq \dlyap(A_{K_0},I + K_{0}^\top K_0 )) \preceq \dlyap(A_{K_0},\Rx + K_{0}^\top \Ru K_0 )) = J_{K_0}$ for $\Ru,\Rx \ge I$. This proves the trace upper bound.
  \item Using the second to last inequality in the above display, we see that for 
  \begin{align*}
  \tau_{k} - \tau_{k-1} = \frac{1}{2}\tau_k &\ge \underbrace{\frac{2000}{9}\left(2d\log \tfrac{100}{3} + d \log (8(1+1\|K_0\|)^2(1+\Psibst^2)J_{K_0})\right)}_{:=\underline{\tau} },
  \end{align*}
  \Cref{lem:covariance_lb} implies (taking $\calE=\Omega$ to be the probability space and $T = \tau_{k}/2$) that, if $\tau_{k} \gtrsimst \tau_{\min}$, we have
  \begin{align*}
  \Pr\left[\matLam_{k_0} \not\succeq \frac{9\tau_{k}}{3200}\Sigma_0 \right] \le  2\exp\left( - \tfrac{9}{4000(d+1)}\tau_k\right).
  \end{align*}
  Routine manipulations of give $\dlyap$,  $1 + \|K_0\|^2 \le \dlyap(A_{K_0}, I + \|K_0\|^2 ) \le \dlyap(A_{K_0}, \Rx + K_0^\top \Ru K_0) = J_{K_0}$ for $\Ru,\Rx \succeq I$. Hence, with a bit of algebra, we can bound
  \begin{align*}
  \underline{\tau} \lesssim \tau_{\min} := d \log (1+ \Psibst J_{K_0}).
  \end{align*}
  Using the lower bound on $\Sigma_0$ concludes the proof.
  \end{enumerate}
  \qed

\subsection{Proof of Lemma~\ref{lem:cost_lem}\label{app:proof_lem_cost_lem}}

  In order to prove Lemma~\ref{lem:cost_lem}, we first show that we can represent the $\Cost$ functional as a quadratic form in Gaussian variables.
 \begin{lem}\label{lem:quadratic_forms_cost} Let $(\matx_1,\matx_2,\sigma_u)$ denote the linear dynamical system described by the evolution of $\calD(K,x_1)$. Then for any $t \ge 1$, there exists a standard Gaussian? vector $\matgbar \in \R^{\BigOh{td}}$ such that for any cost matrices $R_1,R_2 \succeq 0$, we have
        \begin{align*}
       \Cost(R_1,R_2;x_1,\sigma_u,t) = \matgbar^{\top}\Lamgbar\,\matgbar + x_1^\top \Lamxone \,x_1 + 2\matgbar^{\top}\Lamcross x_1,
        \end{align*}
        where, letting $R_K = R_1 + K^\top R_2 K$, $A_K = \Ast + \Bst K$, $P_K = \dlyap(A_K,R_K)$ $J_K := \tr(P_K)$, and $\deff := \min\{\dimu,\dim(R_1) + \dim(R_2)\}$,
        \begin{align*}
          &\tr(\Lamgbar) \le t J_k + 2\sigma_u^2 t \deff \left(\opnorm{R_2} + \opnorm{\Bst}^2\opnorm{P_K}\right),\\
          &\|\Lamgbar\|_{\op} \lesssim (1+\sigma_u^2\|\Bst\|_{\op}^2) \|R_K\|_{\op}\Hinf{A_{K}}^2 + \sigma_u^2 \|R_2\|_{\op}^2,\\
          &\Lamxone \preceq P_K,\\
          &\|\Lamcross x_1\|_2\le \sqrt{\|\Lamgbar\|_{\op} \cdot x_1^\top P_K x  }.
        \end{align*}
        \end{lem}

        Let us continue to prove Lemma~\ref{lem:cost_lem}. The expectation result follow since $\Exp[\Cost(R_1,R_2;x_1,\sigma_u,t)] = \tr(\Lamgbar) + x_1^\top \Lamxone x_1$ for a Gaussian quadratic form. 

        For the high probability result, observe that by Gaussian concentration and Lemma~\ref{lem:quadratic_forms_cost}, we have with probability $1-\delta$
        \begin{align*}
        2\matgbar^{\top}\Lamcross x_1, \lesssim \sqrt{\log(1/\delta) \|\Lamcross x_1\|_2 } 
        &\lesssim \sqrt{\log(1/\delta) \|\Lamgbar\|_{\op} \cdot x_1^\top P_K x_1}.
        \end{align*}
        Hence, by AM-GM, $2\matgbar^{\top}\Lamcross x_1 \le \BigOh{\log(1/\delta) \|\Lamgbar\|_{\op}} + x_1^\top P_K x_1$. On the other hand, by Hanson-Wright
        \begin{align}
        \matgbar^{\top}\Lamgbar\,\matgbar &\le \tr(\tr(\Lamgbar)) + \BigOh{\fronorm{\Lamgbar}\sqrt{\log(1/\delta)} + \opnorm{\Lamgbar}\log(1/\delta)} \label{eq:hanson_intermediate}\\
        &\le\tr(\Lamgbar) + \BigOh{\sqrt{td\log(1/\delta)} + \opnorm{\Lamgbar}\log(1/\delta)},\nonumber
        \end{align}
        where we use the dimension of $\Lamgbar$ in the last line. Combining with the previous result, and adding in $x_1^\top \Lamxone \,x_1 \le x_1^\top P_K x_1$, we have that with probability $1-\delta$,
        \begin{align*}
        \Cost(R_1,R_2;x_1,\sigma_u,t)  &\le \tr(\Lamgbar) + \BigOh{(\sqrt{td\log(1/\delta)} + \log(1/\delta))\opnorm{\Lamgbar}+  x_1^\top P_K x_1}.
        \end{align*}
        The first high-probability statement follows by substituing in $\tr(\Lamgbar)$ and $\opnorm{\Lamgbar}$. Then second statement follows from returning to Eq.~\ref{eq:hanson_intermediate} and using $\opnorm{X},\fronorm{X} \le \tr(X)$ for $X \succeq 0$. \qed

        We shall now prove Lemma~\ref{lem:quadratic_forms_cost}, but first, we establish some useful preliminaries.

\subsubsection{Linear Algebra Preliminaries}
 \begin{defn}[Toeplitz Operator] For $\ell \in \N$, and $j,\ell \ge i$, define the matrices
  \begin{align*}
  \Toep_{i,j,\ell}(A) &:= \begin{bmatrix} A^{i}\I_{i \ge 0} & A^{i+1}\I_{i \ge -1} & \dots & A^{i + \ell}\I_{i \ge -\ell}\\
  A^{i-1}\I_{i \ge 1} & A^{i}\I_{i \ge 0} & \dots & A^{i + \ell-1}\I_{i \ge 1-\ell}\\
  \dots & \dots & \dots & \dots \\
  A^{i - j}\I_{i \ge j} & \dots & \dots & A^{i+\ell - j} \I_{i + \ell  - j \ge 0}\\
  \end{bmatrix}, \quad \ToepCol_{i,j}(A) &:= \begin{bmatrix} A^{j - 1}\\
  A^{j - 2}\\
  \dots\\
   \I_{i \ge 1} A^{i - 1}
  \end{bmatrix}.
  \end{align*}
  \end{defn}
  We shall use the following lemma.
  \begin{lem}\label{lem:toep_norm_bound} For any $i \le j,\ell$, we have $\opnorm{\ToepCol_{i,j}} \le \opnorm{\Toep_{i,j,\ell}(A)} \le \Hinf{A}$, and, for $Y \in \R^{\dimx^2}$, and $\diag_{j-i}(Y)$ denoting a $j-i$-block block matrix with blocks $Y$ on the diagonal, we have the bound
  \begin{align*}
 \tr(\ToepCol_{i,j}(A)^\top \diag_{j-i}(Y) \ToepCol_{i,j}(A)) \preceq (j-i)\cdot\tr(\dlyap(Y,A))\\
  \end{align*}
  \end{lem}
  \begin{proof} The first bound is a consequence of the fact that $\Toep_{i,j,\ell}(A)$ is a submatrix of the infinite-dimensional linear operator mapping inputs sequences in $\ell_2(\R^{\dimx})$ to outputs $\ell_2(\R^{\dimx})$; thus, the operator norm of $\Toep_{i,j,\ell}(A)$ is bounded by the operator norm of this infinite dimensional linear operator, which is equal to $\Hinf{A}$( see e.g. \citet[Corollary 4.2]{tilli1998singular}). The second bound follows from direct computation, as 
  \begin{align*}
  \tr(\ToepCol_{i,j}(A)^\top \diag_{j-i}(Y) \ToepCol_{i,j}(A)) \le \tr(\sum_{s=0}^{\infty} A^\top Y A) = \tr(\dlyap(A,Y)).
  \end{align*}
  \end{proof}

\subsubsection{Proof of \Cref{lem:quadratic_forms_cost}}

\begin{lem}[Form of the Covariates]\label{lem:quadratic_form_bounds_matrix_form} 
  Introduce the vector $\matx_{[t]} = (\matx_{t},\dots,\matx_1)$ and $\matu_{[t]} := (\matu_t,\dots,\matu_{1})$, set $\matwbar_{[t-1]} = (\matw_{t-1},\dots,\matw_1)$ and $\matg_{[t]} = (\matg_t,\dots,\matg_1)$. Then, we can write
  \begin{align*}
    \begin{bmatrix}
    \matx_{[t]}\\
    \matu_{[t]}
    \end{bmatrix} &= M_{K,t}\begin{bmatrix}
    \matw_{[t-1]}\\
    \matg_{[t]}
    \end{bmatrix}  + \underbrace{\begin{bmatrix} I_{t} \\ \diag_t(K) \end{bmatrix} \ToepCol_{1,t}(0)}_{:=M_{0,t}} \matx_1,
    \end{align*}
    where we have defined the matrix
    \begin{align*}
    M_{K,t} =  \begin{bmatrix}\Toep_{0,t,t-1}(A_{K})   & \sigma_u\Toep_{-1,t,t}(A_{K})\diag_{t}(\Bst) \\
    K\Toep_{0,t,t-1}(A_{K}) & \quad\sigma_u \diag_t(I) + \sigma_u K\Toep_{-1,t,t}(A_{K})\diag_{t}(\Bst)
    \end{bmatrix} .
  \end{align*}
\end{lem}
Further, let
\begin{align*}
    \left[\begin{array}{c} A \\
    \hline 
    B
    \end{array}\right]_{\diag} := \begin{bmatrix} A & 0 \\
    0 & B
    \end{bmatrix}.
    \end{align*}

In light of the the above  lemma, we have for  $\matgbar := \begin{bmatrix} \matw_{[t-1]} \\ \matg_{[t]}\end{bmatrix}$, we have that
\begin{align*}
&\sum_{s=1}^t \matx_s^{\top}R_1 \matx_s + \matu_s^{\top}R_2 \matu_s \\
&= 
    \begin{bmatrix}\matx_{[t]}\\ \matu_{[t]} \end{bmatrix}^{\top} 
    \Rdiag \begin{bmatrix}\matx_{[t]}\\ \matu_{[t]} \end{bmatrix}\\
    &= (M_{K,t}\matgbar_t + M_{0,t}\matx_1)^{\top} \Rdiag(M_{K,t}\matgbar_t + M_{0,t}\matx_1)\\
&= \matgbar_{t}^\top \underbrace{M_{K,t}^{\top} \Rdiag M_{K,t}}_{:=\Lamgbar} \matgbar_t + 2\matx_{1}^\top \underbrace{M_{0,t}^{\top}\Rdiag M_{K,t}}_{:=\Lamcross}  \matgbar_t \\
&\qquad + \matx_{1}^\top \underbrace{M_{0,t}^{\top} \Rdiag M_{0,t}}_{:=\Lamxone} \matx_1.
\end{align*}
We can evaluate each term separately. 

\paragraph{Bounding $\tr(\Lamgbar)$.}  Let us recall
\begin{align*}
  M_{K,t} =  \begin{bmatrix}\Toep_{0,t,t-1}(A_{K})   & \sigma_u\Toep_{-1,t,t}(A_{K})\diag_{t}(\Bst) \\
  K\Toep_{0,t,t-1}(A_{K}) & \quad\sigma_u \diag_t(I) + \sigma_u K\Toep_{-1,t,t}(A_{K})\diag_{t}(\Bst)
  \end{bmatrix}.
\end{align*}
  Recall $R_K := R_1 + K^\top R_2 K$. We find that the diagonal terms of $\Lamgbar$ coincinde with the diagonals of the matrix $\Lamgdiag$ defined as
\begin{align*} 
&\left[\begin{array}{c} \Toep_{0,t,t-1}(A_{K})^\top  \diag_{t-1}(R_K)\Toep_{0,t,t-1}(A_{K}) \\
\hline 
 \sigma_u^2\diag_{t}(\Bst)^{\top} \Toep_{-1,t,t}(A_{K})^{\top}\diag_t(R_K) \Toep_{-1,t,t}(A_{K})\diag_{t}(\Bst) + \sigma_u^2 \diag_{t}(R_2) 
 +
 \text{(cross term)}
 \end{array}\right]_{\diag} \\
& \preceq
 \left[\begin{array}{c} \Toep_{0,t,t-1}(A_{K})^\top  \diag_{t-1}(R_K)\Toep_{0,t,t-1}(A_{K}) \\
\hline 
 2\sigma_u^2\diag_{t}(\Bst)^{\top} \Toep_{-1,t,t}(A_{K})^{\top}\diag_t(R_K) \Toep_{-1,t,t}(A_{K})\diag_{t}(\Bst) + 2\sigma_u^2 \diag_{t}(R_2) 
 \end{array}\right]_{\diag},
\end{align*}
where $\text{(cross term)}$ denotes the cross term between the term $\sigma_u^2\diag_{t}(\Bst)^{\top} \Toep_{-1,t,t}(A_{K})^{\top}\diag_t(R_K) \Toep_{-t,t,t}(A_{K})\diag_{t}(\Bst) + \sigma_u^2 \diag_{t}(R_2) $, which we bound in the second inequality by Young's inequality.

By Lemma~\ref{lem:toep_norm_bound}, we have
\begin{align*}
\tr(\Toep_{0,t,t-1}(A_{K})^\top  \diag_{t-1}(R_K)\Toep_{0,t,t-1}(A_{K}) ) \le t\cdot \tr(\dlyap(A_K,R_K)) = J_K.
\end{align*}
Similarly, since $\dlyap(A_K,R_K) = P_K$, and thus $\rank(P_K) \le \rank(R_K) \le \rank(R_1) + \rank(R_2)$, 
\begin{align*}
&\tr(\diag_{t}(\Bst)^\top\Toep_{-1,t,t}(A_{K}))^\top  \diag_{t}(R_K)\Toep_{-1,t,t}(A_{K}))\diag_{t}(\Bst))\\ \quad&\le t\cdot \tr(\Bst^\top \dlyap(A_K,R_K) \Bst)\\  \quad&= t\cdot \tr(\Bst^\top P_K \Bst)\\
\quad&\le t\opnorm{\Bst}^2 \opnorm{\Pst}\min\{\rank(\Bst),\rank(P_K)\} \le t\deff\opnorm{\Bst}^2 \opnorm{\Pst}.
\end{align*}
Finally, we can bound $\tr(2\sigma_u^2 \diag_{t}(R_2)) \le 2t \sigma^2_u \rank(R_2)\nrm{R_2}_{\op} \le  2\deff t \sigma^2_u\nrm{R_2}_{\op}$, yielding
\begin{align*}
\tr(\Lamgbar) = \tr(\Lamgdiag) &\le t J_k + 2\sigma_u^2 t\deff \left(\opnorm{R_2} + \opnorm{\Bst}^2\opnorm{P_K}\right).
\end{align*}
\paragraph{Bounding $\opnorm{\Lamgbar}$.}

Observe that, for any PSD matrix $M = \begin{bmatrix} A & X \\ 
X^\top & B \end{bmatrix}$, we have that
\begin{align*}
    M \preceq 2\left[\begin{array}{c} A \\
    \hline 
    B
    \end{array}\right]_{\diag}.
    \end{align*}
Since $\Lamgbar \succeq 0$ (it is a non-negative form), in particular, we hae $\Lamgbar \preceq 2\Lamgdiag$. Thus
\begin{align*}
\opnorm{\Lamgbar} &\lesssim \opnorm{\Lamgdiag} \\
&\lesssim \sigma_u^2 \left(\|R_2\|_{\op}^2 + \|R_K\|_{\op}\|\Bst\|_{\op}^2\|\|\Toep_{-1,t,t}(A_{K})\|_{\op}^2\right) +  \|R_K\|_{\op}\|\Toep_{0,t,t-1}(A_{K})\|_{\op}^{2}.
\end{align*}
Since we can bound  $\|\Toep_{-1,t,t}(A_{K})\|_{\op}^2 \le \Hinf{A_K}$ by Lemma~\ref{lem:toep_norm_bound}, we obtain
\begin{align*}
\opnorm{\Lamgbar} &\lesssim  \|R_K\|\Hinf{A_{K}}^2 + \sigma_u^2 \left(\|R_2\|_{\op} + \|R_K\|_{\op}\|\Bst\|_2^2 \Hinf{A_{K}}^2\right),  \\
\end{align*}
where we use that $\sigma_u \le 1$.

\paragraph{Bounding $\Lamxone$.}
Let us recall that
\begin{align*}
M_{0,t} := \begin{bmatrix} I_{t} \\ \diag_t(K) \end{bmatrix} \ToepCol_{1,t}(A_K).
\end{align*}
Thus, 
\begin{align*}
\Lamxone = M_{0,t}^\top \Rdiag M_{0,t} &= \ToepCol_{1,t}(A_K)^{\top}\diag_{t}(R_1 + K^\top R_2 K)\ToepCol_{1,t}(A_K) \\
&=\ToepCol_{1,t}(A_K)^{\top}\diag_{t}(R_K)\ToepCol_{1,t}(A_K)\\
&\preceq \dlyap(A_K,R_K) = P_K.
\end{align*}
\paragraph{Bounding $\Lamcross$.}
We can directly verify that there exists a matrix $A$ with $AA^\top = \Lamgbar$ and a matrix $B$ with $BB^\top  = \Lamxone$ such that $\Lamcross = 2AB^\top$. Hence, 
\begin{align*}
\|\Lamcross x_1\|_{\op} \le \sqrt{\opnorm{\Lamgbar} \cdot x_1^\top \Lamxone x_1} \le \sqrt{\opnorm{\Lamgbar} \cdot x_1^\top P_K x_1}.
\end{align*}

\subsection{Proof of \Cref{lem:xt_bound}\label{ssec:proof:lem_xt_bound}}

  Set $\matgbar = \begin{bmatrix} \matw_{[t-1]} \\ \matg_{[t-1]} \end{bmatrix}$. Then we have
  \begin{align*}
  \matx_t - A_{K}^{t-1}\matx_1 =  \ToepCol_{1,t-1}(A_K) \matw_{[t-1]} + \sigma_{u}\ToepCol_{1,t-1}(A_K)\diag(\Bst)\matg_{[t]}\matg_{[t]}.
  \end{align*}

We now observe that   
  \begin{align*}
    \Exp[\|\matx_t - A_{K}^{t-1}\matx_1\|_2^2 \mid \matx_1] 
  &= \tr( \ToepCol_{1,t-1}(A_K)\ToepCol_{1,t-1}(A_K)^{\top}) \\
  &\qquad + \sigma_{u}^2 \tr( \diag_{t-1}(\Bst^{\top}) \ToepCol_{1,t-1}(A_K)\ToepCol_{1,t-1}(A_K)^{\top} \diag_{t-1}(\Bst))\\
  &\le ( 1+ \sigma_{u}^2\|\Bst\|_2^2) \tr( \ToepCol_{1,t-1}(A_K)\ToepCol_{1,t-1}(A_K)^{\top})\\
  &\le ( 1+ \sigma_{u}^2\|\Bst\|_2^2) \|\ToepCol_{1,t-1}(A_K)\|_{\fro}^2 \\
  &\le  ( 1+ \sigma_{u}^2\|\Bst\|_2^2)\tr(\dlyap(A_K,I))\\
   &\le ( 1+ \sigma_{u}^2\|\Bst\|_2^2)J_K,
  \end{align*}
  where the last inequality uses Lemma~\ref{lem:closed_loop_dlyap}. Since $\matx_t - A_{K}^{t-1}\matx_1$ is a Gaussian quadratic form, the simplified Hanson Wright inequality (\Cref{cor:crude_HS}) gives
  \begin{align*}
  \|\matx_t - A_{K}^{t-1}\matx_1\|_2^2 \lesssim ( 1+ \sigma_{u}^2\|\Bst\|_2^2)J_K \log \frac{1}{\delta}.
  \end{align*}
  \qed

\subsection{Extension to General Noise Models \label{app:general_noise}}

Our upper bounds hold for general noise distributions with the following properties:
\begin{enumerate}
  \item The noise satisfies a Hanson-Wright style inequality, so that an analogue of \Cref{lem:cost_lem} holds. Recall that \Cref{lem:cost_lem} establishes that the true costs concentrate around their expectations.
  \item The noise process is a $\sigma_+$-sub-Gaussian martingale difference sequence, in the sense that $\En\brk*{\matw_t\mid{}\matw_1,\ldots,\matw_{t-1}}=0$ and for any $v \in \R^{\dimx}$, $\Exp[\exp(\langle v, \matw_t) \mid \matw_{1},\dots,\matw_{t-1}] \le \exp(\frac{1}{2}\|v\|^2 \sigma_+^2)$. This is necessary for the self-normalized tail bound (\Cref{lem:self_normalized} of \citet{abbasi2011improved}).
  \item The noise satisfies the block-martingale small ball condition from \citet{simchowitz2018learning}, which ensures the covariates are well-conditioned during the estimation phase (in particular, that an analogue of \Cref{lem:covariance_lb} holds)
\end{enumerate}
In more detail, suppose that the noise is $\sigma_+$-sub-Gaussian, and that $\Exp[\matw_t\matw_t^\top \mid \matw_{1},\dots,\matw_{t-1}] \succeq \Sigma_- \succ 0$. Then by applying the Paley-Zygmund inequality (analogously to Eq. 3.12 in \cite{simchowitz2018learning}), one can show that the $(1, \frac{1}{2}\Sigma_-,p)-$block-martingale small-ball property holds with 
\begin{align*}
p &= \frac{1}{4}\cdot \min_{v \ne 0} \frac{\Exp[\langle \matw_t, z \rangle^2 \mid \matw_{1:t-1}]^2}{\Exp[\langle \matw_t, z \rangle^4 \mid \matw_{1:t-1}]} \\
&\gtrsim \frac{\lambda_{\min}(\Sigma_-)^2}{\sigma_+^4},
\end{align*}
where in the last inequality, we upper bound $\Exp[\langle \matw_t, z \rangle^4]$ using the standard moment bound for sub-Gaussian variables. Hence, a sub-Gaussian upper bound and covariance lower bound are enough to guarantee point 3 above holds.

Point 1 is more delicate, because Hanson-Wright inequalities are known under only restrictive assumptions: namely, for vectors which have independent sub-Gaussian coordinates \citep{rudelson2013hanson}, or for those satisfying a Lipschitz-concentration property \citep{adamczak2015note}. For the first condition to be satisfied, we need to assume that there exists a matrix $\Sigma_{+} \succ 0$ such that the vectors $\tilde{\matw}_t :=\Sigma_+^{-1/2}\matw_t$ are (a) jointly independent, and (b) have jointly independent, sub-Gaussian coordinates. For the second condition to hold, we must assume that the concatenated vectors $(\tilde{\matw}_1,\dots,\tilde{\matw}_{t})$ satisfy the Lipschitz-concentration property \citep[Definition 2.1]{adamczak2015note}. If either condition holds, then we can obtain the same regret as in our main theorem by modifying \Cref{lem:quadratic_forms_cost} to use a quadratic form for the sequence $(\tilde{\matw}_1,\dots,\tilde{\matw}_t)$, and then applying one of the Hanson-Wright variants above to attain \Cref{lem:cost_lem}. 

In general, it is not known if sub-Gaussian martingale noise satisfies a Hanson-Wright inequality. In this case, we can demonstrate the concentration of costs around their expectation via a combination of the Azuma-Hoeffding/Azuma-Bernstein inequality with truncation and mixing arguments. This type of argument bounds the fluctuations of the costs around their mean as roughly $(\dimx + \dimu)\sqrt{T}$, which is worse than the square root scaling $\sqrt{\dimx + \dimu}\cdot \sqrt{T}$ enjoyed by the Hanson-Wright inequality. Up to logarithmic factors, this would yield regret of $(\dimx + \dimu)\sqrt{T} + \sqrt{\dimx \dimu^2 T} = \sqrt{\dimx \max\{\dimx,\dimu^2\} T}$, which is sub-optimal for $\dimx \gg \dimu^2$. It is not clear if \emph{any} algorithm can do better in this regime (without a sharper inequality for the concentration of costs around their means), since it is not clear how to ameliorate these random fluctuations. Nevertheless, the final regret bound of $ \sqrt{\dimx \max\{\dimx,\dimu^2\} T}$ still improves upon the dimension dependence in the upper bound of $\sqrt{(\dimx+\dimu)^3 T}$ attained by \cite{mania2019certainty}.

%% file: main.bbl
\begin{thebibliography}{51}
\providecommand{\natexlab}[1]{#1}
\providecommand{\url}[1]{\texttt{#1}}
\expandafter\ifx\csname urlstyle\endcsname\relax
  \providecommand{\doi}[1]{doi: #1}\else
  \providecommand{\doi}{doi: \begingroup \urlstyle{rm}\Url}\fi

\bibitem[Abbasi-Yadkori and Szepesv{\'a}ri(2011)]{abbasi2011regret}
Yasin Abbasi-Yadkori and Csaba Szepesv{\'a}ri.
\newblock Regret bounds for the adaptive control of linear quadratic systems.
\newblock In \emph{Proceedings of the 24th Annual Conference on Learning
  Theory}, pages 1--26, 2011.

\bibitem[Abbasi-Yadkori et~al.(2011)Abbasi-Yadkori, P{\'a}l, and
  Szepesv{\'a}ri]{abbasi2011improved}
Yasin Abbasi-Yadkori, D{\'a}vid P{\'a}l, and Csaba Szepesv{\'a}ri.
\newblock Improved algorithms for linear stochastic bandits.
\newblock In \emph{Advances in Neural Information Processing Systems}, pages
  2312--2320, 2011.

\bibitem[Abeille and Lazaric(2017)]{abeille2017thompson}
Marc Abeille and Alessandro Lazaric.
\newblock Thompson sampling for linear-quadratic control problems.
\newblock In \emph{Artificial Intelligence and Statistics}, pages 1246--1254,
  2017.

\bibitem[Abeille and Lazaric(2018)]{abeille2018improved}
Marc Abeille and Alessandro Lazaric.
\newblock Improved regret bounds for thompson sampling in linear quadratic
  control problems.
\newblock In \emph{International Conference on Machine Learning}, pages 1--9,
  2018.

\bibitem[Adamczak et~al.(2015)]{adamczak2015note}
Radoslaw Adamczak et~al.
\newblock A note on the hanson-wright inequality for random vectors with
  dependencies.
\newblock \emph{Electronic Communications in Probability}, 20, 2015.

\bibitem[Agarwal et~al.(2019{\natexlab{a}})Agarwal, Bullins, Hazan, Kakade, and
  Singh]{agarwal2019online}
Naman Agarwal, Brian Bullins, Elad Hazan, Sham Kakade, and Karan Singh.
\newblock Online control with adversarial disturbances.
\newblock In \emph{International Conference on Machine Learning}, pages
  111--119, 2019{\natexlab{a}}.

\bibitem[Agarwal et~al.(2019{\natexlab{b}})Agarwal, Hazan, and
  Singh]{agarwal2019logarithmic}
Naman Agarwal, Elad Hazan, and Karan Singh.
\newblock Logarithmic regret for online control.
\newblock In \emph{Advances in Neural Information Processing Systems 32}, pages
  10175--10184. 2019{\natexlab{b}}.

\bibitem[Arias-Castro et~al.(2012)Arias-Castro, Candes, and
  Davenport]{arias2012fundamental}
Ery Arias-Castro, Emmanuel~J Candes, and Mark~A Davenport.
\newblock On the fundamental limits of adaptive sensing.
\newblock \emph{IEEE Transactions on Information Theory}, 59\penalty0
  (1):\penalty0 472--481, 2012.

\bibitem[Assouad(1983)]{assouad1983deux}
Patrice Assouad.
\newblock Deux remarques sur l'estimation.
\newblock \emph{Comptes rendus des s{\'e}ances de l'Acad{\'e}mie des sciences.
  S{\'e}rie 1, Math{\'e}matique}, 296\penalty0 (23):\penalty0 1021--1024, 1983.

\bibitem[Azar et~al.(2017)Azar, Osband, and Munos]{azar2017minimax}
Mohammad~Gheshlaghi Azar, Ian Osband, and R{\'e}mi Munos.
\newblock Minimax regret bounds for reinforcement learning.
\newblock In \emph{Proceedings of the 34th International Conference on Machine
  Learning-Volume 70}, pages 263--272. JMLR. org, 2017.

\bibitem[Bertsekas(2005)]{bertsekas2005dynamic}
Dimitri~P Bertsekas.
\newblock \emph{Dynamic Programming and Optimal Control, Vol. I}.
\newblock Athena Scientific, 2005.

\bibitem[Bof et~al.(2018)Bof, Carli, and Schenato]{bof2018lyapunov}
Nicoletta Bof, Ruggero Carli, and Luca Schenato.
\newblock Lyapunov theory for discrete time systems.
\newblock \emph{arXiv preprint arXiv:1809.05289}, 2018.

\bibitem[Boyd(2008)]{BoydNotes}
Stephen Boyd.
\newblock Lecture 13: Linear quadratic lyapunov theory.
\newblock \emph{EE363 Course Notes, Stanford University}, 2008.

\bibitem[Cohen et~al.(2018)Cohen, Hasidim, Koren, Lazic, Mansour, and
  Talwar]{cohen2018online}
Alon Cohen, Avinatan Hasidim, Tomer Koren, Nevena Lazic, Yishay Mansour, and
  Kunal Talwar.
\newblock Online linear quadratic control.
\newblock In \emph{International Conference on Machine Learning}, pages
  1028--1037, 2018.

\bibitem[Cohen et~al.(2019)Cohen, Koren, and Mansour]{cohen2019learning}
Alon Cohen, Tomer Koren, and Yishay Mansour.
\newblock Learning linear-quadratic regulators efficiently with only $\sqrt{T}$
  regret.
\newblock In \emph{International Conference on Machine Learning}, pages
  1300--1309, 2019.

\bibitem[Dann and Brunskill(2015)]{dann2015sample}
Christoph Dann and Emma Brunskill.
\newblock Sample complexity of episodic fixed-horizon reinforcement learning.
\newblock In \emph{Advances in Neural Information Processing Systems}, pages
  2818--2826, 2015.

\bibitem[Dean et~al.()Dean, Mania, Matni, Recht, and Tu]{dean2017sample}
Sarah Dean, Horia Mania, Nikolai Matni, Benjamin Recht, and Stephen Tu.
\newblock On the sample complexity of the linear quadratic regulator.
\newblock \emph{Foundations of Computational Mathematics}, pages 1--47.

\bibitem[Dean et~al.(2018)Dean, Mania, Matni, Recht, and Tu]{dean2018regret}
Sarah Dean, Horia Mania, Nikolai Matni, Benjamin Recht, and Stephen Tu.
\newblock Regret bounds for robust adaptive control of the linear quadratic
  regulator.
\newblock In \emph{Advances in Neural Information Processing Systems}, pages
  4188--4197, 2018.

\bibitem[Faradonbeh et~al.(2018{\natexlab{a}})Faradonbeh, Tewari, and
  Michailidis]{faradonbeh2018input}
Mohamad Kazem~Shirani Faradonbeh, Ambuj Tewari, and George Michailidis.
\newblock Input perturbations for adaptive regulation and learning.
\newblock \emph{arXiv preprint arXiv:1811.04258}, 2018{\natexlab{a}}.

\bibitem[Faradonbeh et~al.(2018{\natexlab{b}})Faradonbeh, Tewari, and
  Michailidis]{faradonbeh2018optimality}
Mohamad Kazem~Shirani Faradonbeh, Ambuj Tewari, and George Michailidis.
\newblock On optimality of adaptive linear-quadratic regulators.
\newblock \emph{arXiv preprint arXiv:1806.10749}, 2018{\natexlab{b}}.

\bibitem[Fazel et~al.(2018)Fazel, Ge, Kakade, and Mesbahi]{fazel2018global}
Maryam Fazel, Rong Ge, Sham Kakade, and Mehran Mesbahi.
\newblock Global convergence of policy gradient methods for the linear
  quadratic regulator.
\newblock In \emph{International Conference on Machine Learning}, pages
  1466--1475, 2018.

\bibitem[Hazan et~al.(2007)Hazan, Agarwal, and Kale]{hazan2007logarithmic}
Elad Hazan, Amit Agarwal, and Satyen Kale.
\newblock Logarithmic regret algorithms for online convex optimization.
\newblock \emph{Machine Learning}, 69\penalty0 (2):\penalty0 169--192, 2007.

\bibitem[Hazan et~al.(2017)Hazan, Singh, and Zhang]{hazan2017learning}
Elad Hazan, Karan Singh, and Cyril Zhang.
\newblock Learning linear dynamical systems via spectral filtering.
\newblock In \emph{Advances in Neural Information Processing Systems}, pages
  6702--6712, 2017.

\bibitem[Hazan et~al.(2018)Hazan, Lee, Singh, Zhang, and
  Zhang]{hazan2018spectral}
Elad Hazan, Holden Lee, Karan Singh, Cyril Zhang, and Yi~Zhang.
\newblock Spectral filtering for general linear dynamical systems.
\newblock In \emph{Advances in Neural Information Processing Systems}, pages
  4634--4643, 2018.

\bibitem[Hsu et~al.(2012)Hsu, Kakade, Zhang, et~al.]{hsu2012tail}
Daniel Hsu, Sham Kakade, Tong Zhang, et~al.
\newblock A tail inequality for quadratic forms of subgaussian random vectors.
\newblock \emph{Electronic Communications in Probability}, 17, 2012.

\bibitem[Jaksch et~al.(2010)Jaksch, Ortner, and Auer]{jaksch2010near}
Thomas Jaksch, Ronald Ortner, and Peter Auer.
\newblock Near-optimal regret bounds for reinforcement learning.
\newblock \emph{Journal of Machine Learning Research}, 11\penalty0
  (Apr):\penalty0 1563--1600, 2010.

\bibitem[Jiang et~al.(2017)Jiang, Krishnamurthy, Agarwal, Langford, and
  Schapire]{jiang2017contextual}
Nan Jiang, Akshay Krishnamurthy, Alekh Agarwal, John Langford, and Robert~E
  Schapire.
\newblock Contextual decision processes with low bellman rank are
  pac-learnable.
\newblock In \emph{Proceedings of the 34th International Conference on Machine
  Learning-Volume 70}, pages 1704--1713. JMLR. org, 2017.

\bibitem[Jin et~al.(2020)Jin, Yang, Wang, and Jordan]{jin2019provably}
Chi Jin, Zhuoran Yang, Zhaoran Wang, and Michael~I Jordan.
\newblock Provably efficient reinforcement learning with linear function
  approximation.
\newblock \emph{Conference on Learning Theory (COLT)}, 2020.

\bibitem[Kakade et~al.(2003)Kakade, Kearns, and
  Langford]{kakade2003exploration}
Sham Kakade, Michael~J Kearns, and John Langford.
\newblock Exploration in metric state spaces.
\newblock In \emph{Proceedings of the 20th International Conference on Machine
  Learning (ICML-03)}, pages 306--312, 2003.

\bibitem[Kearns et~al.(2000)Kearns, Mansour, and Ng]{kearns2000approximate}
Michael~J Kearns, Yishay Mansour, and Andrew~Y Ng.
\newblock Approximate planning in large pomdps via reusable trajectories.
\newblock In \emph{Advances in Neural Information Processing Systems}, pages
  1001--1007, 2000.

\bibitem[Langford and Zhang(2007)]{langford2007epoch}
John Langford and Tong Zhang.
\newblock The epoch-greedy algorithm for contextual multi-armed bandits.
\newblock In \emph{Proceedings of the 20th International Conference on Neural
  Information Processing Systems}, pages 817--824. Citeseer, 2007.

\bibitem[Lillicrap et~al.(2015)Lillicrap, Hunt, Pritzel, Heess, Erez, Tassa,
  Silver, and Wierstra]{lillicrap2015continuous}
Timothy~P Lillicrap, Jonathan~J Hunt, Alexander Pritzel, Nicolas Heess, Tom
  Erez, Yuval Tassa, David Silver, and Daan Wierstra.
\newblock Continuous control with deep reinforcement learning.
\newblock \emph{arXiv preprint arXiv:1509.02971}, 2015.

\bibitem[Lincoln and Rantzer(2006)]{lincoln2006relaxing}
Bo~Lincoln and Anders Rantzer.
\newblock Relaxing dynamic programming.
\newblock \emph{IEEE Transactions on Automatic Control}, 51\penalty0
  (8):\penalty0 1249--1260, 2006.

\bibitem[Mania et~al.(2019)Mania, Tu, and Recht]{mania2019certainty}
Horia Mania, Stephen Tu, and Benjamin Recht.
\newblock Certainty equivalence is efficient for linear quadratic control.
\newblock In \emph{Advances in Neural Information Processing Systems}, pages
  10154--10164, 2019.

\bibitem[Mnih et~al.(2015)Mnih, Kavukcuoglu, Silver, Rusu, Veness, Bellemare,
  Graves, Riedmiller, Fidjeland, Ostrovski, et~al.]{mnih2015human}
Volodymyr Mnih, Koray Kavukcuoglu, David Silver, Andrei~A Rusu, Joel Veness,
  Marc~G Bellemare, Alex Graves, Martin Riedmiller, Andreas~K Fidjeland, Georg
  Ostrovski, et~al.
\newblock Human-level control through deep reinforcement learning.
\newblock \emph{Nature}, 518\penalty0 (7540):\penalty0 529, 2015.

\bibitem[Munos and Szepesv{\'a}ri(2008)]{munos2008finite}
R{\'e}mi Munos and Csaba Szepesv{\'a}ri.
\newblock Finite-time bounds for fitted value iteration.
\newblock \emph{Journal of Machine Learning Research}, 9\penalty0
  (May):\penalty0 815--857, 2008.

\bibitem[Ouyang et~al.(2017)Ouyang, Gagrani, and Jain]{ouyang2017control}
Yi~Ouyang, Mukul Gagrani, and Rahul Jain.
\newblock Control of unknown linear systems with thompson sampling.
\newblock In \emph{2017 55th Annual Allerton Conference on Communication,
  Control, and Computing (Allerton)}, pages 1198--1205. IEEE, 2017.

\bibitem[Rakhlin and Sridharan(2014)]{rakhlin2014online}
Alexander Rakhlin and Karthik Sridharan.
\newblock Online nonparametric regression.
\newblock In \emph{Conference on Learning Theory}, 2014.

\bibitem[Ran and Vreugdenhil(1988)]{ran1988existence}
ACM Ran and R~Vreugdenhil.
\newblock Existence and comparison theorems for algebraic riccati equations for
  continuous-and discrete-time systems.
\newblock \emph{Linear Algebra and its applications}, 99:\penalty0 63--83,
  1988.

\bibitem[Rudelson et~al.(2013)Rudelson, Vershynin, et~al.]{rudelson2013hanson}
Mark Rudelson, Roman Vershynin, et~al.
\newblock Hanson-wright inequality and sub-gaussian concentration.
\newblock \emph{Electronic Communications in Probability}, 18, 2013.

\bibitem[Sarkar and Rakhlin(2019)]{sarkar2018fast}
Tuhin Sarkar and Alexander Rakhlin.
\newblock Near optimal finite time identification of arbitrary linear dynamical
  systems.
\newblock In \emph{International Conference on Machine Learning}, pages
  5610--5618, 2019.

\bibitem[{Sarkar} et~al.(2019){Sarkar}, {Rakhlin}, and
  {Dahleh}]{sarkar2019fast}
Tuhin {Sarkar}, Alexander {Rakhlin}, and Munther~A. {Dahleh}.
\newblock {Finite-Time System Identification for Partially Observed LTI Systems
  of Unknown Order}.
\newblock \emph{arXiv preprint arXiv:1902.01848}, 2019.

\bibitem[Shamir(2013)]{shamir2013complexity}
Ohad Shamir.
\newblock On the complexity of bandit and derivative-free stochastic convex
  optimization.
\newblock In \emph{Conference on Learning Theory}, pages 3--24, 2013.

\bibitem[Silver et~al.(2016)Silver, Huang, Maddison, Guez, Sifre, Van
  Den~Driessche, Schrittwieser, Antonoglou, Panneershelvam, Lanctot,
  et~al.]{silver2016mastering}
David Silver, Aja Huang, Chris~J Maddison, Arthur Guez, Laurent Sifre, George
  Van Den~Driessche, Julian Schrittwieser, Ioannis Antonoglou, Veda
  Panneershelvam, Marc Lanctot, et~al.
\newblock Mastering the game of go with deep neural networks and tree search.
\newblock \emph{nature}, 529\penalty0 (7587):\penalty0 484, 2016.

\bibitem[Simchowitz et~al.(2018)Simchowitz, Mania, Tu, Jordan, and
  Recht]{simchowitz2018learning}
Max Simchowitz, Horia Mania, Stephen Tu, Michael~I Jordan, and Benjamin Recht.
\newblock Learning without mixing: Towards a sharp analysis of linear system
  identification.
\newblock In \emph{Conference On Learning Theory}, pages 439--473, 2018.

\bibitem[Simchowitz et~al.(2019)Simchowitz, Boczar, and
  Recht]{simchowitz2019learning}
Max Simchowitz, Ross Boczar, and Benjamin Recht.
\newblock Learning linear dynamical systems with semi-parametric least squares.
\newblock In \emph{Conference on Learning Theory}, pages 2714--2802, 2019.

\bibitem[Sutton and Barto(2018)]{sutton2018reinforcement}
Richard~S Sutton and Andrew~G Barto.
\newblock \emph{Reinforcement learning: An introduction}.
\newblock 2018.

\bibitem[Tilli(1998)]{tilli1998singular}
Paolo Tilli.
\newblock Singular values and eigenvalues of non-hermitian block toeplitz
  matrices.
\newblock \emph{Linear Algebra and its Applications}, 272\penalty0
  (1-3):\penalty0 59--89, 1998.

\bibitem[Tu and Recht(2018)]{tu2017least}
Stephen Tu and Benjamin Recht.
\newblock Least-squares temporal difference learning for the linear quadratic
  regulator.
\newblock In \emph{International Conference on Machine Learning}, pages
  5005--5014, 2018.

\bibitem[Vovk(2001)]{vovk2001competitive}
Volodya Vovk.
\newblock Competitive on-line statistics.
\newblock \emph{International Statistical Review}, 69\penalty0 (2):\penalty0
  213--248, 2001.

\bibitem[Yu(1997)]{yu1997assouad}
Bin Yu.
\newblock Assouad, fano, and le cam.
\newblock In \emph{Festschrift for Lucien Le Cam}, pages 423--435. Springer,
  1997.

\end{thebibliography}
